\documentclass{article}
\usepackage{ijcai26}

\usepackage{times}
\usepackage{soul}
\usepackage{url}
\usepackage[hidelinks]{hyperref}
\usepackage[utf8]{inputenc}
\usepackage[small]{caption}
\usepackage{graphicx}
\usepackage{amsmath}
\usepackage{amsthm}
\usepackage{booktabs}
\usepackage{algorithm}
\usepackage{algorithmic}
\usepackage[switch]{lineno}

\usepackage{amssymb}
\usepackage{newtxmath}
\usepackage{multirow}
\usepackage[normalem]{ulem}
\useunder{\uline}{\ul}{}
\usepackage[table,xcdraw]{xcolor}
\usepackage{makecell}
\usepackage{marvosym}
\usepackage{rotating}
\usepackage{hyperref}
\usepackage{amsmath}
\usepackage{amssymb}
\usepackage{float}
\usepackage {pifont}

\newtheorem{theorem}{Theorem}

\title{vLinear: A Powerful Linear Model for Multivariate Time Series Forecasting}

\author{
    Anonymous Submission
}

\author{
Wenzhen Yue$^1$
\and
Ruohao Guo$^1$\and
Ji Shi$^1$\and
Zihan Hao$^1$\and
Shiyu Hu$^1$\And
Xianghua Ying$^1$\footnote{Corresponding Author}\\
\affiliations
$^1$State Key Laboratory of General Artificial Intelligence, School of Intelligence Science and Technology, Peking University\\
\emails
yuewenzhen@stu.pku.edu.cn,
xhying@pku.edu.cn
}

\begin{document}

\maketitle

\begin{abstract}
    In this paper, we present \textbf{vLinear}, an effective yet efficient \textbf{linear}-based multivariate time series forecaster featuring two components: the \textbf{v}ecTrans module and the WFMLoss objective. Many state-of-the-art forecasters rely on self-attention or its variants to capture multivariate correlations, typically incurring $\mathcal{O}(N^2)$ computational complexity with respect to the number of variates $N$. To address this, we propose vecTrans, a lightweight module that utilizes a learnable vector to model multivariate correlations, reducing the complexity to $\mathcal{O}(N)$. Notably, vecTrans can be seamlessly integrated into Transformer-based forecasters, delivering up to 5$\times$ inference speedups and consistent performance gains. Furthermore, we introduce WFMLoss (Weighted Flow Matching Loss) as the objective. In contrast to typical \textbf{velocity-oriented} flow matching objectives, we demonstrate that a \textbf{final-series-oriented} formulation yields significantly superior forecasting accuracy. WFMLoss also incorporates path- and horizon-weighted strategies to focus learning on more reliable paths and horizons. Empirically, vLinear achieves state-of-the-art performance across 22 benchmarks and 124 forecasting settings. Moreover, WFMLoss serves as an effective plug-and-play objective, consistently improving existing forecasters. The code is available at \url{https://anonymous.4open.science/r/vLinear}.
    
\end{abstract}

\section{Introduction}

\begin{figure}[t]
   \centering
   \includegraphics[width=1.0\linewidth]{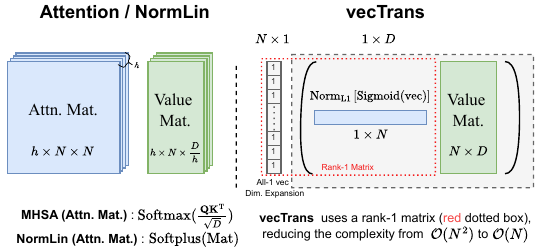}
   \caption{Comparison of multi-head self-attention (MHSA), NormLin \protect\cite{olinear}, and vecTrans. vecTrans employs a learnable rank-1 matrix to model token dependencies. By further re-arranging the computational order, vecTrans reduces the complexity from $\mathcal{O}(N^2)$ to $\mathcal{O}(N)$. When integrated into Transformer-based forecasters (e.g., iTransformer), vecTrans yields up to $5\times$ inference speedup while consistently improving forecasting accuracy.}
   \label{fig1_vectrans}
\end{figure}

Time series forecasting plays a pivotal role in domains such as weather forecasting \cite{nature_weather}, energy systems \cite{informer}, and transportation \cite{traffic}. Recently, many studies have focused on the linear-based forecasters \cite{olinear} and Transformer-based methods \cite{timemixer++}. While Transformer-based models often achieve strong performance, they typically entail high computational costs. Conversely, linear models are significantly more efficient but may suffer from limited expressiveness, especially in modeling multivariate correlations. To balance this efficiency-performance trade-off, we aim to develop a powerful yet efficient linear model for multivariate time series forecasting.

While self-attention \cite{transformer} captures multivariate correlations effectively \cite{itransformer,samformer,fredformer}, it suffers from quadratic $\mathcal{O}(N^2)$ complexity. Inspired by the inherent low-rank structure in attention matrices \cite{freeformer}, we propose the vecTrans module, which replaces the complex input-dependent attention matrix with a learnable rank-1 matrix to model token dependencies (Figure~\ref{fig1_vectrans}). Specifically, vecTrans aggregates token features using a learnable vector of length $N$  and broadcasts the results to all tokens. This approach effectively reduces the complexity to linear  $\mathcal{O}(N)$. \textbf{Empirically, vecTrans exhibits superior performance in time series forecasting despite its minimalist design}. 


Recently, Flow Matching (FM) \cite{flow_matching,rectified_flow,fm_scaling} has emerged as a powerful paradigm in generative modeling by learning a time-dependent velocity field to transform noise into data. Typical FM objectives are \textbf{velocity-oriented}, aiming to match the instantaneous flow direction (Figure~\ref{fig2_wfmloss}). However, we argue that such indirect supervision is suboptimal for time series forecasting. Instead, we propose a \textbf{final-series-oriented} FM objective that explicitly prioritizes the accuracy of the final predictions, yielding significant performance gains (see Table~\ref{tab2_vel_wfm_comp}). To further refine this objective, we incorporate path- and horizon-weighted strategies to focus on the more reliable stages of the generation path and the near-future horizons. Theoretically, we justify the horizon-weighted formulation as the optimal objective under a maximum-likelihood framework assuming a Laplace distribution. We term this \textbf{W}eighted \textbf{F}low \textbf{M}atching objective as \textbf{WFMLoss} and the overall model as \textbf{vLinear}. Our contributions are summarized as follows:

\begin{itemize}
    \item We propose \textbf{vecTrans}, a lightweight token dependency learning module with linear computational complexity. Despite its simplicity, vecTrans consistently outperforms self-attention and its variants. When integrated as a plug-in into Transformer-based forecasters, vecTrans achieves up to a $5\times$ speedup along with significant performance improvements.
    
    \item Based on flow matching theory, we introduce a final-series-oriented, path- and horizon-weighted time series flow matching loss, termed \textbf{WFMLoss}. WFMLoss outperforms existing objective designs and yields consistent performance gains when applied to existing forecasters.
    
    \item Extensive experiments on 22 benchmarks and 124 forecasting settings demonstrate that vLinear consistently achieves state-of-the-art performance while maintaining superior computational efficiency.
\end{itemize}

\section{Related Work}

\subsection{Time Series Forecasters}

Linear-based forecasters (e.g., DLinear \cite{linear}, RLinear \cite{rlinear}, and TimeMixer \cite{timemixer}) often exhibit competitive performance with high computational efficiency. To further enhance modeling capacity, many state-of-the-art time series forecasters rely on the classic self-attention mechanism (e.g., iTransformer \cite{itransformer}, PatchTST \cite{patchtst}, Fredformer \cite{fredformer}, Leddam \cite{Leddam_icml}, and TimeMixer++ \cite{timemixer++}) or its variants (e.g., FreEformer \cite{freeformer}, Informer \cite{informer}) to model multivariate correlations and temporal dynamics. A recent work, OLinear \cite{olinear}, directly learns an attention-like matrix to model token dependencies, achieving a good trade-off between performance and efficiency. Despite their effectiveness, both self-attention mechanisms and the NormLin module suffer from \textit{quadratic} computational complexity. In this work, we introduce the lightweight vecTrans module, achieving state-of-the-art performance while maintaining \textit{linear} complexity with respect to the number of variates. Moreover, vecTrans exhibits strong generalizability and consistently improves both performance and efficiency when plugged into Transformer-based forecasters.

\subsection{Loss Designs for Time Series Forecasting}

The most common loss function in time series forecasting is the mean squared error (MSE) loss \cite{itransformer,patchtst,timemixer++}. As an improvement, CARD \cite{card} proposes a weighted MAE loss to emphasize important time steps. FreDF identifies the bias induced by the temporal MSE loss and instead applies the MAE loss in the frequency domain. TransDF \cite{transdf} further computes the loss in a transformed domain where the series is temporally decorrelated. DBLoss \cite{dbloss} decomposes the predicted series into trend and seasonal components and computes the losses on these components separately. Unlike the above point-wise losses, Sundial \cite{sundial} introduces the TimeFlow loss, which is a velocity-oriented MSE-based flow matching loss.
In this work, we demonstrate  that a final-series-oriented flow matching loss significantly outperforms its velocity-oriented counterpart and other state-of-the-art losses.


\begin{figure}[t]
   \centering
   \includegraphics[width=1.0\linewidth]{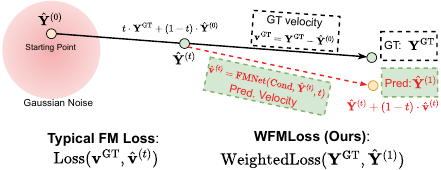}
   \caption{Illustration of the typical flow matching loss and WFMLoss. Unlike velocity-oriented alignment, WFMLoss focuses on calibrating the final predicted series and incorporates path- and horizon-weighted strategies to improve forecasting performance.}
   \label{fig2_wfmloss}
\end{figure}

\section{Preliminaries: Flow Matching for Time Series Forecasting}


Flow Matching (FM) \cite{flow_matching} transforms the forecasting task from static point estimation to modeling optimal straight-line trajectories, typically yielding structural coherence and better generalization \cite{fm_ot}.


Time series forecasting can be formulated as conditional flow matching \cite{fm_ot}. As illustrated in Figure~\ref{fig2_wfmloss}, given conditional representations $\mathbf{Cond} \in \mathbb{R}^H$ extracted from historical observations, the current state $\hat{\mathbf{Y}}^{(t)}$, and the  time $t \in \left [ 0,1 \right ] $ \footnote{In this work, $t\in \left [ 0,1 \right ] $ denotes the continuous time in the velocity field, distinct from the discrete time step $i$ in the series.}, the neural network
$\mathrm{FM\text{-}Net}(\mathbf{Cond}, \hat{\mathbf{Y}}^{(t)}, t)$
is trained to predict the velocity field $\hat{\mathbf{v}}^{(t)}$,  approximating the ground-truth velocity directed towards the target series $\mathbf{Y}^{\mathrm{GT}}$.

During training, the intermediate state $\hat{\mathbf{Y}}^{(t)}$ is constructed by linearly interpolating between a Gaussian noise sample $\hat{\mathbf{Y}}^{(0)}$ and the ground-truth series $\mathbf{Y}^{\mathrm{GT}}$ \cite{fm_ot,flow_match_interpolants,rectified_flow}:
$\hat{\mathbf{Y}}^{(t)} = t  \mathbf{Y}^{\mathrm{GT}} + (1 - t) \hat{\mathbf{Y}}^{(0)}$ .



Accordingly, the ground-truth velocity is defined as
$\mathbf{v} ^{\mathrm{GT}} = \mathbf{Y} ^{\mathrm{GT} }- \hat{\mathbf{Y}} ^{(0)}$. The network FM-Net can be trained by minimizing the discrepancy between $\hat{\mathbf{v}}^{(t)}$ and $\mathbf{v}^{\mathrm{GT} }$ \cite{sundial}:

\begin{equation}
\mathcal{L}_{\mathrm{FM}} = \left \| \hat{\mathbf{v}}^{(t)}-\mathbf{v}^{\mathrm{GT} } \right \|^2 .
\label{eq2_loss_fm}
\end{equation}

During inference, predictions are generated by integrating the predicted velocity field starting from Gaussian noise using a discrete push-forward process with $K$ steps:


\begin{figure*}[t]
   \centering
   \includegraphics[width=1.0\linewidth]{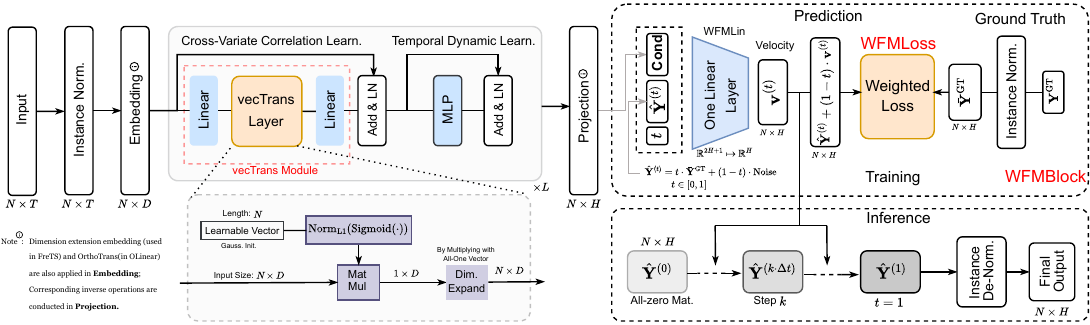}
   \caption{Overall architecture of vLinear, featuring the vecTrans module and the WFMLoss design. vecTrans uses a single learnable vector to aggregate multivariate features and broadcasts the aggregated representation to each variate by expanding along the variate dimension. WFMLoss is based on final-series-oriented flow matching and incorporates path- and horizon-weighted strategies.}
   \label{fig3_overall_arch}
\end{figure*}


\begin{equation}
\hat{\mathbf{Y}} = \mathbf{E} \left [ \hat{\mathbf{Y}}^{(0)} + \Delta t \sum_{k=0}^{K-1} \hat{\mathbf{v}}(\hat{\mathbf{Y}}^{(k \Delta t)}, k \Delta t, \mathbf{Cond} )  \right ] ,  \Delta t=\frac{1}{K},
\label{eq3_fm_infer}
\end{equation}

\noindent where the expectation $\mathbf{E} \left [ \cdot \right ] $ is taken over the predictions generated from different initial Gaussian noise samples.

\section{Method}

For multivariate time series forecasting, given a historical sequence $\mathbf{X} =\left \{ \mathbf{x} _{1}, \cdots, \mathbf{x} _{T}  \right \} \in \mathbb{R}^{N \times T} $ with $T$ time steps and $N$ variates, our goal is to  predict the future $H$ time steps $\mathbf{Y} =\left \{ \mathbf{x}_{T+1},\cdots, \mathbf{x}_{T+H}   \right \} \triangleq \left \{ \mathbf{y}_{1},\cdots, \mathbf{y}_{H}   \right \} \in \mathbb{R}^{N \times H}$. $\mathbf{Y} ^ \mathrm{GT} $ and $\hat{\mathbf{Y}} $ denote ground truth and prediction, respectively.

\subsection{Overall Architecture} \label{overall_arch}

As shown in Figure~\ref{fig3_overall_arch}, vLinear adopts a simple architecture.
Given the input series $\mathbf{X}$, we first apply instance normalization to mitigate non-stationarity \cite{non-stationary}. Then, the OrthoTrans layer
\cite{olinear} is used to decorrelate the series, which has been proven effective
for representation learning. OrthoTrans is implemented by multiplying
with the orthogonal matrix that diagonalizes the series autocorrelation matrix,
and is thus computationally efficient. Its inverse operation,
$\mathrm{OrthoTrans}^{\mathrm{inv}}$, which corresponds to multiplication with
the transposed orthogonal matrix, is applied during decoding.
After OrthoTrans, a dimension extension module is applied to enhance expressivity \cite{frets}, which
transforms the input from $N \times T$ to $N \times d \times T$ by multiplying
with a learnable vector $\boldsymbol{\phi}_d \in \mathbb{R}^d$. The series is then flattened and embedded
into a hidden representation $\mathbf{H}^0 \in \mathbb{R}^{N \times D}$.

During representation learning, we use the vecTrans module (detailed in Section~\ref{vecTrans_sec}) and the two-layer MLP to learn multivariate correlations and temporal dynamics, respectively. After $L$ stacked blocks, we obtain the final representation $\mathbf{H}^L\in \mathbb{R}^{N \times D}$. Then $\mathrm{OrthoTrans} ^{\mathrm{inv}}$ and linear layers are used to output the conditional representation $\mathbf{Cond}\in \mathbb{R}^{N \times H}$. Then $\mathbf{Cond}$ is fed into the WFMBlock (detailed in Section~\ref{WFMLoss_sec}) to generate the final prediction $\hat{\mathbf{Y}}$. The overall process can be formulated as:

\begin{equation}
\begin{aligned}
\mathbf{H}^0  &= \mathrm{Embed}(\mathrm{InstanceNorm} (\mathbf{X} )) \in \mathbb{R}^{N \times  D}, \\
\mathbf{H}^{l+1}_{\mathrm{Var}} &= \mathrm{LN}(\mathbf{H}^{l}+\mathrm{vecTransModule} (\mathbf{H}^{l})), \\
\mathbf{H}^{l+1} &= \mathrm{LN}(\mathbf{H}^{l+1}_{\mathrm{Var}}+\mathrm{MLP} (\mathbf{H}^{l+1}_{\mathrm{Var}})),
 l = 0, \cdots, L-1,
\\
\mathbf{Cond} &= \mathrm{Projection} \left ( \mathbf{H}^{L} \right )   \in \mathbb{R}^{N \times H},\\
\hat{\mathbf{Y}} &= \mathrm{DeNorm}(\mathrm{WFMBlock} ^{\mathrm{inf} }(\mathbf{Cond})) \in \mathbb{R}^{N \times H},
\end{aligned}
\label{eq_overall_proc}
\end{equation}

\noindent where $\mathrm{LN}$ denotes layer normalization, and $\mathrm{WFMBlock} ^{\mathrm{inf} }$ denotes the inference phase of our WFMBlock (Figure~\ref{fig3_overall_arch}).

\subsection{vecTrans} \label{vecTrans_sec}

It has been well established that modeling multivariate correlations via the self-attention mechanism significantly improves forecasting performance \cite{itransformer,fredformer,freeformer}. 
To reduce the computational burden, NormLin \cite{olinear} removes the query and key projections and directly employs a learnable matrix to model cross-variate dependencies. However, both self-attention and NormLin entail quadratic computational complexity. Meanwhile, it has been observed that attention matrices often exhibit a low-rank structure (see Figure~\ref{fig_vec_visual}), indicating substantial information redundancy. Motivated by this observation, we impose a rank-1 constraint\footnote{More general rank-$k$ settings are discussed in Section~\ref{more_ranks}.}  on the learnable matrix that captures multivariate correlations.

\begin{theorem}[Rank-1 Row-Normalized Matrix] \label{theorem1}
Suppose that $\mathbf{A} \in \mathbb{R}^{N \times N}$ is a row-wise L1-normalized matrix with no zero rows. If $\mathrm{rank}(\mathbf{A}) = 1$, then $\mathbf{A}$ must be of the form $\mathbf{A} = \mathbf{1}\mathbf{a}^{\mathsf{T}},$
where $\mathbf{1} \in \mathbb{R}^N$ is the all-one vector, $\mathbf{a} \in \mathbb{R}^N$, and $\|\mathbf{a}\|_1 = 1$.
\end{theorem}

The proof is provided in Appendix~\ref{appd_a}.
Let $\mathbf{a} \in \mathbb{R}^N$ denote the learnable vector.
To ensure that the rank-1 matrix $(\mathbf{1} \mathbf{a}^{\mathsf{T}}) \in \mathbb{R}^{N \times N}$ is row-wise L1-normalized with positive entries, we apply the Sigmoid function followed by L1 normalization to  $\mathbf{a}$.
Specifically, we define:

\begin{equation}
\mathrm{vecTrans}\left ( \mathbf{H}  \right )  = \mathbf{1}   \left [ \mathrm{Norm} _{\mathrm{L1} }\left ( \mathrm{Sigmoid} \left ( \mathbf{a}  \right )  \right )^{\mathsf{T} }   \mathbf{H} \right ],
\label{eq:vecTrans}
\end{equation}

\noindent where $\mathbf{H} \in \mathbb{R}^{N\times D}$.
In Eq.~\eqref{eq:vecTrans}, we re-arrange the order of operations to multiply with $\mathbf{H}$ first, avoiding explicitly computing the rank-1 matrix $(\mathbf{1} \mathbf{a}^{\mathsf{T}}) \in \mathbb{R}^{N \times N}$.
The variants of Eq.~\eqref{eq:vecTrans} are discussed in Appendix~\ref{appd_ablation_vectrans}. Incorporating the pre- and post-linear layers, the vecTrans module is formulated as: 

\begin{equation}
\mathrm{vecTransModule}\left ( \cdot \right ) \triangleq \mathrm{Linear} \left ( \mathtt{vecTrans} \left ( \mathrm{Linear} \left ( \cdot \right )  \right )  \right ),
\label{eq:vecTransMod}
\end{equation}

\noindent where $\mathrm{Linear}(\cdot)$ and $\mathrm{vecTrans}(\cdot)$ operate on the temporal and variate dimensions, respectively. Ablation studies on these two linear layers are provided in Appendix~\ref{appd_pre_post_lin}.

\begin{table}[t]
\begin{center}
\renewcommand{\arraystretch}{0.9}
\begin{tabular}{@{}ccc@{}}
\toprule
Module                                                         & FLOPs & Memory \\ \midrule
\begin{tabular}[c]{@{}c@{}}MHSA\\  \shortcite{transformer} \end{tabular}         &  $ \mathcal{O} \left ( 2N^2D+4ND^2 \right )$     &    $\mathcal{O} \left ( hN^2+ND \right )$     \\
\begin{tabular}[c]{@{}c@{}}Lin.Attn.\\ \shortcite{flattentrans} \end{tabular}    &  $ \mathcal{O} \left ( 4ND^2 + \frac{2ND^2}{h}\right )$       &   $\mathcal{O} \left ( \frac{D^2}{h} +ND \right )$      \\
\begin{tabular}[c]{@{}c@{}}NormLin\\  \shortcite{olinear} \end{tabular}      &   $ \mathcal{O} \left ( N^2D+2ND^2 \right )$     &   $ \mathcal{O} \left ( N^2+ND \right )$      \\
\begin{tabular}[c]{@{}c@{}}vecTrans\\      (Ours)\end{tabular} &    $ \mathcal{O} \left ( 2ND^2+ND \right )$    &    $ \mathcal{O} \left ( ND \right )$     \\ \bottomrule
\end{tabular}

\caption{FLOPs and memory usage of MHSA, linear attentions, NormLin, and vecTrans. $h$ denotes the number of attention heads. vecTrans exhibits $\mathcal{O}(N)$ computational and memory complexity.}
\label{tab1_flops_vecTrans}
\end{center}
\end{table}

\paragraph{Comparison with Linear Attention}
Linear attention methods \cite{flattentrans,linear_softmax} typically adopt a multi-head mechanism and employ linear projections to generate input-dependent query and key matrices. In contrast, vecTrans utilizes an input-independent aggregation and broadcasting strategy without multiple heads, leading to a simpler design and improved computational efficiency. As shown in Table~\ref{tab1_flops_vecTrans}, vecTrans incurs fewer FLOPs and consumes less memory. Furthermore, it empirically outperforms linear attention mechanisms (see Table~\ref{tab_vec_attn}).

\subsection{WFMLoss and Inference} \label{WFMLoss_sec}

The representation learning described above yields $\mathbf{Cond}\in \mathbb{R}^{N \times H}$, which encapsulates the rich context of the history series. Given $\mathbf{Cond}$, the intermediate state at time $t$ $\hat{\mathbf{Y}}^{(t)}$, and the time $t \in [0,1]$,  we employ a single linear layer\footnote{The multi-layer setting is discussed in Appendix~\ref{appd_wfmlin}.}, denoted as WFMLin, to generate the velocity  $\hat{\mathbf{v}}^{(t)}$:

\begin{equation}
    \hat{\mathbf{v}}^{(t)}=\mathrm{Linear} (\mathrm{Concat} (\mathbf{Cond} , \hat{\mathbf{Y}}^{(t)},t)) \in \mathbb{R} ^{N \times H}.
    \label{eq_vt}
\end{equation}

This velocity $\hat{\mathbf{v}}^{(t)}$ guides $\hat{\mathbf{Y}}^{(t)}$ toward the prediction: $\hat{\mathbf{Y}} ^{(1)}= \hat{\mathbf{Y}}  ^{(t)} + (1-t) \hat{\mathbf{v}} ^{(t)}$. We then quantify the discrepancy between the predicted $\hat{\mathbf{Y}} ^{(1)}$ and the normalized ground truth $\bar{\mathbf{Y}}^{\mathrm{GT}} $. By incorporating path- ($(2-t)^{-0.5},0\le t \le 1 $) and horizon-weighted ($i^{-0.5},1 \le i \le H$) strategies, the final WFMLoss is formulated as:

\begin{equation}
\begin{aligned}
    \mathrm{WFMLoss}^{(t)} & = (2-t)^{-0.5} \mathrm{WeightedLoss} (\hat{\mathbf{Y}}^{(1)},\bar{\mathbf{Y}}^{\mathrm{GT}}  ) \\
    &=\frac{1}{H} (2-t)^{-0.5}\sum_{i=1}^{H}i^{-0.5}\left \| \hat{\mathrm{Y}}^{(1)}_{:,i}- \bar{\mathrm{Y}}^{\mathrm{GT}}_{:,i}\right \|  _1,
\end{aligned}
\label{eq_wfmloss}
\end{equation}

\noindent where the subscript ($:,i$) denotes the time step $i$. 
During training, $t$ is uniformly sampled from the interval $[0,1]$.
The hyperparameters in WFMLoss are discussed in Appendix~\ref{hyper-para}. 

\paragraph{Discussion}
The rationale behind the path-weighted design ($(2-t)^{-0.5}$) is to emphasize the later stages of the density path (i.e., larger $t$) to facilitate effective learning. Since smaller $t$ corresponds to higher noise levels, obtaining an accurate $\hat{\mathbf{Y}} ^{(1)}$ is relatively more difficult. Regarding the horizon-weighted MAE loss, we present the following theorem:

\begin{theorem}[Horizon-Weighted MAE Loss] 
\label{theorem2}
Consider the univariate series $\mathbf{y}^{\mathrm{GT}} \in \mathbb{R}^H$. Suppose that the series is a first-order Markov process and the conditional probability density follows a Laplace distribution. Then, horizon-weighted MAE loss $\sum_{i=1}^{H}i^{-0.5} \left | \hat{\mathrm{y}_i}- \mathrm{y}^{\mathrm{GT}}_i  \right | $ is optimal under the maximum likelihood criterion.
\end{theorem}

\begin{table}[t]
\begin{center}
\setlength{\tabcolsep}{5pt}
\renewcommand{\arraystretch}{1.0} 
{\fontsize{8}{10}\selectfont
\begin{tabular}{@{}lcccccc@{}}
\toprule
Dataset & \multicolumn{1}{l}{ETTh1}             & \multicolumn{1}{l}{ETTm2}             & \multicolumn{1}{l}{ECL}               & \multicolumn{1}{l}{Solar}             & \multicolumn{1}{l}{Weather}           & \multicolumn{1}{l}{PEMS03}            \\ \midrule
WFMLoss & {\color[HTML]{FF0000} \textbf{0.416}} & {\color[HTML]{FF0000} \textbf{0.268}} & {\color[HTML]{FF0000} \textbf{0.153}} & {\color[HTML]{FF0000} \textbf{0.209}} & {\color[HTML]{FF0000} \textbf{0.233}} & {\color[HTML]{FF0000} \textbf{0.093}} \\
Vel.$\dagger$   & {\color[HTML]{0000FF} {\ul 0.434}}    & {\color[HTML]{0000FF} {\ul 0.278}}    & {\color[HTML]{0000FF} {\ul 0.168}}    & {\color[HTML]{0000FF} {\ul 0.280}}    & {\color[HTML]{0000FF} {\ul 0.247}}    & {\color[HTML]{0000FF} {\ul 0.097}}    \\
Vel.$\ddagger$  & 0.439                                 & 0.331                                 & 0.693                                 & 0.963                                 & 3.257                                 & 0.359                                 \\ \bottomrule
\end{tabular}
}
\caption{Comparison between WFMLoss and velocity-oriented FM losses. Average MSEs across four prediction lengths are reported. $\dagger$ denotes the path- and horizon-weighted MAE loss, while $\ddagger$ denotes the plain MSE loss (Eq.~\eqref{eq2_loss_fm}). WFMLoss exhibits a significant performance advantage over its velocity-oriented counterparts.}

\label{tab2_vel_wfm_comp}
\end{center}
\end{table}

\noindent The proof is provided in Appendix~\ref{appd_b}. As shown in Table~\ref{tab2_vel_wfm_comp}, WFMLoss outperforms the velocity-oriented counterpart (`Vel.$\dagger$') by 8.6\% on average. \textbf{This performance gap may stem from the velocity-based objective's heavy dependence on the initial noise} $\hat{\mathbf{Y}}^{(0)}$ (recall that $\mathbf{v} ^{\mathrm{GT}} = \mathbf{Y} ^{\mathrm{GT} }- \hat{\mathbf{Y}} ^{(0)}$).
It is counter-intuitive to require the model to memorize the starting noise $\hat{\mathbf{Y}}^{(0)}$ when the current state $\hat{\mathbf{Y}}^{(t)}$ (at a large $t$) is already close to the ground truth.
Furthermore, the plain MSE version (`Vel.$\ddagger$' in Table~\ref{tab2_vel_wfm_comp}) suffers from training instability and significantly underperforms WFMLoss, highlighting the effectiveness of our weighting strategy.

\begin{table}[b]
\begin{center}
\setlength{\tabcolsep}{5pt}
\renewcommand{\arraystretch}{1.0} 
{\fontsize{9}{10}\selectfont
\begin{tabular}{@{}c|cccccc@{}}
\toprule
Starting   & ETTh1 & ETTm2 & ECL   & Traffic & Solar & Weather \\ \midrule
All-Zero     & 0.416 & 0.268 & 0.153 & 0.440   & 0.209 & 0.233   \\ \midrule
Size: 10 & 0.417 & 0.268 & 0.153 & 0.441   & 0.210 & 0.234   \\
Size: 30 & 0.416 & 0.268 & 0.153 & 0.440   & 0.209 & 0.233   \\
Size: 50 & 0.416 & 0.268 & 0.153 & 0.440   & 0.209 & 0.233   \\ \bottomrule
\end{tabular}
}
\caption{Comparison of inference starting from the all-zero state and various Gaussian sampling sizes. Average MSEs across four prediction horizons $\{96,192,336,720\}$ are reported.}

\label{tab3_ori_iters}
\end{center}
\end{table}

\subsubsection{Inference}

As shown in Eq.~\eqref{eq3_fm_infer}, the standard inference process is to start from Gaussian noise and compute the expected values of the outputs. However, since WFMLin is a linear layer, we derive the following simplified inference process.

\begin{table*}[t]
\begin{center}
{\fontsize{7}{8.5}\selectfont
\setlength{\tabcolsep}{1.7pt}
\begin{tabular}{@{}c|cc|cc|cc|cc|cc|cc|cc|cc|cc|cc|cc|cc@{}}
\toprule
Model                          & \multicolumn{2}{c|}{\begin{tabular}[c]{@{}c@{}}vLinear\\      (Ours)\end{tabular}} & \multicolumn{2}{c|}{\begin{tabular}[c]{@{}c@{}}OLinear\\  \shortcite{olinear} \end{tabular}} 
& \multicolumn{2}{c|}{\begin{tabular}[c]{@{}c@{}}TimeMixer\\      \shortcite{timemixer} \end{tabular}} 
& \multicolumn{2}{c|}{\begin{tabular}[c]{@{}c@{}}FilterNet\\   \shortcite{filternet} \end{tabular}} 
& \multicolumn{2}{c|}{\begin{tabular}[c]{@{}c@{}}FITS\\      \shortcite{fits} \end{tabular}} 
& \multicolumn{2}{c|}{\begin{tabular}[c]{@{}c@{}}DLinear\\      \shortcite{linear} \end{tabular}} 
& \multicolumn{2}{c|}{\begin{tabular}[c]{@{}c@{}}TimeMix.++\\    \shortcite{timemixer++} \end{tabular}} 
& \multicolumn{2}{c|}{\begin{tabular}[c]{@{}c@{}}Leddam\\      \shortcite{Leddam_icml} \end{tabular}} 
& \multicolumn{2}{c|}{\begin{tabular}[c]{@{}c@{}}Fredformer\\   \shortcite{fredformer} \end{tabular}} 
& \multicolumn{2}{c|}{\begin{tabular}[c]{@{}c@{}}iTrans.\\      \shortcite{itransformer} \end{tabular}} 
& \multicolumn{2}{c|}{\begin{tabular}[c]{@{}c@{}}PatchTST\\     \shortcite{patchtst} \end{tabular}} 
& \multicolumn{2}{c}{\begin{tabular}[c]{@{}c@{}}TimesNet\\      \shortcite{timesnet} \end{tabular}} \\ \midrule
Metric                         & MSE                                      & MAE                                     & MSE                                  & MAE                                     & MSE                                     & MAE                                    & MSE                                     & MAE                                    & MSE                                  & MAE                                  & MSE                                    & MAE                                   & MSE                                      & MAE                                     & MSE                     & MAE                                                 & MSE                                     & MAE                                     & MSE                                                  & MAE                     & MSE                                    & MAE                                    & MSE                                    & MAE                                   \\ \midrule
ETTm1                          & {\color[HTML]{FF0000} \textbf{0.369}}    & {\color[HTML]{0000FF} {\ul 0.378}}      & {\color[HTML]{0000FF} {\ul 0.374}}   & {\color[HTML]{FF0000} \textbf{0.377}}   & 0.381                                   & 0.395                                  & 0.392                                   & 0.401                                  & 0.493                                & 0.452                                & 0.403                                  & 0.407                                 & {\color[HTML]{FF0000} \textbf{0.369}}    & 0.378                                   & 0.386                   & 0.397                                               & 0.384                                   & 0.395                                   & 0.407                                                & 0.410                   & 0.387                                  & 0.400                                  & 0.400                                  & 0.406                                 \\
ETTm2                          & {\color[HTML]{FF0000} \textbf{0.268}}    & {\color[HTML]{FF0000} \textbf{0.310}}   & 0.270                                & {\color[HTML]{0000FF} {\ul 0.313}}      & 0.275                                   & 0.323                                  & 0.285                                   & 0.328                                  & 0.291                                & 0.333                                & 0.350                                  & 0.401                                 & {\color[HTML]{0000FF} {\ul 0.269}}       & 0.320                                   & 0.281                   & 0.325                                               & 0.279                                   & 0.324                                   & 0.288                                                & 0.332                   & 0.281                                  & 0.326                                  & 0.291                                  & 0.333                                 \\
ETTh1                          & {\color[HTML]{FF0000} \textbf{0.416}}    & {\color[HTML]{FF0000} \textbf{0.420}}   & 0.424                                & {\color[HTML]{0000FF} {\ul 0.424}}      & 0.447                                   & 0.440                                  & 0.441                                   & 0.439                                  & 0.440                                & 0.431                                & 0.456                                  & 0.452                                 & {\color[HTML]{0000FF} {\ul 0.419}}       & 0.432                                   & 0.431                   & 0.429                                               & 0.435                                   & 0.426                                   & 0.454                                                & 0.447                   & 0.469                                  & 0.454                                  & 0.458                                  & 0.450                                 \\
ETTh2                          & {\color[HTML]{0000FF} {\ul 0.358}}       & {\color[HTML]{0000FF} {\ul 0.385}}      & 0.367                                & 0.388                                   & 0.365                                   & 0.395                                  & 0.383                                   & 0.407                                  & 0.376                                & 0.398                                & 0.559                                  & 0.515                                 & {\color[HTML]{FF0000} \textbf{0.339}}    & {\color[HTML]{FF0000} \textbf{0.380}}   & 0.373                   & 0.399                                               & 0.365                                   & 0.393                                   & 0.383                                                & 0.407                   & 0.384                                  & 0.405                                  & 0.414                                  & 0.427                                 \\
ECL                            & {\color[HTML]{FF0000} \textbf{0.153}}    & {\color[HTML]{FF0000} \textbf{0.245}}   & {\color[HTML]{0000FF} {\ul 0.159}}   & {\color[HTML]{0000FF} {\ul 0.248}}      & 0.182                                   & 0.273                                  & 0.173                                   & 0.268                                  & 0.384                                & 0.434                                & 0.212                                  & 0.300                                 & 0.165                                    & 0.253                                   & 0.169                   & 0.263                                               & 0.176                                   & 0.269                                   & 0.178                                                & 0.270                   & 0.208                                  & 0.295                                  & 0.192                                  & 0.295                                 \\
Exchange                       & {\color[HTML]{0000FF} {\ul 0.341}}       & {\color[HTML]{FF0000} \textbf{0.391}}   & 0.355                                & {\color[HTML]{0000FF} {\ul 0.399}}      & 0.387                                   & 0.416                                  & 0.388                                   & 0.419                                  & 0.365                                & 0.408                                & 0.354                                  & 0.414                                 & 0.357                                    & 0.409                                   & 0.354                   & 0.402                                               & {\color[HTML]{FF0000} \textbf{0.333}}   & {\color[HTML]{FF0000} \textbf{0.391}}   & 0.360                                                & 0.403                   & 0.367                                  & 0.404                                  & 0.416                                  & 0.443                                 \\
Traffic                        & 0.440                                    & {\color[HTML]{0000FF} {\ul 0.252}}      & 0.451                                & {\color[HTML]{FF0000} \textbf{0.247}}   & 0.485                                   & 0.298                                  & 0.463                                   & 0.310                                  & 0.615                                & 0.370                                & 0.625                                  & 0.383                                 & {\color[HTML]{FF0000} \textbf{0.416}}    & 0.264                                   & 0.467                   & 0.294                                               & 0.433                                   & 0.291                                   & {\color[HTML]{0000FF} {\ul 0.428}}                   & 0.282                   & 0.531                                  & 0.343                                  & 0.620                                  & 0.336                                 \\
Weather                        & {\color[HTML]{0000FF} {\ul 0.233}}       & {\color[HTML]{FF0000} \textbf{0.256}}   & 0.237                                & {\color[HTML]{0000FF} {\ul 0.260}}      & 0.240                                   & 0.272                                  & 0.245                                   & 0.272                                  & 0.273                                & 0.292                                & 0.265                                  & 0.317                                 & {\color[HTML]{FF0000} \textbf{0.226}}    & 0.262                                   & 0.242                   & 0.272                                               & 0.246                                   & 0.272                                   & 0.258                                                & 0.279                   & 0.259                                  & 0.281                                  & 0.259                                  & 0.287                                 \\
Solar-Energy                   & {\color[HTML]{0000FF} {\ul 0.209}}       & {\color[HTML]{0000FF} {\ul 0.226}}      & 0.215                                & {\color[HTML]{FF0000} \textbf{0.217}}   & 0.216                                   & 0.280                                  & 0.235                                   & 0.266                                  & 0.376                                & 0.384                                & 0.330                                  & 0.401                                 & {\color[HTML]{FF0000} \textbf{0.203}}    & 0.258                                   & 0.230                   & 0.264                                               & 0.226                                   & 0.262                                   & 0.233                                                & 0.262                   & 0.270                                  & 0.307                                  & 0.301                                  & 0.319                                 \\
PEMS03                        & {\color[HTML]{FF0000} \textbf{0.093}}    & {\color[HTML]{FF0000} \textbf{0.197}}   & {\color[HTML]{0000FF} {\ul 0.095}}   & {\color[HTML]{0000FF} {\ul 0.199}}      & 0.167                                   & 0.267                                  & 0.145                                   & 0.251                                  & 0.489                                & 0.465                                & 0.278                                  & 0.375                                 & 0.165                                    & 0.263                                   & 0.107                   & 0.210                                               & 0.135                                   & 0.243                                   & 0.113                                                & 0.221                   & 0.180                                  & 0.291                                  & 0.147                                  & 0.248                                 \\
PEMS04                        & {\color[HTML]{FF0000} \textbf{0.090}}    & {\color[HTML]{FF0000} \textbf{0.190}}   & {\color[HTML]{0000FF} {\ul 0.091}}   & {\color[HTML]{FF0000} \textbf{0.190}}   & 0.185                                   & 0.287                                  & 0.146                                   & 0.258                                  & 0.531                                & 0.489                                & 0.295                                  & 0.388                                 & 0.136                                    & 0.251                                   & 0.103                   & {\color[HTML]{0000FF} {\ul 0.210}}                  & 0.162                                   & 0.261                                   & 0.111                                                & 0.221                   & 0.195                                  & 0.307                                  & 0.129                                  & 0.241                                 \\
PEMS07                        & {\color[HTML]{FF0000} \textbf{0.076}}    & {\color[HTML]{FF0000} \textbf{0.164}}   & {\color[HTML]{0000FF} {\ul 0.077}}   & {\color[HTML]{FF0000} \textbf{0.164}}   & 0.181                                   & 0.271                                  & 0.123                                   & 0.229                                  & 0.500                                & 0.472                                & 0.329                                  & 0.395                                 & 0.152                                    & 0.258                                   & 0.084                   & {\color[HTML]{0000FF} {\ul 0.180}}                  & 0.121                                   & 0.222                                   & 0.101                                                & 0.204                   & 0.211                                  & 0.303                                  & 0.124                                  & 0.225                                 \\
PEMS08                        & {\color[HTML]{FF0000} \textbf{0.111}}    & {\color[HTML]{0000FF} {\ul 0.195}}      & {\color[HTML]{0000FF} {\ul 0.113}}   & {\color[HTML]{FF0000} \textbf{0.194}}   & 0.226                                   & 0.299                                  & 0.172                                   & 0.260                                  & 0.534                                & 0.487                                & 0.379                                  & 0.416                                 & 0.200                                    & 0.279                                   & 0.122                   & {\color[HTML]{0000FF} {\ul 0.211}}                  & 0.161                                   & 0.250                                   & 0.150                                                & 0.226                   & 0.280                                  & 0.321                                  & 0.193                                  & 0.271                                 \\ \midrule
\multicolumn{1}{l|}{1\textsuperscript{st} Count} & 8                                        & 8                                       & 0                                    & 6                                       & 0                                       & 0                                      & 0                                       & 0                                      & 0                                    & 0                                    & 0                                      & 0                                     & 5                                        & 1                                       & 0                       & 0                                                   & 1                                       & 1                                       & 0                                                    & 0                       & 0                                      & 0                                      & 0                                      & 0                                     \\ \bottomrule
\end{tabular}
}
\caption{Long-term forecasting results under the `Input-96-Predict-\{96, 192, 336, 720\}' setting, except for PEMS, for which $H \in \left \{ 12,24,48,96 \right \}$.  Average results are reported. These settings are used throughout the work. Full results are presented in Tables~\ref{tab_long_term_appd} and \ref{tab_long_term_PEMS}.}
\label{tab_long_term}
\end{center}
\end{table*}

\begin{table*}[t]
\begin{center}
{\fontsize{7}{8.5}\selectfont
\setlength{\tabcolsep}{1.7pt}
\begin{tabular}{@{}c|cc|cc|cc|cc|cc|cc|cc|cc|cc|cc|cc|cc@{}}
\toprule
Model                          & \multicolumn{2}{c|}{\begin{tabular}[c]{@{}c@{}}vLinear\\      (Ours)\end{tabular}} 
& \multicolumn{2}{c|}{\begin{tabular}[c]{@{}c@{}}OLinear\\    \shortcite{olinear} \end{tabular}} 
& \multicolumn{2}{c|}{\begin{tabular}[c]{@{}c@{}}TimeMixer\\    \shortcite{timemixer} \end{tabular}} 
& \multicolumn{2}{c|}{\begin{tabular}[c]{@{}c@{}}FilterNet\\   \shortcite{filternet} \end{tabular}} 
& \multicolumn{2}{c|}{\begin{tabular}[c]{@{}c@{}}DLinear\\  \shortcite{linear} \end{tabular}} 
& \multicolumn{2}{c|}{\begin{tabular}[c]{@{}c@{}}TimeMix.++\\  \shortcite{timemixer++} \end{tabular}} 
& \multicolumn{2}{c|}{\begin{tabular}[c]{@{}c@{}}Leddam\\     \shortcite{Leddam_icml} \end{tabular}} 
& \multicolumn{2}{c|}{\begin{tabular}[c]{@{}c@{}}CARD\\     \shortcite{card} \end{tabular}} 
& \multicolumn{2}{c|}{\begin{tabular}[c]{@{}c@{}}Fredformer\\   \shortcite{fredformer} \end{tabular}} 
& \multicolumn{2}{c|}{\begin{tabular}[c]{@{}c@{}}iTrans.\\      \shortcite{itransformer} \end{tabular}} 
& \multicolumn{2}{c|}{\begin{tabular}[c]{@{}c@{}}PatchTST\\      \shortcite{patchtst} \end{tabular}} 
& \multicolumn{2}{c}{\begin{tabular}[c]{@{}c@{}}TimesNet\\     \shortcite{timesnet} \end{tabular}} \\ \midrule
Metric                         & MSE                                      & MAE                                     & MSE                                    & MAE                                   & MSE                                     & MAE                                    & MSE                                     & MAE                                    & MSE                                                    & MAE                   & MSE                                                      & MAE                     & MSE                                   & MAE                                   & MSE                  & MAE                                                  & MSE                                                     & MAE                     & MSE                                    & MAE                                   & MSE                                    & MAE                                    & MSE                                    & MAE                                   \\ \midrule
ILI                            & {\color[HTML]{FF0000} \textbf{1.391}}    & {\color[HTML]{0000FF} {\ul 0.692}}      & {\color[HTML]{0000FF} {\ul 1.429}}     & {\color[HTML]{FF0000} \textbf{0.690}} & 1.864                                   & 0.806                                  & 1.793                                   & 0.791                                  & 2.742                                                  & 1.126                 & 1.805                                                    & 0.793                   & 1.725                                 & 0.777                                 & 1.959                & 0.822                                                & 1.732                                                   & 0.797                   & 1.715                                  & 0.773                                 & 1.905                                  & 0.804                                  & 1.809                                  & 0.807                                 \\
COVID-19                       & {\color[HTML]{FF0000} \textbf{4.834}}    & {\color[HTML]{FF0000} \textbf{1.187}}   & {\color[HTML]{0000FF} {\ul 5.187}}     & {\color[HTML]{0000FF} {\ul 1.211}}    & 5.919                                   & 1.350                                  & 5.607                                   & 1.322                                  & 8.279                                                  & 1.601                 & 5.974                                                    & 1.369                   & 5.251                                 & 1.285                                 & 5.536                & 1.314                                                & 5.279                                                   & 1.287                   & 5.301                                  & 1.293                                 & 5.836                                  & 1.362                                  & 6.106                                  & 1.369                                 \\
METR-LA                        & {\color[HTML]{0000FF} {\ul 0.571}}       & {\color[HTML]{0000FF} {\ul 0.320}}      & 0.587                                  & {\color[HTML]{FF0000} \textbf{0.311}} & 0.608                                   & 0.372                                  & 0.603                                   & 0.366                                  & 0.580                                                  & 0.422                 & {\color[HTML]{FF0000} \textbf{0.567}}                    & 0.363                   & 0.603                                 & 0.367                                 & 0.639                & 0.350                                                & 0.617                                                   & 0.369                   & 0.627                                  & 0.373                                 & 0.614                                  & 0.372                                  & 0.617                                  & 0.370                                 \\
NASDAQ                         & {\color[HTML]{FF0000} \textbf{0.118}}    & {\color[HTML]{FF0000} \textbf{0.199}}   & 0.121                                  & {\color[HTML]{0000FF} {\ul 0.201}}    & {\color[HTML]{0000FF} {\ul 0.120}}      & 0.204                                  & 0.127                                   & 0.211                                  & 0.150                                                  & 0.251                 & 0.125                                                    & 0.210                   & 0.128                                 & 0.211                                 & 0.125                & 0.207                                                & 0.127                                                   & 0.210                   & 0.133                                  & 0.217                                 & 0.128                                  & 0.209                                  & 0.161                                  & 0.247                                 \\
Wiki                           & {\color[HTML]{0000FF} {\ul 6.379}}       & {\color[HTML]{0000FF} {\ul 0.418}}      & 6.395                                  & {\color[HTML]{FF0000} \textbf{0.415}} & 6.443                                   & 0.439                                  & 6.457                                   & 0.439                                  & 6.420                                                  & 0.510                 & 6.430                                                    & 0.443                   & 6.417                                 & 0.433                                 & 6.419                & 0.427                                                & {\color[HTML]{FF0000} \textbf{6.318}}                   & 0.429                   & 6.422                                  & 0.432                                 & 6.368                                  & 0.424                                  & 7.633                                  & 0.572                                 \\
SP500                          & {\color[HTML]{FF0000} \textbf{0.144}}    & {\color[HTML]{FF0000} \textbf{0.249}}   & {\color[HTML]{0000FF} {\ul 0.146}}     & {\color[HTML]{0000FF} {\ul 0.250}}    & 0.153                                   & 0.265                                  & 0.164                                   & 0.279                                  & 0.178                                                  & 0.298                 & 0.157                                                    & 0.270                   & 0.163                                 & 0.282                                 & 0.147                & 0.252                                                & 0.167                                                   & 0.286                   & 0.161                                  & 0.279                                 & 0.159                                  & 0.277                                  & 0.150                                  & 0.262                                 \\
DowJones                       & {\color[HTML]{0000FF} {\ul 7.688}}       & {\color[HTML]{0000FF} {\ul 0.620}}      & {\color[HTML]{FF0000} \textbf{7.686}}  & {\color[HTML]{FF0000} \textbf{0.619}} & 8.499                                   & 0.633                                  & 8.283                                   & 0.633                                  & 7.893                                                  & 0.626                 & 8.895                                                    & 0.643                   & 8.257                                 & 0.633                                 & 7.699                & {\color[HTML]{FF0000} \textbf{0.619}}                & 8.041                                                   & 0.625                   & 8.177                                  & 0.630                                 & 7.991                                  & 0.626                                  & 10.960                                 & 0.737                                 \\
Power                          & {\color[HTML]{FF0000} \textbf{1.231}}    & {\color[HTML]{FF0000} \textbf{0.835}}   & 1.248                                  & {\color[HTML]{FF0000} \textbf{0.835}} & {\color[HTML]{0000FF} {\ul 1.234}}      & {\color[HTML]{0000FF} {\ul 0.840}}     & 1.309                                   & 0.870                                  & 1.278                                                  & 0.870                 & {\color[HTML]{0000FF} {\ul 1.234}}                       & 0.841                   & 1.295                                 & 0.868                                 & 1.288                & 0.847                                                & 1.302                                                   & 0.870                   & 1.324                                  & 0.874                                 & 1.311                                  & 0.873                                  & 1.317                                  & 0.871                                 \\
Unemp                          & {\color[HTML]{0000FF} {\ul 0.655}}       & {\color[HTML]{FF0000} \textbf{0.427}}   & 0.729                                  & {\color[HTML]{0000FF} {\ul 0.461}}    & 1.581                                   & 0.708                                  & 1.286                                   & 0.627                                  & {\color[HTML]{FF0000} \textbf{0.565}}                  & 0.509                 & 1.506                                                    & 0.678                   & 1.502                                 & 0.689                                 & 1.163                & 0.596                                                & 2.048                                                   & 0.789                   & 1.408                                  & 0.666                                 & 1.237                                  & 0.624                                  & 2.328                                  & 0.852                                 \\ \midrule
\multicolumn{1}{l|}{1\textsuperscript{st} Count} & 5                                        & 5                                       & 1                                      & 5                                     & 0                                       & 0                                      & 0                                       & 0                                      & 1                                                      & 0                     & 1                                                        & 0                       & 0                                     & 0                                     & 0                    & 1                                                    & 1                                                       & 0                       & 0                                      & 0                                     & 0                                      & 0                                      & 0                                      & 0                                     \\ \bottomrule
\end{tabular}
}
\caption{Short-term forecasting results under the `Input-12-Predict-\{3, 6, 9, 12\}' and `Input-36-Predict-\{24, 36, 48, 60\}' settings.  Results are averaged across eight prediction horizons. Full results are presented in Tables~\ref{tab_short_term_p1} and \ref{tab_short_term_p2}.}
\label{tab_short_term}
\end{center}
\end{table*}

\begin{theorem}[Simplified Inference] 
\label{theorem3}
If the FM-Net producing the velocity (e.g., WFMLin in vLinear) is linear with respect to $\hat{\mathbf{Y}}^{(t)}$, then performing inference starting from the origin (i.e., the all-zero state) is equivalent to computing the expectation of the outputs derived from Gaussian noise samples.
\end{theorem}

The proof of Theorem~\ref{theorem3} is provided in Appendix~\ref{appd_c}. It is also verified empirically in Table~\ref{tab3_ori_iters}. Based on this, the inference process with $K$ steps can be formulated as:

\begin{equation}
\hat{\mathbf{Y}} = \Delta t \sum_{k=0}^{K-1} \mathrm{WFMLin}(\hat{\mathbf{Y}}^{(k \Delta t)}, k \Delta t,\mathbf{Cond} )   , \quad \Delta t=\frac{1}{K},
\label{eq_infer_from}
\end{equation}

\noindent with $\hat{\mathbf{Y}}^{(0)}$ being the all-zero matrix at time $t=0$. Eq.~\eqref{eq_infer_from} renders the forecasting model \textbf{deterministic} \footnote{The probabilistic forecasting case is discussed in Appendix~\ref{appd_prob_forecaster}.}, and improves computational efficiency. It is noteworthy that, in contrast to the deterministic \textit{inference}, injecting noise into $\hat{\mathbf{Y}}^{(0)}$ during \textit{training} enhances the robustness of vLinear. Corresponding empirical results are presented in Appendix~\ref{appd_noise}.


\section{Experiments}

\paragraph{Datasets and Implementation Details} 
We extensively evaluate vLinear using 22 real-world datasets from diverse domains: \textbf{ETT}\{h1, h2, m1, m2\}, \textbf{ECL}, \textbf{Traffic}, \textbf{Weather},  \textbf{Exchange}, \textbf{Solar-Energy}, \textbf{PEMS}\{03, 04, 07, 08\}, \textbf{ILI}, \textbf{COVID-19}, \textbf{METR-LA}, \textbf{NASDAQ}, \textbf{Wiki}, \textbf{SP500}, \textbf{DowJones}, \textbf{Power}, and \textbf{Unemp}. vLinear is implemented using PyTorch~\cite{pytorch} and optimized with the ADAM optimizer \cite{adam_opt}.
More dataset and implementation details are presented in Appendices D and E.

\subsection{Forecasting Performance}

We choose 13 well-acknowledged state-of-the-art forecasters as baselines, including (1)  Linear-based models: \textbf{OLinear} \cite{olinear}, \textbf{TimeMixer} \cite{timemixer}, \textbf{FilterNet} \cite{filternet}, \textbf{FITS} \cite{fits}, \textbf{DLinear} \cite{linear}; (2) Transformer-based models: \textbf{TimeMixer++} \cite{timemixer++}, \textbf{Leddam} \cite{Leddam_icml},  \textbf{Fredformer} \cite{fredformer}, \textbf{iTransformer} \cite{itransformer}, \textbf{PatchTST} \cite{patchtst}, \textbf{CARD} \cite{card}; (3) TCN-based model: \textbf{TimesNet} \cite{timesnet}.

\begin{table}[t]
\begin{center}
{\fontsize{7}{8}\selectfont
\setlength{\tabcolsep}{2.7pt}
\begin{tabular}{@{}c|cc|cc|cc|cc|cc@{}}
\toprule
Model        & \multicolumn{2}{c|}{\begin{tabular}[c]{@{}c@{}}vLinear\\      (Ours)\end{tabular}} 
& \multicolumn{2}{c|}{\begin{tabular}[c]{@{}c@{}}SimpleTM\\    \shortcite{simpletm} \end{tabular}} 
& \multicolumn{2}{c|}{\begin{tabular}[c]{@{}c@{}}TQNet\\      \shortcite{tqnet} \end{tabular}} 
& \multicolumn{2}{c|}{\begin{tabular}[c]{@{}c@{}}TimePro\\      \shortcite{timepro} \end{tabular}} 
& \multicolumn{2}{c}{\begin{tabular}[c]{@{}c@{}}TimeBase\\      \shortcite{timebase} \end{tabular}} \\ \midrule
Metric       & MSE                                     & MAE                                     & MSE                                     & MAE                                  & MSE                                  & MAE                                  & MSE                                                   & MAE                   & MSE                                    & MAE                                   \\ \midrule
ETTm1        & {\color[HTML]{FF0000} \textbf{0.369}}   & {\color[HTML]{FF0000} \textbf{0.378}}   & 0.381                                   & 0.396                                & {\color[HTML]{0000FF} {\ul 0.377}}   & {\color[HTML]{0000FF} {\ul 0.393}}   & 0.391                                                 & 0.400                 & 0.431                                  & 0.420                                 \\
ETTm2        & {\color[HTML]{FF0000} \textbf{0.268}}   & {\color[HTML]{FF0000} \textbf{0.310}}   & {\color[HTML]{0000FF} {\ul 0.275}}      & {\color[HTML]{0000FF} {\ul 0.322}}   & 0.277                                & 0.323                                & 0.281                                                 & 0.326                 & 0.290                                  & 0.332                                 \\
ETTh1        & {\color[HTML]{FF0000} \textbf{0.416}}   & {\color[HTML]{FF0000} \textbf{0.420}}   & {\color[HTML]{0000FF} {\ul 0.422}}      & {\color[HTML]{0000FF} {\ul 0.428}}   & 0.441                                & 0.434                                & 0.438                                                 & 0.438                 & 0.463                                  & 0.429                                 \\
ETTh2        & {\color[HTML]{0000FF} {\ul 0.358}}      & {\color[HTML]{FF0000} \textbf{0.385}}   & {\color[HTML]{FF0000} \textbf{0.353}}   & {\color[HTML]{0000FF} {\ul 0.391}}   & 0.378                                & 0.402                                & 0.377                                                 & 0.403                 & 0.408                                  & 0.424                                 \\
ECL          & {\color[HTML]{FF0000} \textbf{0.153}}   & {\color[HTML]{FF0000} \textbf{0.245}}   & 0.166                                   & 0.260                                & {\color[HTML]{0000FF} {\ul 0.164}}   & {\color[HTML]{0000FF} {\ul 0.259}}   & 0.169                                                 & 0.262                 & 0.227                                  & 0.295                                 \\
Traffic      & {\color[HTML]{FF0000} \textbf{0.440}}   & {\color[HTML]{FF0000} \textbf{0.252}}   & {\color[HTML]{0000FF} {\ul 0.444}}      & 0.289                                & 0.445                                & {\color[HTML]{0000FF} {\ul 0.276}}   & {\color[HTML]{FF0000} \textbf{0.440}}                 & 0.286                 & 0.682                                  & 0.374                                 \\
Weather      & {\color[HTML]{FF0000} \textbf{0.233}}   & {\color[HTML]{FF0000} \textbf{0.256}}   & 0.243                                   & 0.271                                & {\color[HTML]{0000FF} {\ul 0.242}}   & {\color[HTML]{0000FF} {\ul 0.269}}   & 0.251                                                 & 0.276                 & 0.252                                  & 0.279                                 \\
Solar & 0.209                                   & {\color[HTML]{FF0000} \textbf{0.226}}   & {\color[HTML]{FF0000} \textbf{0.184}}   & {\color[HTML]{0000FF} {\ul 0.247}}   & {\color[HTML]{0000FF} {\ul 0.198}}   & 0.256                                & 0.232                                                 & 0.266                 & 0.404                                  & 0.388                                 \\ \bottomrule
\end{tabular}
}
\caption{Comparison with additional state-of-the-art forecasters. Average results are reported across four prediction horizons.}
\label{tab_more_baseline}
\end{center}
\end{table}

Tables~\ref{tab_long_term} and \ref{tab_short_term} present comprehensive results of long- and short-term forecasting. vLinear consistently achieves state-of-the-art performance across various benchmarks and prediction settings. Note that vLinear performs competitively on datasets with a large number of variates (e.g., PEMS, ECL, and Traffic), indicating that our vecTrans module effectively captures multivariate correlations even with a single vector. This observation \textbf{challenges the necessity of computationally expensive attention-based approaches} for multivariate correlation modeling. Furthermore, the performance advantage of vLinear on small-scale datasets (e.g., COVID-19, Power, and Unemp) demonstrates its effectiveness under limited training data.

We also compare vLinear against additional competitive baselines in Table~\ref{tab_more_baseline}, where vLinear demonstrates superior performance. In Appendix~\ref{robust}, vLinear demonstrates better robustness to random seeds when compared to state-of-the-art forecasters: TimeMixer++ and iTransformer.

\begin{figure}[t]
   \centering
   \includegraphics[width=0.95\linewidth]{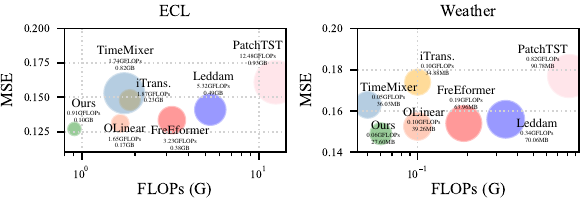}
   \caption{Model efficiency comparison. Bubble areas represent GPU memory usage during inference. The `Input-96-Predict-96' setting is used. Resource footprint data are from Table~\ref{tab_gpu}.}
   \label{fig_gpu}
\end{figure}

\begin{figure}[t]
   \centering
   \includegraphics[width=1.0\linewidth]{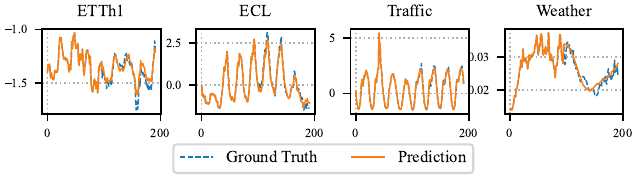}
   \caption{Visualization of our forecasting results.}
   \label{fig_plot}
\end{figure}

\paragraph{Efficiency}
As shown in Figure~\ref{fig_gpu}, vLinear incurs fewer FLOPs and consumes less GPU memory during inference, while achieving superior performance. This observation is consistent with the analysis in Table~\ref{tab1_flops_vecTrans}.

\subsection{Model Analysis}

\paragraph{Representation Learning}
As shown in Figure~\ref{fig3_overall_arch}, we employ the vecTrans module and MLP to capture multivariate dependencies and temporal dynamics, respectively. To verify the effectiveness of this design, we conduct ablation studies on the representation learning of these two dimensions.
Specifically, when the vecTrans layer is applied to the temporal dimension (as a variant), the operation in Eq.~\eqref{eq:vecTrans} is performed on the dimension of size $D$ with the learnable vector $\mathbf{a} \in \mathbb{R}^D$. 
As shown in Table\ref{tab_var_temp}, our design (vecTrans-MLP) outperforms the self-attention counterpart (Attn.-MLP, in the last row) by 3.4\% on average, with reduced computational cost. Furthermore, removing the vecTrans module (second row) leads to an \textbf{11\%} performance degradation, highlighting its effectiveness. Moreover, on datasets with a large number of variates (ECL, Solar, and PEMS03), vecTrans exhibits a larger performance margin (about 5\%) over the second-best variant, demonstrating the efficacy of our multivariate correlation learning scheme in the vecTrans module.

\begin{table}[t]
\begin{center}
{\fontsize{7.5}{8}\selectfont
\setlength{\tabcolsep}{4pt}
\begin{tabular}{@{}cc|cccccc@{}}
\toprule
Var.     & Temp.    & ETTh1                                 & ETTm2                                 & ECL                                   & Solar                                 & Weather                               & PEMS03                                \\ \midrule
w/o      & vecTrans & 0.423                                 & 0.274                                 & 0.186                                 & 0.262                                 & 0.253                                 & 0.192                                 \\
w/o      & MLP      & 0.420                                 & 0.274                                 & 0.173                                 & 0.235                                 & 0.244                                 & 0.146                                 \\ \midrule
MLP      & w/o      & 0.420                                 & 0.274                                 & 0.173                                 & 0.235                                 & 0.244                                 & 0.146                                 \\
MLP      & vecTrans & 0.420                                 & 0.273                                 & 0.174                                 & 0.235                                 & 0.244                                 & 0.148                                 \\
MLP      & MLP      & {\color[HTML]{0000FF} {\ul 0.419}}    & 0.274                                 & 0.174                                 & 0.232                                 & 0.245                                 & 0.153                                 \\ \midrule
vecTrans & w/o      & 0.423                                 & {\color[HTML]{0000FF} {\ul 0.270}}    & 0.168                                 & 0.237                                 & 0.249                                 & 0.115                                 \\
vecTrans & vecTrans & 0.423                                 & 0.271                                 & 0.168                                 & 0.235                                 & 0.249                                 & 0.113                                 \\
vecTrans & MLP      & {\color[HTML]{FF0000} \textbf{0.416}} & {\color[HTML]{FF0000} \textbf{0.268}} & {\color[HTML]{FF0000} \textbf{0.153}} & {\color[HTML]{FF0000} \textbf{0.209}} & {\color[HTML]{FF0000} \textbf{0.233}} & {\color[HTML]{FF0000} \textbf{0.093}} \\ \midrule
Attn.    & MLP      & 0.422                                 & 0.275                                 & {\color[HTML]{0000FF} {\ul 0.159}}    & {\color[HTML]{0000FF} {\ul 0.220}}    & {\color[HTML]{0000FF} {\ul 0.237}}    & {\color[HTML]{0000FF} {\ul 0.100}}    \\ \bottomrule
\end{tabular}
}
\caption{Ablation study on the representation learning of variate (\textit{Var.}) and temporal (\textit{Temp.}) dimensions. Average MSEs across four horizons are reported. Full results are provided in Table~\ref{tab_var_tmp_appd}.}
\label{tab_var_temp}
\end{center}
\end{table}

\begin{figure}[ht]
   \centering
   \includegraphics[width=1.0\linewidth]{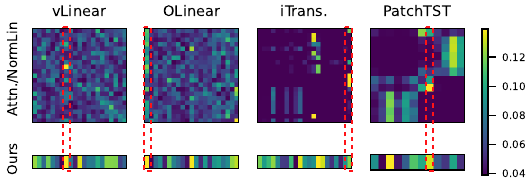}
   \caption{Visualization of Attention/NormLin matrices and the learned vector in vecTrans on the Weather dataset.  As variants, vecTrans is applied to OLinear, iTransformer, and PatchTST, while NormLin is applied to vLinear. The learned vector approximates the column-wise weight distribution of the matrices.}
   \label{fig_vec_visual}
\end{figure}

\paragraph{vecTrans Visualization}
Figure~\ref{fig_vec_visual} shows that the learned vector approximates the column-wise average of the attention and NormLin matrices. This structural similarity justifies replacing computationally expensive matrix-based schemes with our lightweight vector-based strategy.


\begin{table}[ht]
\begin{center}
{\fontsize{7}{8}\selectfont
\setlength{\tabcolsep}{2.1pt}
\begin{tabular}{@{}c|ccccccccc@{}}
\toprule
Dataset & \begin{tabular}[c]{@{}c@{}}vecT.\\      (Ours)\end{tabular} 
& \begin{tabular}[c]{@{}c@{}}Attn.\\ \shortcite{transformer} \end{tabular} 
& \begin{tabular}[c]{@{}c@{}}G.Attn.\\  \shortcite{gatedattn} \end{tabular} 
& \begin{tabular}[c]{@{}c@{}}E.Attn.\\ \shortcite{freeformer} \end{tabular} 
& \begin{tabular}[c]{@{}c@{}}Flow.\\    \shortcite{flowformer} \end{tabular} 
& \begin{tabular}[c]{@{}c@{}}Flash.\\ \shortcite{flashattention} \end{tabular} 
& \begin{tabular}[c]{@{}c@{}}FLattn.\\   \shortcite{flattentrans} \end{tabular} 
& \begin{tabular}[c]{@{}c@{}}L.Attn.\\    \shortcite{linear_softmax} \end{tabular} 
& \begin{tabular}[c]{@{}c@{}}Mamba\\   \shortcite{mamba} \end{tabular} \\ \midrule
ETTh1   & {\color[HTML]{FF0000} \textbf{0.416}}                       & 0.422                                                   & 0.423                                                     & {\color[HTML]{0000FF} {\ul 0.420}}                        & {\color[HTML]{0000FF} {\ul 0.420}}                      & 0.423                                                    & 0.428                                                     & 0.425                                                     & 0.425                                                   \\
ETTm2   & {\color[HTML]{FF0000} \textbf{0.268}}                       & 0.275                                                   & 0.273                                                     & {\color[HTML]{0000FF} {\ul 0.272}}                        & {\color[HTML]{0000FF} {\ul 0.272}}                      & 0.275                                                    & 0.274                                                     & {\color[HTML]{0000FF} {\ul 0.272}}                        & {\color[HTML]{0000FF} {\ul 0.272}}                      \\
ECL     & {\color[HTML]{FF0000} \textbf{0.153}}                       & 0.159                                                   & {\color[HTML]{0000FF} {\ul 0.156}}                        & {\color[HTML]{FF0000} \textbf{0.153}}                     & 0.158                                                   & 0.159                                                    & 0.158                                                     & 0.157                                                     & 0.169                                                   \\
Solar   & {\color[HTML]{FF0000} \textbf{0.209}}                       & 0.220                                                   & 0.226                                                     & {\color[HTML]{0000FF} {\ul 0.211}}                        & 0.218                                                   & 0.220                                                    & 0.223                                                     & 0.223                                                     & 0.224                                                   \\
Weather & {\color[HTML]{FF0000} \textbf{0.233}}                       & 0.237                                                   & 0.237                                                     & 0.235                                                     & {\color[HTML]{0000FF} {\ul 0.234}}                      & 0.237                                                    & 0.241                                                     & 0.239                                                     & 0.256                                                   \\
PEMS03  & {\color[HTML]{FF0000} \textbf{0.093}}                       & 0.100                                                   & 0.099                                                     & {\color[HTML]{0000FF} {\ul 0.095}}                        & 0.099                                                   & 0.100                                                    & 0.102                                                     & 0.101                                                     & 0.106                                                   \\ \bottomrule
\end{tabular}
}
\caption{Performance comparison between vecTrans and various attention mechanisms and Mamba. Average MSEs are reported.}
\label{tab_vec_attn}
\end{center}
\end{table}

\paragraph{vecTrans vs. Attention Variants}
We further compare vecTrans with vanilla self-attention and its variants, including Gated Attention \cite{gatedattn} and linear attentions \cite{flattentrans,linear_softmax}, as well as Mamba \cite{mamba}, within the vLinear framework. Table~\ref{tab_vec_attn} demonstrates that vecTrans consistently outperforms these baselines (with reduced computational complexity), highlighting its strong representation capabilities.

\begin{table}[ht]
\begin{center}
{\fontsize{8}{9}\selectfont
\setlength{\tabcolsep}{4pt}
\begin{tabular}{@{}c|cc|cc|cc|cc@{}}
\toprule
Rank    & \multicolumn{2}{c|}{$k$=1}                                                 & \multicolumn{2}{c|}{$k$=4}                                                      & \multicolumn{2}{c|}{$k$=16}                                                  & \multicolumn{2}{c}{$k$=32}                                                      \\ \midrule
Metric  & MSE                                   & MAE                                   & MSE                                   & MAE                                   & MSE                                   & MAE                                & MSE                                   & MAE                                   \\ \midrule
ETTh1   & {\color[HTML]{FF0000} \textbf{0.416}} & {\color[HTML]{FF0000} \textbf{0.420}} & 0.420                                 & 0.422                                 & {\color[HTML]{0000FF} {\ul 0.419}}    & {\color[HTML]{0000FF} {\ul 0.421}} & {\color[HTML]{0000FF} {\ul 0.419}}    & 0.422                                 \\
ETTm2   & {\color[HTML]{FF0000} \textbf{0.268}} & {\color[HTML]{FF0000} \textbf{0.310}} & {\color[HTML]{0000FF} {\ul 0.271}}    & 0.312                                 & 0.273                                 & 0.312                              & {\color[HTML]{0000FF} {\ul 0.271}}    & {\color[HTML]{0000FF} {\ul 0.311}}    \\
ECL     & {\color[HTML]{FF0000} \textbf{0.153}} & {\color[HTML]{0000FF} {\ul 0.245}}    & {\color[HTML]{FF0000} \textbf{0.153}} & {\color[HTML]{0000FF} {\ul 0.245}}    & {\color[HTML]{FF0000} \textbf{0.153}} & {\color[HTML]{0000FF} {\ul 0.245}} & {\color[HTML]{FF0000} \textbf{0.153}} & {\color[HTML]{FF0000} \textbf{0.244}} \\
Solar   & {\color[HTML]{FF0000} \textbf{0.209}} & {\color[HTML]{FF0000} \textbf{0.226}} & {\color[HTML]{0000FF} {\ul 0.210}}    & {\color[HTML]{FF0000} \textbf{0.226}} & 0.211                                 & {\color[HTML]{0000FF} {\ul 0.227}} & 0.211                                 & {\color[HTML]{0000FF} {\ul 0.227}}    \\
Weather & {\color[HTML]{FF0000} \textbf{0.233}} & {\color[HTML]{FF0000} \textbf{0.256}} & {\color[HTML]{FF0000} \textbf{0.233}} & {\color[HTML]{FF0000} \textbf{0.256}} & {\color[HTML]{0000FF} {\ul 0.234}}    & {\color[HTML]{0000FF} {\ul 0.257}} & 0.245                                 & 0.266                                 \\
PEMS03  & {\color[HTML]{FF0000} \textbf{0.093}} & {\color[HTML]{FF0000} \textbf{0.197}} & {\color[HTML]{0000FF} {\ul 0.095}}    & {\color[HTML]{0000FF} {\ul 0.198}}    & {\color[HTML]{0000FF} {\ul 0.095}}    & 0.199                              & 0.096                                 & 0.199                                 \\ \bottomrule
\end{tabular}
}
\caption{Comparison of different rank settings for vecTrans.}
\label{tab_more_ranks}
\end{center}
\end{table}

\paragraph{Larger Rank in vecTrans} \label{more_ranks}
Relaxing the rank-1 constraint of vecTrans to rank-$k$, Eq.~\eqref{eq:vecTrans} generalizes to:

\begin{equation}
\mathrm{vecTrans}\left ( \mathbf{H}  \right )  = \frac{\mathbf{A}   \left [ \left ( \mathtt{Sigmoid} \left ( \mathbf{B}  \right )  \right )^{\mathsf{T} }   \mathbf{H} \right ]}{\mathbf{A}   \left [ \left ( \mathtt{Sigmoid} \left ( \mathbf{B}  \right )  \right )^{\mathsf{T} }  \mathbf{1} \right ]}  ,
\label{eq:vecTrans_larger_rank}
\end{equation}

\noindent where $\mathbf{A}, \mathbf{B} \in \mathbb{R}^{N \times k}$ are the learnable matrices, $\mathbf{H} \in \mathbb{R}^{N \times D}$, and $\mathbf{1} \in \mathbb{R}^{N}$. The denominator in Eq.~\eqref{eq:vecTrans_larger_rank} is used for the row-wise L1 normalization.  
As shown in Table \ref{tab_more_ranks}, various rank settings yield similar results, with the rank-1 setting performing slightly better than the others. Therefore, we adopt Eq.~\eqref{eq:vecTrans} (rank-1) as the default setting in this work. This observation indicates that a rank-1 matrix is sufficient to capture multivariate correlations.

\begin{table}[hbt]
\begin{center}
{\fontsize{7.5}{8.5}\selectfont
\setlength{\tabcolsep}{4pt}
\begin{tabular}{@{}cccccccc@{}}
\toprule
Dataset & \begin{tabular}[c]{@{}c@{}}WFM\\      (Ours)\end{tabular} 
& \begin{tabular}[c]{@{}c@{}}TransDF\\ \shortcite{transdf} \end{tabular} 
& \begin{tabular}[c]{@{}c@{}}DBLoss\\   \shortcite{dbloss} \end{tabular} 
& \begin{tabular}[c]{@{}c@{}}FreDF\\ \shortcite{fredf} \end{tabular} 
& \begin{tabular}[c]{@{}c@{}}W\_MAE\\  \shortcite{card} \end{tabular} 
& \begin{tabular}[c]{@{}c@{}}MAE\end{tabular} & \begin{tabular}[c]{@{}c@{}}MSE\end{tabular} \\ \midrule
ETTh1   & {\color[HTML]{FF0000} \textbf{0.416}}                     & 0.500                                                     & 0.446                                                    & 0.498                                                   & {\color[HTML]{0000FF} {\ul 0.446}}                       & 0.447                                                 & 0.450                                                 \\
ETTm2   & {\color[HTML]{FF0000} \textbf{0.268}}                     & 0.291                                                     & 0.297                                                    & 0.285                                                   & 0.278                                                    & {\color[HTML]{0000FF} {\ul 0.276}}                    & 0.290                                                 \\
ECL     & {\color[HTML]{FF0000} \textbf{0.153}}                     & 0.164                                                     & 0.168                                                    & 0.172                                                   & {\color[HTML]{0000FF} {\ul 0.160}}                       & {\color[HTML]{0000FF} {\ul 0.160}}                    & 0.162                                                 \\
Solar   & {\color[HTML]{FF0000} \textbf{0.209}}                     & 0.220                                                     & 0.370                                                    & 0.367                                                   & {\color[HTML]{0000FF} {\ul 0.214}}                       & 0.215                                                 & 0.220                                                 \\
Weather & {\color[HTML]{FF0000} \textbf{0.233}}                     & 0.244                                                     & 0.256                                                    & 0.241                                                   & 0.239                                                    & {\color[HTML]{0000FF} {\ul 0.238}}                    & 0.243                                                 \\
PEMS03  & {\color[HTML]{FF0000} \textbf{0.093}}                     & 0.098                                                     & 0.096                                                    & 0.096                                                   & {\color[HTML]{0000FF} {\ul 0.095}}                       & {\color[HTML]{0000FF} {\ul 0.095}}                    & 0.098                                                 \\ \bottomrule
\end{tabular}
}
\caption{WFMLoss vs. other losses. Average MSEs are reported.}
\label{tab_wfmloss_compare}
\end{center}
\end{table}

\paragraph{WFMLoss vs. Other Losses}
We compare our WFMLoss against other objectives under the framework of vLinear. As shown in Table~\ref{tab_wfmloss_compare}, WFMLoss outperforms the second-best loss by 3.2\% on average.

\begin{table}[ht]
\begin{center}
{\fontsize{7}{8}\selectfont
\setlength{\tabcolsep}{3.5pt}
\begin{tabular}{@{}cc|cc|cccccc@{}}
\toprule
P.W & H.W. & MSE & MAE & ETTh1                                 & ETTm2                                 & ECL                                   & Solar                                 & Weather                               & PEMS03                                \\ \midrule
\ding{55}   & \ding{55}    & \ding{52}   &     & 0.444                                 & 0.317                                 & 0.167                                 & 0.220                                 & 0.276                                 & 0.099                                 \\
\ding{55}   & \ding{52}    & \ding{52}   &     & 0.444                                 & 0.312                                 & 0.171                                 & 0.219                                 & 0.276                                 & 0.319                                 \\
\ding{52}   & \ding{55}    & \ding{52}   &     & 0.442                                 & 0.316                                 & 0.165                                 & 0.219                                 & 0.272                                 & 0.098                                 \\
\ding{52}   & \ding{52}    & \ding{52}   &     & 0.444                                 & 0.313                                 & 0.170                                 & 0.218                                 & 0.274                                 & 0.098                                 \\ \midrule
\ding{55}   & \ding{55}    &     & \ding{52}   & 0.419                                 & 0.270                                 & {\color[HTML]{0000FF} {\ul 0.154}}    & 0.211                                 & 0.235                                 & 0.095                                 \\
\ding{55}   & \ding{52}    &     & \ding{52}   & 0.418                                 & {\color[HTML]{0000FF} {\ul 0.269}}    & {\color[HTML]{0000FF} {\ul 0.154}}    & {\color[HTML]{0000FF} {\ul 0.210}}    & {\color[HTML]{0000FF} {\ul 0.234}}    & {\color[HTML]{0000FF} {\ul 0.094}}    \\
\ding{52}   & \ding{55}    &     & \ding{52}   & {\color[HTML]{0000FF} {\ul 0.418}}    & 0.270                                 & {\color[HTML]{0000FF} {\ul 0.154}}    & 0.211                                 & 0.246                                 & 0.095                                 \\
\ding{52}   & \ding{52}    &     & \ding{52}   & {\color[HTML]{FF0000} \textbf{0.416}} & {\color[HTML]{FF0000} \textbf{0.268}} & {\color[HTML]{FF0000} \textbf{0.153}} & {\color[HTML]{FF0000} \textbf{0.209}} & {\color[HTML]{FF0000} \textbf{0.233}} & {\color[HTML]{FF0000} \textbf{0.093}} \\ \bottomrule
\end{tabular}
}
\caption{Ablation study on WFMLoss design. }
\label{tab_wfmloss_ablation}
\end{center}
\end{table}

\paragraph{WFMLoss Design}

We conduct ablation studies to analyze the impact of path-weighting (P.W.), horizon-weighting (H.W.), and the MAE objective. 
As shown in Table~\ref{tab_wfmloss_ablation}, the proposed scheme outperforms the variants without P.W., without H.W., and the MSE-based counterpart by 6\%, 1\%, and 14\%, respectively. These consistent improvements verify the effectiveness of our WFMLoss design.

\subsection{Generalizability of vecTrans and WFMLoss}

\begin{table}[ht]
\begin{center}
{\fontsize{6.5}{7}\selectfont
\setlength{\tabcolsep}{3.5pt}
\begin{tabular}{@{}c|cc|cc|cc|cc@{}}
\toprule
                          & \multicolumn{2}{c|}{iTransformer}                  & \multicolumn{2}{c|}{PatchTST}                 & \multicolumn{2}{c|}{Leddam}                                                   & \multicolumn{2}{c}{Fredformer}                                                \\ \cmidrule(l){2-9} 
\multirow{-2}{*}{Dataset} & Attn. & vecTrans                              & Attn. & vecTrans                              & Attn.                                 & vecTrans                              & Attn.                                 & vecTrans                              \\ \midrule
ETTm1                     & 0.407 & {\color[HTML]{FF0000} \textbf{0.391}} & 0.387 & {\color[HTML]{FF0000} \textbf{0.379}} & 0.386                                 & {\color[HTML]{FF0000} \textbf{0.379}} & {\color[HTML]{FF0000} \textbf{0.384}} & {\color[HTML]{FF0000} \textbf{0.384}} \\
ECL                       & 0.178 & {\color[HTML]{FF0000} \textbf{0.166}} & 0.208 & {\color[HTML]{FF0000} \textbf{0.180}} & 0.169                                 & {\color[HTML]{FF0000} \textbf{0.163}} & 0.176                                 & {\color[HTML]{FF0000} \textbf{0.171}} \\
PEMS03                    & 0.113 & {\color[HTML]{FF0000} \textbf{0.108}} & 0.180 & {\color[HTML]{FF0000} \textbf{0.149}} & 0.107                                 & {\color[HTML]{FF0000} \textbf{0.104}} & 0.134                                 & {\color[HTML]{FF0000} \textbf{0.128}} \\
PEMS07                    & 0.101 & {\color[HTML]{FF0000} \textbf{0.090}} & 0.211 & {\color[HTML]{FF0000} \textbf{0.153}} & 0.084                                 & {\color[HTML]{FF0000} \textbf{0.082}} & 0.121                                 & {\color[HTML]{FF0000} \textbf{0.118}} \\
Solar                     & 0.233 & {\color[HTML]{FF0000} \textbf{0.223}} & 0.270 & {\color[HTML]{FF0000} \textbf{0.234}} & 0.230                                 & {\color[HTML]{FF0000} \textbf{0.223}} & 0.226                                 & {\color[HTML]{FF0000} \textbf{0.225}} \\
Weather                   & 0.258 & {\color[HTML]{FF0000} \textbf{0.248}} & 0.259 & {\color[HTML]{FF0000} \textbf{0.245}} & {\color[HTML]{FF0000} \textbf{0.242}} & {\color[HTML]{FF0000} \textbf{0.242}} & 0.246                                 & {\color[HTML]{FF0000} \textbf{0.239}} \\
METR-LA                   & 0.338 & {\color[HTML]{FF0000} \textbf{0.332}} & 0.335 & {\color[HTML]{FF0000} \textbf{0.333}} & 0.327                                 & {\color[HTML]{FF0000} \textbf{0.320}} & {\color[HTML]{FF0000} \textbf{0.336}} & 0.341                                 \\ \bottomrule
\end{tabular}
}
\caption{Plugging the vecTrans module into Transformer-based forecasters to replace the attention mechanism.}
\label{tab_vecTrans_plugin}
\end{center}
\end{table}

\begin{table}[ht]
\begin{center}
{\fontsize{6.5}{7}\selectfont
\setlength{\tabcolsep}{3.5pt}
\begin{tabular}{@{}c|cc|cc|cc|cc@{}}
\toprule
\multicolumn{1}{c|}{}                          & \multicolumn{2}{c|}{OLinear}                                                  & \multicolumn{2}{c|}{\begin{tabular}[c]{@{}c@{}}iTransformer \end{tabular}} & \multicolumn{2}{c|}{\begin{tabular}[c]{@{}c@{}}PatchTST\end{tabular}} & \multicolumn{2}{c}{DLinear}                   \\ \cmidrule(l){2-9} 
\multicolumn{1}{c|}{\multirow{-2}{*}{Dataset}} & Ori.                                  & WFMLoss                               & Ori.                     & WFMLoss                                                  & Ori.                   & WFMLoss                                                & Ori.  & WFMLoss                               \\ \midrule
ETTh1                                          & 0.424                                 & {\color[HTML]{FF0000} \textbf{0.417}} & 0.454                    & {\color[HTML]{FF0000} \textbf{0.421}}                    & 0.469                  & {\color[HTML]{FF0000} \textbf{0.417}}                  & 0.456 & {\color[HTML]{FF0000} \textbf{0.434}} \\
ETTm2                                          & {\color[HTML]{FF0000} \textbf{0.270}} & {\color[HTML]{FF0000} \textbf{0.270}} & 0.288                    & {\color[HTML]{FF0000} \textbf{0.281}}                    & 0.281                  & {\color[HTML]{FF0000} \textbf{0.278}}                  & 0.350 & {\color[HTML]{FF0000} \textbf{0.285}} \\
ECL                                            & 0.159                                 & {\color[HTML]{FF0000} \textbf{0.152}} & 0.178                    & {\color[HTML]{FF0000} \textbf{0.173}}                    & 0.208                  & {\color[HTML]{FF0000} \textbf{0.193}}                  & 0.212 & {\color[HTML]{FF0000} \textbf{0.211}} \\
Traffic                                        & 0.451                                 & {\color[HTML]{FF0000} \textbf{0.441}} & 0.428                    & {\color[HTML]{FF0000} \textbf{0.420}}                    & 0.531                  & {\color[HTML]{FF0000} \textbf{0.470}}                  & 0.625 & {\color[HTML]{FF0000} \textbf{0.623}} \\
Weather                                        & 0.237                                 & {\color[HTML]{FF0000} \textbf{0.233}} & 0.258                    & {\color[HTML]{FF0000} \textbf{0.250}}                    & 0.259                  & {\color[HTML]{FF0000} \textbf{0.254}}                  & 0.265 & {\color[HTML]{FF0000} \textbf{0.259}} \\ \bottomrule
\end{tabular}
}
\caption{Plugging WFMLoss into existing forecasters. Average MSEs are reported. Full results are presented in Table~\ref{tab_wfmloss_plugin_appd}.}
\label{tab_wfmloss_plugin}
\end{center}
\end{table}

As presented in Tables~\ref{tab_vecTrans_plugin} and \ref{tab_wfmloss_plugin}, integrating the vecTrans module and WFMLoss into existing forecasters consistently yields performance improvements. Specifically, for PatchTST, these two modules contribute \textbf{11.4\%} and \textbf{6.6\%} performance gains, respectively.  Notably, substituting the attention mechanism with the vecTrans module yields \textbf{5.8$\times$} and  \textbf{3.4$\times$} inference speedups for iTransformer on the PEMS07 and ECL datasets, respectively (see Table~\ref{tab_gpu2}).

We provide additional experiments in the appendix, including probabilistic forecasting (Appendix~\ref{appd_prob_forecaster}), vecTrans variants (Appendix~\ref{appd_ablation_vectrans}), hyperparameter sensitivity analysis (Appendix~\ref{hyper-para}), and evaluations with more metrics (Appendix~\ref{appd_more_metrics}).

\section{Conclusion}

In this paper, we present vLinear, a powerful yet efficient linear-based time series forecaster featuring the vecTrans and WFMLoss designs. vecTrans captures multivariate correlations with linear complexity regarding the number of variates. WFMLoss is a final-series-oriented, path- and horizon-weighted flow matching loss that achieves superior performance. Furthermore, both modules can be integrated into existing forecasters to enhance their performance. We hope our designs will inspire future work in the time series community.

\nocite{xu2025neutsflow,fm_ot,quantileformer,CSDI,diff_tsf,wu2025k2vae,wu2025srsnet,li2025towards,wang2025optimal,wang2025effective}

\appendix

\section{Theorem 1 and Its Proof} \label{appd_a}

\begin{theorem}[Rank-1 Row-Normalized Matrix] \label{theorem1_appd}

Suppose that $\mathbf{A} \in \mathbb{R}^{N \times N}$ is a row-wise L1-normalized matrix with no zero rows. If $\mathrm{rank}(\mathbf{A}) = 1$, then $\mathbf{A}$ must be of the form $\mathbf{A} = \mathbf{1}\mathbf{a}^{\mathsf{T}},$
where $\mathbf{1} \in \mathbb{R}^N$ is the all-one vector, $\mathbf{a} \in \mathbb{R}^N$, and $\|\mathbf{a}\|_1 = 1$.
    
\end{theorem}

\begin{proof}
Since $\mathrm{Rank}(\mathbf{A}) = 1$, all rows of $\mathbf{A}$ lie in a one-dimensional subspace of $\mathbb{R}^N$. Hence, there exists a nonzero vector $\mathbf{a} \in \mathbb{R}^N$ and scalars $\{\lambda_i\}_{i=1}^N$ such that
\[
\mathbf{A}_{i,:} = \lambda_i \mathbf{a}^{\mathsf{T}}, \quad i = 1, \dots, N.
\]
Equivalently, $\mathbf{A}$ can be written as
\[
\mathbf{A} = \boldsymbol{\lambda} \, \mathbf{a}^{\mathsf{T}},
\]
where $\boldsymbol{\lambda} = (\lambda_1, \dots, \lambda_N)^{\mathsf{T}} \in \mathbb{R}^N$.

Since $\mathbf{A}$ has no zero rows, we have $\lambda_i \neq 0$ for all $i$. Moreover, the L1 normalization of each row yields
\[
\|\mathbf{A}_{i,:}\|_1 = \|\lambda_i \mathbf{a}^{\mathsf{T}}\|_1 = |\lambda_i| \, \|\mathbf{a}\|_1 = 1, \quad \forall i.
\]
Therefore,
\[
|\lambda_i| = \frac{1}{\|\mathbf{a}\|_1}, \quad \forall i,
\]
which implies that all $\lambda_i$ have the same magnitude. Absorbing any global sign into $\mathbf{a}$ without loss of generality, we may assume
\[
\lambda_i = \frac{1}{\|\mathbf{a}\|_1}, \quad \forall i.
\]
Substituting back, we obtain
\[
\mathbf{A} = \frac{1}{\|\mathbf{a}\|_1} \, \mathbf{1} \, \mathbf{a}^{\mathsf{T}}.
\]
Finally, by redefining $\mathbf{a} \leftarrow \mathbf{a} / \|\mathbf{a}\|_1$, the desired form follows:
\[
\mathbf{A} = \mathbf{1} \mathbf{a}^{\mathsf{T}},
\]
with $\|\mathbf{a}\|_1 = 1$.
\end{proof}

\section{Theorem 2 and Its Proof}\label{appd_b}

\begin{theorem}[Horizon-Weighted MAE Loss] 
\label{theorem2_appd}
Consider the univariate series $\mathbf{y}^{\mathrm{GT}} \in \mathbb{R}^H$. Suppose that the series is a first-order Markov process and the conditional probability density follows a Laplace distribution. Then, horizon-weighted MAE loss $\sum_{i=1}^{H}i^{-0.5} \left | \hat{\mathrm{y}_i}- \mathrm{y}^{\mathrm{GT}}_i  \right | $ is optimal under the maximum likelihood criterion.
\end{theorem}

\begin{proof}
Let $\mathbf{y}^{\mathrm{GT}} = (y_1, \ldots, y_H)$ denote the ground-truth future values, and let $\hat{\mathbf{y}} = (\hat y_1, \ldots, \hat y_H)$ be the model prediction.

Assume that $\{y_i\}_{i=1}^H$ follows a first-order Markov process defined by
\[
y_i = y_{i-1} + \varepsilon_i, \quad i = 1, \ldots, H,
\]
where $\{\varepsilon_i\}$ are i.i.d. random variables with mean zero and variance $\sigma^2$.
Since $y_0$ is observed and fixed, by recursion, we have
\[
y_i = y_0 + \sum_{k=1}^i \varepsilon_k,
\]
which implies
\[
\mathrm{Var}(y_i) = \sum_{k=1}^i \mathrm{Var}(\varepsilon_k) = i\sigma^2.
\]

Suppose the conditional distribution of each $y_i$ given the past is Laplace:
\[
p(y_i \mid y_{i-1}) = \frac{1}{2b_i} \exp\left(-\frac{|y_i - \hat y_i|}{b_i}\right),
\]
where the scale parameter $b_i$ is proportional to the standard deviation of $y_i$.
Since $\mathrm{Var}(y_i) = 2b_i^2$ for a Laplace distribution, we have
\[
b_i = \frac{\sqrt{i}\sigma}{\sqrt{2}} \propto \sqrt{i}.
\]

By the Markov assumption, the joint conditional density of $\mathbf{y}^{\mathrm{GT}}$ can be expressed as
\[
p(\mathbf{y}^{\mathrm{GT}} \mid \hat{\mathbf{y}})
= \prod_{i=1}^H p(y_i \mid y_{i-1})
= \prod_{i=1}^H \frac{1}{2b_i}
\exp\left(-\frac{|y_i - \hat y_i|}{b_i}\right).
\]

Taking the negative log-likelihood and ignoring constants independent of $\hat{\mathbf{y}}$, we obtain
\[
-\log p(\mathbf{y}^{\mathrm{GT}} \mid \hat{\mathbf{y}})
\propto
\sum_{i=1}^H \frac{|y_i - \hat y_i|}{b_i}.
\]

Substituting $b_i \propto \sqrt{i}$ yields
\[
-\log p(\mathbf{y}^{\mathrm{GT}} \mid \hat{\mathbf{y}})
\propto
\sum_{i=1}^H i^{-0.5} |y_i - \hat y_i|.
\]

Therefore, maximizing the likelihood under the Laplace and Markov assumptions is equivalent to minimizing the horizon-weighted MAE loss
\[
\sum_{i=1}^{H}i^{-0.5} \left | \hat{\mathrm{y}_i}- \mathrm{y}^{\mathrm{GT}}_i  \right | ,
\]
which completes the proof.
\end{proof}

\section{Theorem 3 and Its Proof}\label{appd_c}

\begin{theorem}[Simplified Inference] 
\label{theorem3_appd}
If the FM-Net producing the velocity (e.g., WFMLin in vLinear) is linear with respect to $\hat{\mathbf{Y}}^{(t)}$, then performing inference starting from the origin (i.e., the all-zero state) is equivalent to computing the expectation of the outputs derived from Gaussian noise samples.
\end{theorem}

\begin{proof}

    The FM-Net producing the velocity takes a concatenated input of the condition $\mathbf{Cond}$, the current state $\hat{\mathbf{Y}}^{(t)}$, and the time $t$. Let $\mathbf{W}$ and $\boldsymbol{\beta}$ denote the weight matrix and bias vector of this linear layer, respectively. We can block-partition the weight matrix $\mathbf{W}$ corresponding to the input components:
    \begin{equation}
        \hat{\mathbf{v}}^{(t)} = \mathbf{W}_{\mathrm{state}} \hat{\mathbf{Y}}^{(t)} + \mathbf{W}_{\mathrm{cond}} \mathbf{Cond} + \mathrm{W}_{\mathrm{time}} t + \boldsymbol{\beta}.
    \end{equation}
    Since $\mathbf{Cond}$ is fixed for a given sample and $t$ is fixed at a specific time step, we can group all terms independent of the current state $\hat{\mathbf{Y}}^{(t)}$ into a single term $\boldsymbol{\Phi}(t)$:
    \begin{equation}
        \boldsymbol{\Phi}(t) = \mathbf{W}_{\mathrm{cond}} \mathbf{Cond} + \mathrm{W}_{\mathrm{time}} t + \boldsymbol{\beta}.
    \end{equation}
    Thus, the velocity at time $t$ is an affine function of the state:
    \begin{equation}
        \hat{\mathbf{v}}^{(t)} = \mathbf{W}_{\mathrm{state}} \hat{\mathbf{Y}}^{(t)} + \boldsymbol{\Phi}(t).
    \end{equation}

    The inference process uses Euler integration with step size $\Delta t$. The state at the next step $t_{k+1}$ is given by:
    \begin{equation}
        \hat{\mathbf{Y}}^{(t_{k+1})} = \hat{\mathbf{Y}}^{(t_k)} + \Delta t  \hat{\mathbf{v}}^{(t_k)}.
    \end{equation}
    Substituting the velocity function:
    \begin{equation}
        \begin{aligned}
            \hat{\mathbf{Y}}^{(t_{k+1})} &= \hat{\mathbf{Y}}^{(t_k)} + \Delta t \left( \mathbf{W}_{\mathrm{state}} \hat{\mathbf{Y}}^{(t_k)} + \boldsymbol{\Phi}(t_k) \right) \\
            &= (\mathbf{I} + \Delta t \mathbf{W}_{\mathrm{state}}) \hat{\mathbf{Y}}^{(t_k)} + \Delta t \boldsymbol{\Phi}(t_k).
        \end{aligned}
        \label{eq_affine_step}
    \end{equation}
    Let $\mathbf{A} = (\mathbf{I} + \Delta t \mathbf{W}_{\mathrm{state}})$ and $\mathbf{C}_k = \Delta t \boldsymbol{\Phi}(t_k)$. The update rule simplifies to an affine transformation:
    \begin{equation}
    \hat{\mathbf{Y}}^{(t_{k+1})} = \mathbf{A}\hat{\mathbf{Y}}^{(t_k)} + \mathbf{C}_k.
    \label{eq_affine_step2}
    \end{equation}

     Let $\boldsymbol{\mu}_k = \mathbf{E}[\hat{\mathbf{Y}}^{(t_k)}]$ be the expected state at step $k$ where $\hat{\mathbf{Y}}^{(0)} \sim \mathcal{N}(\mathbf{0}, \mathbf{I})$. And let $\mathbf{z}_k$ be the deterministic state at step $k$ where $\mathbf{z}_0 = \mathbf{0}$ (the simplified inference).

    For the expected path, applying expectations to Eq.~\eqref{eq_affine_step2}:
    
    \begin{equation}
        \boldsymbol{\mu}_{k+1} = \mathbf{E}[\mathbf{A}\hat{\mathbf{Y}}^{(t_k)} + \mathbf{C}_k] = \mathbf{A}\mathbf{E}[\hat{\mathbf{Y}}^{(t_k)}] + \mathbf{C}_k = \mathbf{A}\boldsymbol{\mu}_k + \mathbf{C}_k.
    \end{equation}
    
    For the simplified zero-start path:
    \begin{equation}
        \mathbf{z}_{k+1} = \mathbf{A}\mathbf{z}_k + \mathbf{C}_k.
    \end{equation}
    
    Both sequences follow the same recurrence relation. We verify the $k=0$ case: $\boldsymbol{\mu}_0 = \mathbf{E}[\hat{\mathbf{Y}}^{(0)}] = \mathbf{0}$ (since noise is zero-mean), and $\mathbf{z}_0 = \mathbf{0}$ (by definition).

    Since $\boldsymbol{\mu}_0 = \mathbf{z}_0$, it follows by induction that $\boldsymbol{\mu}_K = \mathbf{z}_K$. Thus, the expectation of the output equals the output generated from the zero state.
\end{proof}

\section{Dataset Description}\label{appd_d}

Long-term forecasting experiments are conducted on 13 well-established datasets: \textbf{ETT}\{h1, h2, m1, m2\}, \textbf{Weather}, \textbf{Traffic}, \textbf{Electricity}, \textbf{Exchange}, \textbf{Solar-Energy}, and \textbf{PEMS}\{03, 04, 07, 08\}. The `Input-96-Predict-\{12, 24, 48, 96\}' setting is applied to the PEMS datasets, and the `Input-96-Predict-\{96, 192, 336, 720\}' setting is applied to the other datasets.
Short-term forecasting experiments are conducted on 9 datasets: \textbf{ILI}, \textbf{COVID-19}, \textbf{METR-LA}, \textbf{NASDAQ}, \textbf{Wiki}, \textbf{SP500}, \textbf{DowJones}, \textbf{Power}, and \textbf{Unemp}. To fully evaluate the forecasting performance, two settings are applied: `Input-12-Predict-\{3, 6, 9, 12\}' and `Input-36-Predict-\{24, 36, 48, 60\}'. Overall, 124 prediction settings are used in our experiments. Further details on the datasets, including sizes, data splits (train, validation, test), and domains, are provided in Table~\ref{tab_dataset_desp}.

\begin{table*}[t]
\begin{center}
\setlength{\tabcolsep}{4pt}
\renewcommand{\arraystretch}{1.3}
\begin{tabular}{@{}c|cccccc|cc@{}}
\toprule
Tasks                                      & Datasets                  & Var                 & Frequency              & Total Len.           & Split & Pred. Hor.                                                                  & Domain                  & Source \\ \midrule
\multirow{11}{*}{\rotatebox[origin=c]{90}{Long-term   forecasting}}  & ETTh1, ETTh2              & 7                   & Hourly                 & 17420                & 6:2:2 & \{96,192,336,720\}                                                          & Electricity             & \cite{informer}     \\
 & ETTm1, ETTm2              & 7                   & 15 mins                & 69680                & 6:2:2 & \{96,192,336,720\}                                                          & Electricity             & \cite{informer}     \\
 & Weather                   & 21                  & 10 mins                & 52696                & 7:1:2 & \{96,192,336,720\}                                                          & Weather                 & \cite{autoformer}     \\

 & ECL                       & 321                 & Hourly                 & 26304                & 7:1:2 & \{96,192,336,720\}                                                          & Electricity             & \cite{autoformer}     \\
& Traffic                   & 862                 & Hourly                 & 17544                & 7:1:2 & \{96,192,336,720\}                                                          & Transportation          & \cite{autoformer}      \\
 & Exchange                  & 8                   & Daily                  & 7588                 & 7:1:2 & \{96,192,336,720\}                                                          & Economy                 & \cite{autoformer}     \\
& Solar-Energy              & 137                 & 10 mins                & 52560                & 7:1:2 & \{96,192,336,720\}                                                          & Energy                  & \cite{rnn}   \\
& PEMS03                    & 358                 & 5 mins                 & 26209                & 6:2:2 & \{12,24,48,96\}                                                             & Transportation          & \cite{scinet}     \\
& PEMS04                    & 307                 & 5 mins                 & 16992                & 6:2:2 & \{12,24,48,96\}                                                             & Transportation          & \cite{scinet}     \\
& PEMS07                    & 883                 & 5 mins                 & 28224                & 6:2:2 & \{12,24,48,96\}                                                             & Transportation          & \cite{scinet}    \\
 & PEMS08                    & 170                 & 5 mins                 & 17856                & 6:2:2 & \{12,24,48,96\}                                                             & Transportation          & \cite{scinet}   \\ \midrule
\multirow{18}{*}{\rotatebox[origin=c]{90}{Short-term   forecasting}} & ILI                       & 7                   & Weekly                 & 966                  & 7:1:2 & \begin{tabular}[c]{@{}c@{}}\{3,6,9,12\}\\      \{24,36,48,60\}\end{tabular} & Health                  & \cite{autoformer}    \\
 & \multirow{2}{*}{COVID-19} & \multirow{2}{*}{55} & \multirow{2}{*}{Daily} & \multirow{2}{*}{335} & 7:1:2 & \{3,6,9,12\}                                                                & \multirow{2}{*}{Health} & \multirow{2}{*}{\cite{chen2022tamp}} \\
                                           &                           &                     &                        &                      & 6:2:2 & \{24,36,48,60\}                                                             &                         &                     \\
   & METR-LA                   & 207                 & 5 mins                 & 34272                & 7:1:2 & \begin{tabular}[c]{@{}c@{}}\{3,6,9,12\}\\      \{24,36,48,60\}\end{tabular} & Transportation          & \cite{li2018dcrnn_traffic}     \\
  & NASDAQ                    & 12                  & Daily                  & 3914                 & 7:1:2 & \begin{tabular}[c]{@{}c@{}}\{3,6,9,12\}\\      \{24,36,48,60\}\end{tabular} & Finance                 & \cite{freeformer}     \\
& Wiki                      & 99                  & Daily                  & 730                  & 7:1:2 & \begin{tabular}[c]{@{}c@{}}\{3,6,9,12\}\\      \{24,36,48,60\}\end{tabular} & Web                     & \cite{freeformer}     \\
   & SP500                     & 5                   & Daily                  & 8077                 & 7:1:2 & \begin{tabular}[c]{@{}c@{}}\{3,6,9,12\}\\      \{24,36,48,60\}\end{tabular} & Finance                 & \cite{olinear}     \\
    & DowJones                  & 27                  & Daily                  & 6577                 & 7:1:2 & \begin{tabular}[c]{@{}c@{}}\{3,6,9,12\}\\      \{24,36,48,60\}\end{tabular} & Finance                 & \cite{olinear}     \\
     & Power                     & 2                   & Daily                  & 1186                 & 7:1:2 & \begin{tabular}[c]{@{}c@{}}\{3,6,9,12\}\\      \{24,36,48,60\}\end{tabular} & Energy                  & \cite{olinear}     \\
    & Unemp                     & 53                  & Monthly                & 531                  & 6:2:2 & \begin{tabular}[c]{@{}c@{}}\{3,6,9,12\}\\      \{24,36,48,60\}\end{tabular} & Society                 & \cite{olinear}    \\ \bottomrule
\end{tabular}

\caption{Detailed dataset description. \textit{Var} denotes the number of variates. \textit{Frequency} denotes the time interval between consecutive time steps. \textit{Split} denotes the $\mathtt{(Train, Validation, Test)}$ ratio. \textit{Pred. Hor.} indicates the prediction horizons. In total, 124 forecasting tasks across datasets and horizon settings are evaluated in this work.}
\label{tab_dataset_desp}
\end{center}
\end{table*}

\section{Implementation Details}\label{appd_e}

Our experiments are implemented in PyTorch \cite{pytorch} and conducted on an NVIDIA 4090 GPU. vLinear is optimized using the ADAM optimizer \cite{adam_opt}. The batch size is set to 32, and the model dimension $D$ is set to 512 by default. The number of layers $L$ is set 2 or 3 according to the dataset scale. The other settings are kept the same as OLinear~\cite{olinear}. Early stopping is applied: training is terminated if the validation performance does not improve for 10 consecutive epochs. We adopt a learnable standard deviation for the Gaussian noise in WFMLoss, denoted as $\mathrm{std}$. Both $\mathrm{std}$ and the step number $K$  are selected based on the validation performance.
A hyperparameter sensitivity analysis is provided in Appendix~\ref{hyper-para}. The source code is available at \url{https://anonymous.4open.science/r/vLinear}.

For baseline models, we use the reported values from the original papers when available; otherwise, we reproduce the results using the official code. For FilterNet \cite{filternet}, TexFilter is adopted. When applying the vecTrans module to Leddam, we only update the `cross-channel attention' module. For the TransDF \cite{transdf}, FreDF \cite{fredf}, and DBLoss \cite{dbloss}, the hyperparameters in the objectives are uniformly set to 0.5.

\section{Computational Efficiency}\label{appd_efficiency}

\begin{table*}[t]
\begin{center}
{\fontsize{8}{9}\selectfont
\setlength{\tabcolsep}{4pt}
\begin{tabular}{@{}cc|c|c|c|c|c|c|c|c|c|c|c|c@{}}
\toprule
\multicolumn{2}{c|}{Model}                    & \begin{tabular}[c]{@{}c@{}}vLinear\\      (Ours)\end{tabular} 
& \begin{tabular}[c]{@{}c@{}}OLinear\\ \shortcite{olinear} \end{tabular} 
& \begin{tabular}[c]{@{}c@{}}Leddam\\  \shortcite{Leddam_icml} \end{tabular} 
& \begin{tabular}[c]{@{}c@{}}CARD\\    \shortcite{card} \end{tabular} 
& \begin{tabular}[c]{@{}c@{}}Fredf.m.\\      \shortcite{fredformer} \end{tabular} 
& \begin{tabular}[c]{@{}c@{}}iTrans.\\   \shortcite{itransformer} \end{tabular} 
& \begin{tabular}[c]{@{}c@{}}TimeMix.++\\      \shortcite{timemixer++} \end{tabular} 
& \begin{tabular}[c]{@{}c@{}}TimeMix.\\      \shortcite{timemixer} \end{tabular} 
& \begin{tabular}[c]{@{}c@{}}DLinear\\  \shortcite{linear} \end{tabular} 
& \begin{tabular}[c]{@{}c@{}}TimesNet\\      \shortcite{timesnet} \end{tabular} 
& \begin{tabular}[c]{@{}c@{}}PatchTST\\      \shortcite{patchtst} \end{tabular} 
& \begin{tabular}[c]{@{}c@{}}FreTS\\    \shortcite{frets} \end{tabular} \\ \midrule
\multirow{6}{*}{\rotatebox[origin=c]{90}{ETT}}          & Params(M)      & 4.53                                                          & 4.52                                                      & 8.55                                                     & 0.03                                                   & 8.59                                                       & 4.83                                                      & 1.19                                                         & 0.08                                                       & 0.02                                                      & 299.94                                                     & 3.76                                                       & 0.42                                                    \\
                              & FLOPs(M)       & 33.83                                                         & 33.74                                                     & 111.54                                                   & 1.41                                                   & 135.01                                                     & 33.99                                                     & 1396.90                                                      & 10.72                                                      & 0.13                                                      & 289708                                                     & 272.20                                                     & 2.94                                                    \\
                              & T.T. (ms/iter) & 7.04                                                          & 8.09                                                      & 25.02                                                    & 29.10                                                  & 26.48                                                      & 10.11                                                     & 74.74                                                        & 27.99                                                      & 1.42                                                      & 501.98                                                     & 8.37                                                       & 3.63                                                    \\
                              & T.M.(GB)       & 0.20                                                          & 0.20                                                      & 0.18                                                     & 0.03                                                   & 0.24                                                       & 0.21                                                      & 0.27                                                         & 0.05                                                       & 0.02                                                      & 5.80                                                       & 0.14                                                       & 0.03                                                    \\
                              & I.T.(ms/iter)  & 1.30                                                          & 1.19                                                      & 2.62                                                     & 4.05                                                   & 3.73                                                       & 1.72                                                      & 28.15                                                        & 3.86                                                       & 0.21                                                      & 155.97                                                     & 1.38                                                       & 0.49                                                    \\
                              & I.M. (MB)      & 155.63                                                        & 156.41                                                    & 51.04                                                    & 11.02                                                  & 78.04                                                      & 156.00                                                    & 149.58                                                       & 18.06                                                      & 8.48                                                      & 1396.24                                                    & 51.83                                                      & 14.47                                                   \\ \midrule
\multirow{6}{*}{\rotatebox[origin=c]{90}{ECL}}          & Params(M)      & 2.82                                                          & 4.79                                                      & 8.56                                                     & 1.39                                                   & 12.12                                                      & 4.83                                                      & 1.19                                                         & 0.12                                                       & 0.02                                                      & 300.58                                                     & 3.76                                                       & 0.42                                                    \\
                              & FLOPs(G)       & 0.91                                                          & 1.65                                                      & 5.32                                                     & 5.07                                                   & 5.55                                                       & 1.87                                                      & 64.02                                                        & 1.74                                                       & 0.01                                                      & 293.98                                                     & 12.48                                                      & 0.13                                                    \\
                              & T.T. (ms/iter) & 10.42                                                         & 7.75                                                      & 65.19                                                    & 78.25                                                  & 26.59                                                      & 10.66                                                     & 106.06                                                       & 57.09                                                      & 1.37                                                      & 509.17                                                     & 65.83                                                      & 3.78                                                    \\
                              & T.M.(GB)       & 0.39                                                          & 0.45                                                      & 1.44                                                     & 5.08                                                   & 1.24                                                       & 0.70                                                      & 1.02                                                         & 4.24                                                       & 0.04                                                      & 5.82                                                       & 3.02                                                       & 0.30                                                    \\
                              & I.T.(ms/iter)  & 2.21                                                          & 2.11                                                      & 14.78                                                    & 28.21                                                  & 7.08                                                       & 3.46                                                      & 34.10                                                        & 22.79                                                      & 0.20                                                      & 157.13                                                     & 20.52                                                      & 0.48                                                    \\
                              & I.M. (GB)      & 0.10                                                          & 0.17                                                      & 0.49                                                     & 0.63                                                   & 0.26                                                       & 0.23                                                      & 0.36                                                         & 0.82                                                       & 0.02                                                      & 1.41                                                       & 0.93                                                       & 0.23                                                    \\ \midrule
\multirow{6}{*}{\rotatebox[origin=c]{90}{Exchange}}     & Params(M)      & 1.03                                                          & 1.74                                                      & 8.56                                                     & 1.39                                                   & 8.59                                                       & 4.83                                                      & 1.19                                                         & 0.08                                                       & 0.02                                                      & 299.94                                                     & 3.76                                                       & 0.42                                                    \\
                              & FLOPs(M)       & 10.60                                                         & 16.27                                                     & 127.49                                                   & 114.83                                                 & 135.01                                                     & 38.88                                                     & 1595.72                                                      & 12.25                                                      & 0.15                                                      & 287909                                                     & 311.08                                                     & 3.36                                                    \\
                              & T.T. (ms/iter) & 6.92                                                          & 7.33                                                      & 24.81                                                    & 35.57                                                  & 28.56                                                      & 10.14                                                     & 97.05                                                        & 27.46                                                      & 1.47                                                      & 505.39                                                     & 9.75                                                       & 3.51                                                    \\
                              & T.M.(GB)       & 0.16                                                          & 0.17                                                      & 0.18                                                     & 0.17                                                   & 0.24                                                       & 0.21                                                      & 0.28                                                         & 0.06                                                       & 0.02                                                      & 5.80                                                       & 0.15                                                       & 0.03                                                    \\
                              & I.T.(ms/iter)  & 1.36                                                          & 1.18                                                      & 2.62                                                     & 6.02                                                   & 6.86                                                       & 1.70                                                      & 28.46                                                        & 3.81                                                       & 0.21                                                      & 156.33                                                     & 2.47                                                       & 0.48                                                    \\
                              & I.M. (GB)      & 0.14                                                          & 0.02                                                      & 0.05                                                     & 0.03                                                   & 0.08                                                       & 0.16                                                      & 0.15                                                         & 0.02                                                       & 0.01                                                      & 1.40                                                       & 0.05                                                       & 0.02                                                    \\ \midrule
\multirow{6}{*}{\rotatebox[origin=c]{90}{Traffic}}      & Params(M)      & 4.70                                                          & 6.17                                                      & 8.56                                                     & 0.98                                                   & 11.09                                                      & 4.83                                                      & 4.73                                                         & 0.12                                                       & 0.02                                                      & 301.69                                                     & 3.76                                                       & 0.42                                                    \\
                              & FLOPs(G)       & 4.17                                                          & 4.91                                                      & 15.24                                                    & 10.47                                                  & 14.78                                                      & 6.45                                                      & 588.47                                                       & 4.66                                                       & 0.02                                                      & 288.72                                                     & 33.52                                                      & 0.36                                                    \\
                              & T.T. (s/iter)  & 0.02                                                          & 0.02                                                      & 0.19                                                     & 0.18                                                   & 0.10                                                       & 0.04                                                      & OOM                                                          & 0.17                                                       & 0.001                                                     & 0.50                                                       & 0.17                                                       & 0.01                                                    \\
                              & T.M.(GB)       & 0.90                                                          & 1.01                                                      & 3.74                                                     & 9.50                                                   & 5.76                                                       & 2.95                                                      & OOM                                                          & 11.33                                                      & 0.07                                                      & 5.86                                                       & 8.01                                                       & 0.78                                                    \\
                              & I.T.(ms/iter)  & 5.70                                                          & 5.71                                                      & 45.68                                                    & 61.21                                                  & 31.60                                                      & 13.45                                                     & 1610.66                                                      & 67.21                                                      & 0.20                                                      & 15.64                                                      & 55.10                                                      & 3.48                                                    \\
                              & I.M. (GB)      & 0.36                                                          & 0.43                                                      & 1.25                                                     & 1.69                                                   & 1.84                                                       & 1.27                                                      & 9.35                                                         & 2.20                                                       & 0.04                                                      & 1.42                                                       & 2.44                                                       & 0.59                                                    \\ \midrule
\multirow{6}{*}{\rotatebox[origin=c]{90}{Weather}}      & Params(M)      & 2.77                                                          & 4.52                                                      & 8.56                                                     & 0.03                                                   & 0.50                                                       & 4.83                                                      & 2.37                                                         & 0.10                                                       & 0.02                                                      & 299.97                                                     & 3.76                                                       & 0.42                                                    \\
                              & FLOPs(G)       & 0.06                                                          & 0.10                                                      & 0.34                                                     & 0.004                                                  & 0.01                                                       & 0.10                                                      & 8.36                                                         & 0.05                                                       & 0.0004                                                    & 291.21                                                     & 0.82                                                       & 0.01                                                    \\
                              & T.T. (ms/iter) & 9.82                                                          & 7.47                                                      & 27.55                                                    & 28.69                                                  & 37.85                                                      & 9.45                                                      & 102.04                                                       & 33.30                                                      & 1.38                                                      & 505.88                                                     & 9.89                                                       & 3.53                                                    \\
                              & T.M.(GB)       & 0.19                                                          & 0.21                                                      & 0.21                                                     & 0.05                                                   & 0.05                                                       & 0.22                                                      & 0.86                                                         & 0.16                                                       & 0.02                                                      & 5.80                                                       & 0.27                                                       & 0.04                                                    \\
                              & I.T.(ms/iter)  & 1.43                                                          & 1.18                                                      & 3.15                                                     & 4.11                                                   & 7.19                                                       & 1.66                                                      & 30.18                                                        & 5.10                                                       & 0.20                                                      & 156.27                                                     & 2.96                                                       & 0.48                                                    \\
                              & I.M. (MB)      & 27.60                                                         & 39.26                                                     & 70.06                                                    & 14.45                                                  & 18.28                                                      & 34.88                                                     & 316.49                                                       & 36.03                                                      & 9.06                                                      & 1397.18                                                    & 90.78                                                      & 24.37                                                   \\ \midrule
\multirow{6}{*}{\rotatebox[origin=c]{90}{Solar-Energy}} & Params(M)      & 4.56                                                          & 4.58                                                      & 8.56                                                     & 1.39                                                   & 4.61                                                       & 4.83                                                      & 2.37                                                         & 0.12                                                       & 0.02                                                      & 300.21                                                     & 3.76                                                       & 0.42                                                    \\
                              & FLOPs(G)       & 0.66                                                          & 0.68                                                      & 2.22                                                     & 2.05                                                   & 1.20                                                       & 0.72                                                      & 46.81                                                        & 0.74                                                       & 0.003                                                     & 290.64                                                     & 5.33                                                       & 0.06                                                    \\
                              & T.T. (ms/iter) & 10.06                                                         & 9.12                                                      & 48.64                                                    & 36.28                                                  & 31.77                                                      & 9.97                                                      & 515.44                                                       & 35.90                                                      & 1.44                                                      & 506.58                                                     & 28.89                                                      & 3.60                                                    \\
                              & T.M.(GB)       & 0.29                                                          & 0.30                                                      & 0.68                                                     & 2.22                                                   & 0.61                                                       & 0.37                                                      & 5.11                                                         & 1.83                                                       & 0.03                                                      & 5.81                                                       & 1.34                                                       & 0.15                                                    \\
                              & I.T.(ms/iter)  & 1.59                                                          & 1.20                                                      & 7.65                                                     & 10.44                                                  & 5.16                                                       & 1.96                                                      & 189.20                                                       & 8.54                                                       & 0.20                                                      & 156.66                                                     & 8.49                                                       & 0.48                                                    \\
                              & I.M. (GB)      & 0.19                                                          & 0.21                                                      & 0.23                                                     & 0.28                                                   & 0.16                                                       & 0.18                                                      & 1.94                                                         & 0.36                                                       & 0.01                                                      & 1.40                                                       & 0.42                                                       & 0.11                                                    \\ \midrule
\multirow{6}{*}{\rotatebox[origin=c]{90}{PEMS07}}       & Params(M)      & 5.65                                                          & 6.25                                                      & 8.56                                                     & 1.39                                                   & 13.44                                                      & 4.83                                                      & 2.38                                                         & 0.12                                                       & 0.02                                                      & 301.73                                                     & 3.76                                                       & 0.42                                                    \\
                              & FLOPs(G)       & 5.09                                                          & 5.05                                                      & 15.65                                                    & 16.10                                                  & 20.53                                                      & 6.66                                                      & 301.92                                                       & 4.77                                                       & 0.02                                                      & 290.97                                                     & 34.34                                                      & 0.37                                                    \\
                              & T.T. (s/iter)  & 0.02                                                          & 0.02                                                      & 0.19                                                     & 0.28                                                   & 0.11                                                       & 0.04                                                      & OOM                                                          & 0.17                                                       & 0.002                                                     & 0.51                                                       & 0.19                                                       & 0.01                                                    \\
                              & T.M.(GB)       & 1.04                                                          & 1.05                                                      & 3.85                                                     & 13.99                                                  & 6.42                                                       & 3.07                                                      & OOM                                                          & 11.62                                                      & 0.07                                                      & 5.86                                                       & 9.18                                                       & 0.79                                                    \\
                              & I.T.(ms/iter)  & 6.42                                                          & 6.33                                                      & 47.06                                                    & 94.01                                                  & 36.54                                                      & 13.92                                                     & 1370.91                                                      & 68.96                                                      & 0.20                                                      & 156.78                                                     & 56.49                                                      & 0.50                                                    \\
                              & I.M. (GB)      & 0.38                                                          & 0.44                                                      & 1.28                                                     & 1.74                                                   & 2.03                                                       & 1.32                                                      & 12.37                                                        & 2.25                                                       & 0.04                                                      & 1.43                                                       & 2.50                                                       & 0.61                                                    \\ \bottomrule
\end{tabular}
}
\caption{Computational efficiency of vLinear and state-of-the-art forecasters. `T.T.', `T.M.', `I.T.', and `I.M.' denote training time, training memory, inference time, and inference memory, respectively. `OOM' indicates out-of-memory on a 24GB GPU. Experiments are conducted with an input and prediction length of 96 (`Input-96-Predict-96') and a batch size of 16. For fairness, the number of layers is set to 2 for all models except DLinear, which is inherently a single-layer model. Overall, vLinear demonstrates significantly lower FLOPs and reduced GPU memory consumption.}
\label{tab_gpu}
\end{center}
\end{table*}

\begin{table*}[t]
\begin{center}
{\fontsize{8}{9}\selectfont
\setlength{\tabcolsep}{10pt}
\begin{tabular}{@{}cc|cc|cc|cc|cc@{}}
\toprule
\multicolumn{2}{c|}{\multirow{2}{*}{Model}}    
& \multicolumn{2}{c|}{\begin{tabular}[c]{@{}c@{}}Leddam\\     \shortcite{Leddam_icml} \end{tabular}} 
& \multicolumn{2}{c|}{\begin{tabular}[c]{@{}c@{}}iTrans.\\      \shortcite{itransformer} \end{tabular}} 
& \multicolumn{2}{c|}{\begin{tabular}[c]{@{}c@{}}PatchTST\\    \shortcite{patchtst} \end{tabular}} 
& \multicolumn{2}{c}{\begin{tabular}[c]{@{}c@{}}Fredformer\\   \shortcite{fredformer} \end{tabular}} \\ \cmidrule(l){3-10} 
\multicolumn{2}{c|}{}                          & Attn.                                & vecTrans                               & Attn.                                 & vecTrans                               & Attn.                                 & vecTrans                                & Attn.                                  & vecTrans                                \\ \midrule
\multirow{6}{*}{ETTm1}        & Params(M)      & 8.55                                 & 7.51                                   & 4.83                                  & 2.21                                   & 3.76                                  & 2.71                                    & 8.59                                   & 8.39                                    \\
                              & FLOPs(M)       & 111.54                               & 104.10                                 & 33.99                                 & 15.46                                  & 272.20                                & 182.06                                  & 135.01                                 & 79.09                                   \\
                              & T.T. (ms/iter) & 25.02                                & 12.05                                  & 10.11                                 & 8.01                                   & 8.37                                  & 6.80                                    & 26.48                                  & 9.52                                    \\
                              & T.M.(GB)       & 0.18                                 & 0.16                                   & 0.21                                  & 0.16                                   & 0.14                                  & 0.10                                    & 0.24                                   & 0.18                                    \\
                              & I.T.(ms/iter)  & 2.62                                 & 2.30                                   & 1.72                                  & 1.05                                   & 1.38                                  & 1.15                                    & 3.73                                   & 1.39                                    \\
                              & I.M. (MB)      & 51.04                                & 47.03                                  & 156.00                                & 137.59                                 & 51.83                                 & 27.95                                   & 78.04                                  & 51.63                                   \\ \midrule
\multirow{6}{*}{ECL}          & Params(M)      & 8.56                                 & 7.51                                   & 4.83                                  & 2.21                                   & 3.76                                  & 2.71                                    & 12.12                                  & 12.61                                   \\
                              & FLOPs(G)       & 5.32                                 & 4.77                                   & 1.87                                  & 0.71                                   & 12.48                                 & 8.35                                    & 5.55                                   & 5.84                                    \\
                              & T.T. (ms/iter) & 65.19                                & 62.23                                  & 10.66                                 & 5.84                                   & 65.83                                 & 44.33                                   & 26.59                                  & 15.32                                   \\
                              & T.M.(GB)       & 1.44                                 & 1.41                                   & 0.70                                  & 0.37                                   & 3.02                                  & 2.46                                    & 1.24                                   & 1.26                                    \\
                              & I.T.(ms/iter)  & 14.78                                & 14.51                                  & 3.46                                  & 1.01                                   & 20.52                                 & 14.40                                   & 7.08                                   & 6.60                                    \\
                              & I.M. (GB)      & 0.49                                 & 0.48                                   & 0.23                                  & 0.08                                   & 0.93                                  & 0.87                                    & 0.26                                   & 0.25                                    \\ \midrule
\multirow{6}{*}{Weather}      & Params(M)      & 8.56                                 & 7.52                                   & 4.83                                  & 2.21                                   & 3.76                                  & 2.71                                    & 0.50                                   & 0.48                                    \\
                              & FLOPs(G)       & 0.34                                 & 0.31                                   & 0.10                                  & 0.05                                   & 0.82                                  & 0.55                                    & 0.01                                   & 0.003                                   \\
                              & T.T. (ms/iter) & 27.55                                & 27.50                                  & 9.45                                  & 6.97                                   & 9.89                                  & 8.27                                    & 37.85                                  & 25.57                                   \\
                              & T.M.(GB)       & 0.21                                 & 0.19                                   & 0.22                                  & 0.17                                   & 0.27                                  & 0.22                                    & 0.05                                   & 0.03                                    \\
                              & I.T.(ms/iter)  & 3.15                                 & 3.08                                   & 1.66                                  & 1.02                                   & 2.96                                  & 1.12                                    & 7.19                                   & 3.66                                    \\
                              & I.M. (MB)      & 70.06                                & 66.05                                  & 34.88                                 & 16.47                                  & 90.78                                 & 65.39                                   & 18.28                                  & 13.52                                   \\ \midrule
\multirow{6}{*}{Solar-Energy} & Params(M)      & 8.56                                 & 7.51                                   & 4.83                                  & 2.21                                   & 3.76                                  & 2.71                                    & 4.61                                   & 4.27                                    \\
                              & FLOPs(G)       & 2.22                                 & 2.04                                   & 0.72                                  & 0.30                                   & 5.33                                  & 3.56                                    & 1.20                                   & 0.90                                    \\
                              & T.T. (ms/iter) & 48.64                                & 47.86                                  & 9.97                                  & 7.90                                   & 28.89                                 & 18.51                                   & 31.77                                  & 10.80                                   \\
                              & T.M.(GB)       & 0.68                                 & 0.67                                   & 0.37                                  & 0.24                                   & 1.34                                  & 1.08                                    & 0.61                                   & 0.46                                    \\
                              & I.T.(ms/iter)  & 7.65                                 & 7.56                                   & 1.96                                  & 1.05                                   & 8.49                                  & 5.67                                    & 5.16                                   & 1.74                                    \\
                              & I.M. (GB)      & 0.23                                 & 0.22                                   & 0.18                                  & 0.16                                   & 0.42                                  & 0.38                                    & 0.16                                   & 0.09                                    \\ \midrule
\multirow{6}{*}{PEMS07}       & Params(M)      & 8.56                                 & 7.51                                   & 4.83                                  & 2.21                                   & 3.76                                  & 2.71                                    & 13.44                                  & 12.96                                   \\
                              & FLOPs(G)       & 15.65                                & 13.13                                  & 6.66                                  & 1.95                                   & 34.34                                 & 22.96                                   & 20.53                                  & 16.07                                   \\
                              & T.T. (s/iter)  & 0.19                                 & 0.18                                   & 0.04                                  & 0.008                                  & 0.19                                  & 0.12                                    & 0.11                                   & 0.05                                    \\
                              & T.M.(GB)       & 3.85                                 & 3.71                                   & 3.07                                  & 0.71                                   & 9.18                                  & 6.72                                    & 6.42                                   & 3.11                                    \\
                              & I.T.(ms/iter)  & 47.06                                & 44.57                                  & 13.92                                 & 2.42                                   & 56.49                                 & 39.11                                   & 36.54                                  & 17.68                                   \\
                              & I.M. (GB)      & 1.28                                 & 1.27                                   & 1.32                                  & 0.22                                   & 2.50                                  & 2.38                                    & 2.03                                   & 0.58                                    \\ \bottomrule
\end{tabular}
}
\caption{Comparison of computational efficiency when replacing the attention mechanism with vecTrans. Significant improvements are observed, especially on datasets with a large number of variates. Note that for Leddam, only the `cross-channel attention' module is replaced. Additionally, Fredformer employs Nystromformer \protect \cite{nystromformer}, an efficient self-attention variant, on the ETTm1, ECL, and Weather datasets. Despite this, the vecTrans module consistently improves training and inference efficiency. Specifically, for iTransformer, this substitution achieves an average \textbf{inference speedup of 2.9$\times$}, reaching up to \textbf{5.8$\times$} on the PEMS07 dataset.}
\label{tab_gpu2}
\end{center}
\end{table*}

Table~\ref{tab_gpu} presents the resource footprint of vLinear and other state-of-the-art forecasters. As shown, vLinear exhibits superior computational efficiency and requires lower GPU memory. Specifically, vLinear achieves inference speedups of 1.6$\times$ and 5.4$\times$ compared to iTransformer and PatchTST, respectively.

Table~\ref{tab_gpu2} compares the resource footprint when replacing the attention mechanism with vecTrans. When applied to iTransformer, this substitution achieves an average \textbf{inference speedup of 2.9$\times$}. Notably, \textbf{a 5.8$\times$ speedup} is observed on the PEMS07 dataset.

\section{Robustness to Random Seeds} \label{robust}

Tables~\ref{tab_robust} and \ref{tab_robust_cmp} present the standard deviations across different datasets and prediction lengths, calculated over 7 random seeds. Specifically, the approximate 99\% confidence intervals of the average performance across four horizons are reported in Table~\ref{tab_robust_cmp}. vLinear exhibits superior robustness compared to state-of-the-art Transformer-based forecasters: TimeMixer++ and iTransformer.

\begin{table*}[t]
\begin{center}
\begin{tabular}{@{}cc|cc|cc|cc|cc@{}}
\toprule
\multicolumn{2}{c|}{Dataset}   & \multicolumn{2}{c|}{ECL}    & \multicolumn{2}{c|}{Traffic}    & \multicolumn{2}{c|}{ETTm1}        & \multicolumn{2}{c}{ETTm2}     \\ \midrule
\multicolumn{2}{c|}{Metric}    & MSE           & MAE         & MSE             & MAE           & MSE              & MAE            & MSE           & MAE           \\ \midrule
\multirow{4}{*}{\rotatebox[origin=c]{90}{Horizon}} & 96  & 0.129±6e-4    & 0.221±6e-4  & 0.396±2e-3      & 0.233±5e-4    & 0.301±9e-4       & 0.337±1e-3     & 0.168±6e-4    & 0.245±5e-4    \\
                         & 192 & 0.147±8e-4    & 0.238±7e-4  & 0.427±3e-3      & 0.243±4e-4    & 0.349±3e-3       & 0.365±2e-2     & 0.231±1e-3    & 0.287±7e-4    \\
                         & 336 & 0.158±6e-4    & 0.250±8e-4  & 0.449±4e-3      & 0.255±3e-4    & 0.381±2e-3       & 0.387±7e-4     & 0.289±9e-4    & 0.326±4e-4    \\
                         & 720 & 0.178±9e-4    & 0.271±1e-3  & 0.487±3e-3      & 0.276±3e-4    & 0.444±4e-3       & 0.424±2e-3     & 0.386±2e-3    & 0.382±1e-3    \\ \midrule
\multicolumn{2}{c|}{Dataset}   & \multicolumn{2}{c|}{ETTh1}  & \multicolumn{2}{c|}{ETTh2}      & \multicolumn{2}{c|}{Solar-Energy} & \multicolumn{2}{c}{Weather}   \\ \midrule
\multicolumn{2}{c|}{Metric}    & MSE           & MAE         & MSE             & MAE           & MSE              & MAE            & MSE           & MAE           \\ \midrule
\multirow{4}{*}{\rotatebox[origin=c]{90}{Horizon}} & 96  & 0.356±7e-4    & 0.379±6e-4  & 0.280±2e-3      & 0.327±2e-3    & 0.175±9e-4       & 0.196±6e-4     & 0.150±7e-4    & 0.186±9e-4    \\
                         & 192 & 0.409±1e-3    & 0.411±7e-4  & 0.351±4e-3      & 0.376±2e-3    & 0.202±6e-4       & 0.220±4e-4     & 0.198±1e-3    & 0.234±1e-3    \\
                         & 336 & 0.449±2e-3    & 0.433±1e-3  & 0.400±4e-3      & 0.411±3e-3    & 0.224±9e-4       & 0.240±7e-5     & 0.252±1e-3    & 0.275±6e-4    \\
                         & 720 & 0.451±2e-3    & 0.456±1e-3  & 0.404±5e-3      & 0.427±3e-3    & 0.234±7e-4       & 0.247±7e-4     & 0.332±1e-3    & 0.328±9e-4    \\ \midrule
\multicolumn{2}{c|}{Dataset}   & \multicolumn{2}{c|}{PEMS03} & \multicolumn{2}{c|}{Power (S1)} & \multicolumn{2}{c|}{NASDAQ (S2)}  & \multicolumn{2}{c}{Wiki (S1)} \\ \midrule
\multicolumn{2}{c|}{Metric}    & MSE           & MAE         & MSE             & MAE           & MSE              & MAE            & MSE           & MAE           \\ \midrule
\multirow{4}{*}{\rotatebox[origin=c]{90}{Horizon}} & H1  & 0.059±5e-4    & 0.158±5e-4  & 0.838±5e-3      & 0.683±4e-3    & 0.120±1e-3       & 0.215±7e-4     & 6.135±5e-2    & 0.375±8e-3    \\
                         & H2  & 0.075±6e-4    & 0.177±6e-4  & 0.978±6e-3      & 0.735±4e-3    & 0.162±9e-4       & 0.259±8e-4     & 6.434±1e-2    & 0.386±3e-3    \\
                         & H3  & 0.102±1e-3    & 0.208±9e-4  & 1.046±1e-3      & 0.768±5e-3    & 0.200±2e-3       & 0.291±1e-3     & 6.657±5e-2    & 0.400±9e-3    \\
                         & H4  & 0.138±1e-3    & 0.246±7e-4  & 1.124±6e-3      & 0.797±1e-3    & 0.244±4e-3       & 0.327±4e-3     & 6.825±2e-2    & 0.407±4e-3    \\ \bottomrule
\end{tabular}
\caption{Robustness analysis of vLinear. The standard deviations are calculated over seven random seeds. \textit{S1} and \textit{S2} denote the `Input-12-Predict-\{3, 6, 9, 12\}' and `Input-36-Predict-\{24, 36, 48, 60\}' settings, respectively. }
\label{tab_robust}
\end{center}
\end{table*}

\begin{table*}[t]
\begin{center}
{
\setlength{\tabcolsep}{10pt}
\begin{tabular}{@{}c|cc|cc|cc@{}}
\toprule
Model   & \multicolumn{2}{c|}{vLinear}                                                              & \multicolumn{2}{c|}{TimeMixer++}                                                          & \multicolumn{2}{c}{iTransformer}                          \\ \midrule
Dataset & MSE                                         & MAE                                         & MSE                                         & MAE                                         & MSE                                         & MAE         \\ \midrule
Weather & {0.233±\color[HTML]{FF0000} \textbf{0.002}} & 0.256±{\color[HTML]{FF0000} \textbf{0.002}} & 0.226±0.008                                 & 0.262±0.007                                 & 0.258±0.009                                 & 0.278±0.006 \\
Solar-Energy   & {0.209±\color[HTML]{FF0000} \textbf{0.001}} & {0.226±\color[HTML]{FF0000} \textbf{9e-4}}  & {0.203±\color[HTML]{FF0000} \textbf{0.001}} & 0.238±0.010                                 & 0.233±0.009                                 & 0.262±0.007 \\
ECL     & {0.153±\color[HTML]{FF0000} \textbf{0.001}} & {0.245±\color[HTML]{FF0000} \textbf{0.001}} & 0.165±0.011                                 & {0.253±\color[HTML]{FF0000} \textbf{0.001}} & 0.178±0.002                                 & 0.270±0.005 \\
Traffic & {0.440±\color[HTML]{FF0000} \textbf{0.005}} & {0.252±\color[HTML]{FF0000} \textbf{5e-4}}  & 0.416±0.015                                 & 0.264±0.013                                 & 0.428±0.008                                 & 0.282±0.002 \\
ETTh1   & {0.416±\color[HTML]{FF0000} \textbf{0.003}} & {0.420±\color[HTML]{FF0000} \textbf{0.002}} & 0.419±0.011                                 & 0.432±0.015                                 & 0.454±0.004                                 & 0.447±0.007 \\
ETTh2   & 0.358±0.008                                 & 0.385±0.004                                 & 0.339±0.009                                 & {0.380±\color[HTML]{FF0000} \textbf{0.002}} & {0.383±\color[HTML]{FF0000} \textbf{0.004}} & 0.407±0.007 \\
ETTm1   & {0.369±\color[HTML]{FF0000} \textbf{0.004}} & {0.378±\color[HTML]{FF0000} \textbf{0.003}} & 0.369±0.005                                 & 0.378±0.007                                 & {0.407±\color[HTML]{FF0000} \textbf{0.004}} & 0.410±0.009 \\
ETTm2   & {0.268±\color[HTML]{FF0000} \textbf{0.002}} & {0.310±\color[HTML]{FF0000} \textbf{0.001}} & {0.269±\color[HTML]{FF0000} \textbf{0.002}} & 0.320±0.012                                 & 0.288±0.010                                 & 0.332±0.003 \\ \bottomrule
\end{tabular}
}
\caption{Approximate 99\% confidence intervals (estimated as 
$3 \times$ standard deviation) of the average performance across four prediction lengths. The smallest deviations (indicating best stability) are highlighted in {\color[HTML]{FF0000} \textbf{red bold}}.}
\label{tab_robust_cmp}
\end{center}
\end{table*}

\section{Hyperparameter Sensitivity Analysis} \label{hyper-para}

\begin{figure}[ht]
   \centering
   \includegraphics[width=1.0\linewidth]{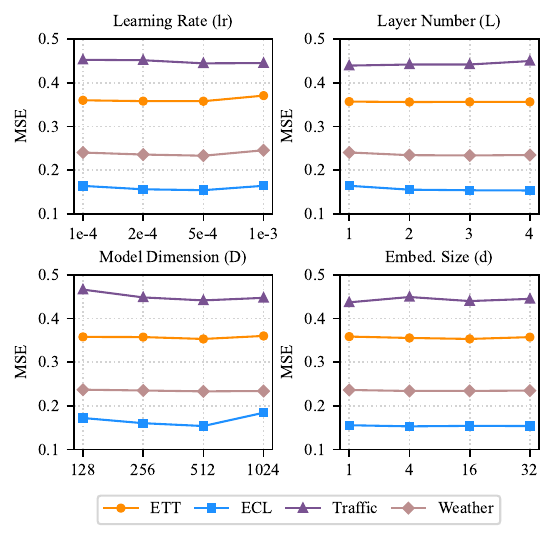}
   \caption{Hyperparameter sensitivity of vLinear with respect to the learning rate $\mathrm{lr}$, the number of layers $L$, the model dimension $D$, and the embedding size ($d$). Performance is averaged across $H \in \left \{ 96,192,336,720 \right \}$. ETT denotes the average performance over  $\mathrm{ETT} \left \{ \mathrm{h} 1,\mathrm{h} 2,\mathrm{m} 1,\mathrm{m} 2 \right \} $. The lookback length $T$ is fixed at 96.}
   \label{fig_hyperpara}
\end{figure}

We conduct extensive experiments to evaluate the hyperparameter sensitivity of vLinear with respect to the  following factors: the learning rate ($\mathrm{lr}$), the number of layers ($L$), the model dimension ($D$), the embedding size ($d$, see Section~\ref{overall_arch} in the main paper) \cite{frets}, the steps $K$ in WFMLoss, and the initial noise standard deviation (std). 

Figure~\ref{fig_hyperpara} illustrates the performance variation concerning the learning rate ($\mathrm{lr}$), the number of layers ($L$), the model dimension ($D$), the embedding size ($d$). Overall, vLinear exhibits robust performance across these four hyperparameters. Notably, for $D \in \left \{ 128,256,512,1024 \right \} $, the MSE variation is only 1.2\% and 1.6\% on the ETTh1 and Weather datasets, respectively. 

\begin{figure}[ht]
   \centering
   \includegraphics[width=0.8\linewidth]{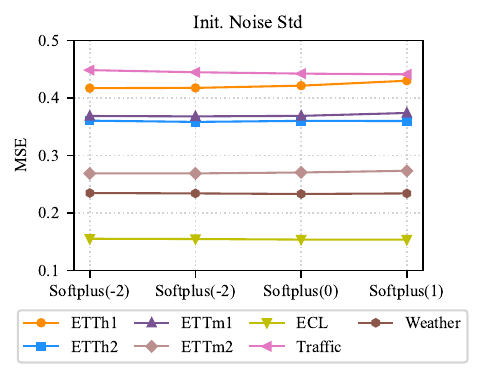}
   \caption{Hyperparameter sensitivity of vLinear with respect to the initial noise standard deviation (std). Performance is averaged across four horizons $H \in \left \{ 96,192,336,720 \right \}$. The lookback length $T$ is fixed at 96.}
   \label{fig_hyperpara_std}
\end{figure}

In vLinear, we adopt a learnable noise standard deviation (denoted as $\mathrm{std}$) in WFMLoss. Since the initialization of such learnable parameters may affect performance, we evaluate the model's sensitivity to the initial $\mathrm{std}$ value.
Figure~\ref{fig_hyperpara_std} depicts the impact of the initial noise standard deviation. Results show that vLinear is robust to changes in this hyperparameter.  For $\mathrm{std} \in \{ \mathrm{Softplus}(k) \mid k \in \{-2, -1, 0, 1\} \}$, the MSE variation is less than 1\% on the ETTh2, ECL, and Weather datasets.

\begin{figure}[ht]
   \centering
   \includegraphics[width=1.0\linewidth]{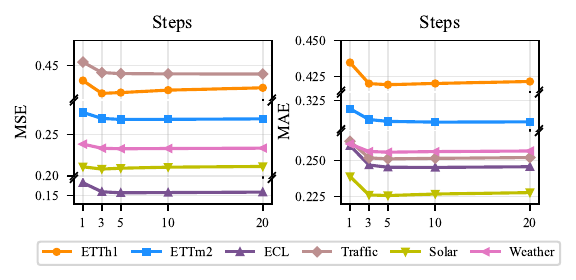}
   \caption{Hyperparameter sensitivity of vLinear with respect to the inference steps ($K$) in WFMBlock. Performance is averaged across four prediction horizons $H \in \left \{ 96,192,336,720 \right \}$. The lookback length $T$ is fixed at 96.}
   \label{fig_hyperpara_steps}
\end{figure}

Figure~\ref{fig_hyperpara_steps} presents the performance variation with respect to the inference steps ($K$) in WFMBlock. vLinear maintains stability regarding this hyperparameter. For $K\ge 5$, the variation in MSE and MAE is less than 1\%. Specifically, on the ETTm2 dataset, the MSE and MAE variation for $K \in \{5,10,20\}$ is merely 0.2\% and 0.1\%, respectively.

\begin{table*}[t]
\begin{center}
\renewcommand{\arraystretch}{1.3}
{\fontsize{8}{10}\selectfont
\setlength{\tabcolsep}{4.5pt}
\begin{tabular}{@{}cc|cccccccc|cccccc@{}}
\toprule
\multicolumn{2}{c|}{}                          & \multicolumn{8}{c|}{$\beta$=-0.5}                                                                                                                                                                                                                                                                                                                                                               & \multicolumn{6}{c}{$\alpha$=-0.5}                                                                                                                                                                                                                                                    \\ \cmidrule(l){3-16} 
\multicolumn{2}{c|}{\multirow{-2}{*}{Setting}} & \multicolumn{2}{c|}{$\alpha$=-1}                                                                      & \multicolumn{2}{c|}{$\alpha$=-0.5}                                                                    & \multicolumn{2}{c|}{$\alpha$=0}                                                                       & \multicolumn{2}{c|}{$\alpha$=0.5}                                                & \multicolumn{2}{c|}{$\beta$=-1}                                                                       & \multicolumn{2}{c|}{$\beta$=0}                                                                     & \multicolumn{2}{c}{$\beta$=0.5}                                               \\
\bottomrule
\end{tabular}
}
\caption{Performance of vLinear under different $\alpha$ and $\beta$ settings in WFMLoss. The default setting ($\alpha=-0.5, \beta=-0.5$) yields the best performance.}
\label{tab_alpha_beta}
\end{center}
\end{table*}

In vLinear, we employ the following WFMLoss, where $\alpha$ and $\beta$ are set to -0.5 by default:

\begin{equation}
\begin{aligned}
    \mathrm{WFMLoss}^{(t)} =\frac{1}{H} (2-t)^{\alpha}\sum_{i=1}^{H}i^{\beta}\left \| \hat{\mathrm{Y}}^{(1)}_{:,i}- \bar{\mathrm{Y}}^{\mathrm{GT}}_{:,i}\right \|  _1,
\end{aligned}
\label{eq_wfmloss_appd}
\end{equation}

Table~\ref{tab_alpha_beta} presents the performance under various $\alpha$ and $\beta$ configurations. Overall, the default setting ($\alpha=-0.5$, $\beta=-0.5$) yields superior results compared to other settings. Notably, on the PEMS03 dataset, setting $\beta=-0.5$  results in a 2\% reduction in MSE compared to $\beta=0.5$, validating the effectiveness of Theorem~\ref{theorem3_appd}.

\section{vecTrans Visualization}\label{appd_vectrans_visual}

vecTrans employs a learnable vector to model token correlations. We integrate vecTrans into state-of-the-art forecasters, including OLinear, iTransformer, and PatchTST, and compare the weight distributions of the learned vector against the original attention (or NormLin) matrices. Figure~\ref{fig_more_heatmap} presents the visualization results on the ECL, Solar-Energy, and PEMS03 datasets. The learned vector exhibits a distribution similar to the column-wise average of the attention (for iTransformer and PatchTST) or NormLin (for OLinear and vLinear) matrices. Remarkably, despite using only a single vector, vecTrans achieves superior performance compared to the full attention mechanisms (see Table~\ref{tab_vec_attn_appd}). This suggests that a single vector is sufficient to effectively capture token dependencies in time series forecasting.

\begin{figure*}[t]
   \centering
   \includegraphics[width=1.0\linewidth]{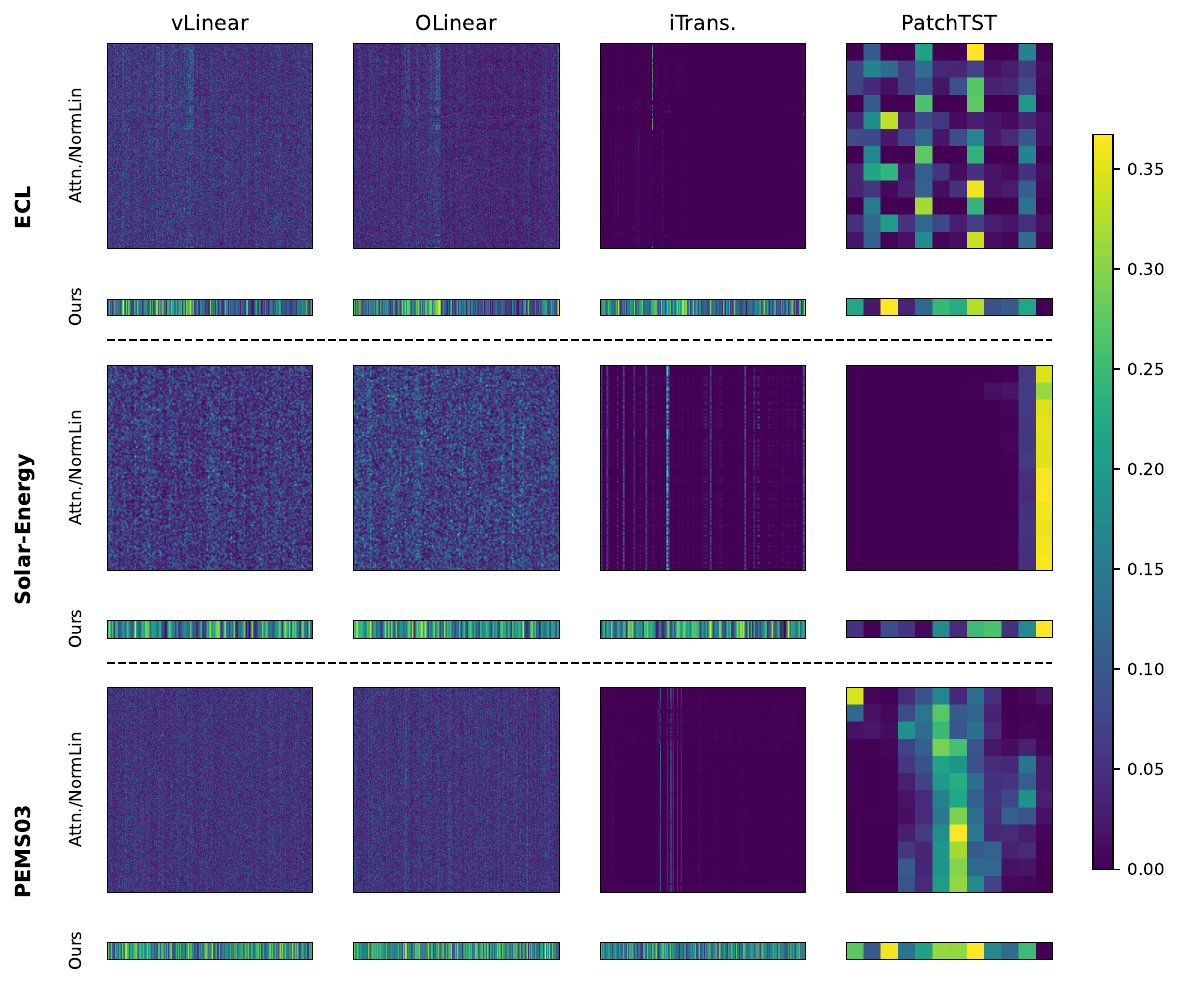}
   \caption{Visualization of Attention/NormLin matrices and the learned vector in vecTrans on the ECL, Solar-Energy and PEMS03 datasets.  As variants, vecTrans is applied to OLinear, iTransformer, and PatchTST, while NormLin is applied to vLinear. Generally, the learned vector approximates the column-wise weight distribution of the matrices, justifying our vector-based design.}
   \label{fig_more_heatmap}
\end{figure*}

\begin{table*}[t]
\begin{center}
\renewcommand{\arraystretch}{1.3}
{\fontsize{8}{10}\selectfont
\setlength{\tabcolsep}{4.5pt}
\begin{tabular}{@{}cc|cc|cc|cc|cc|cc|cc|cc@{}}
\toprule
\multicolumn{2}{c|}{vecTrans   Fun.} & \multicolumn{2}{c|}{Sigmoid}                                                  & \multicolumn{2}{c|}{Softmax}                                                  & \multicolumn{2}{c|}{Softplus}                                                 & \multicolumn{2}{c|}{Exp}                                                      & \multicolumn{2}{c|}{ReLU}                                                     & \multicolumn{2}{c|}{GELU}                                                     & \multicolumn{2}{c}{Identity}                                                  \\ \midrule
\multicolumn{2}{c|}{Metric}          & MSE                                   & MAE                                   & MSE                                   & MAE                                   & MSE                                   & MAE                                   & MSE                                   & MAE                                   & MSE                                   & MAE                                   & MSE                                   & MAE                                   & MSE                                   & MAE                                   \\ \midrule
                             & 96    & {\color[HTML]{FF0000} \textbf{0.356}} & {\color[HTML]{FF0000} \textbf{0.379}} & {\color[HTML]{0000FF} {\ul 0.357}}    & {\color[HTML]{FF0000} \textbf{0.379}} & {\color[HTML]{0000FF} {\ul 0.357}}    & {\color[HTML]{0000FF} {\ul 0.380}}    & {\color[HTML]{0000FF} {\ul 0.357}}    & {\color[HTML]{FF0000} \textbf{0.379}} & 0.358                                 & {\color[HTML]{0000FF} {\ul 0.380}}    & 0.358                                 & {\color[HTML]{0000FF} {\ul 0.380}}    & 0.356                                 & {\color[HTML]{FF0000} \textbf{0.379}} \\
                             & 192   & 0.409                                 & {\color[HTML]{0000FF} {\ul 0.411}}    & 0.409                                 & {\color[HTML]{0000FF} {\ul 0.411}}    & 0.410                                 & {\color[HTML]{0000FF} {\ul 0.411}}    & 0.409                                 & {\color[HTML]{0000FF} {\ul 0.411}}    & 0.409                                 & {\color[HTML]{FF0000} \textbf{0.410}} & {\color[HTML]{0000FF} {\ul 0.408}}    & {\color[HTML]{FF0000} \textbf{0.410}} & {\color[HTML]{FF0000} \textbf{0.407}} & {\color[HTML]{FF0000} \textbf{0.410}} \\
                             & 336   & 0.449                                 & 0.433                                 & 0.449                                 & {\color[HTML]{0000FF} {\ul 0.432}}    & 0.449                                 & 0.433                                 & 0.449                                 & 0.433                                 & 0.449                                 & {\color[HTML]{0000FF} {\ul 0.432}}    & {\color[HTML]{0000FF} {\ul 0.447}}    & {\color[HTML]{0000FF} {\ul 0.432}}    & {\color[HTML]{FF0000} \textbf{0.444}} & {\color[HTML]{FF0000} \textbf{0.431}} \\
                             & 720   & {\color[HTML]{FF0000} \textbf{0.451}} & {\color[HTML]{FF0000} \textbf{0.456}} & {\color[HTML]{0000FF} {\ul 0.452}}    & {\color[HTML]{0000FF} {\ul 0.457}}    & 0.453                                 & 0.458                                 & {\color[HTML]{FF0000} \textbf{0.451}} & {\color[HTML]{0000FF} {\ul 0.457}}    & {\color[HTML]{0000FF} {\ul 0.452}}    & {\color[HTML]{0000FF} {\ul 0.457}}    & {\color[HTML]{0000FF} {\ul 0.452}}    & {\color[HTML]{0000FF} {\ul 0.457}}    & 0.458                                 & 0.460                                 \\ \cmidrule(l){2-16} 
\multirow{-5}{*}{ETTh1}      & Avg   & {\color[HTML]{FF0000} \textbf{0.416}} & {\color[HTML]{FF0000} \textbf{0.420}} & {\color[HTML]{FF0000} \textbf{0.416}} & {\color[HTML]{FF0000} \textbf{0.420}} & {\color[HTML]{0000FF} {\ul 0.417}}    & {\color[HTML]{FF0000} \textbf{0.420}} & {\color[HTML]{FF0000} \textbf{0.416}} & {\color[HTML]{FF0000} \textbf{0.420}} & {\color[HTML]{0000FF} {\ul 0.417}}    & {\color[HTML]{FF0000} \textbf{0.420}} & {\color[HTML]{FF0000} \textbf{0.416}} & {\color[HTML]{FF0000} \textbf{0.420}} & {\color[HTML]{FF0000} \textbf{0.416}} & {\color[HTML]{FF0000} \textbf{0.420}} \\ \midrule
                             & 96    & {\color[HTML]{FF0000} \textbf{0.168}} & {\color[HTML]{FF0000} \textbf{0.245}} & {\color[HTML]{0000FF} {\ul 0.169}}    & {\color[HTML]{0000FF} {\ul 0.246}}    & {\color[HTML]{FF0000} \textbf{0.168}} & {\color[HTML]{0000FF} {\ul 0.246}}    & {\color[HTML]{FF0000} \textbf{0.168}} & {\color[HTML]{0000FF} {\ul 0.246}}    & {\color[HTML]{FF0000} \textbf{0.168}} & {\color[HTML]{0000FF} {\ul 0.246}}    & 0.170                                 & {\color[HTML]{0000FF} {\ul 0.246}}    & {\color[HTML]{0000FF} {\ul 0.169}}    & 0.247                                 \\
                             & 192   & {\color[HTML]{FF0000} \textbf{0.231}} & {\color[HTML]{FF0000} \textbf{0.287}} & {\color[HTML]{FF0000} \textbf{0.231}} & {\color[HTML]{FF0000} \textbf{0.287}} & {\color[HTML]{0000FF} {\ul 0.232}}    & {\color[HTML]{0000FF} {\ul 0.288}}    & {\color[HTML]{FF0000} \textbf{0.231}} & {\color[HTML]{FF0000} \textbf{0.287}} & {\color[HTML]{0000FF} {\ul 0.232}}    & {\color[HTML]{FF0000} \textbf{0.287}} & {\color[HTML]{0000FF} {\ul 0.232}}    & {\color[HTML]{0000FF} {\ul 0.288}}    & {\color[HTML]{0000FF} {\ul 0.232}}    & {\color[HTML]{0000FF} {\ul 0.288}}    \\
                             & 336   & {\color[HTML]{FF0000} \textbf{0.289}} & {\color[HTML]{FF0000} \textbf{0.326}} & 0.292                                 & {\color[HTML]{0000FF} {\ul 0.327}}    & {\color[HTML]{0000FF} {\ul 0.291}}    & {\color[HTML]{FF0000} \textbf{0.326}} & {\color[HTML]{0000FF} {\ul 0.291}}    & {\color[HTML]{FF0000} \textbf{0.326}} & 0.293                                 & {\color[HTML]{0000FF} {\ul 0.327}}    & 0.293                                 & 0.328                                 & 0.292                                 & {\color[HTML]{0000FF} {\ul 0.327}}    \\
                             & 720   & {\color[HTML]{FF0000} \textbf{0.386}} & {\color[HTML]{FF0000} \textbf{0.382}} & 0.388                                 & 0.384                                 & 0.388                                 & {\color[HTML]{0000FF} {\ul 0.383}}    & {\color[HTML]{0000FF} {\ul 0.387}}    & {\color[HTML]{0000FF} {\ul 0.383}}    & 0.390                                 & 0.385                                 & 0.389                                 & 0.384                                 & 0.388                                 & 0.384                                 \\ \cmidrule(l){2-16} 
\multirow{-5}{*}{ETTm2}      & Avg   & {\color[HTML]{FF0000} \textbf{0.268}} & {\color[HTML]{FF0000} \textbf{0.310}} & 0.270                                 & {\color[HTML]{0000FF} {\ul 0.311}}    & 0.270                                 & {\color[HTML]{0000FF} {\ul 0.311}}    & {\color[HTML]{0000FF} {\ul 0.269}}    & {\color[HTML]{FF0000} \textbf{0.310}} & 0.271                                 & {\color[HTML]{0000FF} {\ul 0.311}}    & 0.271                                 & {\color[HTML]{0000FF} {\ul 0.311}}    & 0.270                                 & {\color[HTML]{0000FF} {\ul 0.311}}    \\ \midrule
                             & 96    & {\color[HTML]{FF0000} \textbf{0.129}} & {\color[HTML]{FF0000} \textbf{0.221}} & 0.131                                 & {\color[HTML]{0000FF} {\ul 0.222}}    & {\color[HTML]{0000FF} {\ul 0.130}}    & {\color[HTML]{FF0000} \textbf{0.221}} & {\color[HTML]{0000FF} {\ul 0.130}}    & {\color[HTML]{FF0000} \textbf{0.221}} & 0.131                                 & {\color[HTML]{0000FF} {\ul 0.222}}    & 0.133                                 & 0.224                                 & 0.137                                 & 0.228                                 \\
                             & 192   & {\color[HTML]{FF0000} \textbf{0.147}} & {\color[HTML]{FF0000} \textbf{0.238}} & {\color[HTML]{0000FF} {\ul 0.148}}    & {\color[HTML]{0000FF} {\ul 0.240}}    & {\color[HTML]{FF0000} \textbf{0.147}} & {\color[HTML]{FF0000} \textbf{0.238}} & {\color[HTML]{0000FF} {\ul 0.148}}    & {\color[HTML]{FF0000} \textbf{0.238}} & 0.150                                 & 0.240                                 & 0.152                                 & 0.243                                 & 0.154                                 & 0.244                                 \\
                             & 336   & {\color[HTML]{0000FF} {\ul 0.158}}    & {\color[HTML]{FF0000} \textbf{0.250}} & {\color[HTML]{0000FF} {\ul 0.158}}    & {\color[HTML]{FF0000} \textbf{0.250}} & {\color[HTML]{0000FF} {\ul 0.158}}    & {\color[HTML]{FF0000} \textbf{0.250}} & {\color[HTML]{FF0000} \textbf{0.157}} & {\color[HTML]{FF0000} \textbf{0.250}} & {\color[HTML]{0000FF} {\ul 0.158}}    & {\color[HTML]{0000FF} {\ul 0.251}}    & 0.164                                 & 0.256                                 & 0.165                                 & 0.258                                 \\
                             & 720   & {\color[HTML]{FF0000} \textbf{0.178}} & {\color[HTML]{FF0000} \textbf{0.271}} & {\color[HTML]{0000FF} {\ul 0.179}}    & {\color[HTML]{0000FF} {\ul 0.272}}    & 0.181                                 & 0.274                                 & 0.180                                 & 0.274                                 & {\color[HTML]{0000FF} {\ul 0.179}}    & {\color[HTML]{0000FF} {\ul 0.272}}    & 0.182                                 & 0.275                                 & 0.190                                 & 0.280                                 \\ \cmidrule(l){2-16} 
\multirow{-5}{*}{ECL}        & Avg   & {\color[HTML]{FF0000} \textbf{0.153}} & {\color[HTML]{FF0000} \textbf{0.245}} & {\color[HTML]{0000FF} {\ul 0.154}}    & {\color[HTML]{0000FF} {\ul 0.246}}    & {\color[HTML]{0000FF} {\ul 0.154}}    & {\color[HTML]{0000FF} {\ul 0.246}}    & {\color[HTML]{0000FF} {\ul 0.154}}    & {\color[HTML]{0000FF} {\ul 0.246}}    & {\color[HTML]{0000FF} {\ul 0.154}}    & {\color[HTML]{0000FF} {\ul 0.246}}    & 0.158                                 & 0.249                                 & 0.161                                 & 0.252                                 \\ \midrule
                             & 96    & {\color[HTML]{FF0000} \textbf{0.175}} & {\color[HTML]{FF0000} \textbf{0.196}} & 0.182                                 & 0.200                                 & {\color[HTML]{FF0000} \textbf{0.175}} & {\color[HTML]{0000FF} {\ul 0.197}}    & {\color[HTML]{0000FF} {\ul 0.177}}    & {\color[HTML]{0000FF} {\ul 0.197}}    & 0.181                                 & 0.200                                 & 0.183                                 & 0.201                                 & 0.198                                 & 0.211                                 \\
                             & 192   & {\color[HTML]{FF0000} \textbf{0.202}} & {\color[HTML]{FF0000} \textbf{0.220}} & 0.204                                 & 0.222                                 & {\color[HTML]{0000FF} {\ul 0.203}}    & {\color[HTML]{0000FF} {\ul 0.221}}    & {\color[HTML]{FF0000} \textbf{0.202}} & {\color[HTML]{0000FF} {\ul 0.221}}    & {\color[HTML]{FF0000} \textbf{0.202}} & {\color[HTML]{0000FF} {\ul 0.221}}    & 0.204                                 & 0.222                                 & 0.213                                 & 0.228                                 \\
                             & 336   & {\color[HTML]{FF0000} \textbf{0.224}} & {\color[HTML]{FF0000} \textbf{0.240}} & 0.227                                 & {\color[HTML]{0000FF} {\ul 0.241}}    & {\color[HTML]{0000FF} {\ul 0.225}}    & {\color[HTML]{FF0000} \textbf{0.240}} & {\color[HTML]{0000FF} {\ul 0.225}}    & {\color[HTML]{FF0000} \textbf{0.240}} & 0.227                                 & {\color[HTML]{0000FF} {\ul 0.241}}    & 0.230                                 & 0.243                                 & 0.236                                 & 0.248                                 \\
                             & 720   & {\color[HTML]{FF0000} \textbf{0.234}} & {\color[HTML]{FF0000} \textbf{0.247}} & 0.238                                 & 0.251                                 & {\color[HTML]{0000FF} {\ul 0.235}}    & {\color[HTML]{0000FF} {\ul 0.249}}    & 0.236                                 & {\color[HTML]{0000FF} {\ul 0.249}}    & 0.239                                 & 0.251                                 & 0.241                                 & 0.252                                 & 0.243                                 & 0.253                                 \\ \cmidrule(l){2-16} 
\multirow{-5}{*}{Solar}      & Avg   & {\color[HTML]{FF0000} \textbf{0.209}} & {\color[HTML]{FF0000} \textbf{0.226}} & 0.213                                 & 0.228                                 & {\color[HTML]{FF0000} \textbf{0.209}} & {\color[HTML]{0000FF} {\ul 0.227}}    & {\color[HTML]{0000FF} {\ul 0.210}}    & {\color[HTML]{0000FF} {\ul 0.227}}    & 0.212                                 & 0.228                                 & 0.214                                 & 0.229                                 & 0.222                                 & 0.235                                 \\ \midrule
                             & 96    & {\color[HTML]{0000FF} {\ul 0.150}}    & {\color[HTML]{0000FF} {\ul 0.187}}    & {\color[HTML]{0000FF} {\ul 0.150}}    & {\color[HTML]{FF0000} \textbf{0.186}} & {\color[HTML]{FF0000} \textbf{0.149}} & {\color[HTML]{FF0000} \textbf{0.186}} & 0.151                                 & {\color[HTML]{0000FF} {\ul 0.187}}    & 0.151                                 & 0.188                                 & 0.151                                 & 0.188                                 & 0.154                                 & 0.189                                 \\
                             & 192   & {\color[HTML]{FF0000} \textbf{0.198}} & {\color[HTML]{FF0000} \textbf{0.234}} & {\color[HTML]{0000FF} {\ul 0.199}}    & {\color[HTML]{FF0000} \textbf{0.234}} & {\color[HTML]{0000FF} {\ul 0.199}}    & {\color[HTML]{FF0000} \textbf{0.234}} & {\color[HTML]{0000FF} {\ul 0.199}}    & {\color[HTML]{0000FF} {\ul 0.235}}    & 0.200                                 & {\color[HTML]{0000FF} {\ul 0.235}}    & 0.201                                 & 0.236                                 & 0.200                                 & 0.236                                 \\
                             & 336   & {\color[HTML]{FF0000} \textbf{0.253}} & {\color[HTML]{FF0000} \textbf{0.275}} & {\color[HTML]{0000FF} {\ul 0.254}}    & {\color[HTML]{0000FF} {\ul 0.276}}    & {\color[HTML]{FF0000} \textbf{0.253}} & {\color[HTML]{FF0000} \textbf{0.275}} & {\color[HTML]{FF0000} \textbf{0.253}} & {\color[HTML]{FF0000} \textbf{0.275}} & 0.256                                 & 0.278                                 & 0.258                                 & 0.279                                 & 0.258                                 & 0.279                                 \\
                             & 720   & {\color[HTML]{FF0000} \textbf{0.333}} & {\color[HTML]{FF0000} \textbf{0.329}} & {\color[HTML]{FF0000} \textbf{0.333}} & {\color[HTML]{0000FF} {\ul 0.330}}    & {\color[HTML]{FF0000} \textbf{0.333}} & {\color[HTML]{FF0000} \textbf{0.329}} & {\color[HTML]{FF0000} \textbf{0.333}} & {\color[HTML]{FF0000} \textbf{0.329}} & {\color[HTML]{0000FF} {\ul 0.335}}    & 0.331                                 & 0.338                                 & 0.333                                 & {\color[HTML]{0000FF} {\ul 0.335}}    & 0.332                                 \\ \cmidrule(l){2-16} 
\multirow{-5}{*}{Weather}    & Avg   & {\color[HTML]{FF0000} \textbf{0.233}} & {\color[HTML]{FF0000} \textbf{0.256}} & {\color[HTML]{0000FF} {\ul 0.234}}    & {\color[HTML]{FF0000} \textbf{0.256}} & {\color[HTML]{FF0000} \textbf{0.233}} & {\color[HTML]{FF0000} \textbf{0.256}} & {\color[HTML]{0000FF} {\ul 0.234}}    & {\color[HTML]{FF0000} \textbf{0.256}} & 0.235                                 & {\color[HTML]{0000FF} {\ul 0.258}}    & 0.237                                 & 0.259                                 & 0.237                                 & 0.259                                 \\ \midrule
                             & 12    & {\color[HTML]{FF0000} \textbf{0.059}} & {\color[HTML]{FF0000} \textbf{0.158}} & {\color[HTML]{FF0000} \textbf{0.059}} & {\color[HTML]{0000FF} {\ul 0.159}}    & {\color[HTML]{FF0000} \textbf{0.059}} & {\color[HTML]{FF0000} \textbf{0.158}} & {\color[HTML]{FF0000} \textbf{0.059}} & {\color[HTML]{FF0000} \textbf{0.158}} & {\color[HTML]{FF0000} \textbf{0.059}} & {\color[HTML]{0000FF} {\ul 0.159}}    & {\color[HTML]{0000FF} {\ul 0.060}}    & 0.160                                 & 0.061                                 & 0.161                                 \\
                             & 24    & {\color[HTML]{FF0000} \textbf{0.075}} & {\color[HTML]{FF0000} \textbf{0.177}} & {\color[HTML]{FF0000} \textbf{0.075}} & {\color[HTML]{0000FF} {\ul 0.178}}    & {\color[HTML]{FF0000} \textbf{0.075}} & {\color[HTML]{0000FF} {\ul 0.178}}    & {\color[HTML]{0000FF} {\ul 0.076}}    & {\color[HTML]{0000FF} {\ul 0.178}}    & {\color[HTML]{FF0000} \textbf{0.075}} & {\color[HTML]{0000FF} {\ul 0.178}}    & 0.077                                 & 0.180                                 & 0.079                                 & 0.182                                 \\
                             & 48    & {\color[HTML]{FF0000} \textbf{0.102}} & {\color[HTML]{FF0000} \textbf{0.208}} & 0.105                                 & 0.210                                 & {\color[HTML]{0000FF} {\ul 0.103}}    & {\color[HTML]{FF0000} \textbf{0.208}} & {\color[HTML]{0000FF} {\ul 0.103}}    & {\color[HTML]{FF0000} \textbf{0.208}} & {\color[HTML]{0000FF} {\ul 0.103}}    & {\color[HTML]{0000FF} {\ul 0.209}}    & 0.107                                 & 0.213                                 & 0.120                                 & 0.225                                 \\
                             & 96    & {\color[HTML]{FF0000} \textbf{0.138}} & {\color[HTML]{FF0000} \textbf{0.246}} & 0.143                                 & 0.250                                 & {\color[HTML]{0000FF} {\ul 0.140}}    & {\color[HTML]{0000FF} {\ul 0.247}}    & 0.142                                 & 0.249                                 & 0.142                                 & 0.249                                 & 0.146                                 & 0.253                                 & 0.193                                 & 0.278                                 \\ \cmidrule(l){2-16} 
\multirow{-5}{*}{PEMS03}     & Avg   & {\color[HTML]{FF0000} \textbf{0.093}} & {\color[HTML]{FF0000} \textbf{0.197}} & 0.095                                 & 0.199                                 & {\color[HTML]{0000FF} {\ul 0.094}}    & {\color[HTML]{0000FF} {\ul 0.198}}    & 0.095                                 & {\color[HTML]{0000FF} {\ul 0.198}}    & 0.095                                 & {\color[HTML]{0000FF} {\ul 0.198}}    & 0.097                                 & 0.201                                 & 0.113                                 & 0.212                                 \\ \bottomrule
\end{tabular}
}
\caption{Different transformation functions in vecTrans. \textit{Identity} denotes no transformation applied.}
\label{tab_vectrans_fun}
\end{center}
\end{table*}

\begin{table*}[t]
\begin{center}
\renewcommand{\arraystretch}{1.3}
{\fontsize{9}{10}\selectfont
\setlength{\tabcolsep}{8pt}
\begin{tabular}{@{}cc|cc|cc|cc|cc|cc|cc@{}}
\toprule
                       &                        & \multicolumn{2}{c|}{ETTh1}                                                    & \multicolumn{2}{c|}{ETTm2}                                                    & \multicolumn{2}{c|}{ECL}                                                      & \multicolumn{2}{c|}{Solar}                                                    & \multicolumn{2}{c|}{Weather}                                                  & \multicolumn{2}{c}{PEMS03}                                                    \\ \cmidrule(l){3-14} 
\multirow{-2}{*}{Norm} & \multirow{-2}{*}{Hor.} & MSE                                   & MAE                                   & MSE                                   & MAE                                   & MSE                                   & MAE                                   & MSE                                   & MAE                                   & MSE                                   & MAE                                   & MSE                                   & MAE                                   \\ \midrule
                       & H1                     & {\color[HTML]{FF0000} \textbf{0.356}} & {\color[HTML]{FF0000} \textbf{0.379}} & {\color[HTML]{FF0000} \textbf{0.168}} & {\color[HTML]{FF0000} \textbf{0.245}} & {\color[HTML]{FF0000} \textbf{0.129}} & {\color[HTML]{FF0000} \textbf{0.221}} & {\color[HTML]{FF0000} \textbf{0.175}} & {\color[HTML]{FF0000} \textbf{0.196}} & {\color[HTML]{FF0000} \textbf{0.150}} & {\color[HTML]{FF0000} \textbf{0.186}} & {\color[HTML]{FF0000} \textbf{0.059}} & {\color[HTML]{FF0000} \textbf{0.158}} \\
                       & H2                     & {\color[HTML]{FF0000} \textbf{0.409}} & {\color[HTML]{FF0000} \textbf{0.411}} & {\color[HTML]{FF0000} \textbf{0.231}} & {\color[HTML]{FF0000} \textbf{0.287}} & {\color[HTML]{FF0000} \textbf{0.147}} & {\color[HTML]{FF0000} \textbf{0.238}} & {\color[HTML]{FF0000} \textbf{0.202}} & {\color[HTML]{FF0000} \textbf{0.220}} & {\color[HTML]{FF0000} \textbf{0.198}} & {\color[HTML]{FF0000} \textbf{0.234}} & {\color[HTML]{FF0000} \textbf{0.075}} & {\color[HTML]{FF0000} \textbf{0.177}} \\
                       & H3                     & {\color[HTML]{FF0000} \textbf{0.449}} & {\color[HTML]{FF0000} \textbf{0.433}} & {\color[HTML]{FF0000} \textbf{0.289}} & {\color[HTML]{FF0000} \textbf{0.326}} & {\color[HTML]{FF0000} \textbf{0.158}} & {\color[HTML]{FF0000} \textbf{0.250}} & {\color[HTML]{FF0000} \textbf{0.224}} & {\color[HTML]{FF0000} \textbf{0.240}} & {\color[HTML]{FF0000} \textbf{0.252}} & {\color[HTML]{FF0000} \textbf{0.275}} & {\color[HTML]{FF0000} \textbf{0.102}} & {\color[HTML]{FF0000} \textbf{0.208}} \\
                       & H4                     & {\color[HTML]{FF0000} \textbf{0.451}} & {\color[HTML]{FF0000} \textbf{0.456}} & {\color[HTML]{FF0000} \textbf{0.386}} & {\color[HTML]{FF0000} \textbf{0.382}} & {\color[HTML]{FF0000} \textbf{0.178}} & {\color[HTML]{FF0000} \textbf{0.271}} & {\color[HTML]{FF0000} \textbf{0.234}} & {\color[HTML]{FF0000} \textbf{0.247}} & {\color[HTML]{FF0000} \textbf{0.332}} & {\color[HTML]{FF0000} \textbf{0.328}} & {\color[HTML]{FF0000} \textbf{0.138}} & {\color[HTML]{FF0000} \textbf{0.246}} \\ \cmidrule(l){2-14} 
\multirow{-5}{*}{L1}   & Avg                    & {\color[HTML]{FF0000} \textbf{0.416}} & {\color[HTML]{FF0000} \textbf{0.420}} & {\color[HTML]{FF0000} \textbf{0.268}} & {\color[HTML]{FF0000} \textbf{0.310}} & {\color[HTML]{FF0000} \textbf{0.153}} & {\color[HTML]{FF0000} \textbf{0.245}} & {\color[HTML]{FF0000} \textbf{0.209}} & {\color[HTML]{FF0000} \textbf{0.226}} & {\color[HTML]{FF0000} \textbf{0.233}} & {\color[HTML]{FF0000} \textbf{0.256}} & {\color[HTML]{FF0000} \textbf{0.093}} & {\color[HTML]{FF0000} \textbf{0.197}} \\ \midrule
                       & H1                     & {\color[HTML]{0000FF} {\ul 0.357}}    & {\color[HTML]{0000FF} {\ul 0.380}}    & {\color[HTML]{0000FF} {\ul 0.169}}    & {\color[HTML]{0000FF} {\ul 0.247}}    & {\color[HTML]{0000FF} {\ul 0.239}}    & {\color[HTML]{0000FF} {\ul 0.314}}    & {\color[HTML]{0000FF} {\ul 0.243}}    & {\color[HTML]{0000FF} {\ul 0.259}}    & {\color[HTML]{0000FF} {\ul 0.151}}    & {\color[HTML]{0000FF} {\ul 0.188}}    & 0.096                                 & {\color[HTML]{0000FF} {\ul 0.199}}    \\
                       & H2                     & {\color[HTML]{0000FF} {\ul 0.410}}    & {\color[HTML]{0000FF} {\ul 0.412}}    & {\color[HTML]{0000FF} {\ul 0.232}}    & {\color[HTML]{0000FF} {\ul 0.288}}    & {\color[HTML]{0000FF} {\ul 0.238}}    & {\color[HTML]{0000FF} {\ul 0.317}}    & 0.281                                 & 0.280                                 & {\color[HTML]{0000FF} {\ul 0.201}}    & {\color[HTML]{0000FF} {\ul 0.235}}    & {\color[HTML]{0000FF} {\ul 0.141}}    & 0.250                                 \\
                       & H3                     & {\color[HTML]{0000FF} {\ul 0.450}}    & {\color[HTML]{0000FF} {\ul 0.434}}    & {\color[HTML]{0000FF} {\ul 0.317}}    & {\color[HTML]{0000FF} {\ul 0.348}}    & {\color[HTML]{0000FF} {\ul 0.247}}    & {\color[HTML]{0000FF} {\ul 0.327}}    & {\color[HTML]{0000FF} {\ul 0.324}}    & {\color[HTML]{0000FF} {\ul 0.301}}    & {\color[HTML]{0000FF} {\ul 0.254}}    & {\color[HTML]{0000FF} {\ul 0.276}}    & 1.507                                 & 1.033                                 \\
                       & H4                     & {\color[HTML]{0000FF} {\ul 0.453}}    & {\color[HTML]{0000FF} {\ul 0.458}}    & {\color[HTML]{0000FF} {\ul 0.386}}    & {\color[HTML]{0000FF} {\ul 0.382}}    & 0.268                                 & 0.344                                 & 0.336                                 & 0.304                                 & {\color[HTML]{0000FF} {\ul 0.340}}    & {\color[HTML]{0000FF} {\ul 0.334}}    & {\color[HTML]{0000FF} {\ul 0.279}}    & {\color[HTML]{0000FF} {\ul 0.342}}    \\ \cmidrule(l){2-14} 
\multirow{-5}{*}{L2}   & Avg                    & {\color[HTML]{0000FF} {\ul 0.418}}    & {\color[HTML]{0000FF} {\ul 0.421}}    & {\color[HTML]{0000FF} {\ul 0.276}}    & {\color[HTML]{0000FF} {\ul 0.316}}    & 0.248                                 & {\color[HTML]{0000FF} {\ul 0.325}}    & {\color[HTML]{0000FF} {\ul 0.296}}    & {\color[HTML]{0000FF} {\ul 0.286}}    & {\color[HTML]{0000FF} {\ul 0.236}}    & {\color[HTML]{0000FF} {\ul 0.258}}    & 0.506                                 & 0.456                                 \\ \midrule
                       & H1                     & 0.360                                 & 0.383                                 & 0.195                                 & 0.273                                 & {\color[HTML]{0000FF} {\ul 0.239}}    & {\color[HTML]{0000FF} {\ul 0.314}}    & 0.244                                 & 0.266                                 & 0.156                                 & 0.193                                 & {\color[HTML]{0000FF} {\ul 0.094}}    & {\color[HTML]{0000FF} {\ul 0.199}}    \\
                       & H2                     & 0.412                                 & 0.414                                 & 0.254                                 & 0.308                                 & {\color[HTML]{0000FF} {\ul 0.238}}    & {\color[HTML]{0000FF} {\ul 0.317}}    & {\color[HTML]{0000FF} {\ul 0.279}}    & {\color[HTML]{0000FF} {\ul 0.278}}    & 0.255                                 & 0.289                                 & {\color[HTML]{0000FF} {\ul 0.141}}    & {\color[HTML]{0000FF} {\ul 0.249}}    \\
                       & H3                     & 0.451                                 & {\color[HTML]{0000FF} {\ul 0.434}}    & {\color[HTML]{0000FF} {\ul 0.317}}    & {\color[HTML]{0000FF} {\ul 0.348}}    & 0.248                                 & {\color[HTML]{0000FF} {\ul 0.327}}    & 0.326                                 & {\color[HTML]{0000FF} {\ul 0.301}}    & 0.302                                 & 0.317                                 & {\color[HTML]{0000FF} {\ul 0.203}}    & {\color[HTML]{0000FF} {\ul 0.298}}    \\
                       & H4                     & 0.454                                 & {\color[HTML]{0000FF} {\ul 0.458}}    & 0.388                                 & 0.384                                 & {\color[HTML]{0000FF} {\ul 0.265}}    & {\color[HTML]{0000FF} {\ul 0.342}}    & {\color[HTML]{0000FF} {\ul 0.335}}    & {\color[HTML]{0000FF} {\ul 0.303}}    & 0.378                                 & 0.368                                 & 0.298                                 & 0.351                                 \\ \cmidrule(l){2-14} 
\multirow{-5}{*}{None} & Avg                    & 0.419                                 & 0.422                                 & 0.289                                 & 0.328                                 & {\color[HTML]{0000FF} {\ul 0.247}}    & {\color[HTML]{0000FF} {\ul 0.325}}    & {\color[HTML]{0000FF} {\ul 0.296}}    & 0.287                                 & 0.273                                 & 0.292                                 & {\color[HTML]{0000FF} {\ul 0.184}}    & {\color[HTML]{0000FF} {\ul 0.274}}    \\ \bottomrule
\end{tabular}
}
\caption{Different normalization strategies in vecTrans. \textit{None} denotes no normalization applied. $\left \{ \mathrm{H1}, \mathrm{H2}, \mathrm{H3}, \mathrm{H4 } \right \}$ corresponds to $\{12,24,48,96\}$ for PEMS03, and $\{96,192,336,720\}$ for the other datasets.}
\label{tab_vectrans_normset}
\end{center}
\end{table*}

\section{Ablation Studies on vecTrans Design}
\label{appd_ablation_vectrans}

In vecTrans, we apply $\mathrm{Sigmoid}$ to ensure the positivity of the values in the learnable vector $\mathbf{a}$, followed by L1 normalization. In this section, we compare this default configuration with alternative transformation functions and normalization settings.

Table~\ref{tab_vectrans_fun} compares $\mathrm{Sigmoid}$ with $\mathrm{Softmax}$, $\mathrm{Softplus}$, $\mathrm{Exp}$, $\mathrm{ReLU}$, $\mathrm{GELU}$, and $\mathrm{Identity}$. Results show that while these variants yield comparable performance, $\mathrm{Sigmoid}$ achieves marginally better results.

Table~\ref{tab_vectrans_normset} compares $\mathrm{L1}$ normalization with $\mathrm{L2}$ normalization and the absence of normalization (None). Results demonstrate that $\mathrm{L1}$ consistently outperforms other variants. We hypothesize that $\mathrm{L1}$ normalization encourages the model to focus on the relative weight distribution among tokens rather than the absolute amplitude, which appears less critical.

\section{Multi-Layer WFMLin}
\label{appd_wfmlin}

\begin{table*}[t]
\begin{center}
\renewcommand{\arraystretch}{1.3}
{\fontsize{9}{10}\selectfont
\setlength{\tabcolsep}{10pt}
\begin{tabular}{@{}cc|cc|cc|cc@{}}
\toprule
\multicolumn{2}{c|}{WFMLin}     & \multicolumn{2}{c|}{Layer: 1}                                                 & \multicolumn{2}{c|}{Layer: 2}                                              & \multicolumn{2}{c}{Layer: 3}                                            \\ \midrule
\multicolumn{2}{c|}{Metric}     & MSE                                   & MAE                                   & MSE                                   & MAE                                & MSE                                & MAE                                \\ \midrule
                          & 96  & {\color[HTML]{FF0000} \textbf{0.356}} & {\color[HTML]{FF0000} \textbf{0.379}} & {\color[HTML]{0000FF} {\ul 0.364}}    & {\color[HTML]{0000FF} {\ul 0.386}} & 0.370                              & 0.391                              \\
                          & 192 & {\color[HTML]{FF0000} \textbf{0.409}} & {\color[HTML]{FF0000} \textbf{0.411}} & {\color[HTML]{0000FF} {\ul 0.416}}    & {\color[HTML]{0000FF} {\ul 0.416}} & 0.428                              & 0.421                              \\
                          & 336 & {\color[HTML]{FF0000} \textbf{0.449}} & {\color[HTML]{FF0000} \textbf{0.433}} & {\color[HTML]{0000FF} {\ul 0.452}}    & {\color[HTML]{0000FF} {\ul 0.434}} & 0.473                              & 0.441                              \\
                          & 720 & {\color[HTML]{FF0000} \textbf{0.451}} & {\color[HTML]{FF0000} \textbf{0.456}} & {\color[HTML]{0000FF} {\ul 0.462}}    & {\color[HTML]{0000FF} {\ul 0.459}} & 0.507                              & 0.486                              \\ \cmidrule(l){2-8} 
\multirow{-5}{*}{ETTh1}   & Avg & {\color[HTML]{FF0000} \textbf{0.416}} & {\color[HTML]{FF0000} \textbf{0.420}} & {\color[HTML]{0000FF} {\ul 0.423}}    & {\color[HTML]{0000FF} {\ul 0.424}} & 0.444                              & 0.435                              \\ \midrule
                          & 96  & {\color[HTML]{FF0000} \textbf{0.168}} & {\color[HTML]{FF0000} \textbf{0.245}} & {\color[HTML]{0000FF} {\ul 0.171}}    & {\color[HTML]{0000FF} {\ul 0.248}} & 0.175                              & 0.255                              \\
                          & 192 & {\color[HTML]{FF0000} \textbf{0.231}} & {\color[HTML]{FF0000} \textbf{0.287}} & {\color[HTML]{0000FF} {\ul 0.234}}    & {\color[HTML]{0000FF} {\ul 0.290}} & 0.240                              & 0.298                              \\
                          & 336 & {\color[HTML]{FF0000} \textbf{0.289}} & {\color[HTML]{FF0000} \textbf{0.326}} & {\color[HTML]{0000FF} {\ul 0.299}}    & {\color[HTML]{0000FF} {\ul 0.332}} & 0.323                              & 0.356                              \\
                          & 720 & {\color[HTML]{FF0000} \textbf{0.386}} & {\color[HTML]{FF0000} \textbf{0.382}} & {\color[HTML]{0000FF} {\ul 0.396}}    & {\color[HTML]{0000FF} {\ul 0.390}} & 0.405                              & 0.395                              \\ \cmidrule(l){2-8} 
\multirow{-5}{*}{ETTm2}   & Avg & {\color[HTML]{FF0000} \textbf{0.268}} & {\color[HTML]{FF0000} \textbf{0.310}} & {\color[HTML]{0000FF} {\ul 0.275}}    & {\color[HTML]{0000FF} {\ul 0.315}} & 0.286                              & 0.326                              \\ \midrule
                          & 96  & {\color[HTML]{FF0000} \textbf{0.129}} & {\color[HTML]{FF0000} \textbf{0.221}} & {\color[HTML]{0000FF} {\ul 0.131}}    & {\color[HTML]{0000FF} {\ul 0.222}} & 0.134                              & 0.224                              \\
                          & 192 & {\color[HTML]{FF0000} \textbf{0.147}} & {\color[HTML]{FF0000} \textbf{0.238}} & {\color[HTML]{0000FF} {\ul 0.149}}    & {\color[HTML]{0000FF} {\ul 0.239}} & 0.155                              & 0.243                              \\
                          & 336 & {\color[HTML]{FF0000} \textbf{0.158}} & {\color[HTML]{FF0000} \textbf{0.250}} & {\color[HTML]{0000FF} {\ul 0.164}}    & {\color[HTML]{0000FF} {\ul 0.254}} & 0.172                              & 0.260                              \\
                          & 720 & {\color[HTML]{FF0000} \textbf{0.178}} & {\color[HTML]{FF0000} \textbf{0.271}} & {\color[HTML]{0000FF} {\ul 0.186}}    & {\color[HTML]{0000FF} {\ul 0.275}} & 0.207                              & 0.291                              \\ \cmidrule(l){2-8} 
\multirow{-5}{*}{ECL}     & Avg & {\color[HTML]{FF0000} \textbf{0.153}} & {\color[HTML]{FF0000} \textbf{0.245}} & {\color[HTML]{0000FF} {\ul 0.157}}    & {\color[HTML]{0000FF} {\ul 0.247}} & 0.167                              & 0.255                              \\ \midrule
                          & 96  & {\color[HTML]{FF0000} \textbf{0.175}} & {\color[HTML]{FF0000} \textbf{0.196}} & {\color[HTML]{0000FF} {\ul 0.181}}    & {\color[HTML]{0000FF} {\ul 0.210}} & 0.186                              & 0.216                              \\
                          & 192 & {\color[HTML]{FF0000} \textbf{0.202}} & {\color[HTML]{FF0000} \textbf{0.220}} & {\color[HTML]{0000FF} {\ul 0.209}}    & {\color[HTML]{0000FF} {\ul 0.229}} & 0.212                              & 0.232                              \\
                          & 336 & {\color[HTML]{FF0000} \textbf{0.224}} & {\color[HTML]{FF0000} \textbf{0.240}} & {\color[HTML]{0000FF} {\ul 0.231}}    & {\color[HTML]{0000FF} {\ul 0.244}} & 0.238                              & 0.248                              \\
                          & 720 & {\color[HTML]{FF0000} \textbf{0.234}} & {\color[HTML]{FF0000} \textbf{0.247}} & {\color[HTML]{0000FF} {\ul 0.247}}    & {\color[HTML]{0000FF} {\ul 0.250}} & 0.249                              & 0.252                              \\ \cmidrule(l){2-8} 
\multirow{-5}{*}{Solar}   & Avg & {\color[HTML]{FF0000} \textbf{0.209}} & {\color[HTML]{FF0000} \textbf{0.226}} & {\color[HTML]{0000FF} {\ul 0.217}}    & {\color[HTML]{0000FF} {\ul 0.233}} & 0.221                              & 0.237                              \\ \midrule
                          & 96  & {\color[HTML]{FF0000} \textbf{0.150}} & {\color[HTML]{FF0000} \textbf{0.186}} & {\color[HTML]{FF0000} \textbf{0.150}} & {\color[HTML]{0000FF} {\ul 0.188}} & {\color[HTML]{0000FF} {\ul 0.153}} & 0.189                              \\
                          & 192 & {\color[HTML]{FF0000} \textbf{0.198}} & {\color[HTML]{FF0000} \textbf{0.234}} & {\color[HTML]{0000FF} {\ul 0.201}}    & {\color[HTML]{0000FF} {\ul 0.236}} & 0.206                              & 0.238                              \\
                          & 336 & {\color[HTML]{FF0000} \textbf{0.252}} & {\color[HTML]{FF0000} \textbf{0.275}} & {\color[HTML]{0000FF} {\ul 0.255}}    & {\color[HTML]{0000FF} {\ul 0.277}} & 0.263                              & 0.280                              \\
                          & 720 & {\color[HTML]{FF0000} \textbf{0.332}} & {\color[HTML]{FF0000} \textbf{0.328}} & {\color[HTML]{0000FF} {\ul 0.343}}    & {\color[HTML]{0000FF} {\ul 0.335}} & 0.346                              & {\color[HTML]{0000FF} {\ul 0.335}} \\ \cmidrule(l){2-8} 
\multirow{-5}{*}{Weather} & Avg & {\color[HTML]{FF0000} \textbf{0.233}} & {\color[HTML]{FF0000} \textbf{0.256}} & {\color[HTML]{0000FF} {\ul 0.237}}    & {\color[HTML]{0000FF} {\ul 0.259}} & 0.242                              & 0.261                              \\ \midrule
                          & 12  & {\color[HTML]{FF0000} \textbf{0.059}} & {\color[HTML]{FF0000} \textbf{0.158}} & {\color[HTML]{FF0000} \textbf{0.059}} & {\color[HTML]{0000FF} {\ul 0.159}} & {\color[HTML]{0000FF} {\ul 0.060}} & 0.160                              \\
                          & 24  & {\color[HTML]{FF0000} \textbf{0.075}} & {\color[HTML]{FF0000} \textbf{0.177}} & {\color[HTML]{0000FF} {\ul 0.076}}    & {\color[HTML]{0000FF} {\ul 0.179}} & 0.077                              & {\color[HTML]{0000FF} {\ul 0.179}} \\
                          & 48  & {\color[HTML]{FF0000} \textbf{0.102}} & {\color[HTML]{FF0000} \textbf{0.208}} & {\color[HTML]{0000FF} {\ul 0.103}}    & {\color[HTML]{0000FF} {\ul 0.209}} & 0.106                              & 0.211                              \\
                          & 96  & {\color[HTML]{FF0000} \textbf{0.138}} & {\color[HTML]{FF0000} \textbf{0.246}} & {\color[HTML]{0000FF} {\ul 0.139}}    & {\color[HTML]{0000FF} {\ul 0.249}} & {\color[HTML]{0000FF} {\ul 0.139}} & {\color[HTML]{0000FF} {\ul 0.249}} \\ \cmidrule(l){2-8} 
\multirow{-5}{*}{PEMS03}  & Avg & {\color[HTML]{FF0000} \textbf{0.093}} & {\color[HTML]{FF0000} \textbf{0.197}} & {\color[HTML]{0000FF} {\ul 0.094}}    & {\color[HTML]{0000FF} {\ul 0.199}} & 0.095                              & 0.200                              \\ \bottomrule
\end{tabular}
}
\caption{Multi-layer WFMLin.}
\label{tab_wfmlin}
\end{center}
\end{table*}

In WFMBlock, we employ a one-layer WFMLin to generate the velocity. For comparison, we implement variants that stack multiple linear layers with GELU activation between them.

As shown in Table~\ref{tab_wfmlin}, the single-layer setting consistently outperforms the multi-layer variants. This suggests that overly complex WFMLin designs are unnecessary for effective time series forecasting.

\section{Additional Experiments}
\label{appd_additional_exp}

\subsection{vLinear as a Probabilistic Forecaster}
\label{appd_prob_forecaster}

Probabilistic time series forecasting \cite{quantileformer,CSDI} aims to provide a distribution of potential outcomes to support downstream decision-making. While vLinear typically adopts a deterministic strategy during inference by initializing with an all-zero state, it can also function as a probabilistic forecaster by replacing this initial state with Gaussian noise. In our experiments, we employ the Q-Risk metric to evaluate this capability, which is defined as \cite{rnn_2020}:

\begin{equation}
\begin{aligned}
\text{Q-Risk}_q = \frac{\sum_{n=1}^{N} \sum_{i=1}^{H} \mathcal{L}_q\left(\mathrm{y}_{n,i}, \hat{\mathrm{y}}_{n,i}^{(q)}\right)}{\sum_{n=1}^{N} \sum_{i=1}^{H} |\mathrm{y}_{n,i}|}, 
\end{aligned}
\end{equation}

\noindent where:
\begin{itemize}
    \item $\quad \mathcal{L}_q(y, \hat{y}) = \max\left( q (y - \hat{y}), (1-q) (\hat{y} - y) \right)$ is the quantile loss function,
    \item $q \in (0, 1)$ is the quantile level,
    \item $\mathrm{y}_{n,i}$ represents the ground truth value of the $n$-th variate at time step $i$,
    \item $\hat{\mathrm{y}}_{n,i}^{(q)}$ denotes the predicted value at the $q$-th quantile for the $n$-th series at time step $i$.
\end{itemize}

Table~\ref{tab_vlinear_as_prob} reports the Q-Risk values at various quantiles for vLinear and other probabilistic forecasters. As shown, vLinear exhibits state-of-the-art quantile forecasting performance, demonstrating its strong representation capabilities.

\begin{table*}[ht]
\begin{center}
{\fontsize{8}{9}\selectfont
\setlength{\tabcolsep}{8pt}
\begin{tabular}{@{}c|c|c|c|c|c|c|c|c|c@{}}
\toprule
Dataset                   
& Quantile & \begin{tabular}[c]{@{}c@{}}vLinear\\      (Ours)\end{tabular} 
& \begin{tabular}[c]{@{}c@{}}QuantileFormer\\      \shortcite{quantileformer} \end{tabular} 
& \begin{tabular}[c]{@{}c@{}}iTransformer\\      \shortcite{itransformer} \end{tabular} 
& \begin{tabular}[c]{@{}c@{}}PatchTST\\      \shortcite{patchtst} \end{tabular} 
& \begin{tabular}[c]{@{}c@{}}FEDformer\\      \shortcite{fedformer} \end{tabular} 
& \begin{tabular}[c]{@{}c@{}}Autoformer\\      \shortcite{autoformer} \end{tabular} 
& \begin{tabular}[c]{@{}c@{}}Transformer\\      \shortcite{transformer} \end{tabular} 
& \begin{tabular}[c]{@{}c@{}}DeepAR\\      \shortcite{rnn_2020} \end{tabular} \\ \midrule
                          & 0.5      & {\color[HTML]{0000FF} {\ul 0.4264}}                           & {\color[HTML]{FF0000} \textbf{0.1536}}                           & 0.7514                                                         & 1.4268                                                     & 0.6619                                                      & 1.8463                                                       & 1.0397                                                        & 1.2026                                                   \\
                          & 0.6      & 0.4199                                                        & {\color[HTML]{FF0000} \textbf{0.1642}}                           & {\color[HTML]{0000FF} {\ul 0.4112}}                            & 1.3088                                                     & 0.8673                                                      & 1.3424                                                       & 0.8740                                                        & 1.1749                                                   \\
                          & 0.7      & {\color[HTML]{0000FF} {\ul 0.3954}}                           & {\color[HTML]{FF0000} \textbf{0.2689}}                           & 0.8834                                                         & 1.0240                                                     & 0.4927                                                      & 1.1008                                                       & 0.7372                                                        & 0.7901                                                   \\
                          & 0.8      & {\color[HTML]{FF0000} \textbf{0.3448}}                        & {\color[HTML]{0000FF} {\ul 0.4340}}                              & 0.5824                                                         & 0.5100                                                     & 0.5491                                                      & 0.8392                                                       & 0.4998                                                        & 1.0616                                                   \\
\multirow{-5}{*}{ETTm1}   & 0.9      & {\color[HTML]{0000FF} {\ul 0.2483}}                           & {\color[HTML]{FF0000} \textbf{0.0596}}                           & 0.1228                                                         & 0.2816                                                     & 0.3865                                                      & 0.4774                                                       & 0.3618                                                        & 0.5388                                                   \\ \midrule
                          & 0.5      & {\color[HTML]{0000FF} {\ul 0.4765}}                           & {\color[HTML]{FF0000} \textbf{0.3007}}                           & 0.8850                                                         & 1.4719                                                     & 0.9480                                                      & 1.7221                                                       & 1.1989                                                        & 2.3414                                                   \\
                          & 0.6      & {\color[HTML]{FF0000} \textbf{0.4707}}                        & {\color[HTML]{0000FF} {\ul 0.6130}}                              & 0.9508                                                         & 1.4558                                                     & 0.8875                                                      & 1.2556                                                       & 0.8805                                                        & 0.7631                                                   \\
                          & 0.7      & {\color[HTML]{0000FF} {\ul 0.4423}}                           & {\color[HTML]{FF0000} \textbf{0.2912}}                           & 0.8607                                                         & 1.1307                                                     & 0.8328                                                      & 1.1977                                                       & 0.7284                                                        & 1.2217                                                   \\
                          & 0.8      & {\color[HTML]{FF0000} \textbf{0.3822}}                        & {\color[HTML]{0000FF} {\ul 0.4273}}                              & 0.4721                                                         & 0.4275                                                     & 0.7208                                                      & 0.9091                                                       & 0.4868                                                        & 1.0815                                                   \\
\multirow{-5}{*}{ETTh1}   & 0.9      & {\color[HTML]{FF0000} \textbf{0.2672}}                        & 0.3388                                                           & {\color[HTML]{0000FF} {\ul 0.3129}}                            & 0.3166                                                     & 0.4582                                                      & 0.4569                                                       & 0.5546                                                        & 1.9889                                                   \\ \midrule
                          & 0.5      & {\color[HTML]{FF0000} \textbf{0.2525}}                        & 1.0641                                                           & 1.0705                                                         & 1.0806                                                     & 1.0363                                                      & 1.1641                                                       & 1.0391                                                        & {\color[HTML]{0000FF} {\ul 0.8666}}                      \\
                          & 0.6      & {\color[HTML]{FF0000} \textbf{0.2516}}                        & {\color[HTML]{0000FF} {\ul 1.0480}}                              & 1.1843                                                         & 1.1242                                                     & 1.1708                                                      & 1.2367                                                       & 1.1617                                                        & 1.1173                                                   \\
                          & 0.7      & {\color[HTML]{FF0000} \textbf{0.2490}}                        & 1.1832                                                           & 1.1845                                                         & 1.2547                                                     & {\color[HTML]{0000FF} {\ul 1.0261}}                         & 1.2088                                                       & 1.1381                                                        & 1.2854                                                   \\
                          & 0.8      & {\color[HTML]{FF0000} \textbf{0.2427}}                        & {\color[HTML]{0000FF} {\ul 1.0008}}                              & 1.3705                                                         & 1.1935                                                     & 1.5427                                                      & 1.0030                                                       & 1.0794                                                        & 1.4512                                                   \\
\multirow{-5}{*}{Solar}   & 0.9      & {\color[HTML]{FF0000} \textbf{0.2214}}                        & {\color[HTML]{0000FF} {\ul 0.5883}}                              & 1.6083                                                         & 0.5950                                                     & 0.6414                                                      & 0.6167                                                       & 1.0777                                                        & 1.6117                                                   \\ \midrule
                          & 0.5      & {\color[HTML]{FF0000} \textbf{0.2922}}                        & {\color[HTML]{0000FF} {\ul 0.8489}}                              & 1.8998                                                         & 0.9775                                                     & 2.4497                                                      & 0.9908                                                       & 0.9664                                                        & 1.0502                                                   \\
                          & 0.6      & {\color[HTML]{FF0000} \textbf{0.2974}}                        & {\color[HTML]{0000FF} {\ul 0.8291}}                              & 1.3545                                                         & 1.6937                                                     & 0.9188                                                      & 1.1109                                                       & 0.9325                                                        & 0.8813                                                   \\
                          & 0.7      & {\color[HTML]{FF0000} \textbf{0.2944}}                        & {\color[HTML]{0000FF} {\ul 0.8489}}                              & 1.1941                                                         & 1.1269                                                     & 2.3784                                                      & 0.8686                                                       & 1.0574                                                        & 1.2484                                                   \\
                          & 0.8      & {\color[HTML]{FF0000} \textbf{0.2817}}                        & 0.5998                                                           & 0.8247                                                         & {\color[HTML]{0000FF} {\ul 0.5962}}                        & 1.7356                                                      & 0.6064                                                       & 0.8679                                                        & 0.9394                                                   \\
\multirow{-5}{*}{Traffic} & 0.9      & {\color[HTML]{FF0000} \textbf{0.2542}}                        & {\color[HTML]{0000FF} {\ul 0.4688}}                              & 1.5621                                                         & 1.1450                                                     & 0.8770                                                      & 0.4970                                                       & 0.9247                                                        & 1.1539                                                   \\ \midrule
                          & 0.5      & {\color[HTML]{FF0000} \textbf{0.2663}}                        & {\color[HTML]{0000FF} {\ul 0.7469}}                              & 1.3430                                                         & 1.8354                                                     & 1.9429                                                      & 1.0584                                                       & 1.3703                                                        & 1.0002                                                   \\
                          & 0.6      & {\color[HTML]{FF0000} \textbf{0.2663}}                        & {\color[HTML]{0000FF} {\ul 0.8136}}                              & 1.0348                                                         & 1.3134                                                     & 1.0447                                                      & 0.9191                                                       & 0.8873                                                        & 1.1177                                                   \\
                          & 0.7      & {\color[HTML]{FF0000} \textbf{0.2593}}                        & {\color[HTML]{0000FF} {\ul 0.3330}}                              & 1.2174                                                         & 1.0657                                                     & 0.9669                                                      & 1.0301                                                       & 1.0098                                                        & 1.9544                                                   \\
                          & 0.8      & {\color[HTML]{FF0000} \textbf{0.2442}}                        & {\color[HTML]{0000FF} {\ul 0.4340}}                              & 0.9072                                                         & 0.8800                                                     & 3.0007                                                      & 0.8786                                                       & 0.9005                                                        & 1.2077                                                   \\
\multirow{-5}{*}{ECL}     & 0.9      & {\color[HTML]{FF0000} \textbf{0.2169}}                        & {\color[HTML]{0000FF} {\ul 0.5121}}                              & 1.2742                                                         & 0.7567                                                     & 1.0618                                                      & 0.6420                                                       & 0.9439                                                        & 1.0830                                                   \\ \bottomrule
\end{tabular}
}
\caption{Probabilistic forecasting performance of vLinear. Q-Risk at various quantiles is reported. {\color[HTML]{FF0000} \textbf{Red}}: Best; {\color[HTML]{0000FF} {\ul Blue}}: Second-Best. Note that iTransformer, PatchTST, FEDformer, Autoformer, and Transformer were adapted for quantile prediction by training with the quantile loss \protect \cite{quantileformer}. The results of the other models are taken from QuantileFormer. The `Input-96-Predict-96' setting is applied.}
\label{tab_vlinear_as_prob}
\end{center}
\end{table*}

\begin{figure*}[t]
   \centering
   \includegraphics[width=1.0\linewidth]{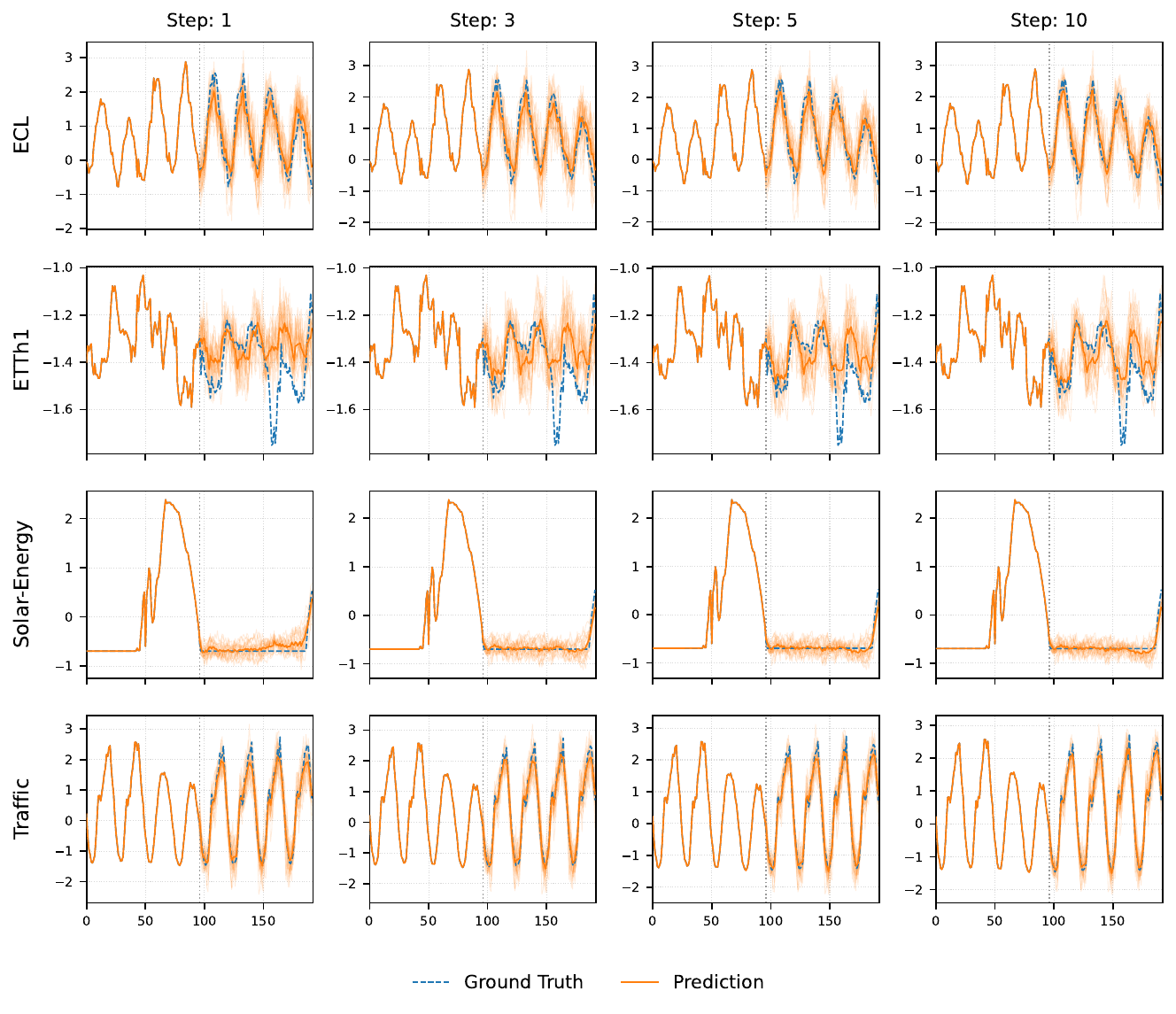}
   \caption{Visualization of probabilistic forecasting results (Input: 96, Predict: 96). The faint lines depict 20 individual prediction samples, while the solid line represents their average. Generally, prediction accuracy improves as the inference steps increase.}
   \label{fig_probts_visual}
\end{figure*}

Figure~\ref{fig_probts_visual} visualizes the probabilistic forecasting results. vLinear exhibits robust prediction accuracy, with the predictions becoming increasingly refined as the number of inference steps increases. This observation aligns well with the quantitative results reported in Figure~\ref{fig_hyperpara_steps}.

\subsection{Importance of Noise in Training}
\label{appd_noise}

\begin{table*}[ht]
\begin{center}
\setlength{\tabcolsep}{15pt}
\renewcommand{\arraystretch}{1.0} 
\renewcommand{\topfraction}{0.9}
{\fontsize{8}{11}\selectfont
\begin{tabular}{@{}cc|cc|cc@{}}
\toprule
\multicolumn{2}{c|}{Model}      & \multicolumn{2}{c|}{vLinear}                                                  & \multicolumn{2}{c}{vLinear$\dagger$}                                                  \\ \midrule
\multicolumn{2}{c|}{Metric}     & MSE                                   & MAE                                   & MSE                                   & MAE                                   \\ \midrule
                          & 96  & {\color[HTML]{FF0000} \textbf{0.356}} & {\color[HTML]{FF0000} \textbf{0.379}} & 0.357                                 & 0.380                                 \\
                          & 192 & {\color[HTML]{FF0000} \textbf{0.409}} & {\color[HTML]{FF0000} \textbf{0.411}} & {\color[HTML]{FF0000} \textbf{0.409}} & {\color[HTML]{FF0000} \textbf{0.411}} \\
                          & 336 & {\color[HTML]{FF0000} \textbf{0.449}} & {\color[HTML]{FF0000} \textbf{0.433}} & {\color[HTML]{FF0000} \textbf{0.449}} & 0.434                                 \\
                          & 720 & {\color[HTML]{FF0000} \textbf{0.451}} & {\color[HTML]{FF0000} \textbf{0.456}} & 0.457                                 & 0.460                                 \\ \cmidrule(l){2-6} 
\multirow{-5}{*}{ETTh1}   & Avg & {\color[HTML]{FF0000} \textbf{0.416}} & {\color[HTML]{FF0000} \textbf{0.420}} & 0.418                                 & 0.421                                 \\ \midrule
                          & 96  & {\color[HTML]{FF0000} \textbf{0.168}} & {\color[HTML]{FF0000} \textbf{0.245}} & 0.169                                 & 0.246                                 \\
                          & 192 & {\color[HTML]{FF0000} \textbf{0.231}} & {\color[HTML]{FF0000} \textbf{0.287}} & 0.232                                 & {\color[HTML]{FF0000} \textbf{0.287}} \\
                          & 336 & {\color[HTML]{FF0000} \textbf{0.289}} & {\color[HTML]{FF0000} \textbf{0.326}} & 0.292                                 & 0.328                                 \\
                          & 720 & {\color[HTML]{FF0000} \textbf{0.386}} & {\color[HTML]{FF0000} \textbf{0.382}} & 0.388                                 & 0.384                                 \\ \cmidrule(l){2-6} 
\multirow{-5}{*}{ETTm2}   & Avg & {\color[HTML]{FF0000} \textbf{0.268}} & {\color[HTML]{FF0000} \textbf{0.310}} & 0.270                                 & 0.311                                 \\ \midrule
                          & 96  & {\color[HTML]{FF0000} \textbf{0.129}} & {\color[HTML]{FF0000} \textbf{0.221}} & 0.134                                 & 0.226                                 \\
                          & 192 & {\color[HTML]{FF0000} \textbf{0.147}} & {\color[HTML]{FF0000} \textbf{0.238}} & 0.152                                 & 0.243                                 \\
                          & 336 & {\color[HTML]{FF0000} \textbf{0.158}} & {\color[HTML]{FF0000} \textbf{0.250}} & 0.159                                 & 0.253                                 \\
                          & 720 & {\color[HTML]{FF0000} \textbf{0.178}} & {\color[HTML]{FF0000} \textbf{0.271}} & 0.192                                 & 0.282                                 \\ \cmidrule(l){2-6} 
\multirow{-5}{*}{ECL}     & Avg & {\color[HTML]{FF0000} \textbf{0.153}} & {\color[HTML]{FF0000} \textbf{0.245}} & 0.159                                 & 0.251                                 \\ \midrule
                          & 96  & {\color[HTML]{FF0000} \textbf{0.396}} & {\color[HTML]{FF0000} \textbf{0.233}} & 0.403                                 & 0.238                                 \\
                          & 192 & {\color[HTML]{FF0000} \textbf{0.427}} & {\color[HTML]{FF0000} \textbf{0.243}} & 0.444                                 & 0.253                                 \\
                          & 336 & {\color[HTML]{FF0000} \textbf{0.449}} & {\color[HTML]{FF0000} \textbf{0.255}} & 0.463                                 & 0.264                                 \\
                          & 720 & {\color[HTML]{FF0000} \textbf{0.487}} & {\color[HTML]{FF0000} \textbf{0.276}} & 0.491                                 & 0.283                                 \\ \cmidrule(l){2-6} 
\multirow{-5}{*}{Traffic} & Avg & {\color[HTML]{FF0000} \textbf{0.440}} & {\color[HTML]{FF0000} \textbf{0.252}} & 0.450                                 & 0.259                                 \\ \midrule
                          & 96  & {\color[HTML]{FF0000} \textbf{0.150}} & {\color[HTML]{FF0000} \textbf{0.186}} & 0.152                                 & 0.189                                 \\
                          & 192 & {\color[HTML]{FF0000} \textbf{0.198}} & {\color[HTML]{FF0000} \textbf{0.234}} & 0.203                                 & 0.238                                 \\
                          & 336 & {\color[HTML]{FF0000} \textbf{0.252}} & {\color[HTML]{FF0000} \textbf{0.275}} & 0.254                                 & 0.276                                 \\
                          & 720 & {\color[HTML]{FF0000} \textbf{0.332}} & {\color[HTML]{FF0000} \textbf{0.328}} & 0.378                                 & 0.368                                 \\ \cmidrule(l){2-6} 
\multirow{-5}{*}{Weather} & Avg & {\color[HTML]{FF0000} \textbf{0.233}} & {\color[HTML]{FF0000} \textbf{0.256}} & 0.247                                 & 0.268                                 \\ \midrule
                          & 12  & {\color[HTML]{FF0000} \textbf{0.059}} & {\color[HTML]{FF0000} \textbf{0.158}} & 0.060                                 & 0.159                                 \\
                          & 24  & {\color[HTML]{FF0000} \textbf{0.075}} & {\color[HTML]{FF0000} \textbf{0.177}} & 0.076                                 & 0.180                                 \\
                          & 48  & {\color[HTML]{FF0000} \textbf{0.102}} & {\color[HTML]{FF0000} \textbf{0.208}} & 0.105                                 & 0.212                                 \\
                          & 96  & {\color[HTML]{FF0000} \textbf{0.138}} & {\color[HTML]{FF0000} \textbf{0.246}} & 0.142                                 & 0.252                                 \\ \cmidrule(l){2-6} 
\multirow{-5}{*}{PEMS03}  & Avg & {\color[HTML]{FF0000} \textbf{0.093}} & {\color[HTML]{FF0000} \textbf{0.197}} & 0.096                                 & 0.200                                 \\ \bottomrule
\end{tabular}
}
\caption{Performance comparison with and without noise in $\hat{\mathbf{Y}} ^{(0)}$ during training. vLinear$\dagger$ denotes the variant where $\hat{\mathbf{Y}} ^{(0)}$ is set to an all-zero matrix in the training phase. The better results are highlighted in {\color[HTML]{FF0000} \textbf{bold}}.}
\label{tab_important_noise}
\end{center}
\end{table*}

In vLinear, the initial state $\hat{\mathbf{Y}} ^{(0)}$ is initialized with Gaussian noise during training. To evaluate the impact of this noise, we conduct an ablation study by setting $\hat{\mathbf{Y}} ^{(0)}$ to an all-zero matrix.
As shown in Table~\ref{tab_important_noise}, training without noise leads to performance degradation. In contrast, incorporating noise into $\hat{\mathbf{Y}} ^{(0)}$ yields average MSE improvements of 3.9\% and 5.6\% on the ECL and Weather datasets, respectively. This observation suggests that the noise plays an important role in enhancing the model's robustness. We hypothesize that relying solely on the trajectory from the origin (all-zero state) to the ground truth limits vLinear's capability to estimate accurate velocity for \textbf{deviated or `outlier' intermediate states} encountered in the unseen test set.

\subsection{Additional Evaluation Metrics}
\label{appd_more_metrics}

We further employ the following scale-free metrics for comprehensive evaluation: the Coefficient of Determination ($R^2$), the Pearson Correlation Coefficient ($r$), and the Mean Absolute Scaled Error (MASE), which are defined as:

\begin{equation}
\begin{aligned}
  &R^2=\frac{1}{N}\sum_{n=1}^{N}\left ( 1-\frac{\left \| \hat{\mathrm{Y}} _{n:}-\mathrm{Y}_{n:} \right \| ^2_2 }{\left \| \mathrm{Y} _{n:}-\overline{\mathrm{Y}_{n:}}  \right \| ^2_2}  \right )  , \\
  &r =  \frac{1}{N}\sum_{n=1}^{N}\frac{ {\textstyle \sum_{i=1}^{H}}  \left ( \hat{\mathrm{y}} _{n,i}-\overline{\hat{\mathrm{Y}} _{n:}}  \right )\left ( \mathrm{y} _{n,i}-\overline{\mathrm{Y} _{n:}}  \right )  }
{  \sqrt{ {\textstyle \sum_{i=1}^{H}} \left ( \hat{\mathrm{y}} _{n,i}-\overline{\hat{\mathrm{Y}} _{n:}}  \right )^2} \sqrt{ {\textstyle \sum_{i=1}^{H}} \left ( \mathrm{y} _{n,i}-\overline{\mathrm{Y} _{n:}}  \right )^2}     } ,   \\
  &\mathrm{MASE} = \frac{1}{N}\sum_{n=1}^{N}\frac{\frac{1}{H} \left \| \hat{\mathrm{Y}} _{n:}-\mathrm{Y}_{n:} \right \|_1  }{\frac{1}{H-1}  {\textstyle \sum_{i=2}^{H} \left | \mathrm{y} _{n,i}-\mathrm{y}_{n,i-1} \right | } },
\end{aligned}
\label{more_metrics}
\end{equation}

\noindent where $\mathrm{Y} _{n:}$ and $\hat{\mathrm{Y}} _{n:}$ denote the $n$-th variate of the ground truth $\mathbf{Y}$ and the prediction $\hat{\mathbf{Y}}$, respectively. $\mathrm{y} _{n,i}$ denotes the $i$-th time step of the $n$-th variate of $\mathbf{Y}$, and $\overline{(\cdot)}$ denotes the mean.

Table~\ref{tab_more_metrics} shows that vLinear consistently outperforms state-of-the-art forecasters in terms of these additional metrics, demonstrating the comprehensive superiority and robustness of our method.

\begin{table*}[t]
\begin{center}
\renewcommand{\arraystretch}{1.2}
{\fontsize{8}{9}\selectfont
\setlength{\tabcolsep}{3pt}
\begin{tabular}{@{}cc|ccc|ccc|ccc|ccc|ccc@{}}
\toprule
\multicolumn{2}{c|}{Model}      & \multicolumn{3}{c|}{\begin{tabular}[c]{@{}c@{}}vLinear\\      (Ours)\end{tabular}}                                    
& \multicolumn{3}{c|}{\begin{tabular}[c]{@{}c@{}}OLinear\\   \shortcite{olinear} \end{tabular}}                         
& \multicolumn{3}{c|}{\begin{tabular}[c]{@{}c@{}}Leddam\\      \shortcite{Leddam_icml} \end{tabular}}  
& \multicolumn{3}{c|}{\begin{tabular}[c]{@{}c@{}}CARD\\     \shortcite{card} \end{tabular}}     
& \multicolumn{3}{c}{\begin{tabular}[c]{@{}c@{}}iTrans.\\      \shortcite{itransformer} \end{tabular}}  \\ \midrule
\multicolumn{2}{c|}{Metric}     & $R^2$ ($\uparrow$)                                   & $r$ ($\uparrow$)                                     & MASE ($\downarrow$)                                 & $R^2$ ($\uparrow$)                                   & $r$ ($\uparrow$)                                     & MASE ($\downarrow$)                                 & $R^2$ ($\uparrow$)             & $r$ ($\uparrow$)                                           & MASE ($\downarrow$)          & $R^2$ ($\uparrow$)              & $r$ ($\uparrow$)              & MASE ($\downarrow$)                                       & $R^2$ ($\uparrow$)                                          & $r$ ($\uparrow$)              & MASE ($\downarrow$)          \\ \midrule
                          & 96  & {\color[HTML]{FF0000} \textbf{0.652}} & {\color[HTML]{FF0000} \textbf{0.916}} & {\color[HTML]{0000FF} {\ul 0.903}}    & {\color[HTML]{0000FF} {\ul 0.640}}    & {\color[HTML]{FF0000} \textbf{0.916}} & {\color[HTML]{FF0000} \textbf{0.902}} & 0.527           & {\color[HTML]{0000FF} {\ul 0.911}}          & 0.964          & 0.533            & 0.905          & 0.975                                       & 0.550                                        & 0.907          & 0.990          \\
                          & 192 & {\color[HTML]{FF0000} \textbf{0.663}} & {\color[HTML]{FF0000} \textbf{0.908}} & {\color[HTML]{0000FF} {\ul 0.984}}    & {\color[HTML]{0000FF} {\ul 0.638}}    & {\color[HTML]{FF0000} \textbf{0.908}} & {\color[HTML]{FF0000} \textbf{0.983}} & 0.593           & {\color[HTML]{0000FF} {\ul 0.902}}          & 1.041          & 0.547            & 0.899          & 1.028                                       & 0.590                                        & 0.900          & 1.057          \\
                          & 336 & {\color[HTML]{FF0000} \textbf{0.717}} & {\color[HTML]{FF0000} \textbf{0.902}} & {\color[HTML]{FF0000} \textbf{1.051}} & {\color[HTML]{0000FF} {\ul 0.710}}    & {\color[HTML]{0000FF} {\ul 0.901}}    & {\color[HTML]{0000FF} {\ul 1.061}}    & 0.670           & 0.893                                       & 1.134          & 0.678            & 0.891          & 1.113                                       & 0.678                                        & 0.893          & 1.133          \\
                          & 720 & {\color[HTML]{FF0000} \textbf{0.709}} & {\color[HTML]{FF0000} \textbf{0.892}} & {\color[HTML]{FF0000} \textbf{1.163}} & {\color[HTML]{0000FF} {\ul 0.701}}    & {\color[HTML]{0000FF} {\ul 0.889}}    & {\color[HTML]{0000FF} {\ul 1.174}}    & 0.661           & 0.882                                       & 1.251          & 0.653            & 0.877          & 1.273                                       & 0.638                                        & 0.872          & 1.315          \\ \cmidrule(l){2-17} 
\multirow{-5}{*}{ECL}     & Avg & {\color[HTML]{FF0000} \textbf{0.685}} & {\color[HTML]{FF0000} \textbf{0.904}} & {\color[HTML]{FF0000} \textbf{1.025}} & {\color[HTML]{0000FF} {\ul 0.672}}    & {\color[HTML]{0000FF} {\ul 0.903}}    & {\color[HTML]{0000FF} {\ul 1.030}}    & 0.613           & 0.897                                       & 1.098          & 0.603            & 0.893          & 1.097                                       & 0.614                                        & 0.893          & 1.124          \\ \midrule
                          & 96  & {\color[HTML]{FF0000} \textbf{0.738}} & {\color[HTML]{FF0000} \textbf{0.900}} & {\color[HTML]{0000FF} {\ul 0.730}}    & {\color[HTML]{0000FF} {\ul 0.737}}    & {\color[HTML]{0000FF} {\ul 0.899}}    & {\color[HTML]{FF0000} \textbf{0.729}} & 0.681           & 0.885                                       & 0.909          & 0.690            & 0.882          & 0.855                                       & 0.690                                        & 0.889          & 0.860          \\
                          & 192 & {\color[HTML]{FF0000} \textbf{0.719}} & {\color[HTML]{FF0000} \textbf{0.885}} & {\color[HTML]{FF0000} \textbf{0.756}} & {\color[HTML]{FF0000} \textbf{0.719}} & {\color[HTML]{0000FF} {\ul 0.884}}    & {\color[HTML]{FF0000} \textbf{0.756}} & 0.666           & 0.865                                       & 0.949          & 0.683            & 0.869          & {\color[HTML]{0000FF} {\ul 0.856}}          & {\color[HTML]{0000FF} {\ul 0.689}}           & 0.877          & 0.879          \\
                          & 336 & {\color[HTML]{FF0000} \textbf{0.731}} & {\color[HTML]{FF0000} \textbf{0.876}} & {\color[HTML]{0000FF} {\ul 0.762}}    & {\color[HTML]{0000FF} {\ul 0.730}}    & {\color[HTML]{0000FF} {\ul 0.875}}    & {\color[HTML]{FF0000} \textbf{0.758}} & 0.707           & 0.867                                       & 0.897          & 0.702            & 0.862          & 0.843                                       & 0.713                                        & 0.867          & 0.883          \\
                          & 720 & {\color[HTML]{FF0000} \textbf{0.705}} & {\color[HTML]{FF0000} \textbf{0.861}} & {\color[HTML]{0000FF} {\ul 0.819}}    & {\color[HTML]{0000FF} {\ul 0.703}}    & {\color[HTML]{0000FF} {\ul 0.860}}    & {\color[HTML]{FF0000} \textbf{0.815}} & 0.687           & 0.853                                       & 0.947          & 0.679            & 0.848          & 0.894                                       & 0.693                                        & 0.857          & 0.913          \\ \cmidrule(l){2-17} 
\multirow{-5}{*}{Traffic} & Avg & {\color[HTML]{FF0000} \textbf{0.723}} & {\color[HTML]{FF0000} \textbf{0.880}} & {\color[HTML]{0000FF} {\ul 0.767}}    & {\color[HTML]{0000FF} {\ul 0.722}}    & {\color[HTML]{0000FF} {\ul 0.879}}    & {\color[HTML]{FF0000} \textbf{0.764}} & 0.685           & 0.868                                       & 0.926          & 0.689            & 0.865          & 0.862                                       & 0.696                                        & 0.873          & 0.884          \\ \midrule
\multicolumn{2}{c|}{Model}      & \multicolumn{3}{c|}{\begin{tabular}[c]{@{}c@{}}TimeMixer\\      \shortcite{timemixer} \end{tabular}}                                      
& \multicolumn{3}{c|}{\begin{tabular}[c]{@{}c@{}}FilterNet\\    \shortcite{filternet} \end{tabular}}                                      
& \multicolumn{3}{c|}{\begin{tabular}[c]{@{}c@{}}DLinear\\      \shortcite{linear} \end{tabular}} 
& \multicolumn{3}{c|}{\begin{tabular}[c]{@{}c@{}}PatchTST\\      \shortcite{patchtst} \end{tabular}} 
& \multicolumn{3}{c}{\begin{tabular}[c]{@{}c@{}}TimesNet\\      \shortcite{timesnet} \end{tabular}} \\ \midrule
\multicolumn{2}{c|}{Metric}     & $R^2$ ($\uparrow$)                                   & $r$ ($\uparrow$)                                     & MASE ($\downarrow$)                                 & $R^2$ ($\uparrow$)                                   & $r$ ($\uparrow$)                                     & MASE ($\downarrow$)                                 & $R^2$ ($\uparrow$)             & $r$ ($\uparrow$)                                           & MASE ($\downarrow$)          & $R^2$ ($\uparrow$)              & $r$ ($\uparrow$)              & MASE ($\downarrow$)                                       & $R^2$ ($\uparrow$)                                          & $r$ ($\uparrow$)              & MASE ($\downarrow$)          \\ \midrule
                          & 96  & 0.456                                 & 0.897                                 & 1.022                                 & 0.537                                 & 0.905                                 & 0.996                                 & 0.120           & 0.870                                       & 1.300          & 0.426            & 0.890          & 1.119                                       & 0.405                                        & 0.892          & 1.118          \\
                          & 192 & 0.575                                 & 0.893                                 & 1.076                                 & 0.576                                 & 0.898                                 & 1.052                                 & 0.215           & 0.867                                       & 1.312          & 0.509            & 0.887          & 1.140                                       & 0.503                                        & 0.885          & 1.195          \\
                          & 336 & 0.640                                 & 0.886                                 & 1.155                                 & 0.670                                 & 0.891                                 & 1.139                                 & 0.247           & 0.858                                       & 1.387          & 0.604            & 0.880          & 1.216                                       & 0.592                                        & 0.875          & 1.262          \\
                          & 720 & 0.634                                 & 0.871                                 & 1.314                                 & 0.628                                 & 0.875                                 & 1.328                                 & 0.148           & 0.844                                       & 1.536          & 0.605            & 0.864          & 1.371                                       & 0.603                                        & 0.864          & 1.388          \\ \cmidrule(l){2-17} 
\multirow{-5}{*}{ECL}     & Avg & 0.576                                 & 0.887                                 & 1.142                                 & 0.603                                 & 0.892                                 & 1.129                                 & 0.182           & 0.860                                       & 1.384          & 0.536            & 0.880          & 1.211                                       & 0.526                                        & 0.879          & 1.241          \\ \midrule
                          & 96  & 0.640                                 & 0.888                                 & 0.979                                 & 0.637                                 & 0.876                                 & 0.987                                 & 0.441           & 0.811                                       & 1.346          & 0.668            & 0.880          & 0.903                                       & 0.625                                        & 0.877          & 1.023          \\
                          & 192 & 0.660                                 & 0.862                                 & 0.959                                 & 0.665                                 & 0.867                                 & 0.950                                 & 0.551           & 0.815                                       & 1.189          & 0.677            & 0.871          & 0.887                                       & 0.635                                        & 0.867          & 1.009          \\
                          & 336 & 0.685                                 & 0.852                                 & 0.960                                 & 0.696                                 & 0.861                                 & 0.930                                 & 0.582           & 0.810                                       & 1.167          & 0.700            & 0.863          & 0.874                                       & 0.654                                        & 0.852          & 1.079          \\
                          & 720 & 0.672                                 & 0.842                                 & 0.986                                 & 0.678                                 & 0.847                                 & 0.982                                 & 0.564           & 0.794                                       & 1.224          & 0.681            & 0.850          & 0.925                                       & 0.661                                        & 0.850          & 1.047          \\ \cmidrule(l){2-17} 
\multirow{-5}{*}{Traffic} & Avg & 0.664                                 & 0.861                                 & 0.971                                 & 0.669                                 & 0.862                                 & 0.962                                 & 0.535           & 0.808                                       & 1.232          & 0.682            & 0.866          & 0.897                                       & 0.644                                        & 0.862          & 1.040          \\ \bottomrule
\end{tabular}
}
\caption{Performance on the scale-free metrics: Coefficient of Determination ($R^2$), Pearson Correlation Coefficient ($r$), and MASE. The symbols $\downarrow$ and $\uparrow$ indicate that lower and higher values are better, respectively. The best and second-best results are highlighted in {\color[HTML]{FF0000} \textbf{bold}} and {\color[HTML]{0000FF} {\ul underlined}}, respectively. The lookback length $T$ is uniformly set as 96.}
\label{tab_more_metrics}
\end{center}
\end{table*}

\subsection{Pre- and Post-Linear Layers in vecTrans Module} \label{appd_pre_post_lin}

In the vecTrans module, linear layers are applied both before and after the core vecTrans operation, formulated as: $\mathrm{Linear} \left ( \mathtt{vecTrans} \left ( \mathrm{Linear} \left ( \cdot \right )  \right )  \right )$.

We conducted an ablation study to evaluate the contributions of these two layers. As shown in Table~\ref{tab:abl_pre_post_lin}, incorporating both pre- and post-linear layers yields an average performance gain of 7.6\% compared to the variant without them. This highlights their effectiveness in refining features for multivariate correlation modeling and facilitating downstream temporal dynamic learning.

\begin{table*}[ht]
\begin{center}
{\fontsize{8}{9}\selectfont
\renewcommand{\arraystretch}{1.3}
\setlength{\tabcolsep}{7pt}
\begin{tabular}{@{}cc|c|cc|cc|cc|cc|cc|cc@{}}
\toprule
                         &                           &                        & \multicolumn{2}{c|}{ETTh1}                                                     & \multicolumn{2}{c|}{ETTm2}                                                     & \multicolumn{2}{c|}{ECL}                                                       & \multicolumn{2}{c|}{Solar}                                                     & \multicolumn{2}{c|}{Weather}                                                   & \multicolumn{2}{c}{PEMS03}                                                    \\ \cmidrule(l){4-15} 
\multirow{-2}{*}{PreLin} & \multirow{-2}{*}{PostLin} & \multirow{-2}{*}{Hor.} & MSE                                   & MAE                                   & MSE                                   & MAE                                   & MSE                                   & MAE                                   & MSE                                   & MAE                                   & MSE                                   & MAE                                   & MSE                                   & MAE                                   \\ \midrule
                         &                           & H1                     & 0.385                                 & 0.404                                 & {\color[HTML]{0000FF} {\ul 0.177}}    & 0.257                                 & 0.141                                 & 0.236                                 & 0.203                                 & 0.222                                 & 0.160                                 & 0.196                                 & 0.067                                 & 0.170                                 \\
                         &                           & H2                     & 0.440                                 & 0.437                                 & 0.246                                 & {\color[HTML]{0000FF} {\ul 0.303}}    & 0.155                                 & 0.248                                 & 0.226                                 & 0.241                                 & 0.209                                 & 0.244                                 & 0.092                                 & 0.198                                 \\
                         &                           & H3                     & 0.473                                 & 0.453                                 & 0.298                                 & 0.335                                 & 0.167                                 & 0.262                                 & 0.249                                 & 0.259                                 & 0.261                                 & 0.282                                 & 0.140                                 & 0.243                                 \\
                         &                           & H4                     & 0.461                                 & 0.465                                 & 0.401                                 & 0.394                                 & 0.194                                 & 0.286                                 & 0.257                                 & 0.264                                 & 0.344                                 & 0.338                                 & 0.185                                 & 0.285                                 \\ \cmidrule(l){3-15} 
\multirow{-5}{*}{\ding{55}}    & \multirow{-5}{*}{\ding{55}}     & Avg                    & 0.440                                 & 0.440                                 & 0.280                                 & 0.322                                 & 0.164                                 & 0.258                                 & 0.234                                 & 0.246                                 & 0.243                                 & 0.265                                 & 0.121                                 & 0.224                                 \\ \midrule
                         &                           & H1                     & {\color[HTML]{0000FF} {\ul 0.358}}    & {\color[HTML]{0000FF} {\ul 0.381}}    & {\color[HTML]{FF0000} \textbf{0.168}} & {\color[HTML]{0000FF} {\ul 0.246}}    & {\color[HTML]{FF0000} \textbf{0.129}} & {\color[HTML]{FF0000} \textbf{0.220}} & {\color[HTML]{0000FF} {\ul 0.180}}    & {\color[HTML]{0000FF} {\ul 0.201}}    & {\color[HTML]{0000FF} {\ul 0.151}}    & {\color[HTML]{0000FF} {\ul 0.187}}    & {\color[HTML]{0000FF} {\ul 0.060}}    & {\color[HTML]{0000FF} {\ul 0.160}}    \\
                         &                           & H2                     & {\color[HTML]{0000FF} {\ul 0.410}}    & {\color[HTML]{0000FF} {\ul 0.412}}    & {\color[HTML]{0000FF} {\ul 0.233}}    & {\color[HTML]{FF0000} \textbf{0.287}} & {\color[HTML]{0000FF} {\ul 0.148}}    & {\color[HTML]{0000FF} {\ul 0.240}}    & {\color[HTML]{0000FF} {\ul 0.207}}    & {\color[HTML]{0000FF} {\ul 0.224}}    & {\color[HTML]{0000FF} {\ul 0.199}}    & {\color[HTML]{0000FF} {\ul 0.235}}    & {\color[HTML]{0000FF} {\ul 0.078}}    & {\color[HTML]{0000FF} {\ul 0.179}}    \\
                         &                           & H3                     & {\color[HTML]{0000FF} {\ul 0.452}}    & {\color[HTML]{0000FF} {\ul 0.434}}    & {\color[HTML]{0000FF} {\ul 0.293}}    & {\color[HTML]{0000FF} {\ul 0.328}}    & {\color[HTML]{0000FF} {\ul 0.159}}    & {\color[HTML]{0000FF} {\ul 0.252}}    & {\color[HTML]{0000FF} {\ul 0.228}}    & {\color[HTML]{0000FF} {\ul 0.243}}    & {\color[HTML]{0000FF} {\ul 0.254}}    & {\color[HTML]{0000FF} {\ul 0.276}}    & {\color[HTML]{0000FF} {\ul 0.105}}    & {\color[HTML]{0000FF} {\ul 0.210}}    \\
                         &                           & H4                     & {\color[HTML]{0000FF} {\ul 0.459}}    & {\color[HTML]{0000FF} {\ul 0.461}}    & {\color[HTML]{0000FF} {\ul 0.390}}    & {\color[HTML]{0000FF} {\ul 0.385}}    & {\color[HTML]{0000FF} {\ul 0.183}}    & {\color[HTML]{0000FF} {\ul 0.276}}    & {\color[HTML]{0000FF} {\ul 0.237}}    & {\color[HTML]{0000FF} {\ul 0.251}}    & 0.336                                 & 0.332                                 & 0.146                                 & 0.254                                 \\ \cmidrule(l){3-15} 
\multirow{-5}{*}{\ding{55}}    & \multirow{-5}{*}{\ding{52}}    & Avg                    & {\color[HTML]{0000FF} {\ul 0.420}}    & {\color[HTML]{0000FF} {\ul 0.422}}    & {\color[HTML]{0000FF} {\ul 0.271}}    & {\color[HTML]{0000FF} {\ul 0.311}}    & {\color[HTML]{0000FF} {\ul 0.155}}    & {\color[HTML]{0000FF} {\ul 0.247}}    & {\color[HTML]{0000FF} {\ul 0.213}}    & {\color[HTML]{0000FF} {\ul 0.230}}    & {\color[HTML]{0000FF} {\ul 0.235}}    & {\color[HTML]{0000FF} {\ul 0.257}}    & {\color[HTML]{0000FF} {\ul 0.097}}    & 0.201                                 \\ \midrule
                         &                           & H1                     & {\color[HTML]{0000FF} {\ul 0.358}}    & {\color[HTML]{0000FF} {\ul 0.381}}    & {\color[HTML]{FF0000} \textbf{0.168}} & {\color[HTML]{0000FF} {\ul 0.246}}    & {\color[HTML]{FF0000} \textbf{0.129}} & {\color[HTML]{FF0000} \textbf{0.220}} & {\color[HTML]{0000FF} {\ul 0.180}}    & {\color[HTML]{0000FF} {\ul 0.201}}    & {\color[HTML]{0000FF} {\ul 0.151}}    & {\color[HTML]{0000FF} {\ul 0.187}}    & {\color[HTML]{0000FF} {\ul 0.060}}    & {\color[HTML]{0000FF} {\ul 0.160}}    \\
                         &                           & H2                     & {\color[HTML]{0000FF} {\ul 0.410}}    & {\color[HTML]{0000FF} {\ul 0.412}}    & {\color[HTML]{0000FF} {\ul 0.233}}    & {\color[HTML]{FF0000} \textbf{0.287}} & 0.149                                 & {\color[HTML]{0000FF} {\ul 0.240}}    & {\color[HTML]{0000FF} {\ul 0.207}}    & {\color[HTML]{0000FF} {\ul 0.224}}    & {\color[HTML]{0000FF} {\ul 0.199}}    & {\color[HTML]{0000FF} {\ul 0.235}}    & {\color[HTML]{0000FF} {\ul 0.078}}    & {\color[HTML]{0000FF} {\ul 0.179}}    \\
                         &                           & H3                     & {\color[HTML]{0000FF} {\ul 0.452}}    & {\color[HTML]{0000FF} {\ul 0.434}}    & {\color[HTML]{0000FF} {\ul 0.293}}    & {\color[HTML]{0000FF} {\ul 0.328}}    & {\color[HTML]{0000FF} {\ul 0.159}}    & {\color[HTML]{0000FF} {\ul 0.252}}    & {\color[HTML]{0000FF} {\ul 0.228}}    & {\color[HTML]{0000FF} {\ul 0.243}}    & {\color[HTML]{0000FF} {\ul 0.254}}    & {\color[HTML]{0000FF} {\ul 0.276}}    & {\color[HTML]{0000FF} {\ul 0.105}}    & {\color[HTML]{0000FF} {\ul 0.210}}    \\
                         &                           & H4                     & 0.460                                 & {\color[HTML]{0000FF} {\ul 0.461}}    & 0.391                                 & {\color[HTML]{0000FF} {\ul 0.385}}    & {\color[HTML]{0000FF} {\ul 0.183}}    & {\color[HTML]{0000FF} {\ul 0.276}}    & {\color[HTML]{0000FF} {\ul 0.237}}    & {\color[HTML]{0000FF} {\ul 0.251}}    & {\color[HTML]{0000FF} {\ul 0.335}}    & {\color[HTML]{0000FF} {\ul 0.331}}    & {\color[HTML]{0000FF} {\ul 0.145}}    & {\color[HTML]{0000FF} {\ul 0.253}}    \\ \cmidrule(l){3-15} 
\multirow{-5}{*}{\ding{52}}   & \multirow{-5}{*}{\ding{55}}     & Avg                    & {\color[HTML]{0000FF} {\ul 0.420}}    & {\color[HTML]{0000FF} {\ul 0.422}}    & {\color[HTML]{0000FF} {\ul 0.271}}    & 0.312                                 & {\color[HTML]{0000FF} {\ul 0.155}}    & {\color[HTML]{0000FF} {\ul 0.247}}    & {\color[HTML]{0000FF} {\ul 0.213}}    & {\color[HTML]{0000FF} {\ul 0.230}}    & {\color[HTML]{0000FF} {\ul 0.235}}    & {\color[HTML]{0000FF} {\ul 0.257}}    & {\color[HTML]{0000FF} {\ul 0.097}}    & {\color[HTML]{0000FF} {\ul 0.200}}    \\ \midrule
                         &                           & H1                     & {\color[HTML]{FF0000} \textbf{0.356}} & {\color[HTML]{FF0000} \textbf{0.379}} & {\color[HTML]{FF0000} \textbf{0.168}} & {\color[HTML]{FF0000} \textbf{0.245}} & {\color[HTML]{FF0000} \textbf{0.129}} & {\color[HTML]{0000FF} {\ul 0.221}}    & {\color[HTML]{FF0000} \textbf{0.175}} & {\color[HTML]{FF0000} \textbf{0.196}} & {\color[HTML]{FF0000} \textbf{0.150}} & {\color[HTML]{FF0000} \textbf{0.186}} & {\color[HTML]{FF0000} \textbf{0.059}} & {\color[HTML]{FF0000} \textbf{0.158}} \\
                         &                           & H2                     & {\color[HTML]{FF0000} \textbf{0.409}} & {\color[HTML]{FF0000} \textbf{0.411}} & {\color[HTML]{FF0000} \textbf{0.231}} & {\color[HTML]{FF0000} \textbf{0.287}} & {\color[HTML]{FF0000} \textbf{0.147}} & {\color[HTML]{FF0000} \textbf{0.238}} & {\color[HTML]{FF0000} \textbf{0.202}} & {\color[HTML]{FF0000} \textbf{0.220}} & {\color[HTML]{FF0000} \textbf{0.198}} & {\color[HTML]{FF0000} \textbf{0.234}} & {\color[HTML]{FF0000} \textbf{0.075}} & {\color[HTML]{FF0000} \textbf{0.177}} \\
                         &                           & H3                     & {\color[HTML]{FF0000} \textbf{0.449}} & {\color[HTML]{FF0000} \textbf{0.433}} & {\color[HTML]{FF0000} \textbf{0.289}} & {\color[HTML]{FF0000} \textbf{0.326}} & {\color[HTML]{FF0000} \textbf{0.158}} & {\color[HTML]{FF0000} \textbf{0.250}} & {\color[HTML]{FF0000} \textbf{0.224}} & {\color[HTML]{FF0000} \textbf{0.240}} & {\color[HTML]{FF0000} \textbf{0.252}} & {\color[HTML]{FF0000} \textbf{0.275}} & {\color[HTML]{FF0000} \textbf{0.102}} & {\color[HTML]{FF0000} \textbf{0.208}} \\
                         &                           & H4                     & {\color[HTML]{FF0000} \textbf{0.451}} & {\color[HTML]{FF0000} \textbf{0.456}} & {\color[HTML]{FF0000} \textbf{0.386}} & {\color[HTML]{FF0000} \textbf{0.382}} & {\color[HTML]{FF0000} \textbf{0.178}} & {\color[HTML]{FF0000} \textbf{0.271}} & {\color[HTML]{FF0000} \textbf{0.234}} & {\color[HTML]{FF0000} \textbf{0.247}} & {\color[HTML]{FF0000} \textbf{0.332}} & {\color[HTML]{FF0000} \textbf{0.328}} & {\color[HTML]{FF0000} \textbf{0.138}} & {\color[HTML]{FF0000} \textbf{0.246}} \\ \cmidrule(l){3-15} 
\multirow{-5}{*}{\ding{52}}   & \multirow{-5}{*}{\ding{52}}    & Avg                    & {\color[HTML]{FF0000} \textbf{0.416}} & {\color[HTML]{FF0000} \textbf{0.420}} & {\color[HTML]{FF0000} \textbf{0.268}} & {\color[HTML]{FF0000} \textbf{0.310}} & {\color[HTML]{FF0000} \textbf{0.153}} & {\color[HTML]{FF0000} \textbf{0.245}} & {\color[HTML]{FF0000} \textbf{0.209}} & {\color[HTML]{FF0000} \textbf{0.226}} & {\color[HTML]{FF0000} \textbf{0.233}} & {\color[HTML]{FF0000} \textbf{0.256}} & {\color[HTML]{FF0000} \textbf{0.093}} & {\color[HTML]{FF0000} \textbf{0.197}} \\ \bottomrule
\end{tabular}
}
\caption{Ablation study on pre- and post-linear layers in the vecTrans module. `Hor.' denotes horizons. $\left \{ \mathrm{H1}, \mathrm{H2}, \mathrm{H3}, \mathrm{H4 } \right \}$ corresponds to $\{12,24,48,96\}$ for PEMS03, and $\{96,192,336,720\}$ for the other datasets.}
\label{tab:abl_pre_post_lin}
\end{center}
\end{table*}

\subsection{Increasing Lookback Length}

\begin{figure}[H]
   \centering
   \includegraphics[width=1.0\linewidth]{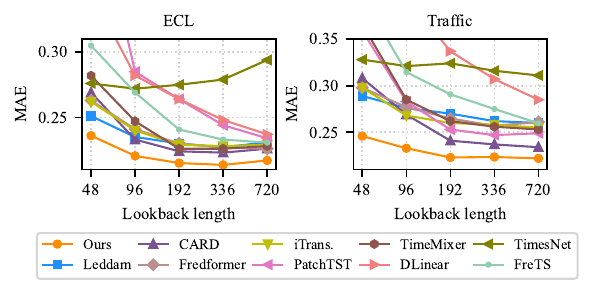}
   \caption{Performance change with increasing lookback lengths. The horizon is uniformly set to $H=96$.}
   \label{fig_lookback}
\end{figure}

Performance under increasing lookback lengths reflects a model’s capability to effectively exploit historical information \cite{itransformer}. While linear-based forecasters theoretically benefit from longer lookback windows \cite{linear}, vLinear empirically demonstrates this advantage. As shown in Figure~\ref{fig_lookback}, vLinear exhibits consistent improvements as the lookback length extends from 48 to 720, consistently surpassing state-of-the-art forecasters.

\subsection{Training with Less Data}

\begin{figure}[H]
   \centering
   \includegraphics[width=1.0\linewidth]{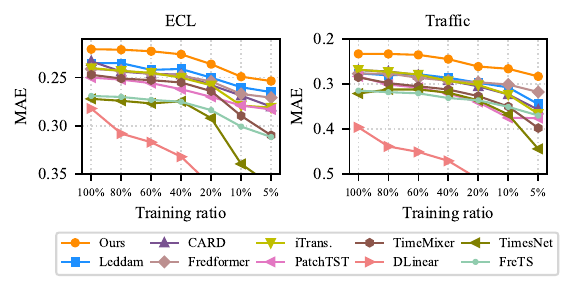}
   \caption{Performance change with less training data. The `Input-96-Predict-96' setting is applied. The Y-axis is inverted for clarity.}
   \label{fig_part_train}
\end{figure}

A model's performance under limited training data reflects its adaptability and learning efficacy. As shown in Figure~\ref{fig_part_train}, vLinear adapts well to reduced training set sizes, ranging from 100\% down to 5\%. Our model consistently outperforms others across all training ratios, highlighting its learning efficacy and robustness to data-sparse conditions.

\section{Full Results}\label{appd_full_results}

This section provides the complete results of some experiments in the main paper.

\subsection{Forecasting Performance}
\label{appd_forecast}

Tables~\ref{tab_long_term_appd} and \ref{tab_long_term_PEMS} present the full long-term forecasting results, corresponding to Table~\ref{tab_long_term} in the main paper. Similarly, Tables~\ref{tab_short_term_p1} and \ref{tab_short_term_p2} detail the short-term forecasting performance (corresponding to Table~\ref{tab_short_term} in the main paper). 

We also provide a comparison with additional state-of-the-art forecasters in Table~\ref{tab_more_baseline_appd} (corresponding to Table~\ref{tab_more_baseline} in the main paper).

\subsection{Ablation Studies}
\label{appd_full_ablation}

Table~\ref{tab_var_tmp_appd} details the ablation studies on variate (\textit{Var.}) and temporal (\textit{Temp.}) dimension representation learning (corresponding to Table~\ref{tab_var_temp} in the main paper). 

Furthermore, Table~\ref{tab_vec_attn_appd} compares the vecTrans module against various attention mechanisms (including the recent Gated Attention \cite{gatedattn}) and Mamba (corresponding to Table~\ref{tab_vec_attn} in the main paper), while Table~\ref{tab_wfmloss_comp_appd} compares WFMLoss with state-of-the-art losses for time series forecasting (corresponding to Table~\ref{tab_wfmloss_compare} in the main paper).

\subsection{Generalization Analysis}
\label{appd_full_general}
Finally, we demonstrate the generalizability of our components. Table~\ref{tab_vecTrans_plugin_appd} and Table~\ref{tab_wfmloss_plugin_appd} report the performance gains when applying the vecTrans module and WFMLoss to existing forecasters, corresponding to Tables~\ref{tab_vecTrans_plugin} and Table~\ref{tab_wfmloss_plugin} in the main paper, respectively.

\begin{table*}[t]
\begin{center}
{\fontsize{6.7}{8}\selectfont
\setlength{\tabcolsep}{2pt}
\begin{tabular}{@{}cc|cc|cc|cc|cc|cc|cc|cc|cc|cc|cc|cc|cc@{}}
\toprule
\multicolumn{2}{c|}{Model}                          & \multicolumn{2}{c|}{\begin{tabular}[c]{@{}c@{}}vLinear\\      (Ours)\end{tabular}} & \multicolumn{2}{c|}{\begin{tabular}[c]{@{}c@{}}OLinear\\  \shortcite{olinear} \end{tabular}} 
& \multicolumn{2}{c|}{\begin{tabular}[c]{@{}c@{}}TimeMixer\\      \shortcite{timemixer} \end{tabular}} 
& \multicolumn{2}{c|}{\begin{tabular}[c]{@{}c@{}}FilterNet\\   \shortcite{filternet} \end{tabular}} 
& \multicolumn{2}{c|}{\begin{tabular}[c]{@{}c@{}}FITS\\      \shortcite{fits} \end{tabular}} 
& \multicolumn{2}{c|}{\begin{tabular}[c]{@{}c@{}}DLinear\\      \shortcite{linear} \end{tabular}} 
& \multicolumn{2}{c|}{\begin{tabular}[c]{@{}c@{}}TimeMix.++\\    \shortcite{timemixer++} \end{tabular}} 
& \multicolumn{2}{c|}{\begin{tabular}[c]{@{}c@{}}Leddam\\      \shortcite{Leddam_icml} \end{tabular}} 
& \multicolumn{2}{c|}{\begin{tabular}[c]{@{}c@{}}Fredformer\\   \shortcite{fredformer} \end{tabular}} 
& \multicolumn{2}{c|}{\begin{tabular}[c]{@{}c@{}}iTrans.\\      \shortcite{itransformer} \end{tabular}} 
& \multicolumn{2}{c|}{\begin{tabular}[c]{@{}c@{}}PatchTST\\     \shortcite{patchtst} \end{tabular}} 
& \multicolumn{2}{c}{\begin{tabular}[c]{@{}c@{}}TimesNet\\      \shortcite{timesnet} \end{tabular}} \\ \midrule
\multicolumn{2}{c|}{Metric}          & MSE                                      & MAE                                     & MSE                                  & MAE                                     & MSE                                                   & MAE                      & MSE                                     & MAE                                    & MSE                                  & MAE                                  & MSE                                                  & MAE                     & MSE                                      & MAE                                     & MSE                                                  & MAE                    & MSE                                     & MAE                                     & MSE                                                  & MAE                     & MSE                                    & MAE                                    & MSE                                    & MAE                                   \\ \midrule
                               & 96  & {\color[HTML]{FF0000} \textbf{0.301}}    & {\color[HTML]{0000FF} {\ul 0.337}}      & {\color[HTML]{0000FF} {\ul 0.302}}   & {\color[HTML]{FF0000} \textbf{0.334}}   & 0.320                                                 & 0.357                    & 0.321                                   & 0.361                                  & 0.353                                & 0.375                                & 0.345                                                & 0.372                   & 0.310                                    & {\color[HTML]{FF0000} \textbf{0.334}}   & 0.319                                                & 0.359                  & 0.326                                   & 0.361                                   & 0.334                                                & 0.368                   & 0.329                                  & 0.367                                  & 0.338                                  & 0.375                                 \\
                               & 192 & {\color[HTML]{0000FF} {\ul 0.349}}       & 0.365                                   & 0.357                                & {\color[HTML]{0000FF} {\ul 0.363}}      & 0.361                                                 & 0.381                    & 0.367                                   & 0.387                                  & 0.486                                & 0.445                                & 0.380                                                & 0.389                   & {\color[HTML]{FF0000} \textbf{0.348}}    & {\color[HTML]{FF0000} \textbf{0.362}}   & 0.369                                                & 0.383                  & 0.363                                   & 0.380                                   & 0.377                                                & 0.391                   & 0.367                                  & 0.385                                  & 0.374                                  & 0.387                                 \\
                               & 336 & {\color[HTML]{0000FF} {\ul 0.381}}       & {\color[HTML]{0000FF} {\ul 0.387}}      & 0.387                                & {\color[HTML]{FF0000} \textbf{0.385}}   & 0.390                                                 & 0.404                    & 0.401                                   & 0.409                                  & 0.531                                & 0.475                                & 0.413                                                & 0.413                   & {\color[HTML]{FF0000} \textbf{0.376}}    & 0.391                                   & 0.394                                                & 0.402                  & 0.395                                   & 0.403                                   & 0.426                                                & 0.420                   & 0.399                                  & 0.410                                  & 0.410                                  & 0.411                                 \\
                               & 720 & {\color[HTML]{0000FF} {\ul 0.444}}       & {\color[HTML]{0000FF} {\ul 0.424}}      & 0.452                                & 0.426                                   & 0.454                                                 & 0.441                    & 0.477                                   & 0.448                                  & 0.600                                & 0.513                                & 0.474                                                & 0.453                   & {\color[HTML]{FF0000} \textbf{0.440}}    & {\color[HTML]{FF0000} \textbf{0.423}}   & 0.460                                                & 0.442                  & 0.453                                   & 0.438                                   & 0.491                                                & 0.459                   & 0.454                                  & 0.439                                  & 0.478                                  & 0.450                                 \\ \cmidrule(l){2-26} 
\multirow{-5}{*}{\rotatebox[origin=c]{90}{ETTm1}}        & Avg & {\color[HTML]{FF0000} \textbf{0.369}}    & {\color[HTML]{0000FF} {\ul 0.378}}      & {\color[HTML]{0000FF} {\ul 0.374}}   & {\color[HTML]{FF0000} \textbf{0.377}}   & 0.381                                                 & 0.395                    & 0.392                                   & 0.401                                  & 0.493                                & 0.452                                & 0.403                                                & 0.407                   & {\color[HTML]{FF0000} \textbf{0.369}}    & {\color[HTML]{0000FF} {\ul 0.378}}      & 0.386                                                & 0.397                  & 0.384                                   & 0.395                                   & 0.407                                                & 0.410                   & 0.387                                  & 0.400                                  & 0.400                                  & 0.406                                 \\ \midrule
                               & 96  & {\color[HTML]{FF0000} \textbf{0.168}}    & {\color[HTML]{FF0000} \textbf{0.245}}   & {\color[HTML]{0000FF} {\ul 0.169}}   & {\color[HTML]{0000FF} {\ul 0.249}}      & 0.175                                                 & 0.258                    & 0.175                                   & 0.258                                  & 0.182                                & 0.266                                & 0.193                                                & 0.292                   & 0.170                                    & {\color[HTML]{FF0000} \textbf{0.245}}   & 0.176                                                & 0.257                  & 0.177                                   & 0.259                                   & 0.180                                                & 0.264                   & 0.175                                  & 0.259                                  & 0.187                                  & 0.267                                 \\
                               & 192 & {\color[HTML]{0000FF} {\ul 0.231}}       & {\color[HTML]{FF0000} \textbf{0.287}}   & 0.232                                & {\color[HTML]{0000FF} {\ul 0.290}}      & 0.237                                                 & 0.299                    & 0.240                                   & 0.301                                  & 0.253                                & 0.312                                & 0.284                                                & 0.362                   & {\color[HTML]{FF0000} \textbf{0.229}}    & 0.291                                   & 0.243                                                & 0.303                  & 0.243                                   & 0.301                                   & 0.250                                                & 0.309                   & 0.241                                  & 0.302                                  & 0.249                                  & 0.309                                 \\
                               & 336 & {\color[HTML]{FF0000} \textbf{0.289}}    & {\color[HTML]{FF0000} \textbf{0.326}}   & {\color[HTML]{0000FF} {\ul 0.291}}   & {\color[HTML]{0000FF} {\ul 0.328}}      & 0.298                                                 & 0.340                    & 0.311                                   & 0.347                                  & 0.313                                & 0.349                                & 0.369                                                & 0.427                   & 0.303                                    & 0.343                                   & 0.303                                                & 0.341                  & 0.302                                   & 0.340                                   & 0.311                                                & 0.348                   & 0.305                                  & 0.343                                  & 0.321                                  & 0.351                                 \\
                               & 720 & {\color[HTML]{0000FF} {\ul 0.386}}       & {\color[HTML]{FF0000} \textbf{0.382}}   & 0.389                                & {\color[HTML]{0000FF} {\ul 0.387}}      & 0.391                                                 & 0.396                    & 0.414                                   & 0.405                                  & 0.416                                & 0.406                                & 0.554                                                & 0.522                   & {\color[HTML]{FF0000} \textbf{0.373}}    & 0.399                                   & 0.400                                                & 0.398                  & 0.397                                   & 0.396                                   & 0.412                                                & 0.407                   & 0.402                                  & 0.400                                  & 0.408                                  & 0.403                                 \\ \cmidrule(l){2-26} 
\multirow{-5}{*}{\rotatebox[origin=c]{90}{ETTm2}}        & Avg & {\color[HTML]{FF0000} \textbf{0.268}}    & {\color[HTML]{FF0000} \textbf{0.310}}   & 0.270                                & {\color[HTML]{0000FF} {\ul 0.313}}      & 0.275                                                 & 0.323                    & 0.285                                   & 0.328                                  & 0.291                                & 0.333                                & 0.350                                                & 0.401                   & {\color[HTML]{0000FF} {\ul 0.269}}       & 0.320                                   & 0.281                                                & 0.325                  & 0.279                                   & 0.324                                   & 0.288                                                & 0.332                   & 0.281                                  & 0.326                                  & 0.291                                  & 0.333                                 \\ \midrule
                               & 96  & {\color[HTML]{FF0000} \textbf{0.356}}    & {\color[HTML]{FF0000} \textbf{0.379}}   & {\color[HTML]{0000FF} {\ul 0.360}}   & {\color[HTML]{0000FF} {\ul 0.382}}      & 0.375                                                 & 0.400                    & 0.382                                   & 0.402                                  & 0.385                                & 0.394                                & 0.386                                                & 0.400                   & 0.361                                    & 0.403                                   & 0.377                                                & 0.394                  & 0.373                                   & 0.392                                   & 0.386                                                & 0.405                   & 0.414                                  & 0.419                                  & 0.384                                  & 0.402                                 \\
                               & 192 & {\color[HTML]{FF0000} \textbf{0.409}}    & {\color[HTML]{FF0000} \textbf{0.411}}   & {\color[HTML]{0000FF} {\ul 0.416}}   & {\color[HTML]{0000FF} {\ul 0.414}}      & 0.429                                                 & 0.421                    & 0.430                                   & 0.429                                  & 0.434                                & 0.422                                & 0.437                                                & 0.432                   & {\color[HTML]{0000FF} {\ul 0.416}}       & 0.441                                   & 0.424                                                & 0.422                  & 0.433                                   & 0.420                                   & 0.441                                                & 0.436                   & 0.460                                  & 0.445                                  & 0.436                                  & 0.429                                 \\
                               & 336 & {\color[HTML]{0000FF} {\ul 0.449}}       & {\color[HTML]{FF0000} \textbf{0.433}}   & 0.457                                & 0.438                                   & 0.484                                                 & 0.458                    & 0.472                                   & 0.451                                  & 0.476                                & 0.444                                & 0.481                                                & 0.459                   & {\color[HTML]{FF0000} \textbf{0.430}}    & {\color[HTML]{0000FF} {\ul 0.434}}      & 0.459                                                & 0.442                  & 0.470                                   & 0.437                                   & 0.487                                                & 0.458                   & 0.501                                  & 0.466                                  & 0.491                                  & 0.469                                 \\
                               & 720 & {\color[HTML]{FF0000} \textbf{0.451}}    & {\color[HTML]{0000FF} {\ul 0.456}}      & {\color[HTML]{0000FF} {\ul 0.463}}   & 0.462                                   & 0.498                                                 & 0.482                    & 0.481                                   & 0.473                                  & 0.465                                & 0.462                                & 0.519                                                & 0.516                   & 0.467                                    & {\color[HTML]{FF0000} \textbf{0.451}}   & {\color[HTML]{0000FF} {\ul 0.463}}                   & 0.459                  & 0.467                                   & {\color[HTML]{0000FF} {\ul 0.456}}      & 0.503                                                & 0.491                   & 0.500                                  & 0.488                                  & 0.521                                  & 0.500                                 \\ \cmidrule(l){2-26} 
\multirow{-5}{*}{\rotatebox[origin=c]{90}{ETTh1}}        & Avg & {\color[HTML]{FF0000} \textbf{0.416}}    & {\color[HTML]{FF0000} \textbf{0.420}}   & 0.424                                & {\color[HTML]{0000FF} {\ul 0.424}}      & 0.447                                                 & 0.440                    & 0.441                                   & 0.439                                  & 0.440                                & 0.431                                & 0.456                                                & 0.452                   & {\color[HTML]{0000FF} {\ul 0.419}}       & 0.432                                   & 0.431                                                & 0.429                  & 0.435                                   & 0.426                                   & 0.454                                                & 0.447                   & 0.469                                  & 0.454                                  & 0.458                                  & 0.450                                 \\ \midrule
                               & 96  & {\color[HTML]{0000FF} {\ul 0.280}}       & {\color[HTML]{FF0000} \textbf{0.327}}   & 0.284                                & 0.329                                   & 0.289                                                 & 0.341                    & 0.293                                   & 0.343                                  & 0.292                                & 0.340                                & 0.333                                                & 0.387                   & {\color[HTML]{FF0000} \textbf{0.276}}    & {\color[HTML]{0000FF} {\ul 0.328}}      & 0.292                                                & 0.343                  & 0.293                                   & 0.342                                   & 0.297                                                & 0.349                   & 0.292                                  & 0.342                                  & 0.340                                  & 0.374                                 \\
                               & 192 & {\color[HTML]{0000FF} {\ul 0.351}}       & {\color[HTML]{FF0000} \textbf{0.376}}   & 0.360                                & {\color[HTML]{0000FF} {\ul 0.379}}      & 0.372                                                 & 0.392                    & 0.374                                   & 0.396                                  & 0.377                                & 0.391                                & 0.477                                                & 0.476                   & {\color[HTML]{FF0000} \textbf{0.342}}    & {\color[HTML]{0000FF} {\ul 0.379}}      & 0.367                                                & 0.389                  & 0.371                                   & 0.389                                   & 0.380                                                & 0.400                   & 0.387                                  & 0.400                                  & 0.402                                  & 0.414                                 \\
                               & 336 & 0.400                                    & 0.411                                   & 0.409                                & 0.415                                   & 0.386                                                 & 0.414                    & 0.417                                   & 0.430                                  & 0.416                                & 0.425                                & 0.594                                                & 0.541                   & {\color[HTML]{FF0000} \textbf{0.346}}    & {\color[HTML]{FF0000} \textbf{0.398}}   & 0.412                                                & 0.424                  & {\color[HTML]{0000FF} {\ul 0.382}}      & {\color[HTML]{0000FF} {\ul 0.409}}      & 0.428                                                & 0.432                   & 0.426                                  & 0.433                                  & 0.452                                  & 0.452                                 \\
                               & 720 & {\color[HTML]{0000FF} {\ul 0.404}}       & {\color[HTML]{0000FF} {\ul 0.427}}      & 0.415                                & 0.431                                   & 0.412                                                 & 0.434                    & 0.449                                   & 0.460                                  & 0.418                                & 0.437                                & 0.831                                                & 0.657                   & {\color[HTML]{FF0000} \textbf{0.392}}    & {\color[HTML]{FF0000} \textbf{0.415}}   & 0.419                                                & 0.438                  & 0.415                                   & 0.434                                   & 0.427                                                & 0.445                   & 0.431                                  & 0.446                                  & 0.462                                  & 0.468                                 \\ \cmidrule(l){2-26} 
\multirow{-5}{*}{\rotatebox[origin=c]{90}{ETTh2}}        & Avg & {\color[HTML]{0000FF} {\ul 0.358}}       & {\color[HTML]{0000FF} {\ul 0.385}}      & 0.367                                & 0.388                                   & 0.365                                                 & 0.395                    & 0.383                                   & 0.407                                  & 0.376                                & 0.398                                & 0.559                                                & 0.515                   & {\color[HTML]{FF0000} \textbf{0.339}}    & {\color[HTML]{FF0000} \textbf{0.380}}   & 0.373                                                & 0.399                  & 0.365                                   & 0.393                                   & 0.383                                                & 0.407                   & 0.384                                  & 0.405                                  & 0.414                                  & 0.427                                 \\ \midrule
                               & 96  & {\color[HTML]{FF0000} \textbf{0.129}}    & {\color[HTML]{FF0000} \textbf{0.221}}   & {\color[HTML]{0000FF} {\ul 0.131}}   & {\color[HTML]{FF0000} \textbf{0.221}}   & 0.153                                                 & 0.247                    & 0.147                                   & 0.245                                  & 0.198                                & 0.274                                & 0.197                                                & 0.282                   & 0.135                                    & {\color[HTML]{0000FF} {\ul 0.222}}      & 0.141                                                & 0.235                  & 0.147                                   & 0.241                                   & 0.148                                                & 0.240                   & 0.161                                  & 0.250                                  & 0.168                                  & 0.272                                 \\
                               & 192 & {\color[HTML]{FF0000} \textbf{0.147}}    & {\color[HTML]{0000FF} {\ul 0.238}}      & {\color[HTML]{0000FF} {\ul 0.150}}   & {\color[HTML]{0000FF} {\ul 0.238}}      & 0.166                                                 & 0.256                    & 0.160                                   & 0.250                                  & 0.363                                & 0.422                                & 0.196                                                & 0.285                   & {\color[HTML]{FF0000} \textbf{0.147}}    & {\color[HTML]{FF0000} \textbf{0.235}}   & 0.159                                                & 0.252                  & 0.165                                   & 0.258                                   & 0.162                                                & 0.253                   & 0.199                                  & 0.289                                  & 0.184                                  & 0.289                                 \\
                               & 336 & {\color[HTML]{FF0000} \textbf{0.158}}    & {\color[HTML]{0000FF} {\ul 0.250}}      & 0.165                                & 0.254                                   & 0.185                                                 & 0.277                    & 0.173                                   & 0.267                                  & 0.444                                & 0.490                                & 0.209                                                & 0.301                   & {\color[HTML]{0000FF} {\ul 0.164}}       & {\color[HTML]{FF0000} \textbf{0.245}}   & 0.173                                                & 0.268                  & 0.177                                   & 0.273                                   & 0.178                                                & 0.269                   & 0.215                                  & 0.305                                  & 0.198                                  & 0.300                                 \\
                               & 720 & {\color[HTML]{FF0000} \textbf{0.178}}    & {\color[HTML]{FF0000} \textbf{0.271}}   & {\color[HTML]{0000FF} {\ul 0.191}}   & {\color[HTML]{0000FF} {\ul 0.279}}      & 0.225                                                 & 0.310                    & 0.210                                   & 0.309                                  & 0.532                                & 0.551                                & 0.245                                                & 0.333                   & 0.212                                    & 0.310                                   & 0.201                                                & 0.295                  & 0.213                                   & 0.304                                   & 0.225                                                & 0.317                   & 0.256                                  & 0.337                                  & 0.220                                  & 0.320                                 \\ \cmidrule(l){2-26} 
\multirow{-5}{*}{\rotatebox[origin=c]{90}{ECL}}          & Avg & {\color[HTML]{FF0000} \textbf{0.153}}    & {\color[HTML]{FF0000} \textbf{0.245}}   & {\color[HTML]{0000FF} {\ul 0.159}}   & {\color[HTML]{0000FF} {\ul 0.248}}      & 0.182                                                 & 0.273                    & 0.173                                   & 0.268                                  & 0.384                                & 0.434                                & 0.212                                                & 0.300                   & 0.165                                    & 0.253                                   & 0.169                                                & 0.263                  & 0.176                                   & 0.269                                   & 0.178                                                & 0.270                   & 0.208                                  & 0.295                                  & 0.192                                  & 0.295                                 \\ \midrule
                               & 96  & {\color[HTML]{FF0000} \textbf{0.079}}    & {\color[HTML]{FF0000} \textbf{0.197}}   & {\color[HTML]{0000FF} {\ul 0.082}}   & {\color[HTML]{0000FF} {\ul 0.200}}      & 0.086                                                 & 0.205                    & 0.091                                   & 0.211                                  & 0.087                                & 0.208                                & 0.088                                                & 0.218                   & 0.085                                    & 0.214                                   & 0.086                                                & 0.207                  & 0.084                                   & 0.202                                   & 0.086                                                & 0.206                   & 0.088                                  & 0.205                                  & 0.107                                  & 0.234                                 \\
                               & 192 & {\color[HTML]{FF0000} \textbf{0.168}}    & {\color[HTML]{FF0000} \textbf{0.291}}   & {\color[HTML]{0000FF} {\ul 0.171}}   & {\color[HTML]{0000FF} {\ul 0.293}}      & 0.193                                                 & 0.312                    & 0.186                                   & 0.305                                  & 0.185                                & 0.306                                & 0.176                                                & 0.315                   & 0.175                                    & 0.313                                   & 0.175                                                & 0.301                  & 0.178                                   & 0.302                                   & 0.177                                                & 0.299                   & 0.176                                  & 0.299                                  & 0.226                                  & 0.344                                 \\
                               & 336 & 0.321                                    & {\color[HTML]{0000FF} {\ul 0.408}}      & 0.331                                & 0.414                                   & 0.356                                                 & 0.433                    & 0.380                                   & 0.449                                  & 0.342                                & 0.425                                & {\color[HTML]{0000FF} {\ul 0.313}}                   & 0.427                   & 0.316                                    & 0.420                                   & 0.325                                                & 0.415                  & 0.319                                   & {\color[HTML]{0000FF} {\ul 0.408}}      & 0.331                                                & 0.417                   & {\color[HTML]{FF0000} \textbf{0.301}}  & {\color[HTML]{FF0000} \textbf{0.397}}  & 0.367                                  & 0.448                                 \\
                               & 720 & {\color[HTML]{0000FF} {\ul 0.796}}       & {\color[HTML]{0000FF} {\ul 0.670}}      & 0.837                                & 0.688                                   & 0.912                                                 & 0.712                    & 0.896                                   & 0.712                                  & 0.846                                & 0.694                                & 0.839                                                & 0.695                   & 0.851                                    & 0.689                                   & 0.831                                                & 0.686                  & {\color[HTML]{FF0000} \textbf{0.749}}   & {\color[HTML]{FF0000} \textbf{0.651}}   & 0.847                                                & 0.691                   & 0.901                                  & 0.714                                  & 0.964                                  & 0.746                                 \\ \cmidrule(l){2-26} 
\multirow{-5}{*}{\rotatebox[origin=c]{90}{Exchange}}     & Avg & {\color[HTML]{0000FF} {\ul 0.341}}       & {\color[HTML]{FF0000} \textbf{0.391}}   & 0.355                                & {\color[HTML]{0000FF} {\ul 0.399}}      & 0.387                                                 & 0.416                    & 0.388                                   & 0.419                                  & 0.365                                & 0.408                                & 0.354                                                & 0.414                   & 0.357                                    & 0.409                                   & 0.354                                                & 0.402                  & {\color[HTML]{FF0000} \textbf{0.333}}   & {\color[HTML]{FF0000} \textbf{0.391}}   & 0.360                                                & 0.403                   & 0.367                                  & 0.404                                  & 0.416                                  & 0.443                                 \\ \midrule
                               & 96  & 0.396                                    & {\color[HTML]{0000FF} {\ul 0.233}}      & 0.398                                & {\color[HTML]{FF0000} \textbf{0.226}}   & 0.462                                                 & 0.285                    & 0.430                                   & 0.294                                  & 0.601                                & 0.361                                & 0.650                                                & 0.396                   & {\color[HTML]{FF0000} \textbf{0.392}}    & 0.253                                   & 0.426                                                & 0.276                  & 0.406                                   & 0.277                                   & {\color[HTML]{0000FF} {\ul 0.395}}                   & 0.268                   & 0.446                                  & 0.283                                  & 0.593                                  & 0.321                                 \\
                               & 192 & 0.427                                    & {\color[HTML]{0000FF} {\ul 0.243}}      & 0.439                                & {\color[HTML]{FF0000} \textbf{0.241}}   & 0.473                                                 & 0.296                    & 0.452                                   & 0.307                                  & 0.603                                & 0.365                                & 0.598                                                & 0.370                   & {\color[HTML]{FF0000} \textbf{0.402}}    & 0.258                                   & 0.458                                                & 0.289                  & 0.426                                   & 0.290                                   & {\color[HTML]{0000FF} {\ul 0.417}}                   & 0.276                   & 0.540                                  & 0.354                                  & 0.617                                  & 0.336                                 \\
                               & 336 & 0.449                                    & {\color[HTML]{0000FF} {\ul 0.255}}      & 0.464                                & {\color[HTML]{FF0000} \textbf{0.250}}   & 0.498                                                 & 0.296                    & 0.470                                   & 0.316                                  & 0.609                                & 0.366                                & 0.605                                                & 0.373                   & {\color[HTML]{FF0000} \textbf{0.428}}    & 0.263                                   & 0.486                                                & 0.297                  & 0.437                                   & 0.292                                   & {\color[HTML]{0000FF} {\ul 0.433}}                   & 0.283                   & 0.551                                  & 0.358                                  & 0.629                                  & 0.336                                 \\
                               & 720 & 0.487                                    & {\color[HTML]{0000FF} {\ul 0.276}}      & 0.502                                & {\color[HTML]{FF0000} \textbf{0.270}}   & 0.506                                                 & 0.313                    & 0.498                                   & 0.323                                  & 0.648                                & 0.387                                & 0.645                                                & 0.394                   & {\color[HTML]{FF0000} \textbf{0.441}}    & 0.282                                   & 0.498                                                & 0.313                  & 0.462                                   & 0.305                                   & {\color[HTML]{0000FF} {\ul 0.467}}                   & 0.302                   & 0.586                                  & 0.375                                  & 0.640                                  & 0.350                                 \\ \cmidrule(l){2-26} 
\multirow{-5}{*}{\rotatebox[origin=c]{90}{Traffic}}      & Avg & 0.440                                    & {\color[HTML]{0000FF} {\ul 0.252}}      & 0.451                                & {\color[HTML]{FF0000} \textbf{0.247}}   & 0.485                                                 & 0.298                    & 0.463                                   & 0.310                                  & 0.615                                & 0.370                                & 0.625                                                & 0.383                   & {\color[HTML]{FF0000} \textbf{0.416}}    & 0.264                                   & 0.467                                                & 0.294                  & 0.433                                   & 0.291                                   & {\color[HTML]{0000FF} {\ul 0.428}}                   & 0.282                   & 0.531                                  & 0.343                                  & 0.620                                  & 0.336                                 \\ \midrule
                               & 96  & {\color[HTML]{FF0000} \textbf{0.150}}    & {\color[HTML]{FF0000} \textbf{0.186}}   & {\color[HTML]{0000FF} {\ul 0.153}}   & {\color[HTML]{0000FF} {\ul 0.190}}      & 0.163                                                 & 0.209                    & 0.162                                   & 0.207                                  & 0.196                                & 0.236                                & 0.196                                                & 0.255                   & 0.155                                    & 0.205                                   & 0.156                                                & 0.202                  & 0.163                                   & 0.207                                   & 0.174                                                & 0.214                   & 0.177                                  & 0.218                                  & 0.172                                  & 0.220                                 \\
                               & 192 & {\color[HTML]{FF0000} \textbf{0.198}}    & {\color[HTML]{FF0000} \textbf{0.234}}   & {\color[HTML]{0000FF} {\ul 0.200}}   & {\color[HTML]{0000FF} {\ul 0.235}}      & 0.208                                                 & 0.250                    & 0.210                                   & 0.250                                  & 0.240                                & 0.271                                & 0.237                                                & 0.296                   & 0.201                                    & 0.245                                   & 0.207                                                & 0.250                  & 0.211                                   & 0.251                                   & 0.221                                                & 0.254                   & 0.225                                  & 0.259                                  & 0.219                                  & 0.261                                 \\
                               & 336 & 0.252                                    & {\color[HTML]{0000FF} {\ul 0.275}}      & 0.258                                & 0.280                                   & {\color[HTML]{0000FF} {\ul 0.251}}                    & 0.287                    & 0.265                                   & 0.290                                  & 0.292                                & 0.307                                & 0.283                                                & 0.335                   & {\color[HTML]{FF0000} \textbf{0.237}}    & {\color[HTML]{FF0000} \textbf{0.265}}   & 0.262                                                & 0.291                  & 0.267                                   & 0.292                                   & 0.278                                                & 0.296                   & 0.278                                  & 0.297                                  & 0.280                                  & 0.306                                 \\
                               & 720 & {\color[HTML]{0000FF} {\ul 0.332}}       & {\color[HTML]{FF0000} \textbf{0.328}}   & 0.337                                & {\color[HTML]{0000FF} {\ul 0.333}}      & 0.339                                                 & 0.341                    & 0.342                                   & 0.340                                  & 0.365                                & 0.354                                & 0.345                                                & 0.381                   & {\color[HTML]{FF0000} \textbf{0.312}}    & 0.334                                   & 0.343                                                & 0.343                  & 0.343                                   & 0.341                                   & 0.358                                                & 0.349                   & 0.354                                  & 0.348                                  & 0.365                                  & 0.359                                 \\ \cmidrule(l){2-26} 
\multirow{-5}{*}{\rotatebox[origin=c]{90}{Weather}}      & Avg & {\color[HTML]{0000FF} {\ul 0.233}}       & {\color[HTML]{FF0000} \textbf{0.256}}   & 0.237                                & {\color[HTML]{0000FF} {\ul 0.260}}      & 0.240                                                 & 0.272                    & 0.245                                   & 0.272                                  & 0.273                                & 0.292                                & 0.265                                                & 0.317                   & {\color[HTML]{FF0000} \textbf{0.226}}    & 0.262                                   & 0.242                                                & 0.272                  & 0.246                                   & 0.272                                   & 0.258                                                & 0.279                   & 0.259                                  & 0.281                                  & 0.259                                  & 0.287                                 \\ \midrule
                               & 96  & {\color[HTML]{0000FF} {\ul 0.175}}       & {\color[HTML]{0000FF} {\ul 0.196}}      & 0.179                                & {\color[HTML]{FF0000} \textbf{0.191}}   & 0.189                                                 & 0.259                    & 0.205                                   & 0.242                                  & 0.319                                & 0.353                                & 0.290                                                & 0.378                   & {\color[HTML]{FF0000} \textbf{0.171}}    & 0.231                                   & 0.197                                                & 0.241                  & 0.185                                   & 0.233                                   & 0.203                                                & 0.237                   & 0.234                                  & 0.286                                  & 0.250                                  & 0.292                                 \\
                               & 192 & {\color[HTML]{FF0000} \textbf{0.202}}    & {\color[HTML]{0000FF} {\ul 0.220}}      & {\color[HTML]{0000FF} {\ul 0.209}}   & {\color[HTML]{FF0000} \textbf{0.213}}   & 0.222                                                 & 0.283                    & 0.233                                   & 0.265                                  & 0.367                                & 0.387                                & 0.320                                                & 0.398                   & 0.218                                    & 0.263                                   & 0.231                                                & 0.264                  & 0.227                                   & 0.253                                   & 0.233                                                & 0.261                   & 0.267                                  & 0.310                                  & 0.296                                  & 0.318                                 \\
                               & 336 & {\color[HTML]{0000FF} {\ul 0.224}}       & {\color[HTML]{0000FF} {\ul 0.240}}      & 0.231                                & {\color[HTML]{FF0000} \textbf{0.229}}   & 0.231                                                 & 0.292                    & 0.249                                   & 0.278                                  & 0.408                                & 0.403                                & 0.353                                                & 0.415                   & {\color[HTML]{FF0000} \textbf{0.212}}    & 0.269                                   & 0.241                                                & 0.268                  & 0.246                                   & 0.284                                   & 0.248                                                & 0.273                   & 0.290                                  & 0.315                                  & 0.319                                  & 0.330                                 \\
                               & 720 & 0.234                                    & {\color[HTML]{0000FF} {\ul 0.247}}      & 0.241                                & {\color[HTML]{FF0000} \textbf{0.236}}   & {\color[HTML]{0000FF} {\ul 0.223}}                    & 0.285                    & 0.253                                   & 0.281                                  & 0.411                                & 0.395                                & 0.356                                                & 0.413                   & {\color[HTML]{FF0000} \textbf{0.212}}    & 0.270                                   & 0.250                                                & 0.281                  & 0.247                                   & 0.276                                   & 0.249                                                & 0.275                   & 0.289                                  & 0.317                                  & 0.338                                  & 0.337                                 \\ \cmidrule(l){2-26} 
\multirow{-5}{*}{\rotatebox[origin=c]{90}{Solar-Energy}} & Avg & {\color[HTML]{0000FF} {\ul 0.209}}       & {\color[HTML]{0000FF} {\ul 0.226}}      & 0.215                                & {\color[HTML]{FF0000} \textbf{0.217}}   & 0.216                                                 & 0.280                    & 0.235                                   & 0.266                                  & 0.376                                & 0.384                                & 0.330                                                & 0.401                   & {\color[HTML]{FF0000} \textbf{0.203}}    & 0.258                                   & 0.230                                                & 0.264                  & 0.226                                   & 0.262                                   & 0.233                                                & 0.262                   & 0.270                                  & 0.307                                  & 0.301                                  & 0.319                                 \\ \midrule
\multicolumn{2}{c|}{1\textsuperscript{st} Count}                             & 19                                          & 21                                                               & 0                                  & 14                                                               & 0                                          & 0                                  & 0                            & 0                                                 & 0                          & 0                                              & 0                                         & 0                                 & 25                                          & 11                                                               & 0                                         & 0                                 & 2                                     & 2                                                                & 0                                          & 0                                 & 1                                     & 1                                                          & 0                                      & 0    \\ \midrule
\multicolumn{2}{c|}{Top-2 Count}      & 36                                      & 43                                      & 16                                   & 34                                     & 2                                                     & 0                       & 0                                      & 0                                      & 0                                    & 0                                   & 1                                                    & 0                      & 29                                      & 16                                      & 0                                                   & 0                      & 3                                       & 5                                      & 5                                                    & 0                      & 1                                      & 1                                     & 0                                      & 0                                     \\ \bottomrule
\end{tabular}
}
\caption{Full results of long-term forecasting under the `Input-96-Predict-\{96, 192, 336, 720\}' setting (excluding PEMS datasets). This table provides the complete results corresponding to Table~\ref{tab_long_term} in the main paper. Results for PEMS datasets are presented in Table~\ref{tab_long_term_PEMS}.}
\label{tab_long_term_appd}
\end{center}
\end{table*}

\begin{table*}[t]
\begin{center}
\renewcommand{\arraystretch}{1.3}
{\fontsize{6.7}{8}\selectfont
\setlength{\tabcolsep}{2.2pt}
\begin{tabular}{@{}cc|cc|cc|cc|cc|cc|cc|cc|cc|cc|cc|cc|cc@{}}
\toprule
\multicolumn{2}{c|}{Model}                          & \multicolumn{2}{c|}{\begin{tabular}[c]{@{}c@{}}vLinear\\      (Ours)\end{tabular}} & \multicolumn{2}{c|}{\begin{tabular}[c]{@{}c@{}}OLinear\\  \shortcite{olinear} \end{tabular}} 
& \multicolumn{2}{c|}{\begin{tabular}[c]{@{}c@{}}TimeMixer\\      \shortcite{timemixer} \end{tabular}} 
& \multicolumn{2}{c|}{\begin{tabular}[c]{@{}c@{}}FilterNet\\   \shortcite{filternet} \end{tabular}} 
& \multicolumn{2}{c|}{\begin{tabular}[c]{@{}c@{}}FITS\\      \shortcite{fits} \end{tabular}} 
& \multicolumn{2}{c|}{\begin{tabular}[c]{@{}c@{}}DLinear\\      \shortcite{linear} \end{tabular}} 
& \multicolumn{2}{c|}{\begin{tabular}[c]{@{}c@{}}TimeMix.++\\    \shortcite{timemixer++} \end{tabular}} 
& \multicolumn{2}{c|}{\begin{tabular}[c]{@{}c@{}}Leddam\\      \shortcite{Leddam_icml} \end{tabular}} 
& \multicolumn{2}{c|}{\begin{tabular}[c]{@{}c@{}}Fredformer\\   \shortcite{fredformer} \end{tabular}} 
& \multicolumn{2}{c|}{\begin{tabular}[c]{@{}c@{}}iTrans.\\      \shortcite{itransformer} \end{tabular}} 
& \multicolumn{2}{c|}{\begin{tabular}[c]{@{}c@{}}PatchTST\\     \shortcite{patchtst} \end{tabular}} 
& \multicolumn{2}{c}{\begin{tabular}[c]{@{}c@{}}TimesNet\\      \shortcite{timesnet} \end{tabular}} \\ \midrule
\multicolumn{2}{c|}{Metric}     & MSE                                      & MAE                                     & MSE                                    & MAE                                   & MSE                                     & MAE                                    & MSE                                     & MAE                                    & MSE                                  & MAE                                  & MSE                                    & MAE                                   & MSE                                      & MAE                                     & MSE                                   & MAE                                   & MSE                                     & MAE                                     & MSE                                    & MAE                                   & MSE                                    & MAE                                    & MSE                                    & MAE                                   \\ \midrule
                          & 12  & {\color[HTML]{FF0000} \textbf{0.059}}    & {\color[HTML]{FF0000} \textbf{0.158}}   & {\color[HTML]{0000FF} {\ul 0.060}}     & {\color[HTML]{0000FF} {\ul 0.159}}    & 0.076                                   & 0.188                                  & 0.071                                   & 0.177                                  & 0.117                                & 0.226                                & 0.122                                  & 0.243                                 & 0.097                                    & 0.208                                   & 0.063                                 & 0.164                                 & 0.068                                   & 0.174                                   & 0.071                                  & 0.174                                 & 0.099                                  & 0.216                                  & 0.085                                  & 0.192                                 \\
                          & 24  & {\color[HTML]{FF0000} \textbf{0.075}}    & {\color[HTML]{FF0000} \textbf{0.177}}   & {\color[HTML]{0000FF} {\ul 0.078}}     & {\color[HTML]{0000FF} {\ul 0.179}}    & 0.113                                   & 0.226                                  & 0.102                                   & 0.213                                  & 0.235                                & 0.324                                & 0.201                                  & 0.317                                 & 0.120                                    & 0.230                                   & 0.080                                 & 0.185                                 & 0.093                                   & 0.202                                   & 0.093                                  & 0.201                                 & 0.142                                  & 0.259                                  & 0.118                                  & 0.223                                 \\
                          & 48  & {\color[HTML]{FF0000} \textbf{0.102}}    & {\color[HTML]{FF0000} \textbf{0.208}}   & {\color[HTML]{0000FF} {\ul 0.104}}     & {\color[HTML]{0000FF} {\ul 0.210}}    & 0.191                                   & 0.292                                  & 0.162                                   & 0.272                                  & 0.541                                & 0.521                                & 0.333                                  & 0.425                                 & 0.170                                    & 0.272                                   & 0.124                                 & 0.226                                 & 0.146                                   & 0.258                                   & 0.125                                  & 0.236                                 & 0.211                                  & 0.319                                  & 0.155                                  & 0.260                                 \\
                          & 96  & {\color[HTML]{FF0000} \textbf{0.138}}    & {\color[HTML]{FF0000} \textbf{0.246}}   & {\color[HTML]{0000FF} {\ul 0.140}}     & {\color[HTML]{0000FF} {\ul 0.247}}    & 0.288                                   & 0.363                                  & 0.244                                   & 0.340                                  & 1.062                                & 0.790                                & 0.457                                  & 0.515                                 & 0.274                                    & 0.342                                   & 0.160                                 & 0.266                                 & 0.228                                   & 0.330                                   & 0.164                                  & 0.275                                 & 0.269                                  & 0.370                                  & 0.228                                  & 0.317                                 \\ \cmidrule(l){2-26} 
\multirow{-5}{*}{\rotatebox[origin=c]{90}{PEMS03}} & Avg & {\color[HTML]{FF0000} \textbf{0.093}}    & {\color[HTML]{FF0000} \textbf{0.197}}   & {\color[HTML]{0000FF} {\ul 0.095}}     & {\color[HTML]{0000FF} {\ul 0.199}}    & 0.167                                   & 0.267                                  & 0.145                                   & 0.251                                  & 0.489                                & 0.465                                & 0.278                                  & 0.375                                 & 0.165                                    & 0.263                                   & 0.107                                 & 0.210                                 & 0.135                                   & 0.243                                   & 0.113                                  & 0.221                                 & 0.180                                  & 0.291                                  & 0.147                                  & 0.248                                 \\ \midrule
                          & 12  & {\color[HTML]{FF0000} \textbf{0.067}}    & {\color[HTML]{FF0000} \textbf{0.163}}   & {\color[HTML]{0000FF} {\ul 0.068}}     & {\color[HTML]{FF0000} \textbf{0.163}} & 0.092                                   & 0.204                                  & 0.082                                   & 0.190                                  & 0.129                                & 0.239                                & 0.148                                  & 0.272                                 & 0.099                                    & 0.214                                   & 0.071                                 & {\color[HTML]{0000FF} {\ul 0.172}}    & 0.085                                   & 0.189                                   & 0.078                                  & 0.183                                 & 0.105                                  & 0.224                                  & 0.087                                  & 0.195                                 \\
                          & 24  & {\color[HTML]{FF0000} \textbf{0.077}}    & {\color[HTML]{FF0000} \textbf{0.176}}   & {\color[HTML]{0000FF} {\ul 0.079}}     & {\color[HTML]{FF0000} \textbf{0.176}} & 0.128                                   & 0.243                                  & 0.110                                   & 0.224                                  & 0.246                                & 0.337                                & 0.224                                  & 0.340                                 & 0.115                                    & 0.231                                   & 0.087                                 & {\color[HTML]{0000FF} {\ul 0.193}}    & 0.117                                   & 0.224                                   & 0.095                                  & 0.205                                 & 0.153                                  & 0.275                                  & 0.103                                  & 0.215                                 \\
                          & 48  & {\color[HTML]{FF0000} \textbf{0.095}}    & {\color[HTML]{FF0000} \textbf{0.196}}   & {\color[HTML]{FF0000} \textbf{0.095}}  & {\color[HTML]{0000FF} {\ul 0.197}}    & 0.213                                   & 0.315                                  & 0.160                                   & 0.276                                  & 0.568                                & 0.539                                & 0.355                                  & 0.437                                 & 0.144                                    & 0.261                                   & {\color[HTML]{0000FF} {\ul 0.113}}    & 0.222                                 & 0.174                                   & 0.276                                   & 0.120                                  & 0.233                                 & 0.229                                  & 0.339                                  & 0.136                                  & 0.250                                 \\
                          & 96  & {\color[HTML]{FF0000} \textbf{0.120}}    & {\color[HTML]{FF0000} \textbf{0.225}}   & {\color[HTML]{0000FF} {\ul 0.122}}     & {\color[HTML]{0000FF} {\ul 0.226}}    & 0.307                                   & 0.384                                  & 0.234                                   & 0.343                                  & 1.181                                & 0.843                                & 0.452                                  & 0.504                                 & 0.185                                    & 0.297                                   & 0.141                                 & 0.252                                 & 0.273                                   & 0.354                                   & 0.150                                  & 0.262                                 & 0.291                                  & 0.389                                  & 0.190                                  & 0.303                                 \\ \cmidrule(l){2-26} 
\multirow{-5}{*}{\rotatebox[origin=c]{90}{PEMS04}} & Avg & {\color[HTML]{FF0000} \textbf{0.090}}    & {\color[HTML]{FF0000} \textbf{0.190}}   & {\color[HTML]{0000FF} {\ul 0.091}}     & {\color[HTML]{FF0000} \textbf{0.190}} & 0.185                                   & 0.287                                  & 0.146                                   & 0.258                                  & 0.531                                & 0.489                                & 0.295                                  & 0.388                                 & 0.136                                    & 0.251                                   & 0.103                                 & {\color[HTML]{0000FF} {\ul 0.210}}    & 0.162                                   & 0.261                                   & 0.111                                  & 0.221                                 & 0.195                                  & 0.307                                  & 0.129                                  & 0.241                                 \\ \midrule
                          & 12  & {\color[HTML]{FF0000} \textbf{0.052}}    & {\color[HTML]{FF0000} \textbf{0.138}}   & {\color[HTML]{FF0000} \textbf{0.052}}  & {\color[HTML]{FF0000} \textbf{0.138}} & 0.073                                   & 0.184                                  & 0.064                                   & 0.163                                  & 0.109                                & 0.222                                & 0.115                                  & 0.242                                 & 0.090                                    & 0.197                                   & {\color[HTML]{0000FF} {\ul 0.055}}    & {\color[HTML]{0000FF} {\ul 0.145}}    & 0.063                                   & 0.158                                   & 0.067                                  & 0.165                                 & 0.095                                  & 0.207                                  & 0.082                                  & 0.181                                 \\
                          & 24  & {\color[HTML]{FF0000} \textbf{0.064}}    & {\color[HTML]{0000FF} {\ul 0.152}}      & {\color[HTML]{0000FF} {\ul 0.065}}     & {\color[HTML]{FF0000} \textbf{0.151}} & 0.111                                   & 0.219                                  & 0.093                                   & 0.200                                  & 0.230                                & 0.327                                & 0.210                                  & 0.329                                 & 0.110                                    & 0.219                                   & 0.070                                 & 0.164                                 & 0.089                                   & 0.192                                   & 0.088                                  & 0.190                                 & 0.150                                  & 0.262                                  & 0.101                                  & 0.204                                 \\
                          & 48  & {\color[HTML]{FF0000} \textbf{0.082}}    & {\color[HTML]{0000FF} {\ul 0.172}}      & {\color[HTML]{0000FF} {\ul 0.084}}     & {\color[HTML]{FF0000} \textbf{0.171}} & 0.237                                   & 0.328                                  & 0.137                                   & 0.248                                  & 0.551                                & 0.531                                & 0.398                                  & 0.458                                 & 0.149                                    & 0.256                                   & 0.094                                 & 0.192                                 & 0.136                                   & 0.241                                   & 0.110                                  & 0.215                                 & 0.253                                  & 0.340                                  & 0.134                                  & 0.238                                 \\
                          & 96  & {\color[HTML]{FF0000} \textbf{0.107}}    & {\color[HTML]{FF0000} \textbf{0.195}}   & {\color[HTML]{0000FF} {\ul 0.108}}     & {\color[HTML]{0000FF} {\ul 0.196}}    & 0.303                                   & 0.354                                  & 0.198                                   & 0.306                                  & 1.112                                & 0.809                                & 0.594                                  & 0.553                                 & 0.258                                    & 0.359                                   & 0.117                                 & 0.217                                 & 0.197                                   & 0.298                                   & 0.139                                  & 0.245                                 & 0.346                                  & 0.404                                  & 0.181                                  & 0.279                                 \\ \cmidrule(l){2-26} 
\multirow{-5}{*}{\rotatebox[origin=c]{90}{PEMS07}} & Avg & {\color[HTML]{FF0000} \textbf{0.076}}    & {\color[HTML]{FF0000} \textbf{0.164}}   & {\color[HTML]{0000FF} {\ul 0.077}}     & {\color[HTML]{FF0000} \textbf{0.164}} & 0.181                                   & 0.271                                  & 0.123                                   & 0.229                                  & 0.500                                & 0.472                                & 0.329                                  & 0.395                                 & 0.152                                    & 0.258                                   & 0.084                                 & {\color[HTML]{0000FF} {\ul 0.180}}    & 0.121                                   & 0.222                                   & 0.101                                  & 0.204                                 & 0.211                                  & 0.303                                  & 0.124                                  & 0.225                                 \\ \midrule
                          & 12  & {\color[HTML]{FF0000} \textbf{0.067}}    & {\color[HTML]{FF0000} \textbf{0.159}}   & {\color[HTML]{0000FF} {\ul 0.068}}     & {\color[HTML]{FF0000} \textbf{0.159}} & 0.091                                   & 0.201                                  & 0.080                                   & 0.182                                  & 0.122                                & 0.233                                & 0.154                                  & 0.276                                 & 0.119                                    & 0.222                                   & 0.071                                 & {\color[HTML]{0000FF} {\ul 0.171}}    & 0.081                                   & 0.185                                   & 0.079                                  & 0.182                                 & 0.168                                  & 0.232                                  & 0.112                                  & 0.212                                 \\
                          & 24  & {\color[HTML]{FF0000} \textbf{0.086}}    & {\color[HTML]{FF0000} \textbf{0.178}}   & {\color[HTML]{0000FF} {\ul 0.089}}     & {\color[HTML]{FF0000} \textbf{0.178}} & 0.137                                   & 0.246                                  & 0.114                                   & 0.219                                  & 0.236                                & 0.330                                & 0.248                                  & 0.353                                 & 0.149                                    & 0.249                                   & 0.091                                 & {\color[HTML]{0000FF} {\ul 0.189}}    & 0.112                                   & 0.214                                   & 0.115                                  & 0.219                                 & 0.224                                  & 0.281                                  & 0.141                                  & 0.238                                 \\
                          & 48  & {\color[HTML]{FF0000} \textbf{0.119}}    & {\color[HTML]{0000FF} {\ul 0.205}}      & {\color[HTML]{0000FF} {\ul 0.123}}     & {\color[HTML]{FF0000} \textbf{0.204}} & 0.265                                   & 0.343                                  & 0.184                                   & 0.284                                  & 0.562                                & 0.540                                & 0.440                                  & 0.470                                 & 0.206                                    & 0.292                                   & 0.128                                 & 0.219                                 & 0.174                                   & 0.267                                   & 0.186                                  & 0.235                                 & 0.321                                  & 0.354                                  & 0.198                                  & 0.283                                 \\
                          & 96  & {\color[HTML]{FF0000} \textbf{0.173}}    & {\color[HTML]{0000FF} {\ul 0.239}}      & {\color[HTML]{FF0000} \textbf{0.173}}  & {\color[HTML]{FF0000} \textbf{0.236}} & 0.410                                   & 0.407                                  & 0.309                                   & 0.356                                  & 1.216                                & 0.846                                & 0.674                                  & 0.565                                 & 0.329                                    & 0.355                                   & {\color[HTML]{0000FF} {\ul 0.198}}    & 0.266                                 & 0.277                                   & 0.335                                   & 0.221                                  & 0.267                                 & 0.408                                  & 0.417                                  & 0.320                                  & 0.351                                 \\ \cmidrule(l){2-26} 
\multirow{-5}{*}{\rotatebox[origin=c]{90}{PEMS08}} & Avg & {\color[HTML]{FF0000} \textbf{0.111}}    & {\color[HTML]{0000FF} {\ul 0.195}}      & {\color[HTML]{0000FF} {\ul 0.113}}     & {\color[HTML]{FF0000} \textbf{0.194}} & 0.226                                   & 0.299                                  & 0.172                                   & 0.260                                  & 0.534                                & 0.487                                & 0.379                                  & 0.416                                 & 0.200                                    & 0.279                                   & 0.122                                 & 0.211                                 & 0.161                                   & 0.250                                   & 0.150                                  & 0.226                                 & 0.280                                  & 0.321                                  & 0.193                                  & 0.271                                 \\ \midrule
\multicolumn{2}{c|}{1\textsuperscript{st} Count}  & 20                                       & 15                                      & 3                                      & 12                                    & 0                                       & 0                                      & 0                                       & 0                                      & 0                                    & 0                                    & 0                                      & 0                                     & 0                                        & 0                                       & 0                                     & 0                                     & 0                                       & 0                                       & 0                                      & 0                                     & 0                                      & 0                                      & 0                                      & 0                                     \\ \bottomrule
\end{tabular}
}
\caption{Full results of long-term forecasting on the PEMS datasets under the `Input-96-Predict-\{12, 24, 48, 96\}' setting. This table serves as the PEMS-specific extension to Table~\ref{tab_long_term} in the main paper.}
\label{tab_long_term_PEMS}
\end{center}
\end{table*}

\begin{table*}[t]
\begin{center}
\renewcommand{\arraystretch}{1.0}
{\fontsize{6.4}{8}\selectfont
\setlength{\tabcolsep}{1.7pt}
\begin{tabular}{@{}cc|cc|cc|cc|cc|cc|cc|cc|cc|cc|cc|cc|cc@{}}
\toprule
\multicolumn{2}{c|}{Model}        & \multicolumn{2}{c|}{\begin{tabular}[c]{@{}c@{}}vLinear\\      (Ours)\end{tabular}} 
& \multicolumn{2}{c|}{\begin{tabular}[c]{@{}c@{}}OLinear\\     \shortcite{olinear} \end{tabular}} 
& \multicolumn{2}{c|}{\begin{tabular}[c]{@{}c@{}}TimeMixer\\    \shortcite{timemixer} \end{tabular}} 
& \multicolumn{2}{c|}{\begin{tabular}[c]{@{}c@{}}FilterNet\\     \shortcite{filternet}  \end{tabular}} 
& \multicolumn{2}{c|}{\begin{tabular}[c]{@{}c@{}}DLinear\\      \shortcite{linear} \end{tabular}} 
& \multicolumn{2}{c|}{\begin{tabular}[c]{@{}c@{}}TimeMix.++\\    \shortcite{timemixer++} \end{tabular}} 
& \multicolumn{2}{c|}{\begin{tabular}[c]{@{}c@{}}Leddam\\      \shortcite{Leddam_icml} \end{tabular}} 
& \multicolumn{2}{c|}{\begin{tabular}[c]{@{}c@{}}CARD\\      \shortcite{card} \end{tabular}} 
& \multicolumn{2}{c|}{\begin{tabular}[c]{@{}c@{}}Fredformer\\   \shortcite{fredformer} \end{tabular}} 
& \multicolumn{2}{c|}{\begin{tabular}[c]{@{}c@{}}iTrans.\\     \shortcite{itransformer} \end{tabular}} 
& \multicolumn{2}{c|}{\begin{tabular}[c]{@{}c@{}}PatchTST\\     \shortcite{patchtst} \end{tabular}} 
& \multicolumn{2}{c}{\begin{tabular}[c]{@{}c@{}}TimesNet\\      \shortcite{timesnet} \end{tabular}} \\ \midrule
\multicolumn{2}{c|}{Metric}       & MSE                                      & MAE                                     & MSE                                    & MAE                                   & MSE                                      & MAE                                   & MSE                                    & MAE                                     & MSE                                                  & MAE                     & MSE                                       & MAE                                   & MSE                                                   & MAE                   & MSE                                  & MAE                                  & MSE                                       & MAE                                   & MSE                                    & MAE                                   & MSE                                    & MAE                                    & MSE                                    & MAE                                   \\ \midrule
                            & 3   & {\color[HTML]{FF0000} \textbf{0.437}}    & {\color[HTML]{FF0000} \textbf{0.346}}   & {\color[HTML]{0000FF} {\ul 0.468}}     & {\color[HTML]{0000FF} {\ul 0.349}}    & 0.659                                    & 0.435                                 & 0.660                                  & 0.437                                   & 1.280                                                & 0.747                   & 0.658                                     & 0.430                                 & 0.551                                                 & 0.388                 & 0.597                                & 0.392                                & 0.528                                     & 0.403                                 & 0.555                                  & 0.395                                 & 0.646                                  & 0.417                                  & 0.627                                  & 0.420                                 \\
                            & 6   & {\color[HTML]{FF0000} \textbf{0.839}}    & {\color[HTML]{FF0000} \textbf{0.499}}   & {\color[HTML]{0000FF} {\ul 0.923}}     & {\color[HTML]{0000FF} {\ul 0.516}}    & 1.306                                    & 0.643                                 & 1.140                                  & 0.606                                   & 2.054                                                & 0.967                   & 1.273                                     & 0.645                                 & 1.021                                                 & 0.569                 & 1.246                                & 0.610                                & 1.128                                     & 0.604                                 & 1.124                                  & 0.586                                 & 1.269                                  & 0.629                                  & 1.147                                  & 0.610                                 \\
                            & 9   & {\color[HTML]{FF0000} \textbf{1.194}}    & {\color[HTML]{FF0000} \textbf{0.648}}   & {\color[HTML]{0000FF} {\ul 1.289}}     & {\color[HTML]{0000FF} {\ul 0.655}}    & 2.070                                    & 0.842                                 & 1.815                                  & 0.798                                   & 2.771                                                & 1.138                   & 2.009                                     & 0.840                                 & 1.881                                                 & 0.795                 & 2.041                                & 0.829                                & 1.804                                     & 0.786                                 & 1.794                                  & 0.772                                 & 2.021                                  & 0.830                                  & 1.796                                  & 0.785                                 \\
                            & 12  & {\color[HTML]{0000FF} {\ul 1.754}}       & {\color[HTML]{0000FF} {\ul 0.807}}      & {\color[HTML]{FF0000} \textbf{1.698}}  & {\color[HTML]{FF0000} \textbf{0.791}} & 2.792                                    & 1.018                                 & 2.435                                  & 0.953                                   & 3.497                                                & 1.284                   & 2.683                                     & 1.004                                 & 2.421                                                 & 0.964                 & 2.746                                & 0.996                                & 2.610                                     & 0.989                                 & 2.273                                  & 0.884                                 & 2.788                                  & 1.016                                  & 2.349                                  & 0.920                                 \\ \cmidrule(l){2-26} 
                            & Avg & {\color[HTML]{FF0000} \textbf{1.056}}    & {\color[HTML]{FF0000} \textbf{0.575}}   & {\color[HTML]{0000FF} {\ul 1.094}}     & {\color[HTML]{0000FF} {\ul 0.578}}    & 1.707                                    & 0.734                                 & 1.512                                  & 0.698                                   & 2.400                                                & 1.034                   & 1.656                                     & 0.730                                 & 1.468                                                 & 0.679                 & 1.658                                & 0.707                                & 1.518                                     & 0.696                                 & 1.437                                  & 0.659                                 & 1.681                                  & 0.723                                  & 1.480                                  & 0.684                                 \\ \cmidrule(l){2-26} 
                            & 24  & {\color[HTML]{0000FF} {\ul 1.819}}       & 0.853                                   & {\color[HTML]{FF0000} \textbf{1.737}}  & {\color[HTML]{FF0000} \textbf{0.800}} & 2.110                                    & 0.879                                 & 2.190                                  & 0.870                                   & 3.158                                                & 1.243                   & 1.877                                     & {\color[HTML]{0000FF} {\ul 0.826}}    & 2.085                                                 & 0.883                 & 2.407                                & 0.970                                & 2.098                                     & 0.894                                 & 2.004                                  & 0.860                                 & 2.046                                  & 0.849                                  & 2.317                                  & 0.934                                 \\
                            & 36  & {\color[HTML]{FF0000} \textbf{1.667}}    & {\color[HTML]{0000FF} {\ul 0.804}}      & 1.714                                  & {\color[HTML]{FF0000} \textbf{0.795}} & 2.084                                    & 0.890                                 & 1.902                                  & 0.862                                   & 3.009                                                & 1.200                   & 2.276                                     & 0.912                                 & 2.017                                                 & 0.892                 & 2.324                                & 0.948                                & {\color[HTML]{0000FF} {\ul 1.712}}        & 0.867                                 & 1.910                                  & 0.880                                 & 2.344                                  & 0.912                                  & 1.972                                  & 0.920                                 \\
                            & 48  & {\color[HTML]{FF0000} \textbf{1.791}}    & {\color[HTML]{FF0000} \textbf{0.801}}   & {\color[HTML]{0000FF} {\ul 1.821}}     & {\color[HTML]{0000FF} {\ul 0.804}}    & 1.961                                    & 0.866                                 & 2.051                                  & 0.882                                   & 2.994                                                & 1.194                   & 1.921                                     & 0.850                                 & 1.860                                                 & 0.847                 & 2.133                                & 0.911                                & 2.054                                     & 0.922                                 & 2.036                                  & 0.891                                 & 2.123                                  & 0.883                                  & 2.238                                  & 0.940                                 \\
                            & 60  & {\color[HTML]{FF0000} \textbf{1.629}}    & {\color[HTML]{FF0000} \textbf{0.781}}   & 1.785                                  & {\color[HTML]{0000FF} {\ul 0.810}}    & 1.926                                    & 0.878                                 & 2.151                                  & 0.925                                   & 3.172                                                & 1.232                   & {\color[HTML]{0000FF} {\ul 1.745}}        & 0.838                                 & 1.967                                                 & 0.879                 & 2.177                                & 0.921                                & 1.925                                     & 0.913                                 & 2.022                                  & 0.919                                 & 2.001                                  & 0.895                                  & 2.027                                  & 0.928                                 \\ \cmidrule(l){2-26} 
\multirow{-10}{*}{\rotatebox[origin=c]{90}{ILI}}      & Avg & {\color[HTML]{FF0000} \textbf{1.726}}    & {\color[HTML]{0000FF} {\ul 0.810}}      & {\color[HTML]{0000FF} {\ul 1.764}}     & {\color[HTML]{FF0000} \textbf{0.802}} & 2.020                                    & 0.878                                 & 2.073                                  & 0.885                                   & 3.083                                                & 1.217                   & 1.955                                     & 0.857                                 & 1.982                                                 & 0.875                 & 2.260                                & 0.938                                & 1.947                                     & 0.899                                 & 1.993                                  & 0.887                                 & 2.128                                  & 0.885                                  & 2.139                                  & 0.931                                 \\ \midrule
                            & 3   & {\color[HTML]{FF0000} \textbf{1.080}}    & {\color[HTML]{0000FF} {\ul 0.494}}      & {\color[HTML]{0000FF} {\ul 1.100}}     & {\color[HTML]{FF0000} \textbf{0.487}} & 1.237                                    & 0.547                                 & 1.195                                  & 0.555                                   & 2.386                                                & 0.909                   & 1.298                                     & 0.584                                 & 1.216                                                 & 0.570                 & 1.103                                & 0.521                                & 1.165                                     & 0.548                                 & 1.193                                  & 0.561                                 & 1.220                                  & 0.573                                  & 2.021                                  & 0.704                                 \\
                            & 6   & {\color[HTML]{0000FF} {\ul 1.678}}       & {\color[HTML]{0000FF} {\ul 0.644}}      & 1.750                                  & {\color[HTML]{FF0000} \textbf{0.619}} & 2.003                                    & 0.739                                 & 1.839                                  & 0.711                                   & 3.220                                                & 1.053                   & 1.833                                     & 0.682                                 & 1.782                                                 & 0.689                 & 1.919                                & 0.735                                & {\color[HTML]{FF0000} \textbf{1.465}}     & 0.685                                 & 1.933                                  & 0.755                                 & 1.982                                  & 0.762                                  & 2.405                                  & 0.808                                 \\
                            & 9   & {\color[HTML]{0000FF} {\ul 2.164}}       & {\color[HTML]{FF0000} \textbf{0.720}}   & 2.239                                  & {\color[HTML]{0000FF} {\ul 0.734}}    & 2.594                                    & 0.860                                 & 2.537                                  & 0.897                                   & 3.803                                                & 1.160                   & 2.472                                     & 0.822                                 & 2.407                                                 & 0.866                 & 2.358                                & 0.841                                & {\color[HTML]{FF0000} \textbf{2.145}}     & 0.845                                 & 2.441                                  & 0.879                                 & 2.633                                  & 0.916                                  & 2.858                                  & 0.969                                 \\
                            & 12  & {\color[HTML]{FF0000} \textbf{2.500}}    & {\color[HTML]{0000FF} {\ul 0.833}}      & {\color[HTML]{0000FF} {\ul 2.538}}     & {\color[HTML]{FF0000} \textbf{0.831}} & 3.103                                    & 0.981                                 & 2.782                                  & 0.956                                   & 4.524                                                & 1.288                   & 3.273                                     & 1.084                                 & 2.851                                                 & 0.991                 & 2.857                                & 0.971                                & 2.833                                     & 0.984                                 & 2.819                                  & 0.984                                 & 3.050                                  & 1.030                                  & 2.993                                  & 0.964                                 \\ \cmidrule(l){2-26} 
                            & Avg & {\color[HTML]{FF0000} \textbf{1.856}}    & {\color[HTML]{0000FF} {\ul 0.673}}      & 1.907                                  & {\color[HTML]{FF0000} \textbf{0.668}} & 2.234                                    & 0.782                                 & 2.088                                  & 0.780                                   & 3.483                                                & 1.102                   & 2.219                                     & 0.793                                 & 2.064                                                 & 0.779                 & 2.059                                & 0.767                                & {\color[HTML]{0000FF} {\ul 1.902}}        & 0.765                                 & 2.096                                  & 0.795                                 & 2.221                                  & 0.820                                  & 2.569                                  & 0.861                                 \\ \cmidrule(l){2-26} 
                            & 24  & {\color[HTML]{FF0000} \textbf{4.425}}    & {\color[HTML]{0000FF} {\ul 1.214}}      & {\color[HTML]{0000FF} {\ul 4.474}}     & {\color[HTML]{FF0000} \textbf{1.180}} & 6.335                                    & 1.554                                 & 5.926                                  & 1.517                                   & 9.780                                                & 1.851                   & 6.539                                     & 1.618                                 & 4.860                                                 & 1.342                 & 5.133                                & 1.394                                & 4.799                                     & 1.347                                 & 4.715                                  & 1.321                                 & 5.528                                  & 1.450                                  & 5.634                                  & 1.442                                 \\
                            & 36  & {\color[HTML]{FF0000} \textbf{6.563}}    & {\color[HTML]{FF0000} \textbf{1.579}}   & {\color[HTML]{0000FF} {\ul 7.241}}     & {\color[HTML]{0000FF} {\ul 1.670}}    & 8.222                                    & 1.787                                 & 7.696                                  & 1.733                                   & 12.804                                               & 2.083                   & 7.986                                     & 1.770                                 & 7.378                                                 & 1.708                 & 7.377                                & 1.725                                & 7.536                                     & 1.727                                 & 7.299                                  & 1.681                                 & 8.351                                  & 1.830                                  & 9.114                                  & 1.848                                 \\
                            & 48  & {\color[HTML]{FF0000} \textbf{9.372}}    & {\color[HTML]{FF0000} \textbf{1.923}}   & 10.076                                 & 1.985                                 & 11.669                                   & 2.157                                 & 11.572                                 & 2.141                                   & 14.244                                               & 2.189                   & 11.655                                    & 2.156                                 & 10.051                                                & 1.999                 & 11.013                               & 2.103                                & {\color[HTML]{0000FF} {\ul 9.833}}        & {\color[HTML]{0000FF} {\ul 1.951}}    & 10.141                                 & 2.012                                 & 11.259                                 & 2.114                                  & 10.940                                 & 2.033                                 \\
                            & 60  & {\color[HTML]{FF0000} \textbf{10.891}}   & {\color[HTML]{0000FF} {\ul 2.088}}      & 12.079                                 & 2.182                                 & 12.188                                   & 2.173                                 & {\color[HTML]{0000FF} {\ul 11.311}}    & {\color[HTML]{FF0000} \textbf{2.066}}   & 15.472                                               & 2.275                   & 12.734                                    & 2.235                                 & 11.467                                                & 2.119                 & 12.528                               & 2.227                                & 12.455                                    & 2.209                                 & 11.871                                 & 2.156                                 & 12.666                                 & 2.225                                  & 12.888                                 & 2.186                                 \\ \cmidrule(l){2-26} 
\multirow{-10}{*}{\rotatebox[origin=c]{90}{COVID-19}} & Avg & {\color[HTML]{FF0000} \textbf{7.813}}    & {\color[HTML]{FF0000} \textbf{1.701}}   & 8.467                                  & {\color[HTML]{0000FF} {\ul 1.754}}    & 9.604                                    & 1.918                                 & 9.126                                  & 1.864                                   & 13.075                                               & 2.099                   & 9.728                                     & 1.945                                 & {\color[HTML]{0000FF} {\ul 8.439}}                    & 1.792                 & 9.013                                & 1.862                                & 8.656                                     & 1.808                                 & 8.506                                  & 1.792                                 & 9.451                                  & 1.905                                  & 9.644                                  & 1.877                                 \\ \midrule
                            & 3   & {\color[HTML]{FF0000} \textbf{0.202}}    & {\color[HTML]{0000FF} {\ul 0.178}}      & 0.207                                  & {\color[HTML]{FF0000} \textbf{0.171}} & 0.205                                    & 0.192                                 & {\color[HTML]{0000FF} {\ul 0.204}}     & 0.189                                   & 0.218                                                & 0.231                   & {\color[HTML]{0000FF} {\ul 0.204}}        & 0.189                                 & {\color[HTML]{0000FF} {\ul 0.204}}                    & 0.191                 & 0.210                                & 0.180                                & 0.205                                     & 0.188                                 & 0.205                                  & 0.188                                 & {\color[HTML]{0000FF} {\ul 0.204}}     & 0.190                                  & 0.221                                  & 0.204                                 \\
                            & 6   & {\color[HTML]{FF0000} \textbf{0.292}}    & 0.220                                   & 0.301                                  & {\color[HTML]{FF0000} \textbf{0.207}} & 0.297                                    & 0.230                                 & 0.296                                  & 0.228                                   & 0.307                                                & 0.278                   & 0.298                                     & 0.230                                 & {\color[HTML]{0000FF} {\ul 0.293}}                    & 0.227                 & 0.311                                & {\color[HTML]{0000FF} {\ul 0.219}}   & 0.298                                     & 0.227                                 & 0.300                                  & 0.229                                 & 0.298                                  & 0.227                                  & 0.308                                  & 0.238                                 \\
                            & 9   & {\color[HTML]{0000FF} {\ul 0.371}}       & {\color[HTML]{0000FF} {\ul 0.249}}      & 0.382                                  & {\color[HTML]{FF0000} \textbf{0.238}} & 0.381                                    & 0.264                                 & 0.377                                  & 0.261                                   & 0.386                                                & 0.316                   & 0.384                                     & 0.266                                 & {\color[HTML]{FF0000} \textbf{0.369}}                 & 0.264                 & 0.401                                & 0.253                                & 0.385                                     & 0.263                                 & 0.386                                  & 0.265                                 & 0.382                                  & 0.263                                  & 0.387                                  & 0.273                                 \\
                            & 12  & {\color[HTML]{FF0000} \textbf{0.441}}    & {\color[HTML]{0000FF} {\ul 0.276}}      & 0.452                                  & {\color[HTML]{FF0000} \textbf{0.263}} & 0.455                                    & 0.295                                 & 0.449                                  & 0.290                                   & 0.452                                                & 0.353                   & 0.456                                     & 0.294                                 & {\color[HTML]{0000FF} {\ul 0.442}}                    & 0.292                 & 0.474                                & 0.281                                & 0.457                                     & 0.292                                 & 0.460                                  & 0.295                                 & 0.456                                  & 0.292                                  & 0.462                                  & 0.298                                 \\ \cmidrule(l){2-26} 
                            & Avg & {\color[HTML]{FF0000} \textbf{0.326}}    & {\color[HTML]{0000FF} {\ul 0.231}}      & 0.335                                  & {\color[HTML]{FF0000} \textbf{0.220}} & 0.334                                    & 0.245                                 & 0.331                                  & 0.242                                   & 0.341                                                & 0.294                   & 0.335                                     & 0.245                                 & {\color[HTML]{0000FF} {\ul 0.327}}                    & 0.243                 & 0.349                                & 0.233                                & 0.336                                     & 0.242                                 & 0.338                                  & 0.244                                 & 0.335                                  & 0.243                                  & 0.344                                  & 0.253                                 \\ \cmidrule(l){2-26} 
                            & 24  & {\color[HTML]{0000FF} {\ul 0.625}}       & {\color[HTML]{0000FF} {\ul 0.346}}      & 0.650                                  & {\color[HTML]{FF0000} \textbf{0.337}} & 0.671                                    & 0.413                                 & 0.670                                  & 0.402                                   & 0.645                                                & 0.458                   & {\color[HTML]{FF0000} \textbf{0.617}}     & 0.394                                 & 0.680                                                 & 0.405                 & 0.700                                & 0.378                                & 0.676                                     & 0.408                                 & 0.700                                  & 0.413                                 & 0.679                                  & 0.410                                  & 0.698                                  & 0.415                                 \\
                            & 36  & {\color[HTML]{FF0000} \textbf{0.771}}    & {\color[HTML]{0000FF} {\ul 0.394}}      & 0.800                                  & {\color[HTML]{FF0000} \textbf{0.388}} & 0.841                                    & 0.480                                 & 0.824                                  & 0.471                                   & 0.785                                                & 0.533                   & {\color[HTML]{0000FF} {\ul 0.781}}        & 0.457                                 & 0.841                                                 & 0.471                 & 0.874                                & 0.448                                & 0.852                                     & 0.477                                 & 0.867                                  & 0.480                                 & 0.845                                  & 0.484                                  & 0.856                                  & 0.475                                 \\
                            & 48  & 0.886                                    & {\color[HTML]{0000FF} {\ul 0.436}}      & 0.905                                  & {\color[HTML]{FF0000} \textbf{0.427}} & 0.964                                    & 0.531                                 & 0.955                                  & 0.521                                   & {\color[HTML]{0000FF} {\ul 0.885}}                   & 0.585                   & {\color[HTML]{FF0000} \textbf{0.842}}     & 0.520                                 & 0.963                                                 & 0.528                 & 1.017                                & 0.498                                & 0.982                                     & 0.526                                 & 1.017                                  & 0.539                                 & 0.972                                  & 0.536                                  & 0.972                                  & 0.518                                 \\
                            & 60  & 0.980                                    & {\color[HTML]{0000FF} {\ul 0.464}}      & 0.999                                  & {\color[HTML]{FF0000} \textbf{0.457}} & 1.047                                    & 0.573                                 & 1.050                                  & 0.563                                   & {\color[HTML]{0000FF} {\ul 0.959}}                   & 0.623                   & {\color[HTML]{FF0000} \textbf{0.958}}     & 0.551                                 & 1.029                                                 & 0.556                 & 1.126                                & 0.541                                & 1.084                                     & 0.569                                 & 1.079                                  & 0.572                                 & 1.077                                  & 0.578                                  & 1.033                                  & 0.543                                 \\ \cmidrule(l){2-26} 
\multirow{-10}{*}{\rotatebox[origin=c]{90}{METR-LA}}  & Avg & {\color[HTML]{0000FF} {\ul 0.815}}       & {\color[HTML]{0000FF} {\ul 0.410}}      & 0.838                                  & {\color[HTML]{FF0000} \textbf{0.402}} & 0.881                                    & 0.499                                 & 0.875                                  & 0.489                                   & 0.819                                                & 0.550                   & {\color[HTML]{FF0000} \textbf{0.799}}     & 0.480                                 & 0.878                                                 & 0.490                 & 0.929                                & 0.466                                & 0.898                                     & 0.495                                 & 0.916                                  & 0.501                                 & 0.893                                  & 0.502                                  & 0.890                                  & 0.488                                 \\ \midrule
                            & 3   & {\color[HTML]{FF0000} \textbf{0.035}}    & {\color[HTML]{FF0000} \textbf{0.092}}   & {\color[HTML]{0000FF} {\ul 0.036}}     & {\color[HTML]{FF0000} \textbf{0.092}} & {\color[HTML]{FF0000} \textbf{0.035}}    & {\color[HTML]{0000FF} {\ul 0.093}}    & 0.038                                  & 0.100                                   & 0.044                                                & 0.123                   & {\color[HTML]{0000FF} {\ul 0.036}}        & 0.098                                 & 0.040                                                 & 0.103                 & 0.037                                & 0.096                                & 0.039                                     & 0.102                                 & 0.040                                  & 0.105                                 & 0.038                                  & 0.099                                  & 0.049                                  & 0.123                                 \\
                            & 6   & {\color[HTML]{FF0000} \textbf{0.049}}    & {\color[HTML]{0000FF} {\ul 0.118}}      & {\color[HTML]{FF0000} \textbf{0.049}}  & {\color[HTML]{FF0000} \textbf{0.117}} & {\color[HTML]{FF0000} \textbf{0.049}}    & {\color[HTML]{0000FF} {\ul 0.118}}    & 0.052                                  & 0.126                                   & 0.062                                                & 0.155                   & {\color[HTML]{0000FF} {\ul 0.050}}        & 0.121                                 & 0.054                                                 & 0.128                 & 0.052                                & 0.123                                & 0.053                                     & 0.126                                 & 0.054                                  & 0.129                                 & 0.053                                  & 0.124                                  & 0.061                                  & 0.142                                 \\
                            & 9   & {\color[HTML]{FF0000} \textbf{0.061}}    & {\color[HTML]{FF0000} \textbf{0.137}}   & {\color[HTML]{0000FF} {\ul 0.062}}     & {\color[HTML]{FF0000} \textbf{0.137}} & {\color[HTML]{0000FF} {\ul 0.062}}       & {\color[HTML]{0000FF} {\ul 0.139}}    & 0.065                                  & 0.145                                   & 0.082                                                & 0.189                   & 0.063                                     & 0.144                                 & 0.066                                                 & 0.147                 & 0.063                                & 0.141                                & 0.067                                     & 0.147                                 & 0.068                                  & 0.150                                 & 0.065                                  & 0.145                                  & 0.073                                  & 0.161                                 \\
                            & 12  & {\color[HTML]{FF0000} \textbf{0.073}}    & {\color[HTML]{FF0000} \textbf{0.154}}   & {\color[HTML]{FF0000} \textbf{0.073}}  & {\color[HTML]{FF0000} \textbf{0.154}} & {\color[HTML]{FF0000} \textbf{0.073}}    & {\color[HTML]{0000FF} {\ul 0.156}}    & 0.076                                  & 0.162                                   & 0.100                                                & 0.215                   & {\color[HTML]{0000FF} {\ul 0.075}}        & 0.161                                 & 0.078                                                 & 0.164                 & {\color[HTML]{0000FF} {\ul 0.075}}   & 0.158                                & 0.079                                     & 0.165                                 & 0.078                                  & 0.165                                 & 0.077                                  & 0.161                                  & 0.088                                  & 0.179                                 \\ \cmidrule(l){2-26} 
                            & Avg & {\color[HTML]{FF0000} \textbf{0.054}}    & {\color[HTML]{FF0000} \textbf{0.125}}   & {\color[HTML]{0000FF} {\ul 0.055}}     & {\color[HTML]{FF0000} \textbf{0.125}} & {\color[HTML]{0000FF} {\ul 0.055}}       & {\color[HTML]{0000FF} {\ul 0.126}}    & 0.058                                  & 0.133                                   & 0.072                                                & 0.170                   & 0.056                                     & 0.131                                 & 0.059                                                 & 0.135                 & 0.057                                & 0.130                                & 0.059                                     & 0.135                                 & 0.060                                  & 0.137                                 & 0.058                                  & 0.132                                  & 0.068                                  & 0.151                                 \\ \cmidrule(l){2-26} 
                            & 24  & {\color[HTML]{FF0000} \textbf{0.120}}    & {\color[HTML]{FF0000} \textbf{0.215}}   & {\color[HTML]{0000FF} {\ul 0.121}}     & {\color[HTML]{0000FF} {\ul 0.216}}    & 0.122                                    & 0.221                                 & 0.130                                  & 0.230                                   & 0.155                                                & 0.274                   & 0.132                                     & 0.233                                 & 0.125                                                 & 0.222                 & 0.124                                & 0.220                                & 0.128                                     & 0.226                                 & 0.137                                  & 0.237                                 & 0.127                                  & 0.224                                  & 0.198                                  & 0.299                                 \\
                            & 36  & {\color[HTML]{FF0000} \textbf{0.162}}    & {\color[HTML]{FF0000} \textbf{0.259}}   & {\color[HTML]{0000FF} {\ul 0.163}}     & {\color[HTML]{0000FF} {\ul 0.261}}    & 0.183                                    & 0.279                                 & 0.175                                  & 0.273                                   & 0.196                                                & 0.306                   & 0.177                                     & 0.278                                 & 0.174                                                 & 0.271                 & 0.167                                & 0.266                                & 0.170                                     & 0.268                                 & 0.184                                  & 0.280                                 & 0.174                                  & 0.269                                  & 0.229                                  & 0.326                                 \\
                            & 48  & {\color[HTML]{FF0000} \textbf{0.200}}    & {\color[HTML]{FF0000} \textbf{0.291}}   & {\color[HTML]{0000FF} {\ul 0.205}}     & {\color[HTML]{0000FF} {\ul 0.296}}    & {\color[HTML]{FF0000} \textbf{0.200}}    & 0.298                                 & 0.224                                  & 0.314                                   & 0.244                                                & 0.344                   & 0.216                                     & 0.311                                 & 0.222                                                 & 0.312                 & 0.218                                & 0.307                                & 0.218                                     & 0.306                                 & 0.229                                  & 0.318                                 & 0.225                                  & 0.314                                  & 0.267                                  & 0.352                                 \\
                            & 60  & {\color[HTML]{0000FF} {\ul 0.244}}       & {\color[HTML]{FF0000} \textbf{0.327}}   & 0.259                                  & 0.336                                 & {\color[HTML]{FF0000} \textbf{0.238}}    & {\color[HTML]{0000FF} {\ul 0.328}}    & 0.259                                  & 0.340                                   & 0.318                                                & 0.401                   & 0.249                                     & 0.337                                 & 0.264                                                 & 0.341                 & 0.264                                & 0.341                                & 0.262                                     & 0.339                                 & 0.279                                  & 0.352                                 & 0.265                                  & 0.339                                  & 0.327                                  & 0.394                                 \\ \cmidrule(l){2-26} 
\multirow{-10}{*}{\rotatebox[origin=c]{90}{NASDAQ}}   & Avg & {\color[HTML]{FF0000} \textbf{0.181}}    & {\color[HTML]{FF0000} \textbf{0.273}}   & 0.187                                  & {\color[HTML]{0000FF} {\ul 0.277}}    & {\color[HTML]{0000FF} {\ul 0.186}}       & 0.281                                 & 0.197                                  & 0.289                                   & 0.228                                                & 0.331                   & 0.193                                     & 0.290                                 & 0.196                                                 & 0.286                 & 0.193                                & 0.284                                & 0.194                                     & 0.285                                 & 0.207                                  & 0.297                                 & 0.198                                  & 0.286                                  & 0.255                                  & 0.343                                 \\ \midrule
                            & 3   & {\color[HTML]{0000FF} {\ul 6.135}}       & {\color[HTML]{0000FF} {\ul 0.375}}      & 6.161                                  & {\color[HTML]{FF0000} \textbf{0.368}} & 6.209                                    & 0.392                                 & 6.234                                  & 0.402                                   & 6.254                                                & 0.438                   & 6.149                                     & 0.389                                 & 6.148                                                 & 0.383                 & 6.183                                & 0.378                                & 6.190                                     & 0.387                                 & 6.237                                  & 0.393                                 & {\color[HTML]{FF0000} \textbf{6.112}}  & 0.380                                  & 7.597                                  & 0.510                                 \\
                            & 6   & {\color[HTML]{0000FF} {\ul 6.434}}       & {\color[HTML]{0000FF} {\ul 0.386}}      & 6.453                                  & {\color[HTML]{FF0000} \textbf{0.385}} & 6.475                                    & 0.402                                 & 6.460                                  & 0.401                                   & 6.579                                                & 0.467                   & 6.436                                     & 0.401                                 & 6.455                                                 & 0.397                 & 6.465                                & 0.393                                & 6.696                                     & 0.404                                 & 6.484                                  & 0.400                                 & {\color[HTML]{FF0000} \textbf{6.425}}  & 0.395                                  & 7.962                                  & 0.515                                 \\
                            & 9   & {\color[HTML]{FF0000} \textbf{6.657}}    & {\color[HTML]{0000FF} {\ul 0.400}}      & {\color[HTML]{0000FF} {\ul 6.666}}     & {\color[HTML]{FF0000} \textbf{0.398}} & 6.702                                    & 0.418                                 & 6.697                                  & 0.416                                   & 6.776                                                & 0.508                   & 6.714                                     & 0.420                                 & 6.687                                                 & 0.412                 & 6.714                                & 0.415                                & 6.768                                     & 0.411                                 & 6.689                                  & 0.411                                 & 6.743                                  & 0.426                                  & 8.150                                  & 0.524                                 \\
                            & 12  & {\color[HTML]{0000FF} {\ul 6.825}}       & {\color[HTML]{0000FF} {\ul 0.407}}      & 6.834                                  & {\color[HTML]{FF0000} \textbf{0.406}} & 6.902                                    & 0.426                                 & 6.899                                  & 0.426                                   & 6.927                                                & 0.513                   & 6.852                                     & 0.421                                 & 6.899                                                 & 0.424                 & 6.852                                & 0.415                                & 7.168                                     & 0.424                                 & 6.868                                  & 0.419                                 & {\color[HTML]{FF0000} \textbf{6.814}}  & 0.414                                  & 8.117                                  & 0.533                                 \\ \cmidrule(l){2-26} 
                            & Avg & {\color[HTML]{FF0000} \textbf{6.513}}    & {\color[HTML]{0000FF} {\ul 0.392}}      & 6.528                                  & {\color[HTML]{FF0000} \textbf{0.389}} & 6.572                                    & 0.409                                 & 6.572                                  & 0.411                                   & 6.634                                                & 0.481                   & 6.538                                     & 0.408                                 & 6.547                                                 & 0.404                 & 6.553                                & 0.400                                & 6.705                                     & 0.406                                 & 6.569                                  & 0.405                                 & {\color[HTML]{0000FF} {\ul 6.523}}     & 0.404                                  & 7.956                                  & 0.520                                 \\ \cmidrule(l){2-26} 
                            & 24  & 6.879                                    & {\color[HTML]{0000FF} {\ul 0.426}}      & 6.894                                  & {\color[HTML]{FF0000} \textbf{0.423}} & 6.900                                    & 0.446                                 & 6.907                                  & 0.443                                   & 6.883                                                & 0.520                   & 6.902                                     & 0.460                                 & 6.919                                                 & 0.450                 & 6.925                                & 0.440                                & {\color[HTML]{FF0000} \textbf{6.531}}     & 0.432                                 & 6.886                                  & 0.437                                 & {\color[HTML]{0000FF} {\ul 6.858}}     & 0.430                                  & 8.023                                  & 0.612                                 \\
                            & 36  & 6.421                                    & {\color[HTML]{0000FF} {\ul 0.440}}      & 6.446                                  & {\color[HTML]{FF0000} \textbf{0.439}} & 6.520                                    & 0.467                                 & 6.514                                  & 0.467                                   & 6.393                                                & 0.538                   & 6.539                                     & 0.473                                 & 6.456                                                 & 0.457                 & 6.463                                & 0.451                                & {\color[HTML]{FF0000} \textbf{5.935}}     & 0.453                                 & 6.431                                  & 0.452                                 & {\color[HTML]{0000FF} {\ul 6.400}}     & 0.445                                  & 7.229                                  & 0.595                                 \\
                            & 48  & 5.997                                    & 0.452                                   & 6.004                                  & {\color[HTML]{FF0000} \textbf{0.446}} & 6.108                                    & 0.484                                 & 6.135                                  & 0.478                                   & {\color[HTML]{0000FF} {\ul 5.940}}                   & 0.547                   & 6.115                                     & 0.487                                 & 6.031                                                 & 0.468                 & 6.031                                & 0.460                                & {\color[HTML]{FF0000} \textbf{5.871}}     & 0.464                                 & 6.101                                  & 0.483                                 & 5.959                                  & {\color[HTML]{0000FF} {\ul 0.449}}     & 7.184                                  & 0.641                                 \\
                            & 60  & 5.681                                    & 0.459                                   & 5.705                                  & {\color[HTML]{0000FF} {\ul 0.454}}    & 5.732                                    & 0.476                                 & 5.811                                  & 0.482                                   & {\color[HTML]{0000FF} {\ul 5.605}}                   & 0.552                   & 5.736                                     & 0.497                                 & 5.740                                                 & 0.478                 & 5.723                                & 0.463                                & {\color[HTML]{FF0000} \textbf{5.389}}     & 0.463                                 & 5.681                                  & 0.462                                 & 5.633                                  & {\color[HTML]{FF0000} \textbf{0.452}}  & 6.805                                  & 0.645                                 \\ \cmidrule(l){2-26} 
\multirow{-10}{*}{\rotatebox[origin=c]{90}{Wiki}}     & Avg & 6.244                                    & {\color[HTML]{0000FF} {\ul 0.444}}      & 6.262                                  & {\color[HTML]{FF0000} \textbf{0.440}} & 6.315                                    & 0.468                                 & 6.342                                  & 0.467                                   & {\color[HTML]{0000FF} {\ul 6.205}}                   & 0.539                   & 6.323                                     & 0.479                                 & 6.286                                                 & 0.463                 & 6.285                                & 0.453                                & {\color[HTML]{FF0000} \textbf{5.931}}     & 0.453                                 & 6.275                                  & 0.458                                 & 6.212                                  & {\color[HTML]{0000FF} {\ul 0.444}}     & 7.310                                  & 0.623                                 \\ \midrule
\multicolumn{2}{c|}{1\textsuperscript{st} Count}    & 32                                       & 19                                      & 4                                      & 33                                    & 5                                        & 0                                     & 0                                      & 1                                       & 0                                                    & 0                       & 4                                         & 0                                     & 1                                                     & 0                     & 0                                    & 0                                    & 7                                         & 0                                     & 0                                      & 0                                     & 3                                      & 1                                      & 0                                      & 0                                     \\ \midrule
\multicolumn{2}{c|}{Top-2 Count}  & 43                                       & 46                                      & 20                                     & 47                                    & 8                                        & 6                                     & 2                                      & 1                                       & 5                                                    & 0                       & 10                                        & 1                                     & 6                                                     & 0                     & 1                                    & 1                                    & 10                                        & 1                                     & 0                                      & 0                                     & 7                                      & 3                                      & 0                                      & 0                                     \\ \bottomrule
\end{tabular}
}
\caption{Full results of short-term forecasting under the `Input-12-Predict-\{3, 6, 9, 12\}' and `Input-36-Predict-\{24, 36, 48, 60\}' settings (Part 1). This table corresponds to Table~\ref{tab_short_term} in the main paper.}
\label{tab_short_term_p1}
\end{center}
\end{table*}

\begin{table*}[t]
\begin{center}
\renewcommand{\arraystretch}{1.3}
{\fontsize{6.4}{8}\selectfont
\setlength{\tabcolsep}{1.7pt}
\begin{tabular}{@{}cc|cc|cc|cc|cc|cc|cc|cc|cc|cc|cc|cc|cc@{}}
\toprule
\multicolumn{2}{c|}{Model}        & \multicolumn{2}{c|}{\begin{tabular}[c]{@{}c@{}}vLinear\\      (Ours)\end{tabular}} 
& \multicolumn{2}{c|}{\begin{tabular}[c]{@{}c@{}}OLinear\\     \shortcite{olinear} \end{tabular}} 
& \multicolumn{2}{c|}{\begin{tabular}[c]{@{}c@{}}TimeMixer\\    \shortcite{timemixer} \end{tabular}} 
& \multicolumn{2}{c|}{\begin{tabular}[c]{@{}c@{}}FilterNet\\     \shortcite{filternet}  \end{tabular}} 
& \multicolumn{2}{c|}{\begin{tabular}[c]{@{}c@{}}DLinear\\      \shortcite{linear} \end{tabular}} 
& \multicolumn{2}{c|}{\begin{tabular}[c]{@{}c@{}}TimeMix.++\\    \shortcite{timemixer++} \end{tabular}} 
& \multicolumn{2}{c|}{\begin{tabular}[c]{@{}c@{}}Leddam\\      \shortcite{Leddam_icml} \end{tabular}} 
& \multicolumn{2}{c|}{\begin{tabular}[c]{@{}c@{}}CARD\\      \shortcite{card} \end{tabular}} 
& \multicolumn{2}{c|}{\begin{tabular}[c]{@{}c@{}}Fredformer\\   \shortcite{fredformer} \end{tabular}} 
& \multicolumn{2}{c|}{\begin{tabular}[c]{@{}c@{}}iTrans.\\     \shortcite{itransformer} \end{tabular}} 
& \multicolumn{2}{c|}{\begin{tabular}[c]{@{}c@{}}PatchTST\\     \shortcite{patchtst} \end{tabular}} 
& \multicolumn{2}{c}{\begin{tabular}[c]{@{}c@{}}TimesNet\\      \shortcite{timesnet} \end{tabular}} \\ \midrule
\multicolumn{2}{c|}{Metric}       & MSE                                      & MAE                                     & MSE                                    & MAE                                   & MSE                                     & MAE                                    & MSE                                      & MAE                                   & MSE                                    & MAE                                   & MSE                                      & MAE                                     & MSE                                                  & MAE                    & MSE                                    & MAE                                   & MSE                                                    & MAE                      & MSE                                     & MAE                                  & MSE                                      & MAE                                  & MSE                                                  & MAE                     \\ \midrule
                            & 3   & {\color[HTML]{FF0000} \textbf{0.035}}    & {\color[HTML]{FF0000} \textbf{0.125}}   & {\color[HTML]{FF0000} \textbf{0.035}}  & {\color[HTML]{0000FF} {\ul 0.126}}    & 0.038                                   & 0.137                                  & 0.046                                    & 0.154                                 & 0.047                                  & 0.152                                 & 0.040                                    & 0.141                                   & 0.047                                                & 0.155                  & {\color[HTML]{0000FF} {\ul 0.036}}     & 0.130                                 & 0.045                                                  & 0.153                    & 0.046                                   & 0.155                                & 0.042                                    & 0.146                                & 0.037                                                & 0.133                   \\
                            & 6   & {\color[HTML]{FF0000} \textbf{0.053}}    & {\color[HTML]{FF0000} \textbf{0.158}}   & {\color[HTML]{FF0000} \textbf{0.053}}  & {\color[HTML]{FF0000} \textbf{0.158}} & 0.056                                   & 0.167                                  & 0.071                                    & 0.193                                 & 0.070                                  & 0.197                                 & 0.057                                    & 0.168                                   & 0.064                                                & 0.181                  & {\color[HTML]{0000FF} {\ul 0.054}}     & {\color[HTML]{0000FF} {\ul 0.162}}    & 0.066                                                  & 0.187                    & 0.066                                   & 0.187                                & 0.063                                    & 0.183                                & 0.055                                                & 0.165                   \\
                            & 9   & {\color[HTML]{FF0000} \textbf{0.070}}    & {\color[HTML]{0000FF} {\ul 0.183}}      & {\color[HTML]{FF0000} \textbf{0.070}}  & {\color[HTML]{FF0000} \textbf{0.181}} & 0.074                                   & 0.192                                  & 0.080                                    & 0.203                                 & 0.091                                  & 0.220                                 & 0.076                                    & 0.196                                   & 0.081                                                & 0.206                  & 0.072                                  & 0.184                                 & 0.080                                                  & 0.205                    & 0.080                                   & 0.203                                & 0.081                                    & 0.207                                & {\color[HTML]{0000FF} {\ul 0.071}}                   & 0.187                   \\
                            & 12  & {\color[HTML]{FF0000} \textbf{0.088}}    & {\color[HTML]{FF0000} \textbf{0.204}}   & {\color[HTML]{FF0000} \textbf{0.088}}  & {\color[HTML]{FF0000} \textbf{0.204}} & 0.092                                   & 0.213                                  & 0.097                                    & 0.224                                 & 0.113                                  & 0.251                                 & 0.093                                    & 0.213                                   & 0.104                                                & 0.236                  & 0.089                                  & {\color[HTML]{0000FF} {\ul 0.206}}    & 0.103                                                  & 0.234                    & 0.099                                   & 0.228                                & 0.101                                    & 0.233                                & {\color[HTML]{0000FF} {\ul 0.089}}                   & 0.210                   \\ \cmidrule(l){2-26} 
                            & Avg & {\color[HTML]{FF0000} \textbf{0.061}}    & {\color[HTML]{0000FF} {\ul 0.168}}      & {\color[HTML]{FF0000} \textbf{0.061}}  & {\color[HTML]{FF0000} \textbf{0.167}} & 0.065                                   & 0.177                                  & 0.073                                    & 0.193                                 & 0.080                                  & 0.205                                 & 0.066                                    & 0.179                                   & 0.074                                                & 0.194                  & {\color[HTML]{0000FF} {\ul 0.063}}     & 0.170                                 & 0.073                                                  & 0.195                    & 0.073                                   & 0.193                                & 0.072                                    & 0.192                                & {\color[HTML]{0000FF} {\ul 0.063}}                   & 0.174                   \\ \cmidrule(l){2-26} 
                            & 24  & {\color[HTML]{FF0000} \textbf{0.154}}    & {\color[HTML]{0000FF} {\ul 0.273}}      & {\color[HTML]{0000FF} {\ul 0.155}}     & {\color[HTML]{FF0000} \textbf{0.271}} & 0.159                                   & 0.288                                  & 0.181                                    & 0.317                                 & 0.189                                  & 0.330                                 & 0.172                                    & 0.305                                   & 0.175                                                & 0.308                  & 0.156                                  & 0.276                                 & 0.181                                                  & 0.315                    & 0.180                                   & 0.309                                & 0.164                                    & 0.298                                & 0.162                                                & 0.291                   \\
                            & 36  & {\color[HTML]{FF0000} \textbf{0.205}}    & {\color[HTML]{FF0000} \textbf{0.317}}   & 0.209                                  & {\color[HTML]{FF0000} \textbf{0.317}} & 0.218                                   & 0.343                                  & 0.224                                    & 0.341                                 & 0.250                                  & 0.363                                 & 0.227                                    & 0.344                                   & 0.232                                                & 0.358                  & {\color[HTML]{0000FF} {\ul 0.206}}     & {\color[HTML]{0000FF} {\ul 0.319}}    & 0.239                                                  & 0.365                    & 0.225                                   & 0.346                                & 0.221                                    & 0.341                                & 0.219                                                & 0.344                   \\
                            & 48  & {\color[HTML]{FF0000} \textbf{0.253}}    & {\color[HTML]{FF0000} \textbf{0.354}}   & {\color[HTML]{0000FF} {\ul 0.258}}     & {\color[HTML]{0000FF} {\ul 0.358}}    & 0.264                                   & 0.367                                  & 0.280                                    & 0.384                                 & 0.291                                  & 0.398                                 & 0.272                                    & 0.383                                   & 0.276                                                & 0.388                  & {\color[HTML]{0000FF} {\ul 0.258}}     & {\color[HTML]{FF0000} \textbf{0.354}} & 0.283                                                  & 0.394                    & 0.275                                   & 0.383                                & 0.278                                    & 0.397                                & 0.262                                                & 0.371                   \\
                            & 60  & {\color[HTML]{FF0000} \textbf{0.295}}    & {\color[HTML]{FF0000} \textbf{0.381}}   & 0.305                                  & 0.387                                 & 0.322                                   & 0.416                                  & 0.332                                    & 0.416                                 & 0.377                                  & 0.475                                 & 0.319                                    & 0.413                                   & 0.325                                                & 0.423                  & {\color[HTML]{0000FF} {\ul 0.303}}     & {\color[HTML]{0000FF} {\ul 0.385}}    & 0.341                                                  & 0.438                    & 0.322                                   & 0.418                                & 0.321                                    & 0.409                                & 0.305                                                & 0.399                   \\ \cmidrule(l){2-26} 
\multirow{-10}{*}{\rotatebox[origin=c]{90}{SP500}}    & Avg & {\color[HTML]{FF0000} \textbf{0.227}}    & {\color[HTML]{FF0000} \textbf{0.331}}   & {\color[HTML]{0000FF} {\ul 0.231}}     & {\color[HTML]{0000FF} {\ul 0.333}}    & 0.241                                   & 0.353                                  & 0.254                                    & 0.365                                 & 0.277                                  & 0.391                                 & 0.247                                    & 0.361                                   & 0.252                                                & 0.369                  & {\color[HTML]{0000FF} {\ul 0.231}}     & {\color[HTML]{0000FF} {\ul 0.333}}    & 0.261                                                  & 0.378                    & 0.250                                   & 0.364                                & 0.246                                    & 0.361                                & 0.237                                                & 0.351                   \\ \midrule
                            & 3   & {\color[HTML]{FF0000} \textbf{1.532}}    & {\color[HTML]{FF0000} \textbf{0.275}}   & 1.550                                  & {\color[HTML]{0000FF} {\ul 0.276}}    & 1.568                                   & {\color[HTML]{0000FF} {\ul 0.276}}     & 1.587                                    & 0.278                                 & 1.574                                  & 0.280                                 & 1.555                                    & 0.278                                   & {\color[HTML]{0000FF} {\ul 1.538}}                   & 0.277                  & 1.563                                  & {\color[HTML]{0000FF} {\ul 0.276}}    & 1.563                                                  & 0.277                    & 1.551                                   & {\color[HTML]{0000FF} {\ul 0.276}}   & 1.556                                    & {\color[HTML]{0000FF} {\ul 0.276}}   & 2.269                                                & 0.325                   \\
                            & 6   & {\color[HTML]{FF0000} \textbf{2.559}}    & {\color[HTML]{FF0000} \textbf{0.361}}   & {\color[HTML]{0000FF} {\ul 2.569}}     & {\color[HTML]{FF0000} \textbf{0.361}} & 2.593                                   & {\color[HTML]{0000FF} {\ul 0.362}}     & 2.634                                    & 0.364                                 & 2.613                                  & 0.367                                 & 2.591                                    & 0.364                                   & 2.588                                                & 0.364                  & 2.579                                  & {\color[HTML]{0000FF} {\ul 0.362}}    & 2.619                                                  & 0.364                    & 2.589                                   & 0.363                                & 2.588                                    & {\color[HTML]{0000FF} {\ul 0.362}}   & 3.934                                                & 0.430                   \\
                            & 9   & {\color[HTML]{FF0000} \textbf{3.535}}    & {\color[HTML]{FF0000} \textbf{0.430}}   & 3.565                                  & {\color[HTML]{FF0000} \textbf{0.430}} & {\color[HTML]{0000FF} {\ul 3.564}}      & {\color[HTML]{0000FF} {\ul 0.431}}     & 3.651                                    & 0.434                                 & 3.621                                  & 0.436                                 & 3.572                                    & 0.433                                   & 3.644                                                & 0.440                  & 3.569                                  & {\color[HTML]{FF0000} \textbf{0.430}} & 3.575                                                  & 0.432                    & 3.599                                   & 0.433                                & 3.585                                    & 0.432                                & 4.715                                                & 0.492                   \\
                            & 12  & {\color[HTML]{0000FF} {\ul 4.477}}       & {\color[HTML]{0000FF} {\ul 0.492}}      & 4.517                                  & {\color[HTML]{FF0000} \textbf{0.490}} & {\color[HTML]{FF0000} \textbf{4.472}}   & {\color[HTML]{FF0000} \textbf{0.490}}  & 4.634                                    & 0.496                                 & 4.586                                  & 0.497                                 & 4.630                                    & 0.496                                   & 4.622                                                & 0.498                  & 4.528                                  & {\color[HTML]{FF0000} \textbf{0.490}} & 4.605                                                  & 0.493                    & 4.579                                   & 0.493                                & 4.533                                    & {\color[HTML]{0000FF} {\ul 0.492}}   & 5.999                                                & 0.561                   \\ \cmidrule(l){2-26} 
                            & Avg & {\color[HTML]{FF0000} \textbf{3.025}}    & {\color[HTML]{0000FF} {\ul 0.390}}      & 3.050                                  & {\color[HTML]{FF0000} \textbf{0.389}} & {\color[HTML]{0000FF} {\ul 3.049}}      & {\color[HTML]{0000FF} {\ul 0.390}}     & 3.126                                    & 0.393                                 & 3.099                                  & 0.395                                 & 3.087                                    & 0.393                                   & 3.098                                                & 0.395                  & 3.060                                  & {\color[HTML]{FF0000} \textbf{0.389}} & 3.090                                                  & 0.391                    & 3.079                                   & 0.391                                & 3.065                                    & {\color[HTML]{0000FF} {\ul 0.390}}   & 4.229                                                & 0.452                   \\ \cmidrule(l){2-26} 
                            & 24  & 7.444                                    & {\color[HTML]{0000FF} {\ul 0.665}}      & {\color[HTML]{0000FF} {\ul 7.432}}     & {\color[HTML]{FF0000} \textbf{0.664}} & 8.327                                   & 0.683                                  & 8.000                                    & 0.683                                 & 7.590                                  & 0.670                                 & 8.283                                    & 0.689                                   & 8.029                                                & 0.679                  & {\color[HTML]{FF0000} \textbf{7.416}}  & {\color[HTML]{0000FF} {\ul 0.665}}    & 7.758                                                  & 0.672                    & 7.925                                   & 0.677                                & 7.641                                    & 0.670                                & 11.535                                               & 0.834                   \\
                            & 36  & {\color[HTML]{0000FF} {\ul 10.801}}      & 0.801                                   & 10.848                                 & {\color[HTML]{0000FF} {\ul 0.799}}    & 11.192                                  & 0.813                                  & 12.011                                   & 0.823                                 & 10.986                                 & 0.803                                 & 14.754                                   & 0.856                                   & 11.962                                               & 0.828                  & {\color[HTML]{FF0000} \textbf{10.799}} & {\color[HTML]{FF0000} \textbf{0.798}} & 11.456                                                 & 0.808                    & 12.087                                  & 0.827                                & 11.210                                   & 0.807                                & 16.922                                               & 0.982                   \\
                            & 48  & 14.102                                   & 0.915                                   & {\color[HTML]{0000FF} {\ul 14.045}}    & {\color[HTML]{0000FF} {\ul 0.914}}    & 15.278                                  & 0.945                                  & 14.814                                   & 0.933                                 & 14.157                                 & 0.922                                 & 16.893                                   & 0.970                                   & 15.266                                               & 0.937                  & {\color[HTML]{FF0000} \textbf{13.881}} & {\color[HTML]{FF0000} \textbf{0.912}} & 14.696                                                 & 0.921                    & 14.787                                  & 0.930                                & 14.866                                   & 0.935                                & 19.501                                               & 1.093                   \\
                            & 60  & {\color[HTML]{0000FF} {\ul 17.057}}      & 1.023                                   & {\color[HTML]{FF0000} \textbf{16.959}} & {\color[HTML]{FF0000} \textbf{1.017}} & 20.997                                  & 1.067                                  & 18.932                                   & 1.054                                 & 18.018                                 & 1.035                                 & 18.881                                   & 1.059                                   & 18.407                                               & 1.045                  & 17.257                                 & {\color[HTML]{0000FF} {\ul 1.021}}    & 18.058                                                 & 1.032                    & 18.298                                  & 1.041                                & 17.947                                   & 1.036                                & 22.804                                               & 1.177                   \\ \cmidrule(l){2-26} 
\multirow{-10}{*}{\rotatebox[origin=c]{90}{DowJones}} & Avg & 12.351                                   & 0.851                                   & {\color[HTML]{FF0000} \textbf{12.321}} & {\color[HTML]{FF0000} \textbf{0.848}} & 13.948                                  & 0.877                                  & 13.439                                   & 0.873                                 & 12.688                                 & 0.857                                 & 14.703                                   & 0.893                                   & 13.416                                               & 0.872                  & {\color[HTML]{0000FF} {\ul 12.338}}    & {\color[HTML]{0000FF} {\ul 0.849}}    & 12.992                                                 & 0.858                    & 13.274                                  & 0.869                                & 12.916                                   & 0.862                                & 17.690                                               & 1.021                   \\ \midrule
                            & 3   & {\color[HTML]{0000FF} {\ul 0.838}}       & {\color[HTML]{FF0000} \textbf{0.683}}   & 0.864                                  & 0.688                                 & 0.850                                   & {\color[HTML]{FF0000} \textbf{0.683}}  & 0.842                                    & {\color[HTML]{0000FF} {\ul 0.685}}    & 0.876                                  & 0.708                                 & 0.843                                    & {\color[HTML]{0000FF} {\ul 0.685}}      & 0.859                                                & 0.696                  & 0.896                                  & 0.691                                 & 0.865                                                  & 0.697                    & 0.863                                   & 0.696                                & {\color[HTML]{FF0000} \textbf{0.830}}    & 0.687                                & 0.929                                                & 0.716                   \\
                            & 6   & 0.978                                    & {\color[HTML]{FF0000} \textbf{0.735}}   & 0.991                                  & 0.742                                 & {\color[HTML]{0000FF} {\ul 0.971}}      & 0.743                                  & 0.988                                    & 0.747                                 & 0.991                                  & 0.761                                 & 0.997                                    & 0.748                                   & 1.023                                                & 0.761                  & 1.013                                  & 0.750                                 & 0.994                                                  & 0.751                    & 0.990                                   & 0.748                                & {\color[HTML]{FF0000} \textbf{0.948}}    & {\color[HTML]{0000FF} {\ul 0.741}}   & 1.070                                                & 0.773                   \\
                            & 9   & 1.046                                    & 0.768                                   & 1.062                                  & 0.770                                 & {\color[HTML]{FF0000} \textbf{1.024}}   & {\color[HTML]{FF0000} \textbf{0.763}}  & 1.063                                    & 0.778                                 & 1.051                                  & 0.790                                 & 1.050                                    & {\color[HTML]{0000FF} {\ul 0.767}}      & 1.071                                                & 0.783                  & 1.085                                  & 0.781                                 & 1.073                                                  & 0.784                    & 1.065                                   & 0.780                                & {\color[HTML]{0000FF} {\ul 1.027}}       & 0.773                                & 1.104                                                & 0.788                   \\
                            & 12  & 1.124                                    & 0.797                                   & 1.119                                  & {\color[HTML]{0000FF} {\ul 0.789}}    & {\color[HTML]{0000FF} {\ul 1.087}}      & {\color[HTML]{FF0000} \textbf{0.788}}  & 1.125                                    & 0.803                                 & 1.110                                  & 0.814                                 & 1.107                                    & 0.792                                   & 1.130                                                & 0.804                  & 1.124                                  & 0.792                                 & 1.135                                                  & 0.807                    & 1.142                                   & 0.808                                & {\color[HTML]{FF0000} \textbf{1.083}}    & 0.797                                & 1.212                                                & 0.825                   \\ \cmidrule(l){2-26} 
                            & Avg & 0.996                                    & {\color[HTML]{0000FF} {\ul 0.746}}      & 1.009                                  & 0.747                                 & {\color[HTML]{0000FF} {\ul 0.983}}      & {\color[HTML]{FF0000} \textbf{0.744}}  & 1.004                                    & 0.753                                 & 1.007                                  & 0.768                                 & 0.999                                    & 0.748                                   & 1.021                                                & 0.761                  & 1.029                                  & 0.753                                 & 1.017                                                  & 0.760                    & 1.015                                   & 0.758                                & {\color[HTML]{FF0000} \textbf{0.972}}    & 0.749                                & 1.079                                                & 0.775                   \\ \cmidrule(l){2-26} 
                            & 24  & {\color[HTML]{FF0000} \textbf{1.335}}    & {\color[HTML]{0000FF} {\ul 0.877}}      & 1.343                                  & {\color[HTML]{FF0000} \textbf{0.870}} & 1.341                                   & 0.881                                  & 1.410                                    & 0.916                                 & 1.390                                  & 0.916                                 & {\color[HTML]{0000FF} {\ul 1.340}}       & {\color[HTML]{0000FF} {\ul 0.877}}      & 1.397                                                & 0.909                  & 1.406                                  & 0.886                                 & 1.410                                                  & 0.913                    & 1.462                                   & 0.924                                & 1.468                                    & 0.935                                & 1.494                                                & 0.924                   \\
                            & 36  & {\color[HTML]{0000FF} {\ul 1.442}}       & {\color[HTML]{0000FF} {\ul 0.912}}      & 1.445                                  & {\color[HTML]{FF0000} \textbf{0.903}} & {\color[HTML]{FF0000} \textbf{1.420}}   & 0.914                                  & 1.590                                    & 0.968                                 & 1.518                                  & 0.957                                 & 1.446                                    & 0.920                                   & 1.509                                                & 0.951                  & 1.506                                  & 0.921                                 & 1.538                                                  & 0.953                    & 1.582                                   & 0.964                                & 1.593                                    & 0.972                                & 1.526                                                & 0.950                   \\
                            & 48  & {\color[HTML]{0000FF} {\ul 1.510}}       & {\color[HTML]{0000FF} {\ul 0.935}}      & 1.559                                  & 0.946                                 & 1.567                                   & 0.963                                  & 1.680                                    & 1.009                                 & 1.610                                  & 0.995                                 & {\color[HTML]{FF0000} \textbf{1.467}}    & {\color[HTML]{FF0000} \textbf{0.933}}   & 1.646                                                & 0.999                  & 1.583                                  & 0.957                                 & 1.652                                                  & 1.008                    & 1.696                                   & 1.011                                & 1.710                                    & 1.020                                & 1.581                                                & 0.981                   \\
                            & 60  & {\color[HTML]{FF0000} \textbf{1.573}}    & {\color[HTML]{FF0000} \textbf{0.971}}   & {\color[HTML]{0000FF} {\ul 1.602}}     & {\color[HTML]{FF0000} \textbf{0.971}} & 1.609                                   & {\color[HTML]{0000FF} {\ul 0.988}}     & 1.776                                    & 1.053                                 & 1.679                                  & 1.020                                 & 1.626                                    & 1.006                                   & 1.727                                                & 1.043                  & 1.693                                  & 1.003                                 & 1.752                                                  & 1.049                    & 1.796                                   & 1.061                                & 1.829                                    & 1.064                                & 1.625                                                & 1.010                   \\ \cmidrule(l){2-26} 
\multirow{-10}{*}{\rotatebox[origin=c]{90}{Power}}    & Avg & {\color[HTML]{FF0000} \textbf{1.465}}    & {\color[HTML]{0000FF} {\ul 0.924}}      & 1.487                                  & {\color[HTML]{FF0000} \textbf{0.922}} & 1.484                                   & 0.937                                  & 1.614                                    & 0.986                                 & 1.549                                  & 0.972                                 & {\color[HTML]{0000FF} {\ul 1.470}}       & 0.934                                   & 1.570                                                & 0.975                  & 1.547                                  & 0.942                                 & 1.588                                                  & 0.981                    & 1.634                                   & 0.990                                & 1.650                                    & 0.998                                & 1.556                                                & 0.966                   \\ \midrule
                            & 3   & {\color[HTML]{FF0000} \textbf{0.012}}    & {\color[HTML]{0000FF} {\ul 0.048}}      & {\color[HTML]{FF0000} \textbf{0.012}}  & {\color[HTML]{FF0000} \textbf{0.047}} & 0.015                                   & 0.074                                  & {\color[HTML]{FF0000} \textbf{0.012}}    & 0.062                                 & 0.072                                  & 0.200                                 & 0.014                                    & 0.069                                   & {\color[HTML]{0000FF} {\ul 0.013}}                   & 0.061                  & 0.014                                  & 0.070                                 & {\color[HTML]{0000FF} {\ul 0.013}}                     & 0.068                    & {\color[HTML]{FF0000} \textbf{0.012}}   & 0.061                                & {\color[HTML]{FF0000} \textbf{0.012}}    & 0.060                                & 0.062                                                & 0.171                   \\
                            & 6   & {\color[HTML]{FF0000} \textbf{0.038}}    & {\color[HTML]{0000FF} {\ul 0.109}}      & 0.041                                  & {\color[HTML]{FF0000} \textbf{0.108}} & 0.057                                   & 0.154                                  & 0.043                                    & 0.130                                 & 0.115                                  & 0.255                                 & 0.079                                    & 0.194                                   & {\color[HTML]{0000FF} {\ul 0.040}}                   & 0.122                  & 0.042                                  & 0.127                                 & 0.046                                                  & 0.139                    & 0.047                                   & 0.136                                & 0.043                                    & 0.127                                & 0.111                                                & 0.227                   \\
                            & 9   & {\color[HTML]{FF0000} \textbf{0.081}}    & {\color[HTML]{FF0000} \textbf{0.162}}   & {\color[HTML]{0000FF} {\ul 0.084}}     & {\color[HTML]{0000FF} {\ul 0.170}}    & 0.109                                   & 0.213                                  & 0.107                                    & 0.216                                 & 0.191                                  & 0.329                                 & 0.101                                    & 0.200                                   & 0.119                                                & 0.226                  & 0.093                                  & 0.194                                 & 0.095                                                  & 0.198                    & 0.098                                   & 0.199                                & 0.093                                    & 0.190                                & 0.176                                                & 0.281                   \\
                            & 12  & {\color[HTML]{FF0000} \textbf{0.119}}    & {\color[HTML]{FF0000} \textbf{0.210}}   & {\color[HTML]{0000FF} {\ul 0.131}}     & {\color[HTML]{0000FF} {\ul 0.220}}    & 0.195                                   & 0.293                                  & 0.155                                    & 0.255                                 & 0.240                                  & 0.386                                 & 0.196                                    & 0.291                                   & 0.141                                                & 0.239                  & 0.155                                  & 0.256                                 & 0.148                                                  & 0.250                    & 0.158                                   & 0.256                                & 0.164                                    & 0.261                                & 0.240                                                & 0.326                   \\ \cmidrule(l){2-26} 
                            & Avg & {\color[HTML]{FF0000} \textbf{0.062}}    & {\color[HTML]{FF0000} \textbf{0.132}}   & {\color[HTML]{0000FF} {\ul 0.067}}     & {\color[HTML]{0000FF} {\ul 0.136}}    & 0.094                                   & 0.183                                  & 0.079                                    & 0.166                                 & 0.154                                  & 0.292                                 & 0.097                                    & 0.188                                   & 0.078                                                & 0.162                  & 0.076                                  & 0.162                                 & 0.075                                                  & 0.163                    & 0.079                                   & 0.163                                & 0.078                                    & 0.160                                & 0.147                                                & 0.251                   \\ \cmidrule(l){2-26} 
                            & 24  & {\color[HTML]{FF0000} \textbf{0.362}}    & {\color[HTML]{FF0000} \textbf{0.395}}   & {\color[HTML]{0000FF} {\ul 0.458}}     & {\color[HTML]{0000FF} {\ul 0.447}}    & 0.668                                   & 0.549                                  & 0.655                                    & 0.548                                 & 0.951                                  & 0.645                                 & 0.779                                    & 0.578                                   & 0.861                                                & 0.645                  & 0.583                                  & 0.522                                 & 0.545                                                  & 0.496                    & 0.719                                   & 0.595                                & 0.620                                    & 0.536                                & 1.798                                                & 0.967                   \\
                            & 36  & {\color[HTML]{FF0000} \textbf{0.652}}    & {\color[HTML]{FF0000} \textbf{0.527}}   & {\color[HTML]{0000FF} {\ul 0.870}}     & {\color[HTML]{0000FF} {\ul 0.619}}    & 1.580                                   & 0.866                                  & 1.247                                    & 0.769                                 & 0.906                                  & 0.717                                 & 1.318                                    & 0.777                                   & 1.741                                                & 0.938                  & 1.205                                  & 0.762                                 & 1.107                                                  & 0.725                    & 2.212                                   & 1.103                                & 2.245                                    & 1.120                                & 2.820                                                & 1.171                   \\
                            & 48  & {\color[HTML]{0000FF} {\ul 1.484}}       & {\color[HTML]{0000FF} {\ul 0.829}}      & 1.643                                  & 0.900                                 & 3.702                                   & 1.483                                  & 2.700                                    & 1.209                                 & {\color[HTML]{FF0000} \textbf{0.974}}  & {\color[HTML]{FF0000} \textbf{0.747}} & 2.977                                    & 1.258                                   & 3.526                                                & 1.427                  & 2.737                                  & 1.218                                 & 6.192                                                  & 2.055                    & 2.577                                   & 1.178                                & 2.644                                    & 1.191                                & 5.909                                                & 1.620                   \\
                            & 60  & {\color[HTML]{0000FF} {\ul 2.496}}       & {\color[HTML]{0000FF} {\ul 1.133}}      & 2.593                                  & 1.182                                 & 6.320                                   & 2.030                                  & 5.370                                    & 1.830                                 & {\color[HTML]{FF0000} \textbf{1.071}}  & {\color[HTML]{FF0000} \textbf{0.798}} & 6.582                                    & 2.060                                   & 5.574                                                & 1.857                  & 4.475                                  & 1.624                                 & 8.236                                                  & 2.384                    & 5.441                                   & 1.799                                & 4.074                                    & 1.511                                & 7.511                                                & 2.053                   \\ \cmidrule(l){2-26} 
\multirow{-10}{*}{\rotatebox[origin=c]{90}{Unemp}}       & Avg & {\color[HTML]{0000FF} {\ul 1.249}}       & {\color[HTML]{FF0000} \textbf{0.721}}   & 1.391                                  & 0.787                                 & 3.068                                   & 1.232                                  & 2.493                                    & 1.089                                 & {\color[HTML]{FF0000} \textbf{0.975}}  & {\color[HTML]{0000FF} {\ul 0.727}}    & 2.914                                    & 1.168                                   & 2.925                                                & 1.217                  & 2.250                                  & 1.031                                 & 4.020                                                  & 1.415                    & 2.737                                   & 1.169                                & 2.396                                    & 1.089                                & 4.510                                                & 1.452                   \\ \midrule
\multicolumn{2}{c|}{1\textsuperscript{st} Count}    & 24                                       & 19                                      & 20                                     & 19                                    & 8                                       & 5                                      & 1                                        & 0                                     & 3                                      & 2                                     & 3                                        & 1                                       & 3                                                    & 0                      & 11                                     & 6                                     & 1                                                      & 0                        & 1                                       & 0                                    & 6                                        & 0                                    & 3                                                    & 0                       \\ \midrule
\multicolumn{2}{c|}{Top-2 Count}  & 33                                       & 34                                      & 20                                     & 31                                    & 8                                       & 10                                     & 1                                        & 1                                     & 3                                      & 3                                     & 3                                        & 4                                       & 3                                                    & 0                      & 11                                     & 16                                    & 1                                                      & 0                        & 1                                       & 1                                    & 6                                        & 5                                    & 3                                                    & 0                       \\ \bottomrule
\end{tabular}
}
\caption{Full results of short-term forecasting under the `Input-12-Predict-\{3, 6, 9, 12\}' and `Input-36-Predict-\{24, 36, 48, 60\}' settings (Part 2). This table corresponds to Table~\ref{tab_short_term} in the main paper.}
\label{tab_short_term_p2}
\end{center}
\end{table*}


\begin{table*}[t]
\begin{center}
{\fontsize{8}{9}\selectfont
\setlength{\tabcolsep}{8pt}

\begin{tabular}{@{}cc|cc|cc|cc|cc|cc@{}}
\toprule
\multicolumn{2}{c|}{Model}           & \multicolumn{2}{c|}{\begin{tabular}[c]{@{}c@{}}vLinear\\      (Ours)\end{tabular}} 
& \multicolumn{2}{c|}{\begin{tabular}[c]{@{}c@{}}SimpleTM\\     \shortcite{simpletm} \end{tabular}} 
& \multicolumn{2}{c|}{\begin{tabular}[c]{@{}c@{}}TQNet\\      \shortcite{tqnet} \end{tabular}} 
& \multicolumn{2}{c|}{\begin{tabular}[c]{@{}c@{}}TimePro\\      \shortcite{timepro} \end{tabular}} 
& \multicolumn{2}{c}{\begin{tabular}[c]{@{}c@{}}TimeBase\\      \shortcite{timebase} \end{tabular}} \\ \midrule
\multicolumn{2}{c|}{Metric}          & MSE                                      & MAE                                     & MSE                                    & MAE                                    & MSE                                   & MAE                                  & MSE                                                    & MAE                   & MSE                     & MAE                                                  \\ \midrule
                               & 96  & {\color[HTML]{FF0000} \textbf{0.301}}    & {\color[HTML]{FF0000} \textbf{0.337}}   & 0.321                                  & 0.361                                  & {\color[HTML]{0000FF} {\ul 0.311}}    & {\color[HTML]{0000FF} {\ul 0.353}}   & 0.326                                                  & 0.364                 & 0.373                   & 0.388                                                \\
                               & 192 & {\color[HTML]{FF0000} \textbf{0.349}}    & {\color[HTML]{FF0000} \textbf{0.365}}   & 0.360                                  & 0.380                                  & {\color[HTML]{0000FF} {\ul 0.356}}    & {\color[HTML]{0000FF} {\ul 0.378}}   & 0.367                                                  & 0.383                 & 0.411                   & 0.409                                                \\
                               & 336 & {\color[HTML]{FF0000} \textbf{0.381}}    & {\color[HTML]{FF0000} \textbf{0.387}}   & 0.390                                  & 0.404                                  & {\color[HTML]{0000FF} {\ul 0.390}}    & {\color[HTML]{0000FF} {\ul 0.401}}   & 0.402                                                  & 0.409                 & 0.436                   & 0.421                                                \\
                               & 720 & {\color[HTML]{FF0000} \textbf{0.444}}    & {\color[HTML]{FF0000} \textbf{0.424}}   & 0.454                                  & {\color[HTML]{0000FF} {\ul 0.438}}     & {\color[HTML]{0000FF} {\ul 0.452}}    & 0.440                                & 0.469                                                  & 0.446                 & 0.503                   & 0.461                                                \\ \cmidrule(l){2-12} 
\multirow{-5}{*}{ETTm1}        & Avg & {\color[HTML]{FF0000} \textbf{0.369}}    & {\color[HTML]{FF0000} \textbf{0.378}}   & 0.381                                  & 0.396                                  & {\color[HTML]{0000FF} {\ul 0.377}}    & {\color[HTML]{0000FF} {\ul 0.393}}   & 0.391                                                  & 0.400                 & 0.431                   & 0.420                                                \\ \midrule
                               & 96  & {\color[HTML]{FF0000} \textbf{0.168}}    & {\color[HTML]{FF0000} \textbf{0.245}}   & 0.173                                  & 0.257                                  & {\color[HTML]{0000FF} {\ul 0.173}}    & {\color[HTML]{0000FF} {\ul 0.256}}   & 0.178                                                  & 0.260                 & 0.188                   & 0.271                                                \\
                               & 192 & {\color[HTML]{FF0000} \textbf{0.231}}    & {\color[HTML]{FF0000} \textbf{0.287}}   & 0.238                                  & 0.299                                  & {\color[HTML]{0000FF} {\ul 0.238}}    & {\color[HTML]{0000FF} {\ul 0.298}}   & 0.242                                                  & 0.303                 & 0.251                   & 0.309                                                \\
                               & 336 & {\color[HTML]{FF0000} \textbf{0.289}}    & {\color[HTML]{FF0000} \textbf{0.326}}   & {\color[HTML]{0000FF} {\ul 0.296}}     & {\color[HTML]{0000FF} {\ul 0.338}}     & 0.301                                 & 0.340                                & 0.303                                                  & 0.342                 & 0.311                   & 0.346                                                \\
                               & 720 & {\color[HTML]{FF0000} \textbf{0.386}}    & {\color[HTML]{FF0000} \textbf{0.382}}   & {\color[HTML]{0000FF} {\ul 0.393}}     & {\color[HTML]{0000FF} {\ul 0.395}}     & 0.397                                 & 0.396                                & 0.400                                                  & 0.399                 & 0.411                   & 0.401                                                \\ \cmidrule(l){2-12} 
\multirow{-5}{*}{ETTm2}        & Avg & {\color[HTML]{FF0000} \textbf{0.268}}    & {\color[HTML]{FF0000} \textbf{0.310}}   & {\color[HTML]{0000FF} {\ul 0.275}}     & {\color[HTML]{0000FF} {\ul 0.322}}     & 0.277                                 & 0.323                                & 0.281                                                  & 0.326                 & 0.290                   & 0.332                                                \\ \midrule
                               & 96  & {\color[HTML]{FF0000} \textbf{0.356}}    & {\color[HTML]{FF0000} \textbf{0.379}}   & {\color[HTML]{0000FF} {\ul 0.366}}     & {\color[HTML]{0000FF} {\ul 0.392}}     & 0.371                                 & 0.393                                & 0.375                                                  & 0.398                 & 0.399                   & 0.392                                                \\
                               & 192 & {\color[HTML]{FF0000} \textbf{0.409}}    & {\color[HTML]{FF0000} \textbf{0.411}}   & {\color[HTML]{0000FF} {\ul 0.422}}     & {\color[HTML]{0000FF} {\ul 0.421}}     & 0.428                                 & 0.426                                & 0.427                                                  & 0.429                 & 0.455                   & 0.423                                                \\
                               & 336 & {\color[HTML]{0000FF} {\ul 0.449}}       & {\color[HTML]{FF0000} \textbf{0.433}}   & {\color[HTML]{FF0000} \textbf{0.440}}  & {\color[HTML]{0000FF} {\ul 0.438}}     & 0.476                                 & 0.446                                & 0.472                                                  & 0.450                 & 0.501                   & 0.443                                                \\
                               & 720 & {\color[HTML]{FF0000} \textbf{0.451}}    & {\color[HTML]{FF0000} \textbf{0.456}}   & {\color[HTML]{0000FF} {\ul 0.463}}     & 0.462                                  & 0.487                                 & 0.470                                & 0.476                                                  & 0.474                 & 0.498                   & {\color[HTML]{0000FF} {\ul 0.458}}                   \\ \cmidrule(l){2-12} 
\multirow{-5}{*}{ETTh1}        & Avg & {\color[HTML]{FF0000} \textbf{0.416}}    & {\color[HTML]{FF0000} \textbf{0.420}}   & {\color[HTML]{0000FF} {\ul 0.422}}     & {\color[HTML]{0000FF} {\ul 0.428}}     & 0.441                                 & 0.434                                & 0.438                                                  & 0.438                 & 0.463                   & 0.429                                                \\ \midrule
                               & 96  & {\color[HTML]{FF0000} \textbf{0.280}}    & {\color[HTML]{FF0000} \textbf{0.327}}   & {\color[HTML]{0000FF} {\ul 0.281}}     & {\color[HTML]{0000FF} {\ul 0.338}}     & 0.295                                 & 0.343                                & 0.293                                                  & 0.345                 & 0.338                   & 0.376                                                \\
                               & 192 & {\color[HTML]{FF0000} \textbf{0.351}}    & {\color[HTML]{FF0000} \textbf{0.376}}   & {\color[HTML]{0000FF} {\ul 0.355}}     & {\color[HTML]{0000FF} {\ul 0.387}}     & 0.367                                 & 0.393                                & 0.367                                                  & 0.394                 & 0.402                   & 0.405                                                \\
                               & 336 & {\color[HTML]{0000FF} {\ul 0.400}}       & {\color[HTML]{0000FF} {\ul 0.411}}      & {\color[HTML]{FF0000} \textbf{0.365}}  & {\color[HTML]{FF0000} \textbf{0.401}}  & 0.417                                 & 0.427                                & 0.419                                                  & 0.431                 & 0.434                   & 0.440                                                \\
                               & 720 & {\color[HTML]{FF0000} \textbf{0.404}}    & {\color[HTML]{FF0000} \textbf{0.427}}   & {\color[HTML]{0000FF} {\ul 0.413}}     & {\color[HTML]{0000FF} {\ul 0.436}}     & 0.433                                 & 0.446                                & 0.427                                                  & 0.445                 & 0.460                   & 0.477                                                \\ \cmidrule(l){2-12} 
\multirow{-5}{*}{ETTh2}        & Avg & {\color[HTML]{0000FF} {\ul 0.358}}       & {\color[HTML]{FF0000} \textbf{0.385}}   & {\color[HTML]{FF0000} \textbf{0.353}}  & {\color[HTML]{0000FF} {\ul 0.391}}     & 0.378                                 & 0.402                                & 0.377                                                  & 0.403                 & 0.408                   & 0.424                                                \\ \midrule
                               & 96  & {\color[HTML]{FF0000} \textbf{0.129}}    & {\color[HTML]{FF0000} \textbf{0.221}}   & 0.141                                  & 0.235                                  & {\color[HTML]{0000FF} {\ul 0.134}}    & {\color[HTML]{0000FF} {\ul 0.229}}   & 0.139                                                  & 0.234                 & 0.212                   & 0.279                                                \\
                               & 192 & {\color[HTML]{FF0000} \textbf{0.147}}    & {\color[HTML]{FF0000} \textbf{0.238}}   & {\color[HTML]{0000FF} {\ul 0.151}}     & {\color[HTML]{0000FF} {\ul 0.247}}     & 0.154                                 & {\color[HTML]{0000FF} {\ul 0.247}}   & 0.156                                                  & 0.249                 & 0.209                   & 0.281                                                \\
                               & 336 & {\color[HTML]{FF0000} \textbf{0.158}}    & {\color[HTML]{FF0000} \textbf{0.250}}   & 0.173                                  & 0.267                                  & {\color[HTML]{0000FF} {\ul 0.169}}    & {\color[HTML]{0000FF} {\ul 0.264}}   & 0.172                                                  & 0.267                 & 0.222                   & 0.295                                                \\
                               & 720 & {\color[HTML]{FF0000} \textbf{0.178}}    & {\color[HTML]{FF0000} \textbf{0.271}}   & {\color[HTML]{0000FF} {\ul 0.201}}     & {\color[HTML]{0000FF} {\ul 0.293}}     & {\color[HTML]{0000FF} {\ul 0.201}}    & 0.294                                & 0.209                                                  & 0.299                 & 0.264                   & 0.327                                                \\ \cmidrule(l){2-12} 
\multirow{-5}{*}{ECL}          & Avg & {\color[HTML]{FF0000} \textbf{0.153}}    & {\color[HTML]{FF0000} \textbf{0.245}}   & 0.166                                  & 0.260                                  & {\color[HTML]{0000FF} {\ul 0.164}}    & {\color[HTML]{0000FF} {\ul 0.259}}   & 0.169                                                  & 0.262                 & 0.227                   & 0.295                                                \\ \midrule
                               & 96  & {\color[HTML]{FF0000} \textbf{0.396}}    & {\color[HTML]{FF0000} \textbf{0.233}}   & 0.410                                  & 0.274                                  & 0.413                                 & {\color[HTML]{0000FF} {\ul 0.261}}   & {\color[HTML]{0000FF} {\ul 0.408}}                     & 0.269                 & 0.713                   & 0.384                                                \\
                               & 192 & {\color[HTML]{0000FF} {\ul 0.427}}       & {\color[HTML]{FF0000} \textbf{0.243}}   & 0.430                                  & 0.280                                  & 0.432                                 & {\color[HTML]{0000FF} {\ul 0.271}}   & {\color[HTML]{FF0000} \textbf{0.424}}                  & 0.276                 & 0.652                   & 0.362                                                \\
                               & 336 & {\color[HTML]{0000FF} {\ul 0.449}}       & {\color[HTML]{FF0000} \textbf{0.255}}   & {\color[HTML]{0000FF} {\ul 0.449}}     & 0.290                                  & 0.450                                 & {\color[HTML]{0000FF} {\ul 0.277}}   & {\color[HTML]{FF0000} \textbf{0.444}}                  & 0.287                 & 0.660                   & 0.365                                                \\
                               & 720 & 0.487                                    & {\color[HTML]{FF0000} \textbf{0.276}}   & {\color[HTML]{0000FF} {\ul 0.486}}     & 0.309                                  & {\color[HTML]{0000FF} {\ul 0.486}}    & {\color[HTML]{0000FF} {\ul 0.295}}   & {\color[HTML]{FF0000} \textbf{0.484}}                  & 0.312                 & 0.702                   & 0.386                                                \\ \cmidrule(l){2-12} 
\multirow{-5}{*}{Traffic}      & Avg & {\color[HTML]{FF0000} \textbf{0.440}}    & {\color[HTML]{FF0000} \textbf{0.252}}   & {\color[HTML]{0000FF} {\ul 0.444}}     & 0.289                                  & 0.445                                 & {\color[HTML]{0000FF} {\ul 0.276}}   & {\color[HTML]{FF0000} \textbf{0.440}}                  & 0.286                 & 0.682                   & 0.374                                                \\ \midrule
                               & 96  & {\color[HTML]{FF0000} \textbf{0.150}}    & {\color[HTML]{FF0000} \textbf{0.186}}   & 0.162                                  & 0.207                                  & {\color[HTML]{0000FF} {\ul 0.157}}    & {\color[HTML]{0000FF} {\ul 0.200}}   & 0.166                                                  & 0.207                 & 0.170                   & 0.215                                                \\
                               & 192 & {\color[HTML]{FF0000} \textbf{0.198}}    & {\color[HTML]{FF0000} \textbf{0.234}}   & 0.208                                  & 0.248                                  & {\color[HTML]{0000FF} {\ul 0.206}}    & {\color[HTML]{0000FF} {\ul 0.245}}   & 0.216                                                  & 0.254                 & 0.216                   & 0.256                                                \\
                               & 336 & {\color[HTML]{FF0000} \textbf{0.252}}    & {\color[HTML]{FF0000} \textbf{0.275}}   & 0.263                                  & 0.290                                  & {\color[HTML]{0000FF} {\ul 0.262}}    & {\color[HTML]{0000FF} {\ul 0.287}}   & 0.273                                                  & 0.296                 & 0.272                   & 0.297                                                \\
                               & 720 & {\color[HTML]{FF0000} \textbf{0.332}}    & {\color[HTML]{FF0000} \textbf{0.328}}   & {\color[HTML]{0000FF} {\ul 0.340}}     & {\color[HTML]{0000FF} {\ul 0.341}}     & 0.344                                 & 0.342                                & 0.351                                                  & 0.346                 & 0.351                   & 0.348                                                \\ \cmidrule(l){2-12} 
\multirow{-5}{*}{Weather}      & Avg & {\color[HTML]{FF0000} \textbf{0.233}}    & {\color[HTML]{FF0000} \textbf{0.256}}   & 0.243                                  & 0.271                                  & {\color[HTML]{0000FF} {\ul 0.242}}    & {\color[HTML]{0000FF} {\ul 0.269}}   & 0.251                                                  & 0.276                 & 0.252                   & 0.279                                                \\ \midrule
                               & 96  & 0.175                                    & {\color[HTML]{FF0000} \textbf{0.196}}   & {\color[HTML]{FF0000} \textbf{0.163}}  & {\color[HTML]{0000FF} {\ul 0.232}}     & {\color[HTML]{0000FF} {\ul 0.173}}    & 0.233                                & 0.196                                                  & 0.237                 & 0.356                   & 0.363                                                \\
                               & 192 & 0.202                                    & {\color[HTML]{FF0000} \textbf{0.220}}   & {\color[HTML]{FF0000} \textbf{0.182}}  & {\color[HTML]{0000FF} {\ul 0.247}}     & {\color[HTML]{0000FF} {\ul 0.199}}    & 0.257                                & 0.231                                                  & 0.263                 & 0.407                   & 0.404                                                \\
                               & 336 & 0.224                                    & {\color[HTML]{FF0000} \textbf{0.240}}   & {\color[HTML]{FF0000} \textbf{0.193}}  & {\color[HTML]{0000FF} {\ul 0.257}}     & {\color[HTML]{0000FF} {\ul 0.211}}    & 0.263                                & 0.250                                                  & 0.281                 & 0.430                   & 0.398                                                \\
                               & 720 & 0.234                                    & {\color[HTML]{FF0000} \textbf{0.247}}   & {\color[HTML]{FF0000} \textbf{0.199}}  & {\color[HTML]{0000FF} {\ul 0.252}}     & {\color[HTML]{0000FF} {\ul 0.209}}    & 0.270                                & 0.253                                                  & 0.285                 & 0.425                   & 0.388                                                \\ \cmidrule(l){2-12} 
\multirow{-5}{*}{Solar-Energy} & Avg & 0.209                                    & {\color[HTML]{FF0000} \textbf{0.226}}   & {\color[HTML]{FF0000} \textbf{0.184}}  & {\color[HTML]{0000FF} {\ul 0.247}}     & {\color[HTML]{0000FF} {\ul 0.198}}    & 0.256                                & 0.232                                                  & 0.266                 & 0.404                   & 0.388                                                \\ \midrule
\multicolumn{2}{c|}{1\textsuperscript{st} Count}       & 29                                       & 39                                      & 8                                      & 1                                      & 0                                     & 0                                    & 4                                                      & 0                     & 0                       & 0                                                    \\ \bottomrule
\end{tabular}
}
\caption{Comparison between vLinear and more state-of-the-art efficient forecasters under the `Input-96-Predict-\{96, 192, 336, 720\}' setting. This table corresponds to Table~\ref{tab_more_baseline} in the main paper.}
\label{tab_more_baseline_appd}
\end{center}
\end{table*}


\begin{table*}[t]
\begin{center}
{\fontsize{8}{9}\selectfont
\setlength{\tabcolsep}{7pt}
\begin{tabular}{@{}cc|c|cc|cc|cc|cc|cc|cc@{}}
\toprule
                           &                            &                        & \multicolumn{2}{c|}{ETTh1}                                                    & \multicolumn{2}{c|}{ETTm2}                                                    & \multicolumn{2}{c|}{ECL}                                                      & \multicolumn{2}{c|}{Solar}                                                    & \multicolumn{2}{c|}{Weather}                                                  & \multicolumn{2}{c}{PEMS03}                                                    \\ \cmidrule(l){4-15} 
\multirow{-2}{*}{Var.}     & \multirow{-2}{*}{Temp.}    & \multirow{-2}{*}{Hor.} & \multicolumn{1}{l}{MSE}               & \multicolumn{1}{l|}{MAE}              & \multicolumn{1}{l}{MSE}               & \multicolumn{1}{l|}{MAE}              & \multicolumn{1}{l}{MSE}               & \multicolumn{1}{l|}{MAE}              & \multicolumn{1}{l}{MSE}               & \multicolumn{1}{l|}{MAE}              & \multicolumn{1}{l}{MSE}               & \multicolumn{1}{l|}{MAE}              & \multicolumn{1}{l}{MSE}               & \multicolumn{1}{l}{MAE}               \\ \midrule
                           &                            & H1                     & 0.361                                 & 0.382                                 & 0.171                                 & 0.250                                 & 0.160                                 & 0.246                                 & 0.228                                 & 0.241                                 & 0.163                                 & 0.200                                 & 0.075                                 & 0.181                                 \\
\bottomrule
\end{tabular}
}
\caption{Ablation studies on the representation learning of variate (\textit{Var.}) and temporal (\textit{Temp.}) dimensions. $\left \{ \mathrm{H1}, \mathrm{H2}, \mathrm{H3}, \mathrm{H4 } \right \}$ corresponds to $\{12,24,48,96\}$ for PEMS03, and $\{96,192,336,720\}$ for the other datasets.
This table corresponds to Table~\ref{tab_var_temp} in the main paper.}
\label{tab_var_tmp_appd}
\end{center}
\end{table*}


\begin{table*}[t]
\begin{center}
\renewcommand{\arraystretch}{1.3}
{\fontsize{8}{10}\selectfont
\setlength{\tabcolsep}{2.3pt}
\begin{tabular}{@{}cc|cc|cc|cc|cc|cc|cc|cc|cc|cc|cc@{}}
\toprule
&                        & \multicolumn{2}{c|}{\begin{tabular}[c]{@{}c@{}}vecTrans\\      (Ours)\end{tabular}} 
& \multicolumn{2}{c|}{\begin{tabular}[c]{@{}c@{}}Attn.\\      \shortcite{transformer} \end{tabular}} 
& \multicolumn{2}{c|}{\begin{tabular}[c]{@{}c@{}}Gated Attn.\\    \shortcite{gatedattn} \end{tabular}} 
& \multicolumn{2}{c|}{\begin{tabular}[c]{@{}c@{}}Enh. Attn.\\  \shortcite{freeformer} \end{tabular}} 
& \multicolumn{2}{c|}{\begin{tabular}[c]{@{}c@{}}Informer\\      \shortcite{informer} \end{tabular}} 
& \multicolumn{2}{c|}{\begin{tabular}[c]{@{}c@{}}Flowformer\\    \shortcite{flowformer} \end{tabular}} 
& \multicolumn{2}{c|}{\begin{tabular}[c]{@{}c@{}}Flashformer\\    \shortcite{flashattention} \end{tabular}} 
& \multicolumn{2}{c|}{\begin{tabular}[c]{@{}c@{}}Flatten\\      \shortcite{flattentrans} \end{tabular}} 
& \multicolumn{2}{c|}{\begin{tabular}[c]{@{}c@{}}Lin. Attn.\\    \shortcite{linear_softmax} \end{tabular}} 
& \multicolumn{2}{c}{\begin{tabular}[c]{@{}c@{}}Mamba\\      \shortcite{mamba} \end{tabular}} \\ \cmidrule(l){3-22} 
\multirow{-2}{*}{Dataset} & \multirow{-2}{*}{Hor.} & MSE                                      & MAE                                      & MSE                                    & MAE                                 & MSE                                       & MAE                                    & MSE                                     & MAE                                     & MSE                                    & MAE                                    & MSE                                     & MAE                                     & MSE                                      & MAE                                     & MSE                                    & MAE                                   & MSE                                     & MAE                                     & MSE                                                 & MAE                   \\ \midrule
                          & 96                     & {\color[HTML]{FF0000} \textbf{0.356}}    & {\color[HTML]{FF0000} \textbf{0.379}}    & 0.363                                  & 0.385                               & 0.364                                     & 0.386                                  & {\color[HTML]{FF0000} \textbf{0.356}}   & {\color[HTML]{0000FF} {\ul 0.380}}      & {\color[HTML]{0000FF} {\ul 0.357}}     & {\color[HTML]{0000FF} {\ul 0.380}}     & 0.359                                   & 0.381                                   & 0.362                                    & 0.385                                   & 0.367                                  & 0.386                                 & 0.366                                   & 0.386                                   & 0.363                                               & 0.388                 \\
                          & 192                    & {\color[HTML]{FF0000} \textbf{0.409}}    & {\color[HTML]{FF0000} \textbf{0.411}}    & 0.416                                  & 0.416                               & 0.415                                     & 0.416                                  & 0.411                                   & {\color[HTML]{0000FF} {\ul 0.412}}      & {\color[HTML]{0000FF} {\ul 0.410}}     & {\color[HTML]{0000FF} {\ul 0.412}}     & {\color[HTML]{0000FF} {\ul 0.410}}      & 0.413                                   & 0.416                                    & 0.416                                   & 0.418                                  & 0.417                                 & 0.415                                   & 0.416                                   & 0.416                                               & 0.417                 \\
                          & 336                    & {\color[HTML]{FF0000} \textbf{0.449}}    & {\color[HTML]{FF0000} \textbf{0.433}}    & 0.453                                  & 0.437                               & 0.455                                     & 0.438                                  & 0.454                                   & 0.437                                   & {\color[HTML]{0000FF} {\ul 0.450}}     & {\color[HTML]{0000FF} {\ul 0.434}}     & {\color[HTML]{0000FF} {\ul 0.450}}      & {\color[HTML]{0000FF} {\ul 0.434}}      & 0.453                                    & 0.437                                   & 0.458                                  & 0.440                                 & 0.460                                   & 0.440                                   & 0.456                                               & 0.439                 \\
                          & 720                    & {\color[HTML]{FF0000} \textbf{0.451}}    & {\color[HTML]{FF0000} \textbf{0.456}}    & 0.457                                  & 0.462                               & 0.459                                     & 0.463                                  & 0.461                                   & 0.464                                   & {\color[HTML]{0000FF} {\ul 0.456}}     & {\color[HTML]{0000FF} {\ul 0.460}}     & 0.463                                   & 0.466                                   & 0.461                                    & 0.464                                   & 0.469                                  & 0.469                                 & 0.460                                   & 0.463                                   & 0.467                                               & 0.469                 \\ \cmidrule(l){2-22} 
\multirow{-5}{*}{ETTh1}   & Avg                    & {\color[HTML]{FF0000} \textbf{0.416}}    & {\color[HTML]{FF0000} \textbf{0.420}}    & 0.422                                  & 0.425                               & 0.423                                     & 0.426                                  & 0.420                                   & 0.423                                   & {\color[HTML]{0000FF} {\ul 0.418}}     & {\color[HTML]{0000FF} {\ul 0.421}}     & 0.420                                   & 0.423                                   & 0.423                                    & 0.425                                   & 0.428                                  & 0.428                                 & 0.425                                   & 0.426                                   & 0.425                                               & 0.428                 \\ \midrule
                          & 96                     & {\color[HTML]{0000FF} {\ul 0.168}}       & {\color[HTML]{FF0000} \textbf{0.245}}    & 0.172                                  & 0.252                               & 0.169                                     & 0.247                                  & {\color[HTML]{0000FF} {\ul 0.168}}      & {\color[HTML]{0000FF} {\ul 0.246}}      & {\color[HTML]{0000FF} {\ul 0.168}}     & {\color[HTML]{0000FF} {\ul 0.246}}     & {\color[HTML]{FF0000} \textbf{0.167}}   & {\color[HTML]{0000FF} {\ul 0.246}}      & 0.173                                    & 0.251                                   & 0.169                                  & 0.247                                 & 0.167                                   & {\color[HTML]{0000FF} {\ul 0.246}}      & 0.171                                               & 0.249                 \\
                          & 192                    & {\color[HTML]{FF0000} \textbf{0.231}}    & {\color[HTML]{FF0000} \textbf{0.287}}    & 0.233                                  & 0.289                               & 0.233                                     & 0.289                                  & 0.233                                   & {\color[HTML]{0000FF} {\ul 0.288}}      & {\color[HTML]{FF0000} \textbf{0.231}}  & {\color[HTML]{0000FF} {\ul 0.288}}     & {\color[HTML]{0000FF} {\ul 0.232}}      & {\color[HTML]{0000FF} {\ul 0.288}}      & 0.233                                    & 0.289                                   & 0.236                                  & 0.290                                 & 0.234                                   & 0.290                                   & 0.234                                               & 0.291                 \\
                          & 336                    & {\color[HTML]{FF0000} \textbf{0.289}}    & {\color[HTML]{FF0000} \textbf{0.326}}    & 0.303                                  & 0.338                               & {\color[HTML]{0000FF} {\ul 0.292}}        & {\color[HTML]{0000FF} {\ul 0.327}}     & 0.294                                   & 0.328                                   & 0.295                                  & 0.329                                  & 0.295                                   & 0.329                                   & 0.302                                    & 0.336                                   & 0.297                                  & 0.331                                 & 0.295                                   & 0.329                                   & {\color[HTML]{0000FF} {\ul 0.292}}                  & 0.329                 \\
                          & 720                    & {\color[HTML]{FF0000} \textbf{0.386}}    & {\color[HTML]{FF0000} \textbf{0.382}}    & 0.393                                  & 0.387                               & 0.396                                     & 0.388                                  & 0.392                                   & {\color[HTML]{0000FF} {\ul 0.386}}      & 0.392                                  & {\color[HTML]{0000FF} {\ul 0.386}}     & 0.393                                   & 0.387                                   & 0.394                                    & 0.388                                   & 0.394                                  & 0.388                                 & 0.394                                   & 0.388                                   & {\color[HTML]{0000FF} {\ul 0.391}}                  & 0.387                 \\ \cmidrule(l){2-22} 
\multirow{-5}{*}{ETTm2}   & Avg                    & {\color[HTML]{FF0000} \textbf{0.268}}    & {\color[HTML]{FF0000} \textbf{0.310}}    & 0.275                                  & 0.316                               & 0.273                                     & 0.313                                  & 0.272                                   & {\color[HTML]{0000FF} {\ul 0.312}}      & {\color[HTML]{0000FF} {\ul 0.271}}     & {\color[HTML]{0000FF} {\ul 0.312}}     & 0.272                                   & {\color[HTML]{0000FF} {\ul 0.312}}      & 0.275                                    & 0.316                                   & 0.274                                  & 0.314                                 & 0.272                                   & 0.313                                   & 0.272                                               & 0.314                 \\ \midrule
                          & 96                     & {\color[HTML]{FF0000} \textbf{0.129}}    & {\color[HTML]{0000FF} {\ul 0.221}}       & 0.131                                  & 0.222                               & 0.133                                     & 0.223                                  & {\color[HTML]{FF0000} \textbf{0.129}}   & {\color[HTML]{FF0000} \textbf{0.220}}   & {\color[HTML]{FF0000} \textbf{0.129}}  & {\color[HTML]{0000FF} {\ul 0.221}}     & 0.131                                   & 0.222                                   & 0.131                                    & 0.222                                   & 0.131                                  & 0.222                                 & {\color[HTML]{0000FF} {\ul 0.130}}      & {\color[HTML]{0000FF} {\ul 0.221}}      & 0.140                                               & 0.231                 \\
                          & 192                    & {\color[HTML]{FF0000} \textbf{0.147}}    & {\color[HTML]{FF0000} \textbf{0.238}}    & 0.153                                  & 0.244                               & 0.153                                     & {\color[HTML]{0000FF} {\ul 0.243}}     & {\color[HTML]{FF0000} \textbf{0.147}}   & {\color[HTML]{FF0000} \textbf{0.238}}   & {\color[HTML]{0000FF} {\ul 0.148}}     & {\color[HTML]{FF0000} \textbf{0.238}}  & 0.154                                   & 0.245                                   & 0.154                                    & {\color[HTML]{0000FF} {\ul 0.243}}      & 0.154                                  & 0.244                                 & 0.153                                   & {\color[HTML]{0000FF} {\ul 0.243}}      & 0.159                                               & 0.249                 \\
                          & 336                    & {\color[HTML]{0000FF} {\ul 0.158}}       & {\color[HTML]{FF0000} \textbf{0.250}}    & 0.167                                  & 0.258                               & 0.161                                     & {\color[HTML]{0000FF} {\ul 0.252}}     & {\color[HTML]{FF0000} \textbf{0.157}}   & {\color[HTML]{FF0000} \textbf{0.250}}   & 0.159                                  & {\color[HTML]{0000FF} {\ul 0.252}}     & 0.166                                   & 0.259                                   & 0.167                                    & 0.259                                   & 0.165                                  & 0.259                                 & 0.162                                   & 0.256                                   & 0.175                                               & 0.265                 \\
                          & 720                    & {\color[HTML]{FF0000} \textbf{0.178}}    & {\color[HTML]{FF0000} \textbf{0.271}}    & 0.185                                  & 0.280                               & {\color[HTML]{FF0000} \textbf{0.178}}     & {\color[HTML]{0000FF} {\ul 0.274}}     & 0.180                                   & {\color[HTML]{0000FF} {\ul 0.274}}      & {\color[HTML]{0000FF} {\ul 0.179}}     & {\color[HTML]{0000FF} {\ul 0.274}}     & 0.180                                   & 0.277                                   & 0.183                                    & 0.279                                   & 0.183                                  & 0.279                                 & 0.184                                   & 0.281                                   & 0.204                                               & 0.292                 \\ \cmidrule(l){2-22} 
\multirow{-5}{*}{ECL}     & Avg                    & {\color[HTML]{FF0000} \textbf{0.153}}    & {\color[HTML]{FF0000} \textbf{0.245}}    & 0.159                                  & 0.251                               & 0.156                                     & 0.248                                  & {\color[HTML]{FF0000} \textbf{0.153}}   & {\color[HTML]{FF0000} \textbf{0.245}}   & {\color[HTML]{0000FF} {\ul 0.154}}     & {\color[HTML]{0000FF} {\ul 0.246}}     & 0.158                                   & 0.250                                   & 0.159                                    & 0.251                                   & 0.158                                  & 0.251                                 & 0.157                                   & 0.250                                   & 0.169                                               & 0.259                 \\ \midrule
                          & 96                     & {\color[HTML]{FF0000} \textbf{0.175}}    & {\color[HTML]{FF0000} \textbf{0.196}}    & 0.189                                  & 0.208                               & 0.205                                     & 0.209                                  & 0.181                                   & 0.201                                   & {\color[HTML]{0000FF} {\ul 0.176}}     & {\color[HTML]{0000FF} {\ul 0.198}}     & 0.187                                   & 0.201                                   & 0.188                                    & 0.203                                   & 0.191                                  & 0.201                                 & 0.195                                   & 0.204                                   & 0.189                                               & 0.214                 \\
                          & 192                    & {\color[HTML]{FF0000} \textbf{0.202}}    & {\color[HTML]{FF0000} \textbf{0.220}}    & 0.213                                  & 0.229                               & 0.220                                     & 0.228                                  & {\color[HTML]{0000FF} {\ul 0.204}}      & {\color[HTML]{0000FF} {\ul 0.222}}      & {\color[HTML]{0000FF} {\ul 0.204}}     & {\color[HTML]{0000FF} {\ul 0.222}}     & 0.207                                   & 0.224                                   & 0.213                                    & 0.229                                   & 0.219                                  & 0.228                                 & 0.217                                   & 0.227                                   & 0.217                                               & 0.235                 \\
                          & 336                    & {\color[HTML]{FF0000} \textbf{0.224}}    & {\color[HTML]{FF0000} \textbf{0.240}}    & 0.241                                  & 0.247                               & 0.235                                     & 0.245                                  & {\color[HTML]{0000FF} {\ul 0.226}}      & {\color[HTML]{FF0000} \textbf{0.240}}   & 0.227                                  & {\color[HTML]{0000FF} {\ul 0.242}}     & 0.235                                   & {\color[HTML]{0000FF} {\ul 0.242}}      & 0.239                                    & 0.247                                   & 0.236                                  & 0.244                                 & 0.237                                   & 0.245                                   & 0.248                                               & 0.256                 \\
                          & 720                    & {\color[HTML]{FF0000} \textbf{0.234}}    & {\color[HTML]{FF0000} \textbf{0.247}}    & 0.238                                  & 0.253                               & 0.243                                     & 0.252                                  & {\color[HTML]{0000FF} {\ul 0.235}}      & {\color[HTML]{FF0000} \textbf{0.247}}   & 0.238                                  & {\color[HTML]{0000FF} {\ul 0.250}}     & 0.244                                   & 0.252                                   & 0.241                                    & 0.253                                   & 0.245                                  & 0.252                                 & 0.243                                   & 0.252                                   & 0.242                                               & 0.254                 \\ \cmidrule(l){2-22} 
\multirow{-5}{*}{Solar}   & Avg                    & {\color[HTML]{FF0000} \textbf{0.209}}    & {\color[HTML]{FF0000} \textbf{0.226}}    & 0.220                                  & 0.234                               & 0.226                                     & 0.233                                  & {\color[HTML]{0000FF} {\ul 0.211}}      & {\color[HTML]{0000FF} {\ul 0.227}}      & {\color[HTML]{0000FF} {\ul 0.211}}     & 0.228                                  & 0.218                                   & 0.230                                   & 0.220                                    & 0.233                                   & 0.223                                  & 0.231                                 & 0.223                                   & 0.232                                   & 0.224                                               & 0.240                 \\ \midrule
                          & 96                     & {\color[HTML]{FF0000} \textbf{0.150}}    & {\color[HTML]{0000FF} {\ul 0.187}}       & 0.152                                  & {\color[HTML]{0000FF} {\ul 0.187}}  & 0.154                                     & 0.190                                  & {\color[HTML]{FF0000} \textbf{0.150}}   & {\color[HTML]{FF0000} \textbf{0.186}}   & {\color[HTML]{0000FF} {\ul 0.151}}     & {\color[HTML]{0000FF} {\ul 0.187}}     & {\color[HTML]{0000FF} {\ul 0.151}}      & {\color[HTML]{FF0000} \textbf{0.186}}   & 0.152                                    & 0.188                                   & 0.156                                  & 0.192                                 & 0.154                                   & 0.191                                   & 0.160                                               & 0.197                 \\
                          & 192                    & {\color[HTML]{FF0000} \textbf{0.198}}    & {\color[HTML]{FF0000} \textbf{0.234}}    & 0.204                                  & 0.239                               & 0.202                                     & 0.237                                  & 0.201                                   & {\color[HTML]{0000FF} {\ul 0.236}}      & {\color[HTML]{0000FF} {\ul 0.199}}     & {\color[HTML]{FF0000} \textbf{0.234}}  & 0.200                                   & {\color[HTML]{FF0000} \textbf{0.234}}   & 0.203                                    & 0.237                                   & 0.206                                  & 0.239                                 & 0.203                                   & 0.238                                   & 0.262                                               & 0.297                 \\
                          & 336                    & {\color[HTML]{FF0000} \textbf{0.253}}    & {\color[HTML]{FF0000} \textbf{0.275}}    & 0.257                                  & 0.279                               & 0.256                                     & 0.278                                  & {\color[HTML]{0000FF} {\ul 0.255}}      & {\color[HTML]{0000FF} {\ul 0.277}}      & 0.257                                  & {\color[HTML]{0000FF} {\ul 0.277}}     & {\color[HTML]{FF0000} \textbf{0.253}}   & {\color[HTML]{FF0000} \textbf{0.275}}   & 0.257                                    & 0.279                                   & 0.260                                  & 0.281                                 & 0.262                                   & 0.282                                   & 0.261                                               & 0.281                 \\
                          & 720                    & {\color[HTML]{FF0000} \textbf{0.333}}    & {\color[HTML]{FF0000} \textbf{0.329}}    & 0.336                                  & 0.332                               & {\color[HTML]{0000FF} {\ul 0.335}}        & 0.331                                  & {\color[HTML]{0000FF} {\ul 0.335}}      & {\color[HTML]{0000FF} {\ul 0.330}}      & 0.336                                  & 0.331                                  & {\color[HTML]{FF0000} \textbf{0.333}}   & {\color[HTML]{FF0000} \textbf{0.329}}   & {\color[HTML]{0000FF} {\ul 0.335}}       & 0.331                                   & 0.341                                  & 0.335                                 & 0.339                                   & 0.332                                   & 0.341                                               & 0.334                 \\ \cmidrule(l){2-22} 
\multirow{-5}{*}{Weather} & Avg                    & {\color[HTML]{FF0000} \textbf{0.233}}    & {\color[HTML]{FF0000} \textbf{0.256}}    & 0.237                                  & 0.259                               & 0.237                                     & 0.259                                  & 0.235                                   & {\color[HTML]{0000FF} {\ul 0.257}}      & 0.236                                  & {\color[HTML]{0000FF} {\ul 0.257}}     & {\color[HTML]{0000FF} {\ul 0.234}}      & {\color[HTML]{FF0000} \textbf{0.256}}   & 0.237                                    & 0.259                                   & 0.241                                  & 0.262                                 & 0.239                                   & 0.261                                   & 0.256                                               & 0.277                 \\ \midrule
                          & 12                     & {\color[HTML]{FF0000} \textbf{0.059}}    & {\color[HTML]{FF0000} \textbf{0.158}}    & {\color[HTML]{FF0000} \textbf{0.059}}  & {\color[HTML]{0000FF} {\ul 0.159}}  & {\color[HTML]{0000FF} {\ul 0.060}}        & 0.160                                  & {\color[HTML]{FF0000} \textbf{0.059}}   & {\color[HTML]{FF0000} \textbf{0.158}}   & {\color[HTML]{FF0000} \textbf{0.059}}  & {\color[HTML]{0000FF} {\ul 0.159}}     & {\color[HTML]{0000FF} {\ul 0.060}}      & {\color[HTML]{0000FF} {\ul 0.159}}      & {\color[HTML]{FF0000} \textbf{0.059}}    & {\color[HTML]{FF0000} \textbf{0.158}}   & 0.061                                  & 0.160                                 & {\color[HTML]{0000FF} {\ul 0.060}}      & {\color[HTML]{0000FF} {\ul 0.159}}      & 0.062                                               & 0.161                 \\
                          & 24                     & {\color[HTML]{FF0000} \textbf{0.075}}    & {\color[HTML]{FF0000} \textbf{0.177}}    & 0.078                                  & 0.182                               & 0.078                                     & 0.181                                  & {\color[HTML]{0000FF} {\ul 0.076}}      & {\color[HTML]{0000FF} {\ul 0.178}}      & {\color[HTML]{0000FF} {\ul 0.076}}     & {\color[HTML]{0000FF} {\ul 0.178}}     & {\color[HTML]{0000FF} {\ul 0.076}}      & 0.179                                   & 0.077                                    & 0.181                                   & 0.079                                  & 0.182                                 & 0.079                                   & 0.183                                   & 0.078                                               & 0.182                 \\
                          & 48                     & {\color[HTML]{FF0000} \textbf{0.102}}    & {\color[HTML]{FF0000} \textbf{0.208}}    & 0.107                                  & 0.214                               & 0.109                                     & 0.215                                  & {\color[HTML]{0000FF} {\ul 0.104}}      & {\color[HTML]{0000FF} {\ul 0.209}}      & 0.106                                  & 0.211                                  & 0.107                                   & 0.213                                   & 0.108                                    & 0.215                                   & 0.111                                  & 0.217                                 & 0.110                                   & 0.215                                   & 0.113                                               & 0.219                 \\
                          & 96                     & {\color[HTML]{FF0000} \textbf{0.138}}    & {\color[HTML]{FF0000} \textbf{0.246}}    & 0.154                                  & 0.262                               & 0.148                                     & 0.256                                  & {\color[HTML]{0000FF} {\ul 0.141}}      & {\color[HTML]{0000FF} {\ul 0.249}}      & 0.143                                  & 0.251                                  & 0.152                                   & 0.255                                   & 0.156                                    & 0.262                                   & 0.159                                  & 0.261                                 & 0.155                                   & 0.258                                   & 0.170                                               & 0.270                 \\ \cmidrule(l){2-22} 
\multirow{-5}{*}{PEMS03}  & Avg                    & {\color[HTML]{FF0000} \textbf{0.093}}    & {\color[HTML]{FF0000} \textbf{0.197}}    & 0.100                                  & 0.204                               & 0.099                                     & 0.203                                  & {\color[HTML]{0000FF} {\ul 0.095}}      & {\color[HTML]{0000FF} {\ul 0.198}}      & 0.096                                  & 0.200                                  & 0.099                                   & 0.201                                   & 0.100                                    & 0.204                                   & 0.102                                  & 0.205                                 & 0.101                                   & 0.204                                   & 0.106                                               & 0.208                 \\ \bottomrule
\end{tabular}

}
\caption{Full results of performance comparison between vecTrans and various attention mechanisms and Mamba. This table corresponds to Table~\ref{tab_vec_attn} in the main paper.}
\label{tab_vec_attn_appd}
\end{center}
\end{table*}


\begin{table*}[t]
\begin{center}
\renewcommand{\arraystretch}{1.3}
{\fontsize{8}{10}\selectfont
\setlength{\tabcolsep}{5pt}
\begin{tabular}{@{}cc|cc|cc|cc|cc|cc|cc|cc@{}}
\toprule
                          &                        & \multicolumn{2}{c|}{\begin{tabular}[c]{@{}c@{}}WFMLoss\\      (Ours)\end{tabular}} & \multicolumn{2}{c|}{\begin{tabular}[c]{@{}c@{}}TransDF\\      \shortcite{transdf} \end{tabular}} & \multicolumn{2}{c|}{\begin{tabular}[c]{@{}c@{}}DBLoss\\      \shortcite{dbloss} \end{tabular}} & \multicolumn{2}{c|}{\begin{tabular}[c]{@{}c@{}}FreDF\\      \shortcite{fredf} \end{tabular}} & \multicolumn{2}{c|}{\begin{tabular}[c]{@{}c@{}}W\_MAE\\      \shortcite{card} \end{tabular}} & \multicolumn{2}{c|}{MAE}                                                   & \multicolumn{2}{c}{MSE} \\ \cmidrule(l){3-16} 
\multirow{-2}{*}{Dataset} & \multirow{-2}{*}{Hor.} & MSE                                      & MAE                                     & MSE                                                  & MAE                     & MSE                                   & MAE                                   & MSE                                                 & MAE                    & MSE                                  & MAE                                    & MSE                                & MAE                                   & MSE        & MAE        \\ \midrule
                          & 96                     & {\color[HTML]{FF0000} \textbf{0.356}}    & {\color[HTML]{FF0000} \textbf{0.379}}   & 0.372                                                & 0.395                   & 0.377                                 & 0.395                                 & {\color[HTML]{0000FF} {\ul 0.369}}                  & 0.394                  & 0.382                                & {\color[HTML]{0000FF} {\ul 0.391}}     & 0.378                              & {\color[HTML]{0000FF} {\ul 0.391}}    & 0.386      & 0.404      \\
                          & 192                    & {\color[HTML]{FF0000} \textbf{0.409}}    & {\color[HTML]{FF0000} \textbf{0.411}}   & 0.436                                                & 0.436                   & 0.433                                 & 0.430                                 & {\color[HTML]{0000FF} {\ul 0.425}}                  & 0.431                  & 0.430                                & {\color[HTML]{0000FF} {\ul 0.422}}     & 0.429                              & 0.424                                 & 0.431      & 0.436      \\
                          & 336                    & {\color[HTML]{FF0000} \textbf{0.449}}    & {\color[HTML]{FF0000} \textbf{0.433}}   & {\color[HTML]{0000FF} {\ul 0.467}}                   & 0.452                   & 0.468                                 & 0.450                                 & 0.478                                               & 0.462                  & 0.479                                & {\color[HTML]{0000FF} {\ul 0.443}}     & 0.487                              & 0.452                                 & 0.485      & 0.463      \\
                          & 720                    & {\color[HTML]{FF0000} \textbf{0.451}}    & {\color[HTML]{FF0000} \textbf{0.456}}   & 0.726                                                & 0.605                   & 0.508                                 & 0.478                                 & 0.720                                               & 0.601                  & {\color[HTML]{0000FF} {\ul 0.492}}   & {\color[HTML]{0000FF} {\ul 0.467}}     & 0.496                              & 0.470                                 & 0.499      & 0.482      \\ \cmidrule(l){2-16} 
\multirow{-5}{*}{ETTh1}   & Avg                    & {\color[HTML]{FF0000} \textbf{0.416}}    & {\color[HTML]{FF0000} \textbf{0.420}}   & 0.500                                                & 0.472                   & {\color[HTML]{0000FF} {\ul 0.446}}    & 0.438                                 & 0.498                                               & 0.472                  & {\color[HTML]{0000FF} {\ul 0.446}}   & {\color[HTML]{0000FF} {\ul 0.431}}     & 0.447                              & 0.434                                 & 0.450      & 0.446      \\ \midrule
                          & 96                     & {\color[HTML]{FF0000} \textbf{0.168}}    & {\color[HTML]{FF0000} \textbf{0.245}}   & 0.176                                                & 0.254                   & 0.178                                 & 0.252                                 & {\color[HTML]{0000FF} {\ul 0.173}}                  & 0.251                  & 0.176                                & 0.251                                  & 0.174                              & {\color[HTML]{0000FF} {\ul 0.249}}    & 0.177      & 0.257      \\
                          & 192                    & {\color[HTML]{FF0000} \textbf{0.231}}    & {\color[HTML]{FF0000} \textbf{0.287}}   & 0.241                                                & 0.297                   & 0.239                                 & 0.296                                 & 0.238                                               & 0.295                  & {\color[HTML]{0000FF} {\ul 0.237}}   & {\color[HTML]{0000FF} {\ul 0.293}}     & 0.243                              & 0.298                                 & 0.244      & 0.302      \\
                          & 336                    & {\color[HTML]{FF0000} \textbf{0.289}}    & {\color[HTML]{FF0000} \textbf{0.326}}   & 0.339                                                & 0.370                   & 0.338                                 & 0.369                                 & 0.338                                               & 0.369                  & 0.300                                & 0.334                                  & {\color[HTML]{0000FF} {\ul 0.296}} & {\color[HTML]{0000FF} {\ul 0.332}}    & 0.338      & 0.369      \\
                          & 720                    & {\color[HTML]{FF0000} \textbf{0.386}}    & {\color[HTML]{FF0000} \textbf{0.382}}   & 0.410                                                & 0.405                   & 0.433                                 & 0.419                                 & 0.393                                               & 0.390                  & 0.398                                & 0.392                                  & {\color[HTML]{0000FF} {\ul 0.392}} & {\color[HTML]{0000FF} {\ul 0.389}}    & 0.402      & 0.398      \\ \cmidrule(l){2-16} 
\multirow{-5}{*}{ETTm2}   & Avg                    & {\color[HTML]{FF0000} \textbf{0.268}}    & {\color[HTML]{FF0000} \textbf{0.310}}   & 0.291                                                & 0.331                   & 0.297                                 & 0.334                                 & 0.285                                               & 0.326                  & 0.278                                & {\color[HTML]{0000FF} {\ul 0.317}}     & {\color[HTML]{0000FF} {\ul 0.276}} & {\color[HTML]{0000FF} {\ul 0.317}}    & 0.290      & 0.331      \\ \midrule
                          & 96                     & {\color[HTML]{FF0000} \textbf{0.129}}    & {\color[HTML]{FF0000} \textbf{0.221}}   & 0.133                                                & 0.228                   & 0.131                                 & 0.225                                 & 0.132                                               & 0.227                  & {\color[HTML]{0000FF} {\ul 0.130}}   & {\color[HTML]{0000FF} {\ul 0.221}}     & 0.131                              & 0.222                                 & 0.135      & 0.231      \\
                          & 192                    & {\color[HTML]{FF0000} \textbf{0.147}}    & {\color[HTML]{FF0000} \textbf{0.238}}   & 0.152                                                & 0.248                   & 0.163                                 & 0.254                                 & {\color[HTML]{0000FF} {\ul 0.151}}                  & 0.246                  & 0.152                                & {\color[HTML]{0000FF} {\ul 0.240}}     & 0.152                              & 0.241                                 & 0.154      & 0.249      \\
                          & 336                    & {\color[HTML]{FF0000} \textbf{0.158}}    & {\color[HTML]{FF0000} \textbf{0.250}}   & 0.178                                                & 0.275                   & 0.191                                 & 0.279                                 & 0.199                                               & 0.285                  & {\color[HTML]{0000FF} {\ul 0.165}}   & {\color[HTML]{0000FF} {\ul 0.255}}     & 0.165                              & 0.256                                 & 0.170      & 0.268      \\
                          & 720                    & {\color[HTML]{FF0000} \textbf{0.178}}    & {\color[HTML]{FF0000} \textbf{0.271}}   & 0.192                                                & 0.290                   & {\color[HTML]{0000FF} {\ul 0.188}}    & 0.283                                 & 0.205                                               & 0.292                  & 0.194                                & {\color[HTML]{0000FF} {\ul 0.280}}     & 0.191                              & {\color[HTML]{0000FF} {\ul 0.280}}    & 0.192      & 0.292      \\ \cmidrule(l){2-16} 
\multirow{-5}{*}{ECL}     & Avg                    & {\color[HTML]{FF0000} \textbf{0.153}}    & {\color[HTML]{FF0000} \textbf{0.245}}   & 0.164                                                & 0.260                   & 0.168                                 & 0.260                                 & 0.172                                               & 0.262                  & {\color[HTML]{0000FF} {\ul 0.160}}   & {\color[HTML]{0000FF} {\ul 0.249}}     & {\color[HTML]{0000FF} {\ul 0.160}} & {\color[HTML]{0000FF} {\ul 0.249}}    & 0.162      & 0.260      \\ \midrule
                          & 96                     & {\color[HTML]{FF0000} \textbf{0.175}}    & 0.196                                   & 0.189                                                & 0.215                   & 0.181                                 & 0.198                                 & 0.186                                               & 0.208                  & {\color[HTML]{0000FF} {\ul 0.180}}   & {\color[HTML]{FF0000} \textbf{0.190}}  & 0.183                              & {\color[HTML]{0000FF} {\ul 0.193}}    & 0.192      & 0.222      \\
                          & 192                    & {\color[HTML]{FF0000} \textbf{0.202}}    & 0.220                                   & 0.216                                                & 0.247                   & 0.209                                 & 0.221                                 & 0.212                                               & 0.237                  & {\color[HTML]{0000FF} {\ul 0.207}}   & {\color[HTML]{FF0000} \textbf{0.211}}  & 0.210                              & {\color[HTML]{0000FF} {\ul 0.214}}    & 0.217      & 0.250      \\
                          & 336                    & {\color[HTML]{FF0000} \textbf{0.224}}    & 0.240                                   & 0.239                                                & 0.263                   & 0.229                                 & 0.240                                 & 0.230                                               & 0.257                  & {\color[HTML]{0000FF} {\ul 0.229}}   & {\color[HTML]{FF0000} \textbf{0.229}}  & 0.230                              & {\color[HTML]{0000FF} {\ul 0.232}}    & 0.232      & 0.268      \\
                          & 720                    & {\color[HTML]{FF0000} \textbf{0.234}}    & 0.247                                   & {\color[HTML]{0000FF} {\ul 0.238}}                   & 0.268                   & 0.863                                 & 0.595                                 & 0.842                                               & 0.661                  & 0.241                                & {\color[HTML]{0000FF} {\ul 0.237}}     & 0.238                              & {\color[HTML]{FF0000} \textbf{0.236}} & 0.239      & 0.275      \\ \cmidrule(l){2-16} 
\multirow{-5}{*}{Solar}   & Avg                    & {\color[HTML]{FF0000} \textbf{0.209}}    & 0.226                                   & 0.220                                                & 0.248                   & 0.370                                 & 0.313                                 & 0.367                                               & 0.341                  & {\color[HTML]{0000FF} {\ul 0.214}}   & {\color[HTML]{FF0000} \textbf{0.217}}  & 0.215                              & {\color[HTML]{0000FF} {\ul 0.219}}    & 0.220      & 0.254      \\ \midrule
                          & 96                     & {\color[HTML]{FF0000} \textbf{0.150}}    & {\color[HTML]{FF0000} \textbf{0.187}}   & 0.157                                                & 0.201                   & {\color[HTML]{0000FF} {\ul 0.151}}    & 0.189                                 & 0.152                                               & 0.195                  & 0.154                                & 0.189                                  & 0.152                              & {\color[HTML]{0000FF} {\ul 0.189}}    & 0.158      & 0.203      \\
                          & 192                    & {\color[HTML]{FF0000} \textbf{0.198}}    & {\color[HTML]{FF0000} \textbf{0.234}}   & 0.209                                                & 0.249                   & 0.214                                 & 0.249                                 & 0.208                                               & 0.248                  & {\color[HTML]{0000FF} {\ul 0.201}}   & {\color[HTML]{0000FF} {\ul 0.236}}     & {\color[HTML]{0000FF} {\ul 0.201}} & 0.237                                 & 0.207      & 0.248      \\
                          & 336                    & {\color[HTML]{FF0000} \textbf{0.253}}    & {\color[HTML]{FF0000} \textbf{0.275}}   & 0.264                                                & 0.290                   & 0.272                                 & 0.290                                 & 0.262                                               & 0.288                  & 0.265                                & {\color[HTML]{0000FF} {\ul 0.283}}     & {\color[HTML]{0000FF} {\ul 0.261}} & {\color[HTML]{0000FF} {\ul 0.283}}    & 0.263      & 0.288      \\
                          & 720                    & {\color[HTML]{FF0000} \textbf{0.333}}    & {\color[HTML]{FF0000} \textbf{0.329}}   & 0.347                                                & 0.345                   & 0.387                                 & 0.360                                 & 0.343                                               & 0.343                  & {\color[HTML]{0000FF} {\ul 0.336}}   & {\color[HTML]{0000FF} {\ul 0.330}}     & 0.340                              & 0.332                                 & 0.346      & 0.343      \\ \cmidrule(l){2-16} 
\multirow{-5}{*}{Weather} & Avg                    & {\color[HTML]{FF0000} \textbf{0.233}}    & {\color[HTML]{FF0000} \textbf{0.256}}   & 0.244                                                & 0.271                   & 0.256                                 & 0.272                                 & 0.241                                               & 0.268                  & 0.239                                & {\color[HTML]{0000FF} {\ul 0.259}}     & {\color[HTML]{0000FF} {\ul 0.238}} & 0.260                                 & 0.243      & 0.270      \\ \midrule
                          & 12                     & {\color[HTML]{FF0000} \textbf{0.059}}    & {\color[HTML]{FF0000} \textbf{0.158}}   & 0.061                                                & 0.160                   & {\color[HTML]{0000FF} {\ul 0.059}}    & 0.159                                 & 0.060                                               & 0.160                  & {\color[HTML]{0000FF} {\ul 0.059}}   & {\color[HTML]{0000FF} {\ul 0.158}}     & 0.060                              & 0.159                                 & 0.061      & 0.162      \\
                          & 24                     & {\color[HTML]{FF0000} \textbf{0.075}}    & {\color[HTML]{FF0000} \textbf{0.177}}   & 0.078                                                & 0.181                   & {\color[HTML]{0000FF} {\ul 0.077}}    & {\color[HTML]{0000FF} {\ul 0.178}}    & {\color[HTML]{0000FF} {\ul 0.077}}                  & 0.179                  & {\color[HTML]{0000FF} {\ul 0.077}}   & 0.179                                  & {\color[HTML]{0000FF} {\ul 0.077}} & 0.179                                 & 0.080      & 0.183      \\
                          & 48                     & {\color[HTML]{FF0000} \textbf{0.102}}    & {\color[HTML]{FF0000} \textbf{0.208}}   & 0.107                                                & 0.213                   & 0.106                                 & 0.212                                 & 0.105                                               & 0.211                  & {\color[HTML]{0000FF} {\ul 0.104}}   & {\color[HTML]{0000FF} {\ul 0.209}}     & 0.105                              & 0.210                                 & 0.108      & 0.215      \\
                          & 96                     & {\color[HTML]{FF0000} \textbf{0.138}}    & {\color[HTML]{FF0000} \textbf{0.246}}   & 0.145                                                & 0.253                   & 0.144                                 & 0.250                                 & 0.141                                               & 0.251                  & {\color[HTML]{0000FF} {\ul 0.139}}   & {\color[HTML]{0000FF} {\ul 0.246}}     & {\color[HTML]{0000FF} {\ul 0.139}} & 0.248                                 & 0.144      & 0.255      \\ \cmidrule(l){2-16} 
\multirow{-5}{*}{PEMS03}  & Avg                    & {\color[HTML]{FF0000} \textbf{0.093}}    & {\color[HTML]{FF0000} \textbf{0.197}}   & 0.098                                                & 0.202                   & 0.096                                 & 0.200                                 & 0.096                                               & 0.200                  & {\color[HTML]{0000FF} {\ul 0.095}}   & {\color[HTML]{0000FF} {\ul 0.198}}     & {\color[HTML]{0000FF} {\ul 0.095}} & 0.199                                 & 0.098      & 0.204      \\ \bottomrule
\end{tabular}
}
\caption{Comparison between WFMLoss and other state-of-the-art losses. This table corresponds to Table~\ref{tab_wfmloss_compare} in the main paper.}
\label{tab_wfmloss_comp_appd}
\end{center}
\end{table*}


\begin{table*}[t]
\begin{center}
\renewcommand{\arraystretch}{1.3}
{\fontsize{8}{10}\selectfont
\setlength{\tabcolsep}{4.5pt}
\begin{tabular}{@{}cc|cccc|cccc|cccc|cccc@{}}
\toprule
\multicolumn{2}{c|}{}                        
& \multicolumn{4}{c|}{\begin{tabular}[c]{@{}c@{}}iTransformer\\   \shortcite{itransformer} \end{tabular}}                                                                                                
& \multicolumn{4}{c|}{\begin{tabular}[c]{@{}c@{}}PatchTST\\      \shortcite{patchtst} \end{tabular}}                                                                    
& \multicolumn{4}{c|}{\begin{tabular}[c]{@{}c@{}}Leddam\\      \shortcite{Leddam_icml} \end{tabular}}                                                                                                      
& \multicolumn{4}{c}{\begin{tabular}[c]{@{}c@{}}Fredformer\\     \shortcite{fredformer} \end{tabular}}                                                                   \\ \cmidrule(l){3-18} 
\multicolumn{2}{c|}{\multirow{-2}{*}{Model}} & \multicolumn{2}{c|}{Attn.}                                                                         & \multicolumn{2}{c|}{VecTrans}                                                 & \multicolumn{2}{c|}{Attn.}                                         & \multicolumn{2}{c|}{VecTrans}                                                 & \multicolumn{2}{c|}{Attn.}                                                                         & \multicolumn{2}{c|}{VecTrans}                                                 & \multicolumn{2}{c|}{Attn.}                                         & \multicolumn{2}{c}{VecTrans}                                                  \\ \midrule
\multicolumn{2}{c|}{Metric}                  & MSE                                   & \multicolumn{1}{c|}{MAE}                                   & MSE                                   & MAE                                   & MSE   & \multicolumn{1}{c|}{MAE}                                   & MSE                                   & MAE                                   & MSE                                   & \multicolumn{1}{c|}{MAE}                                   & MSE                                   & MAE                                   & MSE                                   & \multicolumn{1}{c|}{MAE}   & MSE                                   & MAE                                   \\ \midrule
                                 & 96        & 0.334                                 & \multicolumn{1}{c|}{0.368}                                 & {\color[HTML]{FF0000} \textbf{0.322}} & {\color[HTML]{FF0000} \textbf{0.359}} & 0.329 & \multicolumn{1}{c|}{0.367}                                 & {\color[HTML]{FF0000} \textbf{0.318}} & {\color[HTML]{FF0000} \textbf{0.359}} & 0.319                                 & \multicolumn{1}{c|}{0.359}                                 & {\color[HTML]{FF0000} \textbf{0.315}} & {\color[HTML]{FF0000} \textbf{0.353}} & 0.326                                 & \multicolumn{1}{c|}{0.361} & {\color[HTML]{FF0000} \textbf{0.315}} & {\color[HTML]{FF0000} \textbf{0.341}} \\
                                 & 192       & 0.377                                 & \multicolumn{1}{c|}{0.391}                                 & {\color[HTML]{FF0000} \textbf{0.369}} & {\color[HTML]{FF0000} \textbf{0.386}} & 0.367 & \multicolumn{1}{c|}{0.385}                                 & {\color[HTML]{FF0000} \textbf{0.360}} & {\color[HTML]{FF0000} \textbf{0.384}} & 0.369                                 & \multicolumn{1}{c|}{0.383}                                 & {\color[HTML]{FF0000} \textbf{0.358}} & {\color[HTML]{FF0000} \textbf{0.379}} & {\color[HTML]{FF0000} \textbf{0.363}} & \multicolumn{1}{c|}{0.380} & 0.364                                 & {\color[HTML]{FF0000} \textbf{0.368}} \\
                                 & 336       & 0.426                                 & \multicolumn{1}{c|}{0.420}                                 & {\color[HTML]{FF0000} \textbf{0.398}} & {\color[HTML]{FF0000} \textbf{0.406}} & 0.399 & \multicolumn{1}{c|}{0.410}                                 & {\color[HTML]{FF0000} \textbf{0.387}} & {\color[HTML]{FF0000} \textbf{0.405}} & 0.394                                 & \multicolumn{1}{c|}{0.402}                                 & {\color[HTML]{FF0000} \textbf{0.392}} & {\color[HTML]{FF0000} \textbf{0.402}} & {\color[HTML]{FF0000} \textbf{0.395}} & \multicolumn{1}{c|}{0.403} & 0.397                                 & {\color[HTML]{FF0000} \textbf{0.392}} \\
                                 & 720       & 0.491                                 & \multicolumn{1}{c|}{0.459}                                 & {\color[HTML]{FF0000} \textbf{0.475}} & {\color[HTML]{FF0000} \textbf{0.447}} & 0.454 & \multicolumn{1}{c|}{{\color[HTML]{FF0000} \textbf{0.439}}} & {\color[HTML]{FF0000} \textbf{0.450}} & 0.444                                 & 0.460                                 & \multicolumn{1}{c|}{0.442}                                 & {\color[HTML]{FF0000} \textbf{0.454}} & {\color[HTML]{FF0000} \textbf{0.439}} & {\color[HTML]{FF0000} \textbf{0.453}} & \multicolumn{1}{c|}{0.438} & 0.459                                 & {\color[HTML]{FF0000} \textbf{0.429}} \\ \cmidrule(l){2-18} 
\multirow{-5}{*}{ETTm1}          & Avg       & 0.407                                 & \multicolumn{1}{c|}{0.410}                                 & {\color[HTML]{FF0000} \textbf{0.391}} & {\color[HTML]{FF0000} \textbf{0.399}} & 0.387 & \multicolumn{1}{c|}{0.400}                                 & {\color[HTML]{FF0000} \textbf{0.379}} & {\color[HTML]{FF0000} \textbf{0.398}} & 0.386                                 & \multicolumn{1}{c|}{0.397}                                 & {\color[HTML]{FF0000} \textbf{0.379}} & {\color[HTML]{FF0000} \textbf{0.393}} & {\color[HTML]{FF0000} \textbf{0.384}} & \multicolumn{1}{c|}{0.396} & {\color[HTML]{FF0000} \textbf{0.384}} & {\color[HTML]{FF0000} \textbf{0.382}} \\ \midrule
                                 & 96        & 0.148                                 & \multicolumn{1}{c|}{0.240}                                 & {\color[HTML]{FF0000} \textbf{0.137}} & {\color[HTML]{FF0000} \textbf{0.232}} & 0.161 & \multicolumn{1}{c|}{0.250}                                 & {\color[HTML]{FF0000} \textbf{0.154}} & {\color[HTML]{FF0000} \textbf{0.244}} & 0.141                                 & \multicolumn{1}{c|}{0.235}                                 & {\color[HTML]{FF0000} \textbf{0.138}} & {\color[HTML]{FF0000} \textbf{0.231}} & 0.147                                 & \multicolumn{1}{c|}{0.241} & {\color[HTML]{FF0000} \textbf{0.140}} & {\color[HTML]{FF0000} \textbf{0.229}} \\
                                 & 192       & 0.162                                 & \multicolumn{1}{c|}{0.253}                                 & {\color[HTML]{FF0000} \textbf{0.154}} & {\color[HTML]{FF0000} \textbf{0.248}} & 0.199 & \multicolumn{1}{c|}{0.289}                                 & {\color[HTML]{FF0000} \textbf{0.166}} & {\color[HTML]{FF0000} \textbf{0.256}} & 0.159                                 & \multicolumn{1}{c|}{0.252}                                 & {\color[HTML]{FF0000} \textbf{0.156}} & {\color[HTML]{FF0000} \textbf{0.249}} & 0.165                                 & \multicolumn{1}{c|}{0.258} & {\color[HTML]{FF0000} \textbf{0.156}} & {\color[HTML]{FF0000} \textbf{0.243}} \\
                                 & 336       & 0.178                                 & \multicolumn{1}{c|}{0.269}                                 & {\color[HTML]{FF0000} \textbf{0.169}} & {\color[HTML]{FF0000} \textbf{0.264}} & 0.215 & \multicolumn{1}{c|}{0.305}                                 & {\color[HTML]{FF0000} \textbf{0.182}} & {\color[HTML]{FF0000} \textbf{0.272}} & 0.173                                 & \multicolumn{1}{c|}{0.268}                                 & {\color[HTML]{FF0000} \textbf{0.169}} & {\color[HTML]{FF0000} \textbf{0.266}} & 0.177                                 & \multicolumn{1}{c|}{0.273} & {\color[HTML]{FF0000} \textbf{0.172}} & {\color[HTML]{FF0000} \textbf{0.260}} \\
                                 & 720       & 0.225                                 & \multicolumn{1}{c|}{0.317}                                 & {\color[HTML]{FF0000} \textbf{0.203}} & {\color[HTML]{FF0000} \textbf{0.297}} & 0.256 & \multicolumn{1}{c|}{0.337}                                 & {\color[HTML]{FF0000} \textbf{0.220}} & {\color[HTML]{FF0000} \textbf{0.306}} & 0.201                                 & \multicolumn{1}{c|}{0.295}                                 & {\color[HTML]{FF0000} \textbf{0.190}} & {\color[HTML]{FF0000} \textbf{0.287}} & {\color[HTML]{FF0000} \textbf{0.213}} & \multicolumn{1}{c|}{0.304} & 0.215                                 & {\color[HTML]{FF0000} \textbf{0.302}} \\ \cmidrule(l){2-18} 
\multirow{-5}{*}{ECL}            & Avg       & 0.178                                 & \multicolumn{1}{c|}{0.270}                                 & {\color[HTML]{FF0000} \textbf{0.166}} & {\color[HTML]{FF0000} \textbf{0.260}} & 0.208 & \multicolumn{1}{c|}{0.295}                                 & {\color[HTML]{FF0000} \textbf{0.180}} & {\color[HTML]{FF0000} \textbf{0.269}} & 0.169                                 & \multicolumn{1}{c|}{0.263}                                 & {\color[HTML]{FF0000} \textbf{0.163}} & {\color[HTML]{FF0000} \textbf{0.258}} & 0.176                                 & \multicolumn{1}{c|}{0.269} & {\color[HTML]{FF0000} \textbf{0.171}} & {\color[HTML]{FF0000} \textbf{0.258}} \\ \midrule
                                 & 12        & 0.071                                 & \multicolumn{1}{c|}{0.174}                                 & {\color[HTML]{FF0000} \textbf{0.064}} & {\color[HTML]{FF0000} \textbf{0.167}} & 0.099 & \multicolumn{1}{c|}{0.216}                                 & {\color[HTML]{FF0000} \textbf{0.071}} & {\color[HTML]{FF0000} \textbf{0.176}} & 0.063                                 & \multicolumn{1}{c|}{0.164}                                 & {\color[HTML]{FF0000} \textbf{0.062}} & {\color[HTML]{FF0000} \textbf{0.164}} & 0.068                                 & \multicolumn{1}{c|}{0.174} & {\color[HTML]{FF0000} \textbf{0.065}} & {\color[HTML]{FF0000} \textbf{0.166}} \\
                                 & 24        & 0.093                                 & \multicolumn{1}{c|}{0.201}                                 & {\color[HTML]{FF0000} \textbf{0.083}} & {\color[HTML]{FF0000} \textbf{0.189}} & 0.142 & \multicolumn{1}{c|}{0.259}                                 & {\color[HTML]{FF0000} \textbf{0.103}} & {\color[HTML]{FF0000} \textbf{0.210}} & 0.080                                 & \multicolumn{1}{c|}{0.185}                                 & {\color[HTML]{FF0000} \textbf{0.079}} & {\color[HTML]{FF0000} \textbf{0.183}} & 0.093                                 & \multicolumn{1}{c|}{0.202} & {\color[HTML]{FF0000} \textbf{0.090}} & {\color[HTML]{FF0000} \textbf{0.196}} \\
                                 & 48        & 0.125                                 & \multicolumn{1}{c|}{0.236}                                 & {\color[HTML]{FF0000} \textbf{0.121}} & {\color[HTML]{FF0000} \textbf{0.228}} & 0.211 & \multicolumn{1}{c|}{0.319}                                 & {\color[HTML]{FF0000} \textbf{0.170}} & {\color[HTML]{FF0000} \textbf{0.268}} & 0.124                                 & \multicolumn{1}{c|}{0.226}                                 & {\color[HTML]{FF0000} \textbf{0.114}} & {\color[HTML]{FF0000} \textbf{0.222}} & 0.146                                 & \multicolumn{1}{c|}{0.258} & {\color[HTML]{FF0000} \textbf{0.144}} & {\color[HTML]{FF0000} \textbf{0.252}} \\
                                 & 96        & {\color[HTML]{FF0000} \textbf{0.164}} & \multicolumn{1}{c|}{0.275}                                 & 0.165                                 & {\color[HTML]{FF0000} \textbf{0.271}} & 0.269 & \multicolumn{1}{c|}{0.370}                                 & {\color[HTML]{FF0000} \textbf{0.252}} & {\color[HTML]{FF0000} \textbf{0.331}} & {\color[HTML]{FF0000} \textbf{0.160}} & \multicolumn{1}{c|}{{\color[HTML]{FF0000} \textbf{0.266}}} & 0.161                                 & 0.267                                 & 0.228                                 & \multicolumn{1}{c|}{0.330} & {\color[HTML]{FF0000} \textbf{0.212}} & {\color[HTML]{FF0000} \textbf{0.313}} \\ \cmidrule(l){2-18} 
\multirow{-5}{*}{PEMS03}         & Avg       & 0.113                                 & \multicolumn{1}{c|}{0.221}                                 & {\color[HTML]{FF0000} \textbf{0.108}} & {\color[HTML]{FF0000} \textbf{0.214}} & 0.180 & \multicolumn{1}{c|}{0.291}                                 & {\color[HTML]{FF0000} \textbf{0.149}} & {\color[HTML]{FF0000} \textbf{0.246}} & 0.107                                 & \multicolumn{1}{c|}{0.210}                                 & {\color[HTML]{FF0000} \textbf{0.104}} & {\color[HTML]{FF0000} \textbf{0.209}} & 0.134                                 & \multicolumn{1}{c|}{0.241} & {\color[HTML]{FF0000} \textbf{0.128}} & {\color[HTML]{FF0000} \textbf{0.232}} \\ \midrule
                                 & 12        & 0.067                                 & \multicolumn{1}{c|}{0.165}                                 & {\color[HTML]{FF0000} \textbf{0.057}} & {\color[HTML]{FF0000} \textbf{0.149}} & 0.095 & \multicolumn{1}{c|}{0.207}                                 & {\color[HTML]{FF0000} \textbf{0.075}} & {\color[HTML]{FF0000} \textbf{0.183}} & 0.055                                 & \multicolumn{1}{c|}{0.145}                                 & {\color[HTML]{FF0000} \textbf{0.054}} & {\color[HTML]{FF0000} \textbf{0.145}} & 0.063                                 & \multicolumn{1}{c|}{0.158} & {\color[HTML]{FF0000} \textbf{0.061}} & {\color[HTML]{FF0000} \textbf{0.154}} \\
                                 & 24        & 0.088                                 & \multicolumn{1}{c|}{0.190}                                 & {\color[HTML]{FF0000} \textbf{0.074}} & {\color[HTML]{FF0000} \textbf{0.169}} & 0.150 & \multicolumn{1}{c|}{0.262}                                 & {\color[HTML]{FF0000} \textbf{0.133}} & {\color[HTML]{FF0000} \textbf{0.245}} & 0.070                                 & \multicolumn{1}{c|}{0.164}                                 & {\color[HTML]{FF0000} \textbf{0.069}} & {\color[HTML]{FF0000} \textbf{0.161}} & 0.089                                 & \multicolumn{1}{c|}{0.190} & {\color[HTML]{FF0000} \textbf{0.087}} & {\color[HTML]{FF0000} \textbf{0.185}} \\
                                 & 48        & 0.110                                 & \multicolumn{1}{c|}{0.215}                                 & {\color[HTML]{FF0000} \textbf{0.094}} & {\color[HTML]{FF0000} \textbf{0.194}} & 0.253 & \multicolumn{1}{c|}{0.340}                                 & {\color[HTML]{FF0000} \textbf{0.158}} & {\color[HTML]{FF0000} \textbf{0.248}} & 0.094                                 & \multicolumn{1}{c|}{0.192}                                 & {\color[HTML]{FF0000} \textbf{0.090}} & {\color[HTML]{FF0000} \textbf{0.188}} & 0.135                                 & \multicolumn{1}{c|}{0.239} & {\color[HTML]{FF0000} \textbf{0.130}} & {\color[HTML]{FF0000} \textbf{0.229}} \\
                                 & 96        & 0.139                                 & \multicolumn{1}{c|}{0.245}                                 & {\color[HTML]{FF0000} \textbf{0.135}} & {\color[HTML]{FF0000} \textbf{0.243}} & 0.346 & \multicolumn{1}{c|}{0.404}                                 & {\color[HTML]{FF0000} \textbf{0.245}} & {\color[HTML]{FF0000} \textbf{0.308}} & 0.117                                 & \multicolumn{1}{c|}{{\color[HTML]{FF0000} \textbf{0.217}}} & {\color[HTML]{FF0000} \textbf{0.116}} & 0.218                                 & 0.196                                 & \multicolumn{1}{c|}{0.294} & {\color[HTML]{FF0000} \textbf{0.194}} & {\color[HTML]{FF0000} \textbf{0.288}} \\ \cmidrule(l){2-18} 
\multirow{-5}{*}{PEMS07}         & Avg       & 0.101                                 & \multicolumn{1}{c|}{0.204}                                 & {\color[HTML]{FF0000} \textbf{0.090}} & {\color[HTML]{FF0000} \textbf{0.188}} & 0.211 & \multicolumn{1}{c|}{0.303}                                 & {\color[HTML]{FF0000} \textbf{0.153}} & {\color[HTML]{FF0000} \textbf{0.246}} & 0.084                                 & \multicolumn{1}{c|}{0.180}                                 & {\color[HTML]{FF0000} \textbf{0.082}} & {\color[HTML]{FF0000} \textbf{0.178}} & 0.121                                 & \multicolumn{1}{c|}{0.220} & {\color[HTML]{FF0000} \textbf{0.118}} & {\color[HTML]{FF0000} \textbf{0.214}} \\ \midrule
                                 & 96        & 0.203                                 & \multicolumn{1}{c|}{0.237}                                 & {\color[HTML]{FF0000} \textbf{0.191}} & {\color[HTML]{FF0000} \textbf{0.226}} & 0.234 & \multicolumn{1}{c|}{0.286}                                 & {\color[HTML]{FF0000} \textbf{0.202}} & {\color[HTML]{FF0000} \textbf{0.244}} & 0.197                                 & \multicolumn{1}{c|}{0.241}                                 & {\color[HTML]{FF0000} \textbf{0.186}} & {\color[HTML]{FF0000} \textbf{0.222}} & {\color[HTML]{FF0000} \textbf{0.185}} & \multicolumn{1}{c|}{0.233} & 0.191                                 & {\color[HTML]{FF0000} \textbf{0.208}} \\
                                 & 192       & 0.233                                 & \multicolumn{1}{c|}{0.261}                                 & {\color[HTML]{FF0000} \textbf{0.221}} & {\color[HTML]{FF0000} \textbf{0.255}} & 0.267 & \multicolumn{1}{c|}{0.310}                                 & {\color[HTML]{FF0000} \textbf{0.234}} & {\color[HTML]{FF0000} \textbf{0.264}} & 0.231                                 & \multicolumn{1}{c|}{0.264}                                 & {\color[HTML]{FF0000} \textbf{0.222}} & {\color[HTML]{FF0000} \textbf{0.257}} & 0.227                                 & \multicolumn{1}{c|}{0.253} & {\color[HTML]{FF0000} \textbf{0.224}} & {\color[HTML]{FF0000} \textbf{0.229}} \\
                                 & 336       & 0.248                                 & \multicolumn{1}{c|}{0.273}                                 & {\color[HTML]{FF0000} \textbf{0.237}} & {\color[HTML]{FF0000} \textbf{0.269}} & 0.290 & \multicolumn{1}{c|}{0.315}                                 & {\color[HTML]{FF0000} \textbf{0.252}} & {\color[HTML]{FF0000} \textbf{0.278}} & 0.241                                 & \multicolumn{1}{c|}{{\color[HTML]{FF0000} \textbf{0.268}}} & {\color[HTML]{FF0000} \textbf{0.239}} & 0.269                                 & 0.246                                 & \multicolumn{1}{c|}{0.284} & {\color[HTML]{FF0000} \textbf{0.239}} & {\color[HTML]{FF0000} \textbf{0.246}} \\
                                 & 720       & 0.249                                 & \multicolumn{1}{c|}{0.275}                                 & {\color[HTML]{FF0000} \textbf{0.242}} & {\color[HTML]{FF0000} \textbf{0.274}} & 0.289 & \multicolumn{1}{c|}{0.317}                                 & {\color[HTML]{FF0000} \textbf{0.250}} & {\color[HTML]{FF0000} \textbf{0.277}} & 0.250                                 & \multicolumn{1}{c|}{0.281}                                 & {\color[HTML]{FF0000} \textbf{0.244}} & {\color[HTML]{FF0000} \textbf{0.276}} & 0.247                                 & \multicolumn{1}{c|}{0.276} & {\color[HTML]{FF0000} \textbf{0.246}} & {\color[HTML]{FF0000} \textbf{0.249}} \\ \cmidrule(l){2-18} 
\multirow{-5}{*}{Solar}          & Avg       & 0.233                                 & \multicolumn{1}{c|}{0.262}                                 & {\color[HTML]{FF0000} \textbf{0.223}} & {\color[HTML]{FF0000} \textbf{0.256}} & 0.270 & \multicolumn{1}{c|}{0.307}                                 & {\color[HTML]{FF0000} \textbf{0.234}} & {\color[HTML]{FF0000} \textbf{0.266}} & 0.230                                 & \multicolumn{1}{c|}{0.264}                                 & {\color[HTML]{FF0000} \textbf{0.223}} & {\color[HTML]{FF0000} \textbf{0.256}} & 0.226                                 & \multicolumn{1}{c|}{0.262} & {\color[HTML]{FF0000} \textbf{0.225}} & {\color[HTML]{FF0000} \textbf{0.233}} \\ \midrule
                                 & 96        & 0.174                                 & \multicolumn{1}{c|}{0.214}                                 & {\color[HTML]{FF0000} \textbf{0.161}} & {\color[HTML]{FF0000} \textbf{0.206}} & 0.177 & \multicolumn{1}{c|}{0.218}                                 & {\color[HTML]{FF0000} \textbf{0.162}} & {\color[HTML]{FF0000} \textbf{0.208}} & {\color[HTML]{FF0000} \textbf{0.156}} & \multicolumn{1}{c|}{{\color[HTML]{FF0000} \textbf{0.202}}} & 0.157                                 & 0.203                                 & 0.163                                 & \multicolumn{1}{c|}{0.207} & {\color[HTML]{FF0000} \textbf{0.153}} & {\color[HTML]{FF0000} \textbf{0.191}} \\
                                 & 192       & 0.221                                 & \multicolumn{1}{c|}{0.254}                                 & {\color[HTML]{FF0000} \textbf{0.212}} & {\color[HTML]{FF0000} \textbf{0.253}} & 0.225 & \multicolumn{1}{c|}{0.259}                                 & {\color[HTML]{FF0000} \textbf{0.208}} & {\color[HTML]{FF0000} \textbf{0.250}} & 0.207                                 & \multicolumn{1}{c|}{0.250}                                 & {\color[HTML]{FF0000} \textbf{0.206}} & {\color[HTML]{FF0000} \textbf{0.248}} & 0.211                                 & \multicolumn{1}{c|}{0.251} & {\color[HTML]{FF0000} \textbf{0.204}} & {\color[HTML]{FF0000} \textbf{0.241}} \\
                                 & 336       & 0.278                                 & \multicolumn{1}{c|}{{\color[HTML]{FF0000} \textbf{0.296}}} & {\color[HTML]{FF0000} \textbf{0.270}} & {\color[HTML]{FF0000} \textbf{0.296}} & 0.278 & \multicolumn{1}{c|}{0.297}                                 & {\color[HTML]{FF0000} \textbf{0.265}} & {\color[HTML]{FF0000} \textbf{0.292}} & 0.262                                 & \multicolumn{1}{c|}{0.291}                                 & {\color[HTML]{FF0000} \textbf{0.260}} & {\color[HTML]{FF0000} \textbf{0.290}} & 0.267                                 & \multicolumn{1}{c|}{0.292} & {\color[HTML]{FF0000} \textbf{0.262}} & {\color[HTML]{FF0000} \textbf{0.284}} \\
                                 & 720       & 0.358                                 & \multicolumn{1}{c|}{0.349}                                 & {\color[HTML]{FF0000} \textbf{0.349}} & {\color[HTML]{FF0000} \textbf{0.347}} & 0.354 & \multicolumn{1}{c|}{0.348}                                 & {\color[HTML]{FF0000} \textbf{0.344}} & {\color[HTML]{FF0000} \textbf{0.344}} & {\color[HTML]{FF0000} \textbf{0.343}} & \multicolumn{1}{c|}{{\color[HTML]{FF0000} \textbf{0.343}}} & 0.345                                 & 0.345                                 & 0.343                                 & \multicolumn{1}{c|}{0.341} & {\color[HTML]{FF0000} \textbf{0.337}} & {\color[HTML]{FF0000} \textbf{0.335}} \\ \cmidrule(l){2-18} 
\multirow{-5}{*}{Weather}        & Avg       & 0.258                                 & \multicolumn{1}{c|}{0.279}                                 & {\color[HTML]{FF0000} \textbf{0.248}} & {\color[HTML]{FF0000} \textbf{0.275}} & 0.259 & \multicolumn{1}{c|}{0.281}                                 & {\color[HTML]{FF0000} \textbf{0.245}} & {\color[HTML]{FF0000} \textbf{0.273}} & {\color[HTML]{FF0000} \textbf{0.242}} & \multicolumn{1}{c|}{0.272}                                 & {\color[HTML]{FF0000} \textbf{0.242}} & {\color[HTML]{FF0000} \textbf{0.271}} & 0.246                                 & \multicolumn{1}{c|}{0.273} & {\color[HTML]{FF0000} \textbf{0.239}} & {\color[HTML]{FF0000} \textbf{0.263}} \\ \midrule
                                 & 3         & 0.205                                 & \multicolumn{1}{c|}{{\color[HTML]{FF0000} \textbf{0.188}}} & {\color[HTML]{FF0000} \textbf{0.204}} & 0.189                                 & 0.204 & \multicolumn{1}{c|}{0.190}                                 & {\color[HTML]{FF0000} \textbf{0.204}} & {\color[HTML]{FF0000} \textbf{0.188}} & 0.204                                 & \multicolumn{1}{c|}{0.191}                                 & {\color[HTML]{FF0000} \textbf{0.202}} & {\color[HTML]{FF0000} \textbf{0.187}} & {\color[HTML]{FF0000} \textbf{0.205}} & \multicolumn{1}{c|}{0.188} & 0.207                                 & {\color[HTML]{FF0000} \textbf{0.178}} \\
                                 & 6         & 0.300                                 & \multicolumn{1}{c|}{0.229}                                 & {\color[HTML]{FF0000} \textbf{0.296}} & {\color[HTML]{FF0000} \textbf{0.227}} & 0.298 & \multicolumn{1}{c|}{{\color[HTML]{FF0000} \textbf{0.227}}} & {\color[HTML]{FF0000} \textbf{0.297}} & 0.229                                 & 0.293                                 & \multicolumn{1}{c|}{0.227}                                 & {\color[HTML]{FF0000} \textbf{0.289}} & {\color[HTML]{FF0000} \textbf{0.226}} & {\color[HTML]{FF0000} \textbf{0.298}} & \multicolumn{1}{c|}{0.227} & 0.300                                 & {\color[HTML]{FF0000} \textbf{0.215}} \\
                                 & 9         & 0.386                                 & \multicolumn{1}{c|}{0.265}                                 & {\color[HTML]{FF0000} \textbf{0.377}} & {\color[HTML]{FF0000} \textbf{0.261}} & 0.382 & \multicolumn{1}{c|}{0.263}                                 & {\color[HTML]{FF0000} \textbf{0.379}} & {\color[HTML]{FF0000} \textbf{0.262}} & 0.369                                 & \multicolumn{1}{c|}{0.264}                                 & {\color[HTML]{FF0000} \textbf{0.366}} & {\color[HTML]{FF0000} \textbf{0.259}} & {\color[HTML]{FF0000} \textbf{0.385}} & \multicolumn{1}{c|}{0.263} & 0.391                                 & {\color[HTML]{FF0000} \textbf{0.249}} \\
                                 & 12        & 0.460                                 & \multicolumn{1}{c|}{0.295}                                 & {\color[HTML]{FF0000} \textbf{0.450}} & {\color[HTML]{FF0000} \textbf{0.291}} & 0.456 & \multicolumn{1}{c|}{0.292}                                 & {\color[HTML]{FF0000} \textbf{0.453}} & {\color[HTML]{FF0000} \textbf{0.292}} & 0.442                                 & \multicolumn{1}{c|}{0.292}                                 & {\color[HTML]{FF0000} \textbf{0.424}} & {\color[HTML]{FF0000} \textbf{0.289}} & {\color[HTML]{FF0000} \textbf{0.457}} & \multicolumn{1}{c|}{0.292} & 0.466                                 & {\color[HTML]{FF0000} \textbf{0.277}} \\ \cmidrule(l){2-18} 
\multirow{-5}{*}{METR-LA}           & Avg       & 0.338                                 & \multicolumn{1}{c|}{0.244}                                 & {\color[HTML]{FF0000} \textbf{0.332}} & {\color[HTML]{FF0000} \textbf{0.242}} & 0.335 & \multicolumn{1}{c|}{0.243}                                 & {\color[HTML]{FF0000} \textbf{0.333}} & {\color[HTML]{FF0000} \textbf{0.243}} & 0.327                                 & \multicolumn{1}{c|}{0.243}                                 & {\color[HTML]{FF0000} \textbf{0.320}} & {\color[HTML]{FF0000} \textbf{0.240}} & {\color[HTML]{FF0000} \textbf{0.336}} & \multicolumn{1}{c|}{0.242} & 0.341                                 & {\color[HTML]{FF0000} \textbf{0.230}} \\ \bottomrule
\end{tabular}
}
\caption{Integration of vecTrans into Transformer-based forecasters by replacing the attention mechanism. Other components and loss functions remain unchanged. This table corresponds to Table~\ref{tab_vecTrans_plugin} in the main paper.}
\label{tab_vecTrans_plugin_appd}
\end{center}
\end{table*}


\begin{table*}[t]
\begin{center}
\renewcommand{\arraystretch}{1.3}
{\fontsize{8}{10}\selectfont
\setlength{\tabcolsep}{4.5pt}
\begin{tabular}{@{}cc|cccc|cccc|cccc|cccc@{}}
\toprule
\multicolumn{2}{c|}{} & \multicolumn{4}{c|}{\begin{tabular}[c]{@{}c@{}}OLinear\\ \shortcite{olinear} \end{tabular}} 
& \multicolumn{4}{c|}{\begin{tabular}[c]{@{}c@{}}iTransformer\\   \shortcite{itransformer} \end{tabular}}                                
& \multicolumn{4}{c|}{\begin{tabular}[c]{@{}c@{}}PatchTST\\      \shortcite{patchtst} \end{tabular}}     
& \multicolumn{4}{c}{\begin{tabular}[c]{@{}c@{}}DLinear\\ \shortcite{linear} \end{tabular}}                                                                           \\ \cmidrule(l){3-18} 
\multicolumn{2}{c|}{\multirow{-2}{*}{Model}} & \multicolumn{2}{c|}{Ori.}                                                                          & \multicolumn{2}{c|}{WFMLoss}                                                   & \multicolumn{2}{c|}{Ori.}          & \multicolumn{2}{c|}{WFMLoss}                                                   & \multicolumn{2}{c|}{Ori.}                                          & \multicolumn{2}{c|}{WFMLoss}                                                   & \multicolumn{2}{c|}{Ori.}                                          & \multicolumn{2}{c}{WFMLoss}                                                    \\ \midrule
\multicolumn{2}{c|}{Metric}                  & MSE                                   & \multicolumn{1}{c|}{MAE}                                   & MSE                                   & MAE                                   & MSE   & \multicolumn{1}{c|}{MAE}   & MSE                                   & MAE                                   & MSE                                   & \multicolumn{1}{c|}{MAE}   & MSE                                   & MAE                                   & MSE                                   & \multicolumn{1}{c|}{MAE}   & MSE                                   & MAE                                   \\ \midrule
                                 & 96        & 0.360                                 & \multicolumn{1}{c|}{0.382}                                 & {\color[HTML]{FF0000} \textbf{0.356}} & {\color[HTML]{FF0000} \textbf{0.379}} & 0.386 & \multicolumn{1}{c|}{0.405} & {\color[HTML]{FF0000} \textbf{0.364}} & {\color[HTML]{FF0000} \textbf{0.386}} & 0.414                                 & \multicolumn{1}{c|}{0.419} & {\color[HTML]{FF0000} \textbf{0.362}} & {\color[HTML]{FF0000} \textbf{0.383}} & 0.386                                 & \multicolumn{1}{c|}{0.400} & {\color[HTML]{FF0000} \textbf{0.376}} & {\color[HTML]{FF0000} \textbf{0.388}} \\
                                 & 192       & 0.416                                 & \multicolumn{1}{c|}{0.414}                                 & {\color[HTML]{FF0000} \textbf{0.409}} & {\color[HTML]{FF0000} \textbf{0.411}} & 0.441 & \multicolumn{1}{c|}{0.436} & {\color[HTML]{FF0000} \textbf{0.414}} & {\color[HTML]{FF0000} \textbf{0.416}} & 0.460                                 & \multicolumn{1}{c|}{0.445} & {\color[HTML]{FF0000} \textbf{0.413}} & {\color[HTML]{FF0000} \textbf{0.413}} & 0.437                                 & \multicolumn{1}{c|}{0.432} & {\color[HTML]{FF0000} \textbf{0.427}} & {\color[HTML]{FF0000} \textbf{0.419}} \\
                                 & 336       & 0.457                                 & \multicolumn{1}{c|}{0.438}                                 & {\color[HTML]{FF0000} \textbf{0.451}} & {\color[HTML]{FF0000} \textbf{0.435}} & 0.487 & \multicolumn{1}{c|}{0.458} & {\color[HTML]{FF0000} \textbf{0.453}} & {\color[HTML]{FF0000} \textbf{0.437}} & 0.501                                 & \multicolumn{1}{c|}{0.466} & {\color[HTML]{FF0000} \textbf{0.449}} & {\color[HTML]{FF0000} \textbf{0.435}} & 0.481                                 & \multicolumn{1}{c|}{0.459} & {\color[HTML]{FF0000} \textbf{0.470}} & {\color[HTML]{FF0000} \textbf{0.441}} \\
                                 & 720       & 0.463                                 & \multicolumn{1}{c|}{0.462}                                 & {\color[HTML]{FF0000} \textbf{0.454}} & {\color[HTML]{FF0000} \textbf{0.457}} & 0.503 & \multicolumn{1}{c|}{0.491} & {\color[HTML]{FF0000} \textbf{0.454}} & {\color[HTML]{FF0000} \textbf{0.456}} & 0.500                                 & \multicolumn{1}{c|}{0.488} & {\color[HTML]{FF0000} \textbf{0.446}} & {\color[HTML]{FF0000} \textbf{0.452}} & 0.519                                 & \multicolumn{1}{c|}{0.516} & {\color[HTML]{FF0000} \textbf{0.463}} & {\color[HTML]{FF0000} \textbf{0.463}} \\ \cmidrule(l){2-18} 
\multirow{-5}{*}{ETTh1}          & Avg       & 0.424                                 & \multicolumn{1}{c|}{0.424}                                 & {\color[HTML]{FF0000} \textbf{0.417}} & {\color[HTML]{FF0000} \textbf{0.420}} & 0.454 & \multicolumn{1}{c|}{0.447} & {\color[HTML]{FF0000} \textbf{0.421}} & {\color[HTML]{FF0000} \textbf{0.424}} & 0.469                                 & \multicolumn{1}{c|}{0.454} & {\color[HTML]{FF0000} \textbf{0.417}} & {\color[HTML]{FF0000} \textbf{0.421}} & 0.456                                 & \multicolumn{1}{c|}{0.452} & {\color[HTML]{FF0000} \textbf{0.434}} & {\color[HTML]{FF0000} \textbf{0.427}} \\ \midrule
                                 & 96        & 0.169                                 & \multicolumn{1}{c|}{0.249}                                 & {\color[HTML]{FF0000} \textbf{0.168}} & {\color[HTML]{FF0000} \textbf{0.245}} & 0.180 & \multicolumn{1}{c|}{0.264} & {\color[HTML]{FF0000} \textbf{0.178}} & {\color[HTML]{FF0000} \textbf{0.260}} & 0.175                                 & \multicolumn{1}{c|}{0.259} & {\color[HTML]{FF0000} \textbf{0.173}} & {\color[HTML]{FF0000} \textbf{0.250}} & 0.193                                 & \multicolumn{1}{c|}{0.292} & {\color[HTML]{FF0000} \textbf{0.183}} & {\color[HTML]{FF0000} \textbf{0.261}} \\
                                 & 192       & 0.232                                 & \multicolumn{1}{c|}{0.290}                                 & {\color[HTML]{FF0000} \textbf{0.232}} & {\color[HTML]{FF0000} \textbf{0.288}} & 0.250 & \multicolumn{1}{c|}{0.309} & {\color[HTML]{FF0000} \textbf{0.241}} & {\color[HTML]{FF0000} \textbf{0.302}} & 0.241                                 & \multicolumn{1}{c|}{0.302} & {\color[HTML]{FF0000} \textbf{0.239}} & {\color[HTML]{FF0000} \textbf{0.293}} & 0.284                                 & \multicolumn{1}{c|}{0.362} & {\color[HTML]{FF0000} \textbf{0.246}} & {\color[HTML]{FF0000} \textbf{0.301}} \\
                                 & 336       & 0.291                                 & \multicolumn{1}{c|}{0.328}                                 & {\color[HTML]{FF0000} \textbf{0.290}} & {\color[HTML]{FF0000} \textbf{0.326}} & 0.311 & \multicolumn{1}{c|}{0.348} & {\color[HTML]{FF0000} \textbf{0.305}} & {\color[HTML]{FF0000} \textbf{0.342}} & 0.305                                 & \multicolumn{1}{c|}{0.343} & {\color[HTML]{FF0000} \textbf{0.299}} & {\color[HTML]{FF0000} \textbf{0.333}} & 0.369                                 & \multicolumn{1}{c|}{0.427} & {\color[HTML]{FF0000} \textbf{0.306}} & {\color[HTML]{FF0000} \textbf{0.339}} \\
                                 & 720       & {\color[HTML]{FF0000} \textbf{0.389}} & \multicolumn{1}{c|}{0.387}                                 & 0.390                                 & {\color[HTML]{FF0000} \textbf{0.386}} & 0.412 & \multicolumn{1}{c|}{0.407} & {\color[HTML]{FF0000} \textbf{0.401}} & {\color[HTML]{FF0000} \textbf{0.396}} & 0.402                                 & \multicolumn{1}{c|}{0.400} & {\color[HTML]{FF0000} \textbf{0.400}} & {\color[HTML]{FF0000} \textbf{0.392}} & 0.554                                 & \multicolumn{1}{c|}{0.522} & {\color[HTML]{FF0000} \textbf{0.406}} & {\color[HTML]{FF0000} \textbf{0.395}} \\ \cmidrule(l){2-18} 
\multirow{-5}{*}{ETTm2}          & Avg       & 0.270                                 & \multicolumn{1}{c|}{0.313}                                 & {\color[HTML]{FF0000} \textbf{0.270}} & {\color[HTML]{FF0000} \textbf{0.311}} & 0.288 & \multicolumn{1}{c|}{0.332} & {\color[HTML]{FF0000} \textbf{0.281}} & {\color[HTML]{FF0000} \textbf{0.325}} & 0.281                                 & \multicolumn{1}{c|}{0.326} & {\color[HTML]{FF0000} \textbf{0.278}} & {\color[HTML]{FF0000} \textbf{0.317}} & 0.350                                 & \multicolumn{1}{c|}{0.401} & {\color[HTML]{FF0000} \textbf{0.285}} & {\color[HTML]{FF0000} \textbf{0.324}} \\ \midrule
                                 & 96        & 0.131                                 & \multicolumn{1}{c|}{0.221}                                 & {\color[HTML]{FF0000} \textbf{0.129}} & {\color[HTML]{FF0000} \textbf{0.220}} & 0.148 & \multicolumn{1}{c|}{0.240} & {\color[HTML]{FF0000} \textbf{0.147}} & {\color[HTML]{FF0000} \textbf{0.235}} & {\color[HTML]{FF0000} \textbf{0.161}} & \multicolumn{1}{c|}{0.250} & 0.168                                 & {\color[HTML]{FF0000} \textbf{0.247}} & {\color[HTML]{FF0000} \textbf{0.197}} & \multicolumn{1}{c|}{0.282} & {\color[HTML]{FF0000} \textbf{0.197}} & {\color[HTML]{FF0000} \textbf{0.272}} \\
                                 & 192       & 0.150                                 & \multicolumn{1}{c|}{0.238}                                 & {\color[HTML]{FF0000} \textbf{0.147}} & {\color[HTML]{FF0000} \textbf{0.237}} & 0.162 & \multicolumn{1}{c|}{0.253} & {\color[HTML]{FF0000} \textbf{0.161}} & {\color[HTML]{FF0000} \textbf{0.249}} & 0.199                                 & \multicolumn{1}{c|}{0.289} & {\color[HTML]{FF0000} \textbf{0.176}} & {\color[HTML]{FF0000} \textbf{0.257}} & {\color[HTML]{FF0000} \textbf{0.196}} & \multicolumn{1}{c|}{0.285} & {\color[HTML]{FF0000} \textbf{0.196}} & {\color[HTML]{FF0000} \textbf{0.275}} \\
                                 & 336       & 0.165                                 & \multicolumn{1}{c|}{0.254}                                 & {\color[HTML]{FF0000} \textbf{0.157}} & {\color[HTML]{FF0000} \textbf{0.250}} & 0.178 & \multicolumn{1}{c|}{0.269} & {\color[HTML]{FF0000} \textbf{0.176}} & {\color[HTML]{FF0000} \textbf{0.264}} & 0.215                                 & \multicolumn{1}{c|}{0.305} & {\color[HTML]{FF0000} \textbf{0.193}} & {\color[HTML]{FF0000} \textbf{0.273}} & 0.209                                 & \multicolumn{1}{c|}{0.301} & {\color[HTML]{FF0000} \textbf{0.208}} & {\color[HTML]{FF0000} \textbf{0.290}} \\
                                 & 720       & 0.191                                 & \multicolumn{1}{c|}{0.279}                                 & {\color[HTML]{FF0000} \textbf{0.177}} & {\color[HTML]{FF0000} \textbf{0.272}} & 0.225 & \multicolumn{1}{c|}{0.317} & {\color[HTML]{FF0000} \textbf{0.207}} & {\color[HTML]{FF0000} \textbf{0.290}} & 0.256                                 & \multicolumn{1}{c|}{0.337} & {\color[HTML]{FF0000} \textbf{0.234}} & {\color[HTML]{FF0000} \textbf{0.307}} & 0.245                                 & \multicolumn{1}{c|}{0.333} & {\color[HTML]{FF0000} \textbf{0.243}} & {\color[HTML]{FF0000} \textbf{0.322}} \\ \cmidrule(l){2-18} 
\multirow{-5}{*}{ECL}            & Avg       & 0.159                                 & \multicolumn{1}{c|}{0.248}                                 & {\color[HTML]{FF0000} \textbf{0.152}} & {\color[HTML]{FF0000} \textbf{0.245}} & 0.178 & \multicolumn{1}{c|}{0.270} & {\color[HTML]{FF0000} \textbf{0.173}} & {\color[HTML]{FF0000} \textbf{0.259}} & 0.208                                 & \multicolumn{1}{c|}{0.295} & {\color[HTML]{FF0000} \textbf{0.193}} & {\color[HTML]{FF0000} \textbf{0.271}} & 0.212                                 & \multicolumn{1}{c|}{0.300} & {\color[HTML]{FF0000} \textbf{0.211}} & {\color[HTML]{FF0000} \textbf{0.289}} \\ \midrule
                                 & 96        & 0.398                                 & \multicolumn{1}{c|}{{\color[HTML]{FF0000} \textbf{0.226}}} & {\color[HTML]{FF0000} \textbf{0.397}} & 0.231                                 & 0.395 & \multicolumn{1}{c|}{0.268} & {\color[HTML]{FF0000} \textbf{0.389}} & {\color[HTML]{FF0000} \textbf{0.240}} & 0.446                                 & \multicolumn{1}{c|}{0.283} & {\color[HTML]{FF0000} \textbf{0.444}} & {\color[HTML]{FF0000} \textbf{0.253}} & 0.650                                 & \multicolumn{1}{c|}{0.396} & {\color[HTML]{FF0000} \textbf{0.645}} & {\color[HTML]{FF0000} \textbf{0.377}} \\
                                 & 192       & 0.439                                 & \multicolumn{1}{c|}{{\color[HTML]{FF0000} \textbf{0.241}}} & {\color[HTML]{FF0000} \textbf{0.429}} & 0.241                                 & 0.417 & \multicolumn{1}{c|}{0.276} & {\color[HTML]{FF0000} \textbf{0.410}} & {\color[HTML]{FF0000} \textbf{0.249}} & 0.540                                 & \multicolumn{1}{c|}{0.354} & {\color[HTML]{FF0000} \textbf{0.455}} & {\color[HTML]{FF0000} \textbf{0.260}} & {\color[HTML]{FF0000} \textbf{0.598}} & \multicolumn{1}{c|}{0.370} & {\color[HTML]{FF0000} \textbf{0.598}} & {\color[HTML]{FF0000} \textbf{0.353}} \\
                                 & 336       & 0.464                                 & \multicolumn{1}{c|}{{\color[HTML]{FF0000} \textbf{0.250}}} & {\color[HTML]{FF0000} \textbf{0.455}} & 0.253                                 & 0.433 & \multicolumn{1}{c|}{0.283} & {\color[HTML]{FF0000} \textbf{0.423}} & {\color[HTML]{FF0000} \textbf{0.256}} & 0.551                                 & \multicolumn{1}{c|}{0.358} & {\color[HTML]{FF0000} \textbf{0.470}} & {\color[HTML]{FF0000} \textbf{0.267}} & 0.605                                 & \multicolumn{1}{c|}{0.373} & {\color[HTML]{FF0000} \textbf{0.604}} & {\color[HTML]{FF0000} \textbf{0.354}} \\
                                 & 720       & 0.502                                 & \multicolumn{1}{c|}{{\color[HTML]{FF0000} \textbf{0.270}}} & {\color[HTML]{FF0000} \textbf{0.484}} & 0.276                                 & 0.467 & \multicolumn{1}{c|}{0.302} & {\color[HTML]{FF0000} \textbf{0.456}} & {\color[HTML]{FF0000} \textbf{0.274}} & 0.586                                 & \multicolumn{1}{c|}{0.375} & {\color[HTML]{FF0000} \textbf{0.510}} & {\color[HTML]{FF0000} \textbf{0.287}} & 0.645                                 & \multicolumn{1}{c|}{0.394} & {\color[HTML]{FF0000} \textbf{0.644}} & {\color[HTML]{FF0000} \textbf{0.373}} \\ \cmidrule(l){2-18} 
\multirow{-5}{*}{Traffic}        & Avg       & 0.451                                 & \multicolumn{1}{c|}{{\color[HTML]{FF0000} \textbf{0.247}}} & {\color[HTML]{FF0000} \textbf{0.441}} & 0.250                                 & 0.428 & \multicolumn{1}{c|}{0.282} & {\color[HTML]{FF0000} \textbf{0.420}} & {\color[HTML]{FF0000} \textbf{0.255}} & 0.531                                 & \multicolumn{1}{c|}{0.343} & {\color[HTML]{FF0000} \textbf{0.470}} & {\color[HTML]{FF0000} \textbf{0.267}} & 0.625                                 & \multicolumn{1}{c|}{0.383} & {\color[HTML]{FF0000} \textbf{0.623}} & {\color[HTML]{FF0000} \textbf{0.364}} \\ \midrule
                                 & 96        & 0.153                                 & \multicolumn{1}{c|}{0.190}                                 & {\color[HTML]{FF0000} \textbf{0.150}} & {\color[HTML]{FF0000} \textbf{0.186}} & 0.174 & \multicolumn{1}{c|}{0.214} & {\color[HTML]{FF0000} \textbf{0.167}} & {\color[HTML]{FF0000} \textbf{0.202}} & 0.177                                 & \multicolumn{1}{c|}{0.218} & {\color[HTML]{FF0000} \textbf{0.173}} & {\color[HTML]{FF0000} \textbf{0.206}} & 0.196                                 & \multicolumn{1}{c|}{0.255} & {\color[HTML]{FF0000} \textbf{0.190}} & {\color[HTML]{FF0000} \textbf{0.225}} \\
                                 & 192       & 0.200                                 & \multicolumn{1}{c|}{0.235}                                 & {\color[HTML]{FF0000} \textbf{0.198}} & {\color[HTML]{FF0000} \textbf{0.234}} & 0.221 & \multicolumn{1}{c|}{0.254} & {\color[HTML]{FF0000} \textbf{0.216}} & {\color[HTML]{FF0000} \textbf{0.247}} & 0.225                                 & \multicolumn{1}{c|}{0.259} & {\color[HTML]{FF0000} \textbf{0.219}} & {\color[HTML]{FF0000} \textbf{0.248}} & 0.237                                 & \multicolumn{1}{c|}{0.296} & {\color[HTML]{FF0000} \textbf{0.232}} & {\color[HTML]{FF0000} \textbf{0.272}} \\
                                 & 336       & 0.258                                 & \multicolumn{1}{c|}{0.280}                                 & {\color[HTML]{FF0000} \textbf{0.253}} & {\color[HTML]{FF0000} \textbf{0.275}} & 0.278 & \multicolumn{1}{c|}{0.296} & {\color[HTML]{FF0000} \textbf{0.270}} & {\color[HTML]{FF0000} \textbf{0.287}} & 0.278                                 & \multicolumn{1}{c|}{0.297} & {\color[HTML]{FF0000} \textbf{0.273}} & {\color[HTML]{FF0000} \textbf{0.288}} & 0.283                                 & \multicolumn{1}{c|}{0.335} & {\color[HTML]{FF0000} \textbf{0.277}} & {\color[HTML]{FF0000} \textbf{0.313}} \\
                                 & 720       & 0.337                                 & \multicolumn{1}{c|}{0.333}                                 & {\color[HTML]{FF0000} \textbf{0.331}} & {\color[HTML]{FF0000} \textbf{0.328}} & 0.358 & \multicolumn{1}{c|}{0.349} & {\color[HTML]{FF0000} \textbf{0.347}} & {\color[HTML]{FF0000} \textbf{0.339}} & 0.354                                 & \multicolumn{1}{c|}{0.348} & {\color[HTML]{FF0000} \textbf{0.349}} & {\color[HTML]{FF0000} \textbf{0.338}} & 0.345                                 & \multicolumn{1}{c|}{0.381} & {\color[HTML]{FF0000} \textbf{0.338}} & {\color[HTML]{FF0000} \textbf{0.362}} \\ \cmidrule(l){2-18} 
\multirow{-5}{*}{Weather}        & Avg       & 0.237                                 & \multicolumn{1}{c|}{0.260}                                 & {\color[HTML]{FF0000} \textbf{0.233}} & {\color[HTML]{FF0000} \textbf{0.256}} & 0.258 & \multicolumn{1}{c|}{0.279} & {\color[HTML]{FF0000} \textbf{0.250}} & {\color[HTML]{FF0000} \textbf{0.269}} & 0.259                                 & \multicolumn{1}{c|}{0.281} & {\color[HTML]{FF0000} \textbf{0.254}} & {\color[HTML]{FF0000} \textbf{0.270}} & 0.265                                 & \multicolumn{1}{c|}{0.317} & {\color[HTML]{FF0000} \textbf{0.259}} & {\color[HTML]{FF0000} \textbf{0.293}} \\ \bottomrule
\end{tabular}
}
\caption{Applying WFMLoss to existing forecasters: OLinear, iTransformer, PatchTST, and DLinear. The model architectures and hyperparameters remain unchanged. This table corresponds to Table~\ref{tab_wfmloss_plugin} in the main paper.}
\label{tab_wfmloss_plugin_appd}
\end{center}
\end{table*}

\clearpage
\bibliographystyle{named}
\bibliography{ijcai26}

\end{document}